\documentclass[11pt]{article}
\usepackage[OT1]{fontenc}
\usepackage{smile}

\def\RL{{\rm RL}}
\def\H{{\rm H}}
\def\L{{\rm L}}

\def\tA{{\tilde A}}

\def\tnu{{\tilde \nu}}

\def\MLE{{\rm MLE}}

\def\confset{{\cC}}
\def\eff{{\mathrm{eff}}}

\DeclareMathAlphabet{\pazocal}{OMS}{zplm}{m}{n}

\def\CI{{\cC}}
\def\DE{{\rm DE}}
\def\E{{\rm E}}
\def\QRE{{\mathtt{QRE}}}

\newif\ifmain
\maintrue
\newif\ifneurips
\neuripsfalse
\title{\huge Actions Speak What You Want:
Provably Sample-Efficient Reinforcement Learning of the Quantal Stackelberg Equilibrium from Strategic Feedbacks}

\author{Siyu Chen\thanks{Department of Statistics and Data Science, Yale University. Email: \texttt{siyu.chen.sc3226@yale.edu}.} \qquad Mengdi Wang\thanks{Department of Electrical and Computer Engineering, Princeton University. Email: \texttt{mengdiw@princeton.edu}.} \qquad Zhuoran Yang\thanks{Department of Statistics and Data Science, Yale University. Email: \texttt{zhuoran.yang@yale.edu}.}}

\date{}
\begin{document}
\maketitle
\begin{abstract}

    We study reinforcement learning (RL) for    learning a Quantal Stackelberg Equilibrium (QSE) in an episodic Markov game with a leader-follower structure. 
In specific, at the outset of the game, the leader announces her policy to the follower and commits to it.
The follower observes the leader's policy and, in turn, adopts a quantal response policy by solving an entropy-regularized policy optimization problem induced by leader's policy. 
The goal of the leader is to find her optimal policy, which  yields the optimal expected total return, by interacting with the follower and learning from data. 
A key challenge of this problem is that the leader cannot observe the follower's reward, and needs to infer the follower's quantal response model from his actions against leader's policies. 
We propose sample-efficient algorithms for both the online and offline settings, in the context of function approximation. 
Our algorithms are based on (i) learning the quantal response model via maximum likelihood estimation and (ii) model-free or model-based RL for solving the leader's decision making problem, and we show that they achieve sublinear regret upper bounds. Moreover, we quantify the uncertainty of these estimators and leverage the uncertainty to implement optimistic and pessimistic algorithms for online and offline settings. 
Besides, when specialized to the linear and myopic setting, our algorithms are also computationally efficient. 
Our theoretical analysis features a novel performance-difference lemma which incorporates the error of quantal response model, which might be of independent interest.  
\end{abstract}

{
  \hypersetup{linkcolor=black}
  \tableofcontents
}
\newpage 

\vspace{-15pt}
\section{Introduction}\label{sec:intro}
\vspace{-5pt}
Multi-agent reinforcement learning (RL) \citep{busoniu2008comprehensive, zhang2021multi} studies sequential decision making problems where  multiple agents interact with each other. 
Such a problem is often modeled as a Markov game,
where at  each timestep, each agent
takes an action at the current state of the environment, receives  an immediate reward, and the environment moves  to a new state according to a Markov transition kernel. 
Here the agents affect each other because the rewards and state transitions depend  on the actions of all players. 
The goal of each agent is to learn her optimal policy that maximizes her expected total return, in the presence of strategic plays of other agents. 

Compared with single-agent RL, one of the most salient challenges of multi-agent RL is \emph{nonstationarity}. 
That is, from the perspective of each agent, she faces a Markov decision process (MDP) induced by the policy of the other agents, which changes abruptly once the other agents adopt new policies. 
To deal with such a challenge, most of the existing literature either 
(i) assumes a \emph{central controller}  that learns the policies for all agents, or 
(ii) designs decentralized methods  based on  adversarial bandits algorithms (see \Cref{sec:related_work} for details). 
The way these works deal with nonstationarity is rather ``passive''.
In particular, nonstationarity is out of the scope given (i), and (ii)   assumes the other agents change arbitrarily and only hopes to compete against the best fixed policy in hindsight.

In contrast, an ``active'' approach of dealing nonstationarity would actively \emph{infer} how other agents adjust their policies from their history actions, and leveraging such information to better design her own policy. 
Such an active approach is particularly promising in human-robot interaction, where we aim to train robots to better assist  human  by learning from human feedback.

In this work, we aim to design provably efficient RL algorithms for Markov games that \emph{actively address the challenge of nonstationarity}. 
As an initial attempt, 
we focus on the setting of a two-player Markov games with a leader-follower structure. 
In specific, in such a game, 
the leader takes on the role of the primary decision-maker who commits to a policy $\pi$ upfront, while the follower, knowing $\pi$, chooses a best-response policy $\nu^{\pi}$ that maximizes his expected  rewards. 
Here we assume the follower's response $\nu^{\pi}$ is unique, which is  obtained by solving an entropy-regularized MDP induced by $\pi$. Here $\nu^{\pi}$ is also called the quantal response to leader's policy $\pi$.
We allow the follower to be either myopic, considering only the immediate reward, or farsighted, planning ahead across all future steps. 
The leader's objective is to 
find her optimal policy $\pi^*$ such that the policy pair  
$(\pi^*, \nu^{\pi^*})$ optimizes her expected total rewards, and  $(\pi^*, \nu^{\pi^*})$ is called a quantal Stackelberg equilibrium (QSE) \citep{bacsar1998dynamic, mckelvey1995quantal}. 
The objectives of the leader and the follower are misaligned as their reward functions $u$ and $r$ are different. 
Our goal is to design online and offline RL algorithms for the leader under the episodic  and function approximation setting, based only on information available to the leader, i.e.,  states, joint actions, and the leader's rewards. In particular, the leader only observes the follower's actions, not rewards.

To efficiently learn $\pi^*$, the leader needs to efficiently optimize her own expected return  by strategically guiding the follower's behavior, which necessitates \emph{actively learning the follower's quantal response mapping} from the data. 
Learning a QSE poses a few unique challenges. {\blue (a)} First, the leader must infer the follower's quantal response mapping from follower's actions, which requires building a response model and inferring follower's reward function by observing how the follower responds  to different leader policies. 
Such a challenge is exacerbated when the follower is farsighted. 
{\blue (b)} Second, QSE is posed as the solution to  a \emph{bilevel optimization} where the leader optimizes her policy subject to the constraint that the follower adopts the quantal response policy. 
Characterizing the performance of the leader's learned policy requires novel analysis that bridges the upper and lower-level problems. 
In specific, we need to determine how the error of the inferred quantal response affects the performance of the learned policy. 
{\blue (c)} Third, we need to understand how to incentivize exploration in the online setting or how to guard against insufficient  data  coverage in the offline setting, in the presence of an inaccurate response model. 
Handling this challenge requires modifying the optimism/pessimism principle for the problem of learning a QSE.

We successfully tackle these challenges and  establish provably sample-efficient RL algorithms for both online and offline settings, in the context of  function approximation, where the follower can be either myopic or farsighted. 
In specific, our algorithms are based on (i) a model-based approach for learning the follower's quantal-response   via maximum likelihood estimation (Challenge (a)), and (ii) a 
single-agent reinforcement learning method for leader's decision making problem. 
In specific, for a myopic follower, 
we adopt a 
 model-free least-squares value iteration (LSVI)  \citep{sutton2018reinforcement} approach to learn the leader's optimal value function. 
 For a farsighted follower, we leverage model-based RL to learn the transition model and the leader's reward function. 
In addition, to promote exploration in the online setting, we construct confidence sets for both the response model and leader's value function (or transition and reward model), and propose to update the leader's policies via optimistic planning \citep{auer2008near}. 
Such an algorithm can be easily modified for the offline setting via pessimistic planning (Challenge (c)). 
Both these algorithms are able to incorporate general function approximation. 
Furthermore,  in the special case with myopic follower and linear function approximation, 
we establish variants of optimistic and pessimistic algorithms that can be efficiently implemented using bonus and penalty functions. 
Furthermore, we prove that all of these methods enjoy  sublinear regret or suboptimality bounds, hence proving statistical efficiency. 
In the case with general function approximation, we introduce novel eluder dimensions that captures the challenge of exploration for learning a QSE. 
Finally, to characterize the suboptimality of leader's policy, we establish a novel performance difference lemma  which relates the suboptimality of the learned policy  to the \emph{Bellman error} of the upper level problem and the  \emph{estimation error of the response model} in the lower level problem (Challenge (b)). Such a result might be of independent interest.

\vspace{5pt}
{\noindent \bf Our Contributions.} In summary, this work proposes provably sample-efficient algorithms for learning the QSE in an episodic Markov game with a leader-follower structure. In such  a game, when the leader announces her policy, the follower adopts the corresponding quantal response, and the leader only observes the follower's actions but not rewards. 
To address the challenge of nonstationarity, the leader needs to actively infer how the follower reacts to her announced policy, and leverage such  information to drive the game towards the QSE. 
We successfully address such a challenge for both  online and offline RL settings with linear and general function approximation. 
We also allow the follower to be both myopic and farsighted. 
Specifically, our contributions are three-fold.
\begin{itemize}
\item [(i)] First, for the case of a  linear Markov game with a   myopic follower, we establish sample-efficient RL algorithms for both  the  offline and online settings. 
These algorithms are based on the principle of pessimism/optimism in the face of uncertainty, where we quantify the estimation uncertainty of leader's value function and follower's quantal response model. 
Thanks to the linear structure, both 
  the offline and online algorithms have versions that are computationally efficient  (Algorithm \ref{alg:PMLE} with {\bf S3} and Algorithm \ref{alg:MLE-OVI} with {\bf S5}).   
 \item [(ii)] Second, for the case  with  general function approximation, we propose provably sample-efficient algorithms for both the online and offline settings, where the follower is allowed to be either  myopic or  farsighted.
 In particular, in the case with a myopic follower, our  algorithms  combine model-free value estimation for the leader and maximum likelihood estimation of the follower's quantal response mapping  (Algorithm \ref{alg:MLE-BCP} and Algorithm \ref{alg:MLE-GOLF}).
For the case with a farsighted follower, we propose  model-based algorithms by combining  pessimism  and optimism with  maximum likelihood estimation for Markov game model (Algorithm \ref{alg:real-PMLE} and Algorithm~\ref{alg:OMLE}). We prove that these algorithm are all sample-efficient by establishing upper bounds on the suboptimality or cumulative regret. 
  
 \item [(iii)] Third, we establish a novel performance difference lemma 
 which relates the suboptimality of the  leader's policy  to leader's Bellman error and the error incurred in estimating the follower's quantal response model, which is measured in terms of the total variation (TV) distance. (See \eqref{eq:performance diff-1} or \Cref{lem:subopt-decomposition} for details.) 
 Moreover, due to the nonlinear nature of the quantal response mapping, we further expand the estimation error of the  quantal response model into a first-order and a second-order term,  which can be further controlled by the maximum likelihood analysis. (See \eqref{eq:QRE-decompose} or \Cref{lem:performance diff} for details.)    
 From the lens of bilevel optimization, our regret analysis connects the Bellman error in the upper level   with the quantal response error in lower level, which might be of independent interest.   \end{itemize} 


\section{Related Works} \label{sec:related_work}

Our work is most related to works that learn Stackelberg equilibria in (Markov) games via online RL \citep{bai2021sample, zhong2021can, kao2022decentralized,  zhao2023online}.
All of these works assume the follower is myopic and perfectly rational. In the following, we discuss these works in detail. Moreover, our work  adds to the large body of literature of of single-agent and multi-agent online and offline RL, as well as the literature of maximum entropy inverse RL. 


\vspace{5pt}
{\noindent \bf Online Single-Agent Reinforcement Learning with Function Approximation.} 
Our works is related to the line of work that designs provably sample-efficient RL algorithms for MDPs in the context of function approximation.
Various works propose  RL algorithms that focus on the linear case, \citep{wang2019optimism, yang2019sample, cai2020provably, jin2020provably, 
zanette2020frequentist, ayoub2020model, modi2020sample, yang2020provably, zhou2021nearly}. 
Among these works, our work is most related to \citet{jin2020provably}. In particular, the linear Markov game model we study for the linear case is a direct extension of the linear MDP model in \citet{jin2020provably}. 
Moreover, the bonus function $\Gamma_h^{(1)}$ 
and the $U$-function (which plays the same role as the Q-function in single-agent RL) constructed in \eqref{eq:linear ridge} are the  same as the UCB bonus function and the optimistic value function in \citet{jin2020provably}. 
Furthermore, more recently, there is a line of research that study  sample-efficient RL in the context of general function approximation. See, e.g., \citet{jiang@2017, sun2019model, wang2020reinforcement, jin2021bellman, du2021bilinear, dann2021provably, zhong2022posterior, foster2021statistical, foster2023tight} and the references therein.
Among these works, our work is particularly related to the work of \citet{jin2021bellman}, which proposes a model-free algorithm named GOLF that is an optimistic variant of least-squares value iteration. 
For the case of online RL with  general function approximation,
our algorithm for updating the leader's value functions is based on a modified version of GOLF, where we additionally take into account the estimation uncertainty of the follower's  quantal response model. 
Moreover, \citet{jin2021bellman} introduce the Bellman eluder dimension that characterizes the exploration 
difficulty of the MDP problem. 
In the regret analysis, we introduce 
similar notions of eluder dimension for learning the leader's optimal policy. 
In particular, we introduce two versions of  eluder dimensions that captures the complexity of leader's Bellman error and the follower's quantal response error.  

\vspace{5pt}
{\noindent \bf Online Multi-Agent Reinforcement Learning.} 
Our work is also related to the literature on online multi-agent RL. 
Most of the existing research focus on learning Markov perfect equilibria in two-agent zero-sum Markov games or correlated or coarse correlated equilibria in general-sum Markov games. 
These works can be divided into two strands depending whether the proposed algorithm is centralized or decentralized. 
When there is a central controller that learns the policies of all agents, 
a few recent works propose extensions of single-agent RL algorithms to Markov games based on the principle of   optimism in the face of uncertainty \citep{bai2020near,bai2020provable,liu2020sharp, jin2021power, huang2021towards, xiong22b, xie2020learning, chen2021almost}. 
The second strand of research develops decentralized online RL algorithm for a single-agent of in a Markov games. 
See, e.g., \citet{jin2021v,liu2022learning, zhan2022decentralized, songcan,tian2021online,  mao2023provably, erez2022regret,  wang2023breaking,cui2023breaking} and the references therein. 
The algorithms proposed in most of these works handle the nonstationary due to other agents by leveraging techniques from adversarial bandit literature. 

Our work is more related to works that learns Stackelberg equilibria in (Markov) games via online RL \citep{bai2021sample, zhong2021can, kao2022decentralized,  zhao2023online}.
All of these works assume the follower is myopic and perfectly rational. 
In specific, \citet{bai2021sample, zhao2023online} focus on the static setting.  
\citet{bai2021sample} consider a centralized setting where central controller can determines the actions taken by both the leader and the follower. \citet{zhao2023online} assume the follower is omniscient in the sense that the follower always plays the best response policy, which is similar to our setting. They show that when the  follower is perfectly rational, the regret of the leader  exhibits different scenarios depending on the relationship between the leader's and follower's rewards. Besides, \citet{kao2022decentralized} assume that the leader and follower are cooperative and design a decentralized algorithm for both the leader and follower, under the tabular setting.  \citet{zhong2021can} study online and offline RL for the leader, assuming the follower's reward function is known, and thus the best response of the follower is known to the leader.

 \vspace{5pt}
{\noindent \bf Offline Reinforcement Learning with Pessimism.} 
Our work is also related to the recent line of research on the efficacy of pessimism in offline RL. See, e.g., \citep{yu2020mopo, kidambi2020morel, kumar2020conservative, buckman2020importance, jin2021pessimism, rashidinejad2021bridging, zanette2021provable, uehara2021pessimistic,  xie2021bellman, lyu2022pessimism, shi2022pessimistic, yan2022efficacy, zhong2022pessimistic, cui2022offline, cui2022provably, yu2022strategic, zhang2023offline} and the references therein for algorithms for MDP or Markov games based on  the pessimism principle. 
Among these works, our work is particularly related to \citet{jin2021pessimism, xie2021bellman, uehara2021pessimistic, yu2022strategic}. 
In particular, in the case of myopic follower and linear function approximation, our offline RL is based on the LSVI-LCB algorithm introduced in \citep{jin2021pessimism}, which uses a penalty function to perform pessimism. 
Furthermore, in the case of myopic follower and   general function approximation, our algorithm is an extension of the value-based pessimistic algorithm proposed in \citet{xie2021bellman} to leader-follower game, and when the follower is farsighted, our algorithm is an extension of the model-based pessimistic algorithm proposed in \citet{uehara2021pessimistic}. 
Finally, \citet{yu2022strategic} studies a different leader-follower game with myopic followers and they propose a model-based and pessimistic algorithm that leverages nonparametric instrumental variable regression. 

 \vspace{5pt}
{\noindent \bf Maximum Entropy Inverse Reinforcement Learning (MaxEnt IRL).} 
Our approach of learning the quantal response mapping via maximum likelihood estimation is related to the existing works on MaxEnt IRL, where the goal is to recover the reward function from expert trajectories based on an energy-based model \citep{ziebart2008maximum, neu2009training, ziebart2010modeling, choi2012nonparametric, gleave2022primer, zhu2023principled}. 
Among these works, our work is more relevant to  \citet{zhu2023principled}, which  establishes the sample complexity of MaxEnt IRL, and learning reward functions from comparison data in both bandits and MDPs. 
The analysis in \citet{zhu2023principled} builds upon the body of literature on learning from comparisons \citep{bradley1952rank, plackett1975analysis, luce2012individual, hajek2014minimax,shah2015estimation, negahban2012iterative}. 
\citet{zhu2023principled} considers maximum likelihood estimation with linear rewards, whereas we also consider general function approximation. More importantly, learning quantal response model from follower's actions is only the lower level problem, and our eventual goal is to learn the optimal policy of the leader. Such a bilevel structure is not considered in  \citet{zhu2023principled}.

\ifneurips\vspace{-10pt}\fi
\section{Preliminaries}
\ifneurips\vspace{-10pt}\fi

\vspace{-5pt}
{\noindent \bf Notations.}
We denote by $\SSS_+^d$ the class of $d$-dimensional symmetric nonnegative definite matrices, $\cF(\cX)$ the class of measurable functions on space $\cX$,  $[N]=\{1, \dots, N\}$, $\Delta(\cX)$ the probability space on $\cX$, $\Cov[v]$ the covariance matrix for random vector $v$, and $\nbr{\cdot}_P$ the vector norm induced by $P\in\SSS_+^d$. We denote by {$\lesssim$} the inequality relationship up to some constants and logarithmic factors. We denote by $|\cB|$ the size of a class $\cB$.

\subsection{Episodic Leader-Follower  Markov Game}\label{sec:markov_game_def}
{\ifneurips
\vspace{-10pt} 
\fi}
We consider an episodic two-player Markov game between a leader and a follower, denoted by 
$
    M  =  \{ \cS , \cA, \cB, H, \rho_0 , P  , u, r  \}. 
$
Here $\cS$ denotes the state space, $\cA$ and $\cB$ are the action spaces of the leader and follower, respectively, $H$ is the horizon length. 
In addition,   $ P = \{  P_h \colon  \cS \times\cA \times\cB\rightarrow\Delta(\cS ) \}_{h \in [H]}$ are the transition kernels, and $u = \{u_h \colon \cS \times\cA \times\cB \rightarrow [0, 1] \}_{h\in [H]}$ and $r = \{r_h \colon \cS \times\cA \times\cB \rightarrow [0, 1] \}_{h\in [H]}$ are the reward functions of the leader and the follower, respectively. 
In such a game,  for any $h \in [H]$, at step $h$, both the leader and follower observe the current state $s_h \in \cS$, take  actions $a_h \in \cA$ and $ b_h \in \cB$, receives rewards $u_h(s_h, a_h, b_h) $ and  $r_h(s_h, a_h, b_h) $ respectively, and the environment moves to a new state $s_{h+1} \sim  P_h (\cdot \given s_h, a_h, b_h)$. Here the initial state $s_1 \sim \rho_0 \in \Delta(\cS)$ and the game  terminates after $s_{H+1}$ is generated. 

\ifneurips
\vspace{-2pt}\fi
\vspace{5pt}
{\noindent \bf Leader-Follower Structure and Policies.} 
We assume that the leader is a more powerful player who is able to coordinate the follower's behaviors. 
In specific, 
at the beginning of the game, the leader announces her policy $\pi = \{ \pi_{h} \}_{h\in [H]}$ for the entire game to the follower,
where $\pi_h \colon \cS\rightarrow \sA $ maps the current state to an element in $\sA$. 
Here $\sA = \{ \cB \rightarrow \Delta(\cA) \} $ denotes the set of functions that maps each action $b$  of the follower to a distribution over $\cA$. 
In other words, each element $\alpha \in \sA$ can be viewed as a prescription \citep{nayyar2014common} that specifies the leader's action contingent on the follower's action.
When the leader announces $\pi$ beforehand, she informs the follower how she will choose her action $a_h$ at each state $s_h$, given the follower's action $b_h$. 
To simplify the notation, in the sequel, we regard $\pi_h$ as a function $\pi_h (\cdot \given s_h, b_h) \in \Delta(\cA)$, which specifies the distribution of~$a_h$.\footnote{Here we assume that the leader is a more powerful player in the sense that her policy $\pi$  takes the follower's action $b_h$ as an input, although she does not observe $b_h$ when announcing $\pi$. 
This can be easily modified for a slightly weaker leader, whose policies $\pi$ does not depend on $b_h$. That is, $\pi_h (\cdot \given s_h)$ maps $s_h$ to a distribution over $\cA$.}


Furthermore, the follower's policy is denoted by $\nu = \{ \nu_h \}_{h\in [H]}$, where $\nu_h \colon \cS \rightarrow \Delta(\cB) $ specifies how the follower takes action $b_h$ at state $s_h$. 
With the leader-follower structure, the actions $\{ a_h, b_h\}$ are generated as  follows. At the beginning, the leader announces her policy $\pi$. The follower observes $\pi$ and chooses a policy $\nu$. For any $h\in [H]$, at the $h$-th time step, the leader commits to the announced policy and samples $\alpha_h =  \pi_h (s_h)$ at state $s_h$, the follower chooses $b_h \sim \nu_h (\cdot \given s_h)$, and then the leader executes $a_h \sim \alpha_h (\cdot \given b_h) = \pi _h ( \cdot \given s_h, b_h) $. 
{\main The structure of this Markov game is depicted in \Cref{fig:Markov game}.}

\begin{figure}[h]
    \centering
    \includegraphics[width=0.7\textwidth]{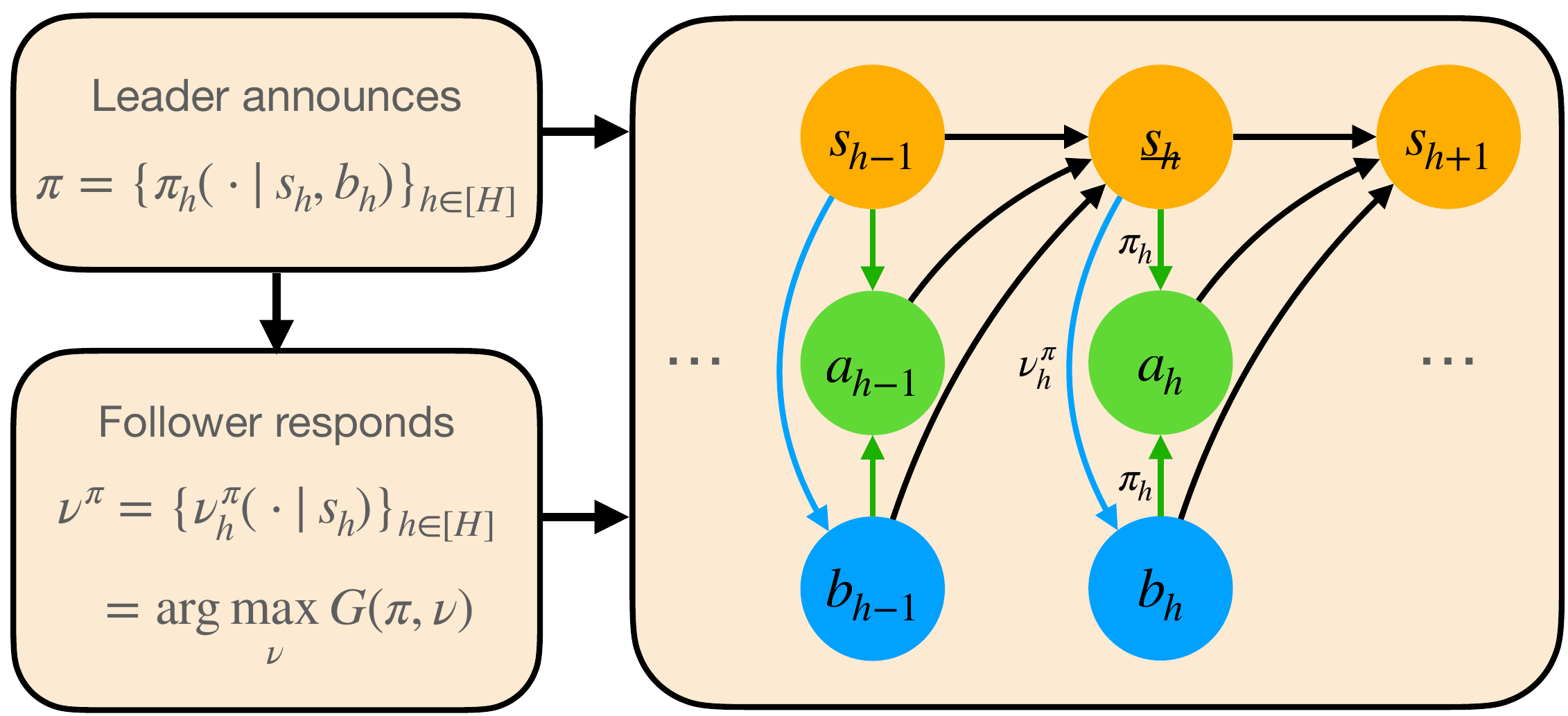}
    \caption{Structure of the Markov game. After the leader announces her $H$-step policies, the follower picks his quantal response according to \eqref{eq:energy}, which can also be computed iteratively via \eqref{eq:quantal_response_policy} and \eqref{eq:qv_pi_qr}. As a result, such a Markov game on the follower's side can also be formulated as a logistic dynamic choice model where the follower assesses his future total rewards and selects an action according to the quantal response model. }\label{fig:Markov game} 
\end{figure}

{\neurips
\vspace{-2pt}\fi  }
{\noindent \bf Follower's Quantal Response.} 
We assume the follower has bounded rationality \citep{simon1955behavioral}  in response to the leader's announced policy $\pi$. 
In particular, let $\eta > 0$ be a parameter and let $\gamma \in [0,1]$ be a discount factor. 
We define  the quantal response policy of the follower with respect to $\pi$, denoted  by $\nu^{\pi}$, as 
the solution to an entropy regularized policy optimization problem:
{\neurips
  \abovebelowskip{0.7}{0.5}
\begin{align}
    \!\!\!\ts\nu^\pi =  \argmax_{\nu}  G(\pi, \nu), G(\pi, \nu) = \EE^{\pi, \nu}\bigsbr{{\ts\sum_{h=1}^H} \gamma^{h-1}\bigrbr{r_h (s_h, a_h, b_h) + \eta^{-1}\cH\rbr{\nu_h(\cdot\given s_h)}}}, \label{eq:energy}
\end{align}\fi}
{\main
\begin{align}
    \nu^\pi =  \argmax_{\nu}  G(\pi, \nu), \quad G(\pi, \nu) = \EE^{\pi, \nu}\sbr{{\sum_{h=1}^H} \gamma^{h-1}\bigrbr{r_h (s_h, a_h, b_h) + \eta^{-1}\cH\rbr{\nu_h(\cdot\given s_h)}}},  \label{eq:energy}
\end{align}
\fi}
\hspace{-6pt} where  we let
$\cH(\cdot)$ denote the Shannon entropy and the expectation  $\EE^{\pi, \nu^\pi} $ is taken over the randomness of the trajectory $\{ (s_h, a_h, b_h )\}_{h\in [H]} $ generated by the policy pair $(\pi, \nu^\pi)$. 
Here $\eta  $ in \eqref{eq:energy} reflects the degree of bounded rationality. 
In particular, when $\eta $ approaches $ + \infty$,  $\nu^\pi$ becomes the optimal  policy  of the MDP induced by $\pi$,  which means the follower is perfectly rational. 
Moreover, $\gamma >0 $ in \eqref{eq:energy} reflects the level of farsightedness of the follower. In particular, a myopic follower with $\gamma = 0$ only maximizes his immediate reward, whereas a farsighted follower   $\gamma = 1$ maximizes the cumulative rewards across the $H$ steps. 

\ifneurips\vspace{-3pt}\fi
Thanks to the entropy regularization in \eqref{eq:energy}, the quantal response policy $\nu^{\pi}$ is unique for any $\pi$. 
Furthermore, by the equivalence between  entropy regularization and soft Q-learning \citep{haarnoja2017reinforcement, geist2019theory}, we can alternatively  characterize $\nu^{\pi}$ using the soft Bellman equation. 
Specifically, 
for any  $h\in [H]$,   $\nu^{\pi}_h \colon \cS \rightarrow \Delta(\cB)$ can be written as the exponential of the advantage (A) function $A_h^\pi:\cS\times\cB\rightarrow \RR$,
\ifneurips
{
    
\abovebelowskip{0.5}{0.5}

\#\label{eq:quantal_response_policy}
\nu_h^\pi(b_h\given s_h) &= \exp\bigl ( \eta \cdot A_h^\pi(s_h, b_h) \bigr), \quad\text{where}\quad A_h^\pi(s_h, b_h) = Q_h^\pi(s_h, b_h) - V_{   h}^\pi(s_h) ,
\#
}
\hspace{-5pt}\fi
\ifmain
\#\label{eq:quantal_response_policy}
\nu_h^\pi(b_h\given s_h) &= \exp\bigl ( \eta \cdot A_h^\pi(s_h, b_h) \bigr), \quad\text{where}\quad A_h^\pi(s_h, b_h) = Q_h^\pi(s_h, b_h) - V_{   h}^\pi(s_h) ,
\#
\fi
and 
  the action-value (Q) function $Q^{\pi}_h \colon \cS \times \cB \rightarrow \RR$ and state-value (V) function $V_h^{\pi} \colon \cS \rightarrow \RR$ are defined respectively as 
\ifneurips
{\abovebelowskip{0.5}{0.5}\fi
\#
    Q_h^\pi (s_h, b_h) & = \bigrbr{r_h^\pi + \gamma  \rbr{ P_h^\pi V_{h+1}^\pi}}(s_h, b_h), 
    {\ifmain \quad \fi}
    V_h^\pi(s_h) = \eta^{-1}  \log \bigrbr{ {\ts\sum_{b\in \cB }}\exp\bigl ( \eta Q_h^\pi(s_h, b_h) \bigr ) }  . \label{eq:qv_pi_qr}
\# 
\ifneurips }
\hspace{-5pt}\fi
Here  we define 
$r_h^\pi(s, b) = \la \pi_h(\cdot\given s, b ), r_h(s, \cdot, b)\ra _{\cA}$ 
and 
$P _h^\pi (s' \given s, b ) = \la \pi_h(\cdot\given s, b ),  P_h (s' \given s, \cdot ,b)\ra _{\cA}$,  which are the reward function and transition kernel of the follower's MDP induced by $\pi$. 
Let  $U_h^\pi:\cS\times\cA\times\cB\rightarrow \RR$ and $W_h^\pi:\cS\rightarrow \cB$ be the leader's action-value (U) function and state-value (W) function under policy $\pi$, which are defined respectively as
\ifneurips
{\abovebelowskip{0.5}{.5}\fi
\begin{gather}
    U_h^\pi(s_h, a_h, b_h)  = u_h(s_h, a_h, b_h) + \rbr{ P_h W_{h+1}^\pi}(s_h, a_h, b_h), \label{eq:U_function} \\ 
    W_h^\pi(s_h)  = \EE^{   \pi_h, \nu_h^\pi} \sbr{U_h^\pi(s_h, a_h, b_h)}  = \la  U_h^{\pi} (s_h, \cdot , \cdot  ) ,  \pi_h \otimes \nu_{h}^{\pi}  ( \cdot, \cdot \given s_h) \cdot  \nu^{\pi}_h  (b  \given s_h), \label{eq:W_function} 
\end{gather}
\ifneurips
}
\hspace{-5pt}\fi
where the expectation in \eqref{eq:W_function} is with respect to $b_h \sim \nu_h^{\pi} (\cdot \given s_h) $ and  $a_h \sim \pi_h(\cdot \given s_h, b_h) $. 
Here we define $\pi_{h}\otimes \nu_{h}^{\pi}(a, b\given s) = \pi_h (a \given s,b) \cdot  \nu^{\pi}_h  (b  \given s) $ for all $h\in [H]$. 
For ease of presentation, we define a quantal Bellman operator $\TT_h^{\pi}:\cF(\cS\times\cA\times\cB) \rightarrow \cF(\cS\times\cA\times\cB)$ for the leader as 
\ifneurips
{\abovebelowskip{0.5}{0.5}\fi
\begin{align}
    (\TT_h^\pi f) (s_h, a_h, b_h) = u_h(s_h, a_h, b_h) + \EE\sbr{\inp{f(s_{h+1}, \cdot,\cdot)}{\pi_{h+1}\otimes \nu_{h+1}^{\pi}(\cdot,\cdot\given s_{h+1})}}, \label{eq:bellman_operator_leader}
\end{align}
\ifneurips }
\hspace{-7pt}\fi
where the expectation is taken over $s_{h+1}\sim P_h(\cdot\given s_h, a_h, b_h)$.
By definition, $U^{\pi} = \{ U_h ^{\pi} \}_{h\in [H]} $  in \eqref{eq:U_function} satisfy the Bellman equation 
$U_h^{\pi} = \TT_h^{\pi} U_{h+1}^{\pi} $. 

\ifneurips\vspace{-5pt}\fi \vspace{5pt}
{\noindent \bf Quantal Stackelberg Equilibrium (QSE).} From the leader's perspective, her goal is to maximize her cumulative rewards under the assumption that 
the follower adopts the quantal response  given by \eqref{eq:energy}. 
Specifically, let $\Pi = \{ \Pi_h \}_{h \in [H]}$ be a class  of leader's policies. 
The leader aims to find $\pi^*$ which maximizes $J(\pi)$ over $\pi \in \Pi$, where 
{\neurips\aboveskip[0.5]\belowskip[0.5]
\begin{align}
    J(\pi) \defeq  F(\pi, \nu^{\pi}), \qquad \textrm{and}~~F(\pi, \nu ) = \EE^{\pi, \nu }\bigsbr{{\ts \sum_{h=1}^H} u_h (s_h, a_h, b_h)}. \label{eq:J}
\end{align}\fi}
{\main
\begin{align}
    J(\pi) \defeq  F(\pi, \nu^{\pi}), \qquad \textrm{and}~~F(\pi, \nu ) = \EE^{\pi, \nu }\sbr{\sum_{h=1}^H u_h (s_h, a_h, b_h)}. \label{eq:J}
\end{align}\fi} 
\hspace{-6pt} Leader's optimal policy $\pi^*$ and its quantal response $\nu^{\pi^*}$, together constitutes a \emph{Quantal Stackelberg Equilibrium} (QSE)  \citep{bacsar1998dynamic, mckelvey1995quantal} of the leader-follower Markov game $M$. Because $J(\pi)$ is finite,   a QSE is guaranteed to exist but might not be unique. 
In the sequel, we  aim to  learn $\pi^*$ within a class $\Pi = \{ \Pi_h \}_{h\in [H]}$ from offline or online  data. 
To characterize the sample complexity of the learning algorithm, we define the suboptimality of any leader's policy $\pi$ as 
\ifneurips
{\aboveskip[0.5]\belowskip[0.5]
\begin{align}\label{eq:subopt}
    \subopt(\pi)\defeq J(\pi^*)-J(\pi),
\end{align}
}\fi
\ifmain
\begin{align}\label{eq:subopt}
    \subopt(\pi)\defeq J(\pi^*)-J(\pi),
\end{align}
\fi
which compares the cumulated rewards received by the leader under a QSE and $(\pi, \nu^{\pi})$. 

\vspace{5pt}
{\main
{\noindent \bf Motivating Examples.} While traditional game theory assumes all players behave perfectly rational, 
real world practice proved that it is necessary to take into account human players' bounded rationality \citep{conlisk1996bounded, camerer1998bounded, karwowski2023sequential,
hernandez2019bounded}.
In the context of consumer decision-making, when faced with a wide array of product choices, consumers often have limited information and cognitive resources to thoroughly evaluate each option. 
Thus, it is more realistic to consider the consumers to choose quantally rather than deterministic among the products.
On the other hand, the choices made by the followers also reveal the followers' goals and interests, which can also be used for steering learning algorithms to align with humans' preferences in the field of robot training, recommender system, and large language models etc \citep{najar2021reinforcement, ouyang2022training, stiennon2020learning, bai2022constitutional, sadigh2017active, christiano2017deep}.
In this work, we target the problem of learning the follower's reward model only from his logistic choices in the context of a Markov game.
\fi}

\ifneurips\vspace{-5pt}\fi
\subsection{Learning QSE from Data: Information Structure and Performance Metrics}
\ifneurips\vspace{-5pt}\fi
We aim to design online and offline RL algorithms that learns $\pi^*$ on behalf of  the leader.
That is, the RL algorithms has access to online or offline data that only contains what the leader is able to observe, when interacting with a boundedly rational   agent. 

\ifneurips\vspace{-2pt}\fi \vspace{5pt}
{\noindent \bf Information Structure.} Let $\cM$ be a class of leader-follower Markov Games specified in Definition \ref{sec:markov_game_def} and let $M^* \in \cM$ denote the true environment. 
In the reinforcement learning setup, the  leader does not know $M^*$ or the quantal response mapping, and need to learn her optimal policy  $\pi^*$. 
We  assume the follower always outputs the quantal response policy $\nu^{\pi}$ 
when the leader commits to a policy $\pi$, where $\nu^{\pi}$ is defined in \eqref{eq:energy}. 
In particular, in this case,  the leader's knowledge is $\{s_h,   a_h, b_h , u_h, \pi_h \}_{h\in [H]} $, a trajectory collected by the policy pair  $(\pi. \nu^{\pi})$, where  $b_h \sim \nu^{\pi}_h (\cdot \given s_h)$,  and $a_h \sim   \pi_h(\cdot \given s_h, b_h)$. An information asymmetry exists in the sense that the leader cannot observe the follower's reward.
 
\ifneurips\vspace{-2pt}\fi \vspace{5pt}
{\noindent \bf Offline RL.} In the offline setting, we aim to learn leader's optimal policy $\pi^*$ from an offline dataset $\mathcal{D} = \{ \tau^t = \{  (s_h^t,   a_h^t, b_h^t, u_h^t, \pi_h^t ) \} _{h\in [H]},   t\in [T]\}$ collected a priori, which contains $T$  trajectories collected on  $M^*$. 
Here each trajectory $\tau^t $ is sampled from a behavior policy $\pi^t$ and its quantal response $\nu^{\pi^t}$. 
Given the dataset $\cD$, we aim to design an  offline RL algorithm that returns a policy $\hat \pi$ for the leader such that the suboptimality $\subopt(\hat \pi)$ defined in  \eqref{eq:subopt} is small.

\ifneurips\vspace{-2pt}\fi \vspace{5pt}
{\noindent \bf Online RL.} In the online setting, the leader learns $\pi^*$ by interacting with the agent for $T$ episodes, without any prior knowledge or data. In specific, for any $t\in [T]$, in the $t$-th episode, leader announces and commits to $\pi^t$, the follower adopts quantal response $\nu^{\pi^t}$. Then the leader observes a new trajectory $ \tau^t = \{  (s_h^t,    a_h^t, b_h^t, u_h^t, \pi_h^t ) \} _{h\in [H]} $. The goal of the online RL algorithm is to design policy sequence $\{\pi^t\}_{t\in [T]}$ such that the regret 
{\neurips\aboveskip[0.5]\belowskip[0.5]
\begin{align}\label{eq:regret_def}
    \Reg (T) = {\ts\sum_{t=1}^T} \subopt(\pi^t) = {\ts\sum_{t=1}^T}  \bigl( J(\pi^*) - J(\pi^t )\bigr)   
\end{align}
\fi}
{\main
\begin{align}\label{eq:regret_def}
    \Reg (T) = \sum_{t=1}^T  \subopt(\pi^t) = \sum_{t=1}^T  \bigl[ J(\pi^*) - J(\pi^t )\bigr]   
\end{align}
\fi}
\hspace{-7pt} is small. 
Besides,   the randomized policy pair that chooses $(\pi^t, \nu^{\pi^t})$ uniformly random constitutes an approximate QSE with error $\Reg(T)/T$. 
In other words, if $\Reg(T) = o(T)$, when $T$ is sufficiently large, the average policy generated by the online RL algorithm constitutes an approximate QSE. 

In the following, we define a linear MDP approximation for this quantal Stackelberg game.
\begin{definition}[{Linear Markov Game}]\label{def:linear MDP}
    We call the episodic leader-follower Markov game linear if there exist mappings $\varpi^*_h:\cS\rightarrow\RR^d$ and feature functions  $\phi_h(\cdot,\cdot):\cS\times\cA\times\cB\rightarrow\RR^d$ for any $h\in[H]$ such that 
the transition probabilities can be expressed as:
$
    P_h(s_{h+1}'\given s_h, a_h, b_h) = \inp[]{\phi_h(s_h, a_h, b_h)}{\varpi^*_h(s_{h+1})},
$  
and
the 
reward functions $r_h$ and $u_h$ are also linear in $\phi_h$. That is, 
$
        u_h(s_h, a_h, b_h) = \inp[]{\phi_h(s_h, a_h, b_h)}{\vartheta^*_h},  $
$
r_h(s_h, a_h, b_h) = \inp[]{\phi_h(s_h, a_h, b_h)}{\theta^*_h}, 
$
where $\vartheta^*_h\in\RR^d$ and $\theta^*_h\in\RR^d$ are parameters.  
\end{definition}
\ifneurips\vspace{-10pt}\fi
\section{Error Decomposition for QSE: Learning Quantal Response}\label{sec:learning QSE}
\ifneurips\vspace{-5pt}\fi
As we have mentioned in \S\ref{sec:intro}, we need to 
(i) learn the follower's  quantal response mapping $\pi\rightarrow \nu^\pi$ and 
(ii) solve the leader's policy optimization problem given that the follower adopts the quantal response. 
In this section, we discuss how to handle these two steps.
For ease of presentation, we focus on the linear Markov Game defined in \Cref{def:linear MDP} with a myopic follower. 
Note that the follower's reward function is sufficient for determining the quantal response mapping in the myopic case.
We let $r_h^\theta(s_h,a_h,b_h)=\la\phi_h(s_h,a_h,b_h), \theta_h\ra$ and use the notation
$ r_h^{\pi, \theta}(s, b) = \odotp{\pi_h(\cdot\given s_h, b_h)}{r_h^\theta(s_h, \cdot, b_h)}_\cB = \inp[]{\phi_h^\pi(s, b)}{\theta_h},$
where $\phi_h^\pi(s, b)=\inp{\phi_h(s, \cdot, b)}{\pi_h(\cdot\given s, b)}_\cA$ is the reward feature mapping under policy $\pi$.
We suppose $\theta_h\in\Theta_h$, which is a bounded subset of $\RR^d$. We let $\theta=\{\theta_h\}_{h\in[H]}$ and  $\Theta=\{\Theta_h\}_{h\in[H]}$.
We will define in the sequel a \emph{\bf quantal response error} (QRE) for myopic follower that characterizes the error  incurred in  learning the quantal response mapping. 

\ifneurips\vspace{-5pt}\fi
\subsection{Performance Difference Lemma for QSE}
\label{sec:subopt decomposition}
\ifneurips\vspace{-5pt}\fi
To quantify how the estimation error of the quantal response affects the suboptimality of the leader's policy (Challenge (b) in \S\ref{sec:intro}),  in the following, we introduce a new version of performance difference lemma that bridges the upper and lower level problems in the QSE.
The idea is to decompose the performance difference into the the leader's Bellman error and the follower's quantal response error.

For any fixed policy $\pi$, recall that the leader's functions under joint policies $(\pi, \nu^{\pi})$ is given by $U^{\pi} $ and $W^{\pi}$ in \eqref{eq:U_function} and  \eqref{eq:W_function}. Suppose we have an estimated parameter $\tilde\theta$ for the follower's reward, and based on the estimated reward $\tilde r = r^{\tilde\theta}$, we have an estimated  quantal response
$\tilde \nu^{\pi} = \nu^{\pi,\tilde\theta}$ under  policy $\pi$. On the leader's side,  we denote by $\tilde U$ and $\tilde W $ the estimates of $U^{\pi} $ and $W^{\pi}$, respectively,
which satisfy 
$
\tilde W_h(s ) = \la \tilde U_h (s, \cdot , \cdot ) , \pi _h \otimes \tilde \nu_h^{\pi} (\cdot , \cdot \given s) \ra_{\cA\times\cB} 
$.
We can hence estimate $J(\pi) $ by $\tilde J(\pi ) \defeq \EE_{s_1 \sim \rho_0} [ \tilde W_1 (s_1)] $. 
We have the following performance difference decomposition, 
\ifneurips{\aboveskip[0.5]\belowskip[0.5]
\begin{align}
\!\!\!\!\!\!\!\!&\tilde J(\pi ) - J(\pi ) \label{eq:performance diff-1}\\
    &~~\le\underbrace{{{\ts\sum_{h=1}^H} \EE^{\pi, \nu^{\pi }}
    \!
    \bigsbr{\bigrbr{\tilde U_h - u_h}(s_h, a_h, b_h)- \tilde W_{h+1}(s_{h+1})}}}_{\ts\small\textrm{Leader's Bellman Error}}+ \underbrace{{\ts\sum_{h=1}^H } H \EE^{\pi, \nu^{\pi }}\! \bigl [ \nbr{\rbr{\tilde \nu_h^{\pi} -\nu_h^{\pi} }(\cdot\given s_h)}_1 \bigr ] }_{\ts\small\textrm{Quantal Response Error}}, \notag
\end{align} 
}
\hspace{-3pt}\fi
\ifmain 
\begin{align}
    &\tilde J(\pi ) - J(\pi ) \label{eq:performance diff-1}\\
        &~~\le\underbrace{{{\sum_{h=1}^H} \EE^{\pi, \nu^{\pi }}
        \!
        \bigsbr{\bigrbr{\tilde U_h - u_h}(s_h, a_h, b_h)- \tilde W_{h+1}(s_{h+1})}}}_{\ts\small\textrm{Leader's Bellman Error}}+ \underbrace{{\sum_{h=1}^H } H \EE^{\pi, \nu^{\pi }}\! \bigl [ \nbr{\rbr{\tilde \nu_h^{\pi} -\nu_h^{\pi} }(\cdot\given s_h)}_1 \bigr ] }_{\ts\small\textrm{Quantal Response Error}}, \notag
    \end{align}
\fi
See \Cref{lem:subopt-decomposition} for a detailed proof.
A unique challenge of our problem is to characterize the quantal response error where we \emph{only have observations of the follower's actions but not  the follower's reward}. The nonlinearity of the quantal response model in \eqref{eq:quantal_response_policy} exacerbates the difficulty. Fortunately, we can linearize the quantal response error by Taylor expansion. With slight abuse of notation, we also refer to the following error term as the quantal response error (QRE),
\ifneurips{
    \abovebelowskip{0.85}{0.85}
\begin{align}
    \QRE(s_h, b_h;\tilde\theta,\pi) =  (\Upsilon_h^{\pi}(\tilde r_h - r_h))\orbr{s_h, b_h}, \label{eq:QRE}
\end{align}
}
\hspace{-5pt}\fi 
\ifmain
\begin{align}
    \QRE(s_h, b_h;\tilde\theta,\pi) =  (\Upsilon_h^{\pi}(\tilde r_h - r_h))\orbr{s_h, b_h}, \label{eq:QRE}
\end{align}\fi
where operator $\Upsilon_h^\pi:\cF(\cS\times\cA\times\cB)\rightarrow \cF(\cS\times\cB)$ is defined as 
\ifneurips{
    \abovebelowskip{0.85}{0.85}
\begin{align}\label{eq:Upsilon}
    \rbr{\Upsilon_h^\pi f}(s_h, b_h) = \dotp{\pi_h(\cdot\given s_h, b_h)}{f(s_h, \cdot, b_h)}_\cA - \dotp{\pi_h\otimes \nu_h^{\pi}(\cdot,\cdot\given s_h)}{f(s_h,\cdot,\cdot)}_{\cA\times\cB}.
\end{align}}
\hspace{-5pt}\fi
\ifmain
\begin{align}\label{eq:Upsilon}
    \rbr{\Upsilon_h^\pi f}(s_h, b_h) = \dotp{\pi_h(\cdot\given s_h, b_h)}{f(s_h, \cdot, b_h)}_\cA - \dotp{\pi_h\otimes \nu_h^{\pi}(\cdot,\cdot\given s_h)}{f(s_h,\cdot,\cdot)}_{\cA\times\cB}.
\end{align}\fi
For myopic follower with a  linear reward, we have for the quantal response error in \eqref{eq:performance diff-1} that 
\ifneurips
{\abovebelowskip{0.6}{0.6}
\begin{align}
    &{\ts\sum_{h=1}^H } H \cdot \EE^{\pi, \nu^{\pi }} \bigl [ \nbr{\rbr{\tilde \nu_h^{\pi} -\nu_h^{\pi} }(\cdot\given s_h)}_1 \bigr ] \label{eq:QRE-decompose}\\
    &\quad \lesssim  2\eta H \cdot 
    {\ts\sum_{h=1}^H
    \Bigrbr{
    \sqrt{\EE^{\pi, \nu^\pi} \bigsbr{\QRE(s_h, b_h;\tilde\theta,\pi)^2}}
    + 
        \EE^{\pi, \nu^\pi} \bigsbr{\QRE(s_h, b_h;\tilde\theta,\pi)^2}
    }}
    , 
    \notag
\end{align}
}
\hspace{-7pt}\fi
\ifmain
\begin{align}
    &{\sum_{h=1}^H } H \cdot \EE^{\pi, \nu^{\pi }} \bigl [ \nbr{\rbr{\tilde \nu_h^{\pi} -\nu_h^{\pi} }(\cdot\given s_h)}_1 \bigr ] \label{eq:QRE-decompose}\\
    &\quad \lesssim  2\eta H \cdot 
    {\sum_{h=1}^H
    \Bigrbr{
    \sqrt{\EE^{\pi, \nu^\pi} \bigsbr{\QRE(s_h, b_h;\tilde\theta,\pi)^2}}
    + 
        \EE^{\pi, \nu^\pi} \bigsbr{\QRE(s_h, b_h;\tilde\theta,\pi)^2}
    }}
    , 
    \notag
\end{align}
\fi
where $\lesssim$ only hides some coefficients in the 2nd-order term on the right-hand side. 
We defer readers to \Cref{lem:response diff-myopic-linear} and its follow-up discussions for more details. 
We remark that   $\mathtt{QRE}$ and the operator $\Upsilon_h^\pi$ capture the comparison nature of the model in the sense that $\EE^{\pi, \nu^\pi}[\QRE(s_h, b_h;\tilde\theta_h,\pi)^2] = 0$ for any admissible $\pi$ if and only if $r_h(s_h, a_h, b_h) = \tilde r_h(s_h, a_h, b_h) + c(s_h)$ for some function $c:\cS\rightarrow\RR$, which matches our intuition that the follower's quantal response should be invariant to any constant shift at a given state. 
With linear function approximation, we define 
$
    \Sigma_{s_h}^{\pi, \theta^*} = \Cov_{s_h}^{\pi, \theta^*}\sbr{\phi_h^{\pi}(s_h, b_h)}
$ with $\Cov_{s_h}^{\pi,\theta}[\cdot] = \Cov^{\pi, \theta}[\cdot\given s_h]$, where the covariance is with respect to $b_h\sim\nu_h^{\pi,\theta}(\cdot\given s_h)$.
We note that    this covariance matrix satisfies $\EE^{\pi,\nu^\pi}[\QRE(s_h,b_h;\tilde\theta,\pi)^2\given s_h] = \onbr{\tilde\theta_h -\theta_h^*}_{\Sigma_{s_h}^{\pi, \theta^*}}$.
The QRE defined in \eqref{eq:QRE} as well as the  covariance matrix $\Sigma_{s_h}^{\pi, \theta^*}$ will be useful in handling the offline distributional shift and 
characterizing the online learning complexity.


{}

{}

{}

\ifneurips\vspace{-10pt}\fi
\subsection{Learning  Quantal Response from Follower's Feedbacks via MLE}\label{sec:MLE for behavior model}
\ifneurips\vspace{-5pt}\fi
In the following, we show that the quantal response mapping defined in \eqref{eq:energy} can be estimated from the follower's history action choices via maximum likelihood estimation (MLE). 
For any $\theta \in \Theta $ and any policy $\pi$ of the leader, we let $\nu^{\pi, \theta}$, $A^{\pi, \theta}$, $Q^{\pi, \theta } $ and $V^{\pi, \theta }$ denote the quantal response of $\pi$, advantage function, and Q- and V-functions under model $r^{\theta}$, which are defined according to  \eqref{eq:quantal_response_policy} and \eqref{eq:qv_pi_qr}.
Thus, given a  (possibly adaptive) dataset $\cD = \{(s_h ^t, a_h ^t, b_h ^t, u_h^t, \pi_h ^t)\}_{t\in[T], h\in [H]}$,
the negative loglikelihood function at step $h$ is given by 
\ifneurips
{\abovebelowskip{0.5}{0.5}
\#\label{eq:loglikelihood}
\cL_h(\theta_h)  & = - {\ts \sum_{i=1}^T }
\log \nu_h^{\pi^i, \theta} (b_h^i \given s_h^i) 
  = - {\ts \sum_{i=1}^T } \eta  \cdot A_h^{\pi^i , \theta}(s_h^i, b_h^i), 
\# 
}
\hspace{-5pt}
\fi
\ifmain
\#\label{eq:loglikelihood}
\cL_h(\theta_h)  & = - { \sum_{i=1}^T }
\log \nu_h^{\pi^i, \theta} (b_h^i \given s_h^i) 
  = - {\sum_{i=1}^T } \eta  \cdot A_h^{\pi^i , \theta}(s_h^i, b_h^i), 
\# 
\fi
where the second equality in \eqref{eq:loglikelihood} is due to \eqref{eq:quantal_response_policy}.  
Note that for the myopic follower case, the right-hand side of \eqref{eq:loglikelihood} only depends on $\theta_h$.
Note that leveraging the  pessimism and optimism principles in offline or online RL necessitates uncertainty quantification. 
Thus, instead of constructing a point estimator for $\theta_h^*$, we  
aim to construct a confidence set that contains $\theta_h^*$ with high probability. 
To this end, we define
\ifneurips
{\abovebelowskip{0.5}{0.5}
\begin{align}
\confset_{h,\Theta}(\beta)=\bigcbr{\theta_h\in\Theta_h: \cL_h(\theta_h)\le {\textstyle \inf_{\theta_h'\in\Theta_h}}\cL_h(\theta_h') + \beta}, 
\label{eq:behavior_model_confset}
\end{align}}
\hspace{-4pt} \fi
\ifmain
\begin{align}
    \confset_{h,\Theta}(\beta)=\bigcbr{\theta_h\in\Theta_h: \cL_h(\theta_h)\le {\textstyle \inf_{\theta_h'\in\Theta_h}}\cL_h(\theta_h') + \beta}, 
    \label{eq:behavior_model_confset}
    \end{align}
\fi
where $\beta > 0$ is a parameter.
Following a standard martingale concentration analysis in \citet{chen2022unified, foster2021statistical} but specialized to the quantal response model in \Cref{lem:MLE-formal}, we are able to establish  a guarantee on the QRE as
$
    {\ts\sum_{t=1}^T} \EE^{\pi^t,\nu^{\pi^t}}\osbr{\QRE(s_h, b_h;\tilde\theta,\pi^t)^2\given s_h^t}\le \cO(\beta)
$
for $\tilde\theta_h\in \CI_{h,\Theta}(\beta)$
with high probability.
Specifically, under the linear approximation, we have 
\ifneurips
{
    \abovebelowskip{0.5}{0.1}
\begin{align}
    {\ts\sum_{h=1}^H} \bignbr{\hat\theta_h-\theta_h^*}_{\Sigma_{h, \cD}^{\theta^*}}^2  \le \cO(\beta T^{-1}),
    \label{eq:bandit-ub-1}
\end{align}
}
\hspace{-5pt}\fi
\ifmain
\begin{align}
    {\sum_{h=1}^H} \bignbr{\hat\theta_h-\theta_h^*}_{\Sigma_{h, \cD}^{\theta^*}}^2  \le \cO(\beta T^{-1}),
    \label{eq:bandit-ub-1}
\end{align}
\fi
where the matrix $\Sigma_{h, \cD}^{\theta}$ depends on the data and is defined as
\ifneurips
{\abovebelowskip{.5}{.5}
    \#\label{eq:cov matrix}
    \Sigma_{h, \cD}^{\theta}= T^{-1} {\ts \sum_{t=1}^T}
    {\Cov_{s_h^t}^{\pi^t, \theta}\osbr{\phi_h^{\pi^t}(s_h, b_h)}}  = T^{-1} {\ts \sum_{t=1}^T} \Sigma_{s_h^t}^{\pi^t, \theta}. 
    \#
}
\hspace{-5pt}\fi
\ifmain
\#\label{eq:cov matrix}
    \Sigma_{h, \cD}^{\theta}= T^{-1} \cdot {\sum_{t=1}^T}
    {\Cov_{s_h^t}^{\pi^t, \theta}\osbr{\phi_h^{\pi^t}(s_h, b_h)}}  = T^{-1} \cdot {\sum_{t=1}^T} \Sigma_{s_h^t}^{\pi^t, \theta}. 
    \#
\fi 
Inequality \eqref{eq:bandit-ub-1} characterizes the accuracy of the confidence set based on the $\Sigma_{h, \cD}^{\theta}$-norm in $\RR^d$ space. 
In the case of linear function approximation, we set $\beta = \tilde \cO(d)$ and thus  we obtain a $\tilde \cO(d T^{-1})$ rate in \eqref{eq:bandit-ub-1} with respect to the weighted norm induced by  $\Sigma_{h, \cD}^{\theta^*}$.
In contrast to Theorem 4 of \citet{shah2015estimation} where they study the MLE estimator of a $K$-wise choice model and only provides an expectation bound, our result characterizes the confidence set and further provides high probability bound.
Moreover, we allow the follower to have infinitely many choices by using a linear class for the follower's reward.
In particular, our results characterize how the leader's past policies affect the learning of the follower's reward parameter through this $\Sigma_{h,\cD}^{\theta^*}$ matrix.
Informally speaking,
the MLE guarantee in \eqref{eq:bandit-ub-1} suggests that the first order quantal response error in \eqref{eq:QRE-decompose} enjoys a $\tilde\cO((d/T)^{1/2})$ rate subject to some concentrability coefficient \citep{munos2008finite, szepesvari2005finite}, e.g., $ \trace^{1/2}\bigrbr{\EE^{\pi, M^*}\Sigma_{s_h}^{\pi,\theta^*} {\Sigma_{h, \cD}^{\theta^*}}^\dagger }$ in the offline setting (Here, we only focus on the first order term since the second order term gives a faster convergence rate in terms of $T$). 
In the online setting, 
we build a self-normalized process with respect to this nonnegative definite matrix $\Sigma_{s_h^i}^{\pi^i,\theta^*}$ and control the online regret by the elliptical potential lemma on nonnegative definite matrices.
{\main
    In the following, we discuss the quantal response learning results in more details.
\fi}

{\main
\paragraph{Covariance Matrix versus  Laplacian Matrix.}
We first point out that unlike the use of the covariance matrix in \eqref{eq:bandit-ub-1}, \citet{shah2015estimation} uses a different Laplacian matrix for the norm, which shall take the form $L_\cD = T^{-1} \sum_{t=1}^T {\Cov_{s_h^t}^{u}[\phi^{\pi^t}(s_h, b_h)]}$ under our settings with $u$ being a uniform distribution over $\cB$. 
We remark that although both $\Sigma_{h, \cD}^{\theta}$ and $L_\cD$ capture the comparison nature of the model (since any constant shifting in the reward function at each state does not influence the myopic follower's decisions), our choice of the covariance matrix $\Sigma_{h, \cD}^\theta$ is indeed more natural in the sense that $\Sigma_{h, \cD}^\theta$ is the Hessian of the negative loglikelihood evaluated at $\theta$.
By using the Laplacian instead of our covariate matrix, the guarantee can be deteriorate by some $\poly(|\cB|\exp(2\eta))$ factor, e.g., Theorem 4 of \citet{shah2015estimation} or Theorem 4.1 of \citet{zhu2023principled}.\footnote{This is due to a change of kernel norm from the covariance matrix to the Laplacian matrix in the bound.} 

\paragraph{Identifiability of the Follower's Reward Function.}
Generally speaking, there is no guarantee that we can learn the absolute value of the follower's utilities when the leader only observes the quantal response. 
Intuitively, any constant shift in the utilities has no effect on the follower's behavior model. Mathematically, consider a tabular case with identity feature mapping $\phi_h=I_{|\cS||\cB||\cA|}$. Recall that by definition, the covariance matrix 
\$
\Sigma_{h, \cD}^{\theta} &= T^{-1} \cdot {\sum_{t=1}^T}
{\Cov_{s_h^t}^{\pi^t, \theta}\osbr{\phi_h^{\pi^t}(s_h, b_h)}} \nend
&= T^{-1} \cdot {\sum_{t=1}^T}
\EE_{s_h^t}^{\pi^t, \theta}\sbr{(\Upsilon_h^{\pi^t}\phi_h)(s_h, b_h) (\Upsilon_h^{\pi^t}\phi_h)(s_h, b_h)^\top}
\$
belongs to $\RR^{|\cS||\cB||\cA|\times |\cS||\cB||\cA|}$ with $v(s_h, a_h, b_h)=\ind(a_h, b_h) c(s_h)$ lying in the null space for any $c\in\cF(\cS)$ (One can easily check that $v$ always lies in the null space of $(\Upsilon_h^{\pi^t}\phi_h)(s_h, b_h)$ by definition of $\Upsilon_h^{\pi}$ in \eqref{eq:Upsilon} and noting that $\phi_h$ is an identity matrix).
The same argument also holds for the Laplacian matrix.
To ensure identifiability, one can impose an additional linear constraint on the utilities such as $\inp{\ind(\cdot)}{r_h(s_h, a_h, \cdot)}_\cB = |\cB|/2$. 
However, such a condition is only needed for the farsighted follower case for technical reasons and  is not without loss of generality.\footnote{Such a linear constraint is without loss of generality for a myopic follower since any constant shifting at a certain state does not affects the follower's behavior. However, the claim does not hold for a farsighted follower since a shifting at state $s_h$ might cause nontrivial fluctuations in the follower's Q function at preceding states.}

\paragraph{Single Offline Policy Fails in Learning the Behavior Model.}
Another observation from \eqref{eq:bandit-ub-1} is that it can fail in learning the follower's behavior model if the data is collected using a single policy, even with guarantee of full coverage of every state and action. 
The intuition is that if the leader commits to the same policy all the time, the follower always faces the same dynamics with the same effective utilities $r^\pi(s_h,b_h)= \inp{\pi(\cdot\given s_h, b_h)}{r(s_h,\cdot,b_h)}_\cA$, which is only a small linear subspace and therefore the follower's action choice  reveals no information about what is going to happen if the leader picks another policy that does not lie within this subspace.
Such an issue suggests that the the offline dataset should contains diverse  leader's policies that are linearly independent to learn follower's behavior model, which in the online setting naturally incurs a trade-off between exploration and exploitation.
\fi}

{}

\ifmain
\section{Offline Learning with Myopic Follower}\label{sec:offline-myopic}
\fi
\ifneurips
\subsection{Offline Learning with Myopic Follower}\label{sec:offline-myopic}
\fi
In this section, we study the problem of offline learning the optimal policy for the leader when the follower is myopic. 
In the offline setting, the offline dataset is collected as $\cD = \{(s_h^i, a_h^i, b_h^i, u_h^i,\pi_h^i)_{h=1}^H\}_{i\in[T]}$.
Here, $\pi^i$ should be thought of as a random variable.
We let $\EE_\cD$ denote the expectation with respect to the data generating distribution (also over the randomness of $\pi^i$). 
We study two function approximation schemes, namely the linear function approximation and the general function approximation.

\ifmain
\subsection{Offline Learning for Linear Markov Game}\label{sec:offline-ML}
\fi
\ifneurips
\subsubsection{Offline Learning for Linear MDP}\label{sec:offline-ML}
\fi
In this subsection, we  develop a computationally efficient and value iteration-based algorithm for the linear Markov game setting which is defined in \Cref{def:linear MDP}.
{}
Recall the guarantee we have for the confidence set based on the negative log-likelihood. A blessing of the myopic follower case is that at each state $s_h$, the follower's quantal response is only a function of the policy $\pi_h$ and the reward $r_h$ with model parameter $\theta_h\in\Theta_h$ at the same step. Therefore, the negative log-likelihood for the follower's behavior at step $h$ is given by
\begin{align}
    \cL_{h}(\theta_h) = - \sum_{i=1}^T \rbr{\eta r_h^{\pi^i, \theta}(s_h^i, b_h^i) - \log \rbr{\sum_{b'\in\cB} \exp\rbr{\eta r_h^{\pi^i, \theta}(s_h^i, b')}}}. \label{eq:myopic-offline-general-MLE loss}
\end{align}
Here, the follower's reward function $r_h^{\pi^i,\theta}$ only depends on $\theta_h$. 
One can 
construct confidence sets 
$$\CI_{h,\Theta}(\beta)=\cbr{\theta_h\in\Theta_h: \cL_h(\theta_h)\le \inf_{\theta'\in\RR^d}\cL_h(\theta')+\beta}, \quad \forall h\in[H]$$
following the same manner as \eqref{eq:behavior_model_confset}.
Here, we can let each parameter class $\Theta_h$ be a bounded subset of $\RR^d$ with $\nbr{\theta_h}_2\le B_\Theta$.  
In the following, we seperately discuss how to deal with the uncertainty in the environment model and the behavior model by adding penalties in the leader's value functions.

\paragraph{Environment Model Uncertainty Quantification.}
The value interation follows a very similar idea as \citet{zhong2021can, jin2020provably}, but the main difference is that we need to handle the uncertainty in the behavior model parameter $\theta_h$. 
We first give the update of the state-action value functions $\hat U_h(\cdot,\cdot,\cdot)$ at each step.
The idea is to exploit the linear structure $U_h(s_h, a_h, b_h) = u_h(s_h, a_h, b_h) + \TT_h W_{h+1}(s_h, a_h, b_h)=\inp{\phi_h(s_h, a_h, b_h)}{\omega_h}$ with $\omega_h = \vartheta_h^* + \sum_{s_{h+1}}\mu_h(s_{h+1})W_{h+1}(s_{h+1})$,
and solve for the state-action value function $\hat U_h(\cdot,\cdot,\cdot)$ by the following ridge regression,
\#
    & \hat\omega_h  =\argmin_{\omega\in\RR^d} \sum_{i=1}^T \rbr{\phi_h(s_h^i, a_h^i, b_h^i)^\top \omega - u_h^i - \hat W_h(s_{h+1}^i)}^2 + \nbr{\omega}_2^2, \label{eq:linear ridge}\\
    & \hat U_h(s_h, a_h, b_h)  = \phi_h(s_h,a_h,b_h)^\top \hat\omega_h - \Gamma^{(1)}_h(s_h, a_h, b_h),\nonumber
\# 
where $\Gamma^{(1)}_h$ is an uncertainty quantifier \citep{jin2021pessimism,zhong2021can} for the uncertainty in the environment model, and this term is included to ensure pessimism. Here, we can choose $\Gamma^{(1)}_h(s_h, a_h, b_h)= \tilde \cO\big(\sqrt{\phi_h(s_h, a_h, b_h)^\top\Lambda_h^{\dagger}\phi_h(s_h, a_h, b_h)}\big)$ where $$\Lambda_h=\sum_{i=1}^T \phi_h(s_h^i, a_h^i, b_h^i) \phi_h(s_h^i, a_h^i, b_h^i)^\top \allowbreak+ I_d$$ is the kernel obtained from the ridge regression problem \eqref{eq:linear ridge}. One should also be aware that the ridge regression problem \eqref{eq:linear ridge} has a closed form solution $\hat \omega_h\leftarrow \Lambda_h^{-1} (\sum_{i=1}^T \phi_h(s_h^i, a_h^i, b_h^i) (u_h^i + \hat W_{h+1}(s_{h+1}^i))) $. Plugging this closed form solution into \eqref{eq:linear ridge}, we get an update for the leader's U function.

\paragraph{Behavior Model Uncertainty Quantification.}
We next show how to deal with the behavior model uncertainty and find a good policy for the leader.
Recall the confidence set $\CI_{h,\Theta}(\beta)$ we construct for the follower's behavior model. Given the fact that the  follower is myopic, the behavior model at step $h$ is fully characterized by $\theta_h$ and the leader can decides on what policy to use simply by looking at $\hat U_h$ given by \eqref{eq:linear ridge} and $\CI_{h,\Theta}(\beta)$, and take the policy that maximizes the one-step value subject to the \emph{pessimistic} estimation, which gives us the first scheme
\begin{align}
    \textbf{S1:}\quad  \hat\pi_h(s_h)  = \argmax_{\pi_h(s_h) \in\sA} \min_{\theta_h\in\confset_{h, \Theta}(\beta)} \inp[\big]{\hat U_h(s_h, \cdot,\cdot)}{\pi_h \otimes\nu_h^{\pi, \theta}(\cdot,\cdot\given s_h)}_{\cA\times\cB},\label{eq:scheme-1}
\end{align}
where $\hat W_h(s_h)$ is just the optimal value to the maximin problem and we remind the readers that $\sA$ is the prescription space.
However, we should note that the problem is highly nonlinear and it is often hard to compute this maximin problem. Note that the inner minimization is included just to ensure pessimism. We should ask ourselves if we can turn the inner minimization problem into a maximization problem by adding penalties for the uncertainty in the behavior model $\theta_h$. By an analysis of the TV distance $D_\TV(\nu_h^{\pi,\theta_h}(\cdot\given s_h), \nu_h^{\pi,\theta_h^*}(\cdot\given s_h))$
specialized to the linear case,\footnote{See \Cref{lem:response diff-myopic-linear} and \Cref{cor:formal-MLE confset-linear myopic} for more details.} we are able to include a penalty term to ensure pessimism and turn the problem into a maximization for $\forall s_h\in\cS$.
\begin{align}
    \textbf{S2:}\quad  \hat\pi_h(s_h) = \argmax_{\pi_h(s_h) \in\sA, \atop \theta_h\in\confset_{h,\Theta}(\beta)} \inp[\big]{\hat U_h(s_h, \cdot,\cdot)}{\pi_h \otimes\nu_h^{\pi, \theta}(\cdot,\cdot\given s_h)}_{\cA\times\cB} - \Gamma^{(2)}_h(s_h;\pi_h , \theta_h),\label{eq:scheme-2}
\end{align}
where $
\Gamma^{(2)}_h(s_h;\pi_h , \theta_h) = 2 B_U(\eta \xi(s_h;\pi_h, \theta_h)  + C^{(3)} \xi(s_h;\pi_h, \theta_h)^2 )
$ with 
$$C^{(3)}=\rbr{2+\exp\rbr{2\eta B_A}\eta B_A} \eta^2 \exp\rbr{2\eta B_A}/{2}, $$ 
and
\begin{equation}
    \xi(s_h;\pi_h, \theta_h) = \sqrt{\trace\Bigrbr{\bigrbr{T\Sigma_{h,\cD}^{\theta_h} + I_d}^\dagger \Sigma_{s_h}^{\pi_h , \theta_h}}} \cdot \sqrt{8 C_{\eta}^2 \beta + 4B_\Theta^2}.\label{eq:Gamma^2}
\end{equation}
Here, $\Sigma_{h,\cD}^{\theta} = T^{-1} \sum_{i=1}^T \Cov_{s_h^i}^{\pi^i, \theta} \allowbreak [\phi_h^{\pi^i}(s_h^i, b_h)]$ is the data-dependent covariance matrix defined in \eqref{eq:cov matrix},  and $\Sigma_{s_h}^{\pi,\theta}=\Cov_{s_h}^{\pi,\theta} \allowbreak [\phi_h^{\pi_h}(s_h, b_h)]$ is the covariance matrix that actually only depends on $\pi_h(s_h)$ and  parameter $\theta_h$. 
Here, we remark that $\Gamma^{(2)}_h$ is a valid uncertainty quantifier which also captures the nonlinear effect of the TV distance $D_\TV(\nu_h^{\pi,\theta_h}(\cdot\given s_h), \nu_h^{\pi,\theta_h^*}(\cdot\given s_h))$ by the second order term and the analysis of $\Gamma^{(2)}_h$ is available in \Cref{lem:response diff-myopic-linear}.

Although Scheme 2 already avoids the maximin optimization in Scheme 1, we are still not satisfied since the  data-dependent covariance matrix $\Sigma_{h,\cD}^{\theta}$ also depends on the optimization variable $\theta_h$, which poses challenges in computation. 
One way to deal with the problem is considering a fixed $\theta_h$. As a matter of fact, we are able to replace the $\theta_h$ in \eqref{eq:scheme-2} with the MLE estimator $\hat\theta_{h,\MLE} = \argmin_{\theta_h\in\Theta} \cL_{h}(\theta_h)$, which gives the following scheme,
\begin{align}
    \textbf{S3:}\quad  \hat\pi_h(s_h) = \argmax_{\pi_h(s_h) \in\sA}  \inp[\big]{\hat U_h(s_h, \cdot,\cdot)}{\pi_h \otimes\nu^{\pi, \hat\theta_{h,\MLE}}(\cdot,\cdot\given s_h)}_{\cA\times\cB} - \Gamma^{(2)}_h(s_h;\pi_h , \hat\theta_{h, \MLE}). \label{eq:scheme-3}
\end{align}  
Here, the uncertainty quantifier $\Gamma_h^{(2)}$ is still needed to ensure pessimism.
To bridge the estimation error of $\hat \theta_{h,\MLE}$ in Scheme 3 to the previous two schemes, one can show that both $\Sigma_{h,\cD}^{\hat\theta_{h,\MLE}}$ and $\Sigma_{s_h}^{\pi,\hat\theta_{h,\MLE}}$ are upper and lower bounded by their correspondences with $\theta^*$ plugged in. 
The analysis is available in \Cref{prop:Hessian-ulb}.
Finally, the \textbf{M}aximal \textbf{L}ikelihood \textbf{E}stimation with \textbf{P}essimistic \textbf{V}alue \textbf{I}teration (MLE-PVI) algorithm is summarized as the following.
\begin{algorithm}[H]
    \begin{algorithmic}[1]
    \Require {$\eta, \cD$}
    \State Initialize $\hat W_{H+1}=0$.
    \For{$h=H, H-1,\dots, 1$}
    \State Obtain kernel $\Lambda_h\leftarrow \sum_{i=1}^T \phi_h(s_h^i, a_h^i, b_h^i) \phi_h(s_h^i, a_h^i, b_h^i)^\top + I$. 
    \State Solve the ridge regression for $\hat \omega_h\leftarrow \Lambda_h^{-1} \rbr{\sum_{i=1}^T \phi_h(s_h^i, a_h^i, b_h^i) \bigrbr{u_h^i + \hat W_{h+1}(s_{h+1}^i)}} $.
    \State Update $\hat U_h(\cdot,\cdot,\cdot)\leftarrow \phi_h(\cdot,\cdot,\cdot)^\top \hat \omega_h - \Gamma_h^{(1)}(\cdot,\cdot,\cdot)$ and truncate to $[0, H-h+1]$.
    \State Compute $(\hat W_h(s_h), \hat\pi_h(s_h) )$ as the optimal value and optimal solution to S1 \eqref{eq:scheme-1}, S2 \eqref{eq:scheme-2}, or S3 \eqref{eq:scheme-3} for each $s_h\in\cS$.
    \EndFor
    \Ensure $\hat\pi=(\hat\pi_h)_{h\in[H]}$.
    \end{algorithmic}
    \caption{Offline MLE-PVI Algorithm for Myopic Follower under Linear Markov Game}
    \label{alg:PMLE}
\end{algorithm}
We have the following theoretical guarantee for \Cref{alg:PMLE}.
\begin{theorem}[{Suboptimality for MLE-PVI}]\label{thm:PMLE-VI-myopic}
    Suppose the data compliance condition \begin{align}\label{eq:data compliance}
        &\PP_\cD(u_h^i=u, s_{h+1}^i=s\given \tau^{i-1}, \{s_{h'}^i, \pi_{h'}^i, a_{h'}^i, b_{h'}^i\}_{h'\in[h]}) \nend
        &\quad =\PP(u_h=u, s_{h+1}=s\given s_h=s_h^i, a_h=a_h^i, b_h=b_h^i), \quad \forall h\in[H], i\in[T], 
    \end{align} 
    holds.
    We choose $\Gamma^{(1)}_h(\cdot,\cdot,\cdot) \ge C_1 d H \allowbreak \sqrt{\log(2d H T/\delta)}\cdot \sqrt{\phi_h(\cdot,\cdot,\cdot)^\top \Lambda_h^{-1}\phi_h(\cdot,\cdot,\cdot)}$ for some universal constant $C_1>0$ and $\beta \ge C_2 d\log(H(1+\eta T^2)\delta^{-1})$ for some universal constant $C_2>0$. For the PMLE-VI algorithm under the above three schemes, we have with probability at least $1-2\delta$ that
    \begin{align*}
        \subopt(\hat\pi) \le \sum_{h=1}^H 2\EE^{\pi^*, \nu^{\pi^*}} \sbr{\Gamma^{(1)}_h(s_h, a_h, b_h) + \Gamma^{(2)}_h(s_h;\pi_h^*, \theta_h')},
    \end{align*}
    where $\theta_h' = \theta_h^*$ for Scheme 1 and Scheme 2, and $\theta_h' = \hat\theta_{h, \MLE}$ for Scheme 3.
    \begin{proof}
        See \Cref{sec:proof-PMLE-VI-myopic} for a detailed proof.
    \end{proof}
\end{theorem}
We give the following corollary that characterizes the distribution shift issue. 
\begin{corollary}[Distribution shift]\label{rmk:MLE-PVI-dist-shift}
    Suppose for the leader's side, we have with high probability that 
    $
    \Lambda_h\succeq I + c_1 T \EE^{\pi^*,\nu^{\pi^*}}[\phi_h\phi_h^\top]
    $ for some constant $c_1>0$, 
    and the for the follower's side, we have with high probability 
    \begin{align}
    &I + {\ts\sum_{t=1}^T }\EE^{\pi^i, \nu^{\pi^i}} \bigsbr{(\Upsilon_h^{\pi^i} \phi_h)  (\Upsilon_h^{\pi^i} \phi_h)^\top \given s_h^i}  \succeq I + c_2 T\EE^{\pi^*, \nu^{\pi^*}} \bigsbr{(\Upsilon_h^{\pi^*} \phi_h) (\Upsilon_h^{\pi^*} \phi_h)^\top} , 
    \label{eq:MLE-PVI-coverage}
    \end{align}
    for some constant $c_2>0$,
    where $\Upsilon_h^{\pi}\phi $ is a short hand of $(\Upsilon_h^{\pi}\phi)(s_h, b_h)$. We then have for Scheme 1 and Scheme 2, 
    \begin{align*}
        {\subopt(\hat\pi) 
        \lesssim \frac{d^{3/2}H^2} {\sqrt{c_1 T}} +  \eta C_\eta H^{2}d \cdot \sqrt{\frac{1}{c_2 T}} +  e^{4\eta B_A} (\eta C_\eta)^3 H^2  d^2 \cdot  \frac{1}{c_2 T}, }
    \end{align*}
    and for Scheme 3, 
    \begin{align*}
    { \subopt(\hat\pi) 
    \lesssim \frac{d^{3/2}H^2} {\sqrt{c_1 T}} + e^{2\eta B_A} \eta C_\eta H^{2}d \cdot \sqrt{\frac{1}{c_2 T}} +  e^{8\eta B_A} (\eta C_\eta)^3 H^2  d^2 \cdot  \frac{1}{c_2 T}.}
    \end{align*}
\end{corollary}
\begin{proof}
    See \Cref{sec:proof-MLE-PVI-dist-shift} for a detailed proof.
\end{proof}
We note that \eqref{eq:MLE-PVI-coverage} is similar to the standard sufficient coverage condition in linear MDP but 
customized for linear QRE, where the operator $\Upsilon_h^{\pi^*}$  defined in \eqref{eq:Upsilon} plays a key role in the distribution shift. In particular, \eqref{eq:MLE-PVI-coverage} not only requires coverage over the trajectory induced by $\pi^*$, but also requires richness in the leader's prescription $\pi^i(s_h)$ at those states visited under $\pi^*$.
To understand this point, we note that if the leader announces the same policy $\pi^0$ for all the time, the follower always acts according to the same reward $r_h^{\pi^0} (s_h, b_h)= \la r_h(s_h,\cdot,b_h), \pi^0_h(\cdot\given s_h, b_h)\ra_\cB$, which is only a linear subspace of the reward function and the leader cannot anticipate the follower's quantal response for other policies. 

We next understand the first order terms, i.e., $\cO(T^{-1/2})$ terms in the suboptimality. The first term characterizes the leader's Bellman error, which is standard for RL problems. 
The second term characterizes the follower's first-order quantal response error (QRE). The first-order QRE term suffers from an $\exp(2\eta B_A)$ coefficient only in Scheme 3. We remark that this is because we fix the follower's quantal response using the MLE estimator in \eqref{eq:scheme-3} while Scheme 1 and Scheme 2 allow us to pick a more refined estimator $\hat\theta_h$ in the confidence set at the cost of heavier computation.
\ifneurips
\subsubsection{Offline Learning with General Function Class}\label{sec:Offline-MG}
\fi
\ifmain
\subsection{Offline Learning with General Function Class}\label{sec:Offline-MG}
\fi
In this subsection, we carry out the offline learning scheme with general function approximation.
For the leader's side, we propose to learn the environment model by minimizing the squared loss of the Bellman error over the U function for each policy $\pi$. 
For consistency, we still assume that the follower's reward function at step $h\in[H]$ lies in some general function class parameterized by $\theta_h\in\Theta_h$. We let $\Theta=\{\Theta_h\}_{h\in[H]}$.
Suppose $\cU:\cS\times\cA\times\cB\rightarrow\RR$ is a given function class for the leader's state-action value function. 
Following the idea of Bellman-consistent pessimism \citep{xie2021bellman}, we define a loss function for the environment model error as
\begin{align}
    &\ell_h(U_h', U_{h+1}, \theta_{h+1}, \pi) \nend
    &\quad = \sum_{i=1}^T \rbr{U_h'(s_h^i, a_h^i, b_h^i) - u_h^i -  \inp[\big]{U_{h+1}(s_{h+1}^i, \cdot, \cdot)}{\pi_{h+1}\otimes \nu_{h+1}^{\pi, \theta}(\cdot, \cdot\given s_{h+1}^i)}}^2. \label{eq:myopic-offline-general-Bellman loss}
\end{align}
Intuitively, the loss $\ell_h$ going to zero means no Bellman error for the value functions between step $h$ and step $h+1$ under $\pi$ and $\nu^{\pi,\theta}$.
Note that the unknown parameters are $\theta=\{\theta_h\}_{h\in[H]}\in\Theta$ and $\{U_h\}_{h\in[H]}\in\cU^{\otimes H}$.
Based on the loss functions defined in \eqref{eq:myopic-offline-general-MLE loss} and \eqref{eq:myopic-offline-general-Bellman loss}, we can construct a confidence set for each leader's policy $\pi$ as
\begin{align}
    &\CI_{\cU, \Theta}^\pi(\beta) \nend
    &\quad= \cbr{
    (U,\theta)\in\cU^{\otimes H}\times\Theta:
    \rbr{ \ds
        \cL_h(\theta_h)-\inf_{\theta'\in\Theta_h}\cL_h(\theta') \le \beta 
    \atop \ds
        \ell_h(U_h, U_{h+1}, \theta_{h+1}, \pi) - \inf_{U'\in\cU_h} \ell_h(U', U_{h+1}, \theta_{h+1}, \pi)\le H^2\beta}, 
    \forall h\in[H]}. \label{eq:myopic-offline-general-confset}
\end{align}
The first condition in \eqref{eq:myopic-offline-general-confset} characterizes a valid and accurate confidence set for the follower's behavior model as we have done in \eqref{eq:bandit-ub-1}. For the second condition in \eqref{eq:myopic-offline-general-confset}, if certain realizability and completeness conditions are satisfied, we have guarantee on small leader's Bellman errors \citep{xie2021bellman, lyu2022pessimism}, which characterizes the uncertainty in the environment model.
Combining these two guarantees, we can therefore expect $\CI_{\cU, \Theta}^\pi (\beta)$ to be a valid and accurate confidence set for both the environment and the behavior models.
Following the principle of pessimism, we can output the policy $\hat\pi$ as, 
\begin{align}
    \hat\pi=\argmax_{\pi\in\Pi} \min_{(U, \theta)\in\CI_{\cU,\Theta}^\pi(\beta)} \EE_{s_1\sim\rho_0} \sbr{\inp[\big]{U_1(s_1, \cdot, \cdot)}{\pi_1\otimes\nu_1^{\pi, \theta}(\cdot,\cdot\given s_1)}_{\cA \times \cB}}.\label{eq:offline-MG-pi^hat}
\end{align}
To present our results, we first define an optimistic Bellman operator $\TT_h^{*, \theta}:\cF(\cS\times\cA\times\cB)\rightarrow \cF(\cS\times\cA\times\cB)$ for the leader as
\begin{align}\label{eq:define optimistic Bellman opt}
    &\bigrbr{\TT_h^{*,\theta} f} (s_h, a_h, b_h) \nend
    &\quad = u_h(s_h, a_h, b_h) +  \EE\sbr{\max_{\pi_{h+1}(s_{h+1})\in\sA}\bigdotp{f(s_{h+1}, \cdot, \cdot)}{\pi_{h+1}\otimes\nu_{h+1}^{\pi,\theta}(\cdot,\cdot\given s_{h+1})}_{\cA\times\cB}\Biggiven s_h, a_h, b_h}.
\end{align}
Here, the expectation is taken with respect to $s_{h+1}\sim P_h(\cdot\given s_h, a_h, b_h)$. 
We now summarize our offline policy learning with \textbf{M}aximum \textbf{L}ikelihood with \textbf{B}ellman \textbf{C}onsistent \textbf{P}essimism (MLE-BCP) algorithm together with its theoretical guarantee.
\begin{algorithm}[H]
    \begin{algorithmic}[1]
    \Require {$\eta, \cD$}
    \State Construct confidence set $C_{\cU,\Theta}^\pi(\beta)$ by \eqref{eq:myopic-offline-general-confset}. 
    \State Solve for the policy $\hat\pi$ with Bellman consistent pessimism in \eqref{eq:offline-MG-pi^hat}. 
    \Ensure $\hat\pi=(\hat\pi_h)_{h\in[H]}$.
    \end{algorithmic}
    \caption{Offline MLE-BCP for Myopic Follower under General Function Approximation}
    \label{alg:MLE-BCP}
\end{algorithm}
\begin{theorem}[{Suboptimality for MLE-BCP}]\label{thm:Offline-MG}
    Suppose that each trajectory in the offline dataset is independently collected.
    Suppose that the following conditions hold for model class $\Theta$ and function class $\cU$:
\begin{itemize}
    \item[(i)] (\textit{Realizability}) There exists $\theta^*\in\Theta$ such that $r_h^{\theta^*}=r_h$ for any $h\in[H]$. For any $\pi\in\Pi, \theta\in\Theta$, there exists $U\in\cU$ such that $U_h = \TT_h^{\pi,\theta} U_{h+1}$ for any $h\in[H]$;
    \item[(ii)] (\textit{Completeness}) For any $U\in\cU$ and $\pi\in\Pi, \theta\in\Theta$, there exists $U'\in\cU$ such that $U'=\TT_h^{\pi,\theta} U_{h+1}$ for any $h\in[H]$. 
\end{itemize}
    By choosing $\beta = c\cdot \max\cbr{\log(H \cN_\rho(\cY, T^{-1})\delta^{-1}), \log(H\cN_\rho(\Theta, T^{-1})/\delta)}$ for some universal constant $c>0$, where the covering number for $\Theta$ and $\cY$ are defined in \eqref{eq:cN-Theta-myopic} and \eqref{eq:cN-cY}, respectively, 
    we have for the offline algorithm \eqref{eq:offline-MG-pi^hat} that 
    \begin{align*}
        \subopt(\hat\pi) 
        &\lesssim \max_{U\in\cU,\theta\in\Theta, h\in[H]}
            \sqrt{\frac{{{\bignbr{{ U_h  -  \TT_h^{\pi^*,\theta}  U_{h+1}} }_{2, d^{\pi^*}}^2}}}{{{\bignbr{ U_{h} - \TT_{h}^{\pi^*,\theta}  U_{h + 1}}_{2,\cD}^2}}}} 
        \cdot H^2\sqrt{\beta T^{-1}}  \nend
        &\qquad + \max_{\theta\in\Theta, h\in[H]}\sqrt{ \frac{{\bignbr{\Upsilon_h^{\pi^*} (r_h^\theta - r_h^{\theta^*})}_{2, d^{\pi^*}}^2}}{\bignbr{{\Upsilon_h^{\pi^i} (r_h^\theta - r_h^{\theta^*})}}_{2,\cD}^2}} \cdot  H^2 \eta C_\eta \sqrt{\beta T^{-1}}\nend
        &\qquad  +   \max_{\theta\in\Theta, h\in[H]} {
            \frac{{\bignbr{\Upsilon_h^{\pi^*} (r_h^\theta - r_h^{\theta^*})}_{2, d^{\pi^*}}^2}}{\bignbr{{\Upsilon_h^{\pi^i} (r_h^\theta - r_h^{\theta^*})}}_{2,\cD}^2}
        }\cdot  H^2 \exp(4\eta B_A) (\eta C_\eta)^3 \beta T^{-1},
    \end{align*}
    where $C_\eta = \eta^{-1}+B_A$ and $\lesssim$ only hides universal constants.
    \begin{proof}
        See \Cref{sec:proof-offline-MG} for a detailed proof.
    \end{proof}
\end{theorem}
\Cref{thm:Offline-MG} establishes the suboptimality for offline learning the optimal policy using general function approximation.
Similar to the linear case, the first two terms characterizes the leader's Bellman error and the follower's first order quantal response error, respectively. 
The only exponential term appears in the follower's second order quantal response error term, which is roughly of order $\cO(T^{-1})$.
In particular, the concentrability coefficients that address the distribution shift issue are characterized by the ratio in both the Bellman error and the QRE error.

\ifmain
\section{Online Learning with Myopic Follower} \label{sec:myopic-online}
\fi 
\ifneurips
\subsection{Online Learning with Myopic Follower} \label{sec:myopic-online}
\fi
In the previous section, we address the problem of offline learning the QSE with myopic follower under both general function class and linear MDP setting. In this section, we move on to the online scenario with myopic follower. Specifically, the game proceeds as the following. At state $s_h$ in episode $t$, the leader announces her prescription $\alpha_h^t:\cB\rightarrow\Delta(\cA)$,  and the myopic follower picks an action $b_h^t$. The leader then pick an action $a_h^t\sim \alpha_h^t(\cdot\given b_h^t)$ and the state then transits to $s_{h+1}^t$. The game at episode $t$ ends at the $H$-th step and a new episode begins next. We also study the online problem both under the general function class and the linear MDP setting.

\ifmain
\subsection{Online Learning for Linear Markov Game}
\label{sec:online-myopic-linear}
\fi
\ifneurips
\subsubsection{Online Learning for Linear MDP}
\label{sec:online-myopic-linear}
\fi
The case for online learning with linear MDP is not so different from the offline one, except for the fact that we have to incorporate optimism for exploration. This is nothing much but just flipping the sign of the penalties and turn them into bonuses. Following the same spirit of \Cref{sec:offline-ML}, we simply present our algorithm here for the sake of completeness. 
For updating the U function, we add a bonus to the ridge regression result,
\begin{gather}
    \hat \omega_h^t\leftarrow (\Lambda_h^t)^{-1} \rbr{\sum_{i=1}^{t-1} \phi_h(s_h^i, a_h^i, b_h^i) \bigrbr{u_h^i + \hat W_{h+1}(s_{h+1}^i)}}, \label{eq:linear ridge}\\
    \hat U_h^t(s_h, a_h, b_h) = \phi_h(s_h,a_h,b_h)^\top \hat\omega_h^t + \Gamma^{(1,t)}_h(s_h, a_h, b_h),\nonumber
\end{gather}
where we choose $\Gamma^{(1,t)}_h(s_h, a_h, b_h)= \cO(\sqrt{\phi_h(s_h, a_h, b_h)^\top\Lambda_h^{\dagger}\phi_h(s_h, a_h, b_h)})$ where $\Lambda_h=\sum_{i=1}^T  \phi_h(s_h^i, a_h^i, b_h^i)\allowbreak \phi_h(s_h^i, a_h^i, b_h^i)^\top \allowbreak+ I_d$. For updating the W function, 
we still define the negative loglikelihood as the one given in \eqref{eq:myopic-offline-general-MLE loss},
\begin{align}
    \cL_{h}^t(\theta_h) = - \sum_{i=1}^{t-1} \rbr{\eta r_h^{\pi^i, \theta}(s_h^i, b_h^i) - \log \rbr{\sum_{b'\in\cB} \exp\rbr{\eta r_h^{\pi^i, \theta}(s_h^i, b')}}}. \label{eq:online-MG-MLE loss}
\end{align}
We obtain an optimistic policy via the following two schemes,
\begin{align}
    \textbf{S4:}\quad  \hat\pi_h^t(s_h)  &= \argmax_{\pi_h(s_h) \in\sA
    \atop \theta_h\in\confset_{h, \Theta}^t(\beta)} \inp[\big]{\hat U_h^t(s_h, \cdot,\cdot)}{\pi_h \otimes\nu_h^{\pi , \theta}(\cdot,\cdot\given s_h)}_{\cA\times \cB}, \quad \forall s_h\in\cS_h,\label{eq:scheme-4}\\
    \textbf{S5:}\quad  \hat\pi_h^t(s_h)  &= \argmax_{\pi_h(s_h) \in\sA}  \inp[\big]{\hat U_h^t(s_h, \cdot,\cdot)}{\pi_h \otimes\nu^{\pi , \hat\theta_{h,\MLE}^t}(\cdot,\cdot\given s_h)}_{\cA\times \cB} + \Gamma^{(2,t)}_h(s_h;\pi_h , \hat\theta_{h, \MLE}^t). \label{eq:scheme-5}
\end{align}
which follow from \eqref{eq:scheme-1} and \eqref{eq:scheme-3}, respectively. Here, we denote by $\CI_{h,\Theta}^t(\beta) = \{\theta_h\in\Theta_h: \cL_h^t(\theta)\le \min_{\theta_h'}\cL_h^t(\theta_h')+\beta\}$ the confidence set at step $t$ with $\cL_h^t(\theta)$ defined in \eqref{eq:online-MG-MLE loss}. Moreover, $\hat\theta_{h,\MLE}^t$ is the MLE estimator that minimizes the negative log-likelihood $\cL_h^t(\cdot)$, and $\Gamma_h^{(2,t)}(s_h;\pi_h,\theta_h)=2 H(\eta \xi  + C^{(3)} \xi^2 )
$ with $C^{(3)}=\eta^2 \exp(2\eta B_A) (2+\eta B_A \exp(2\eta B_A))/2$ and, 
\begin{equation}
    \xi = \sqrt{\trace\rbr{\bigrbr{\Sigma_{h, t}^{\theta_h} + I_d}^\dagger \Sigma_{s_h}^{\pi_h , \theta_h}}} \cdot \sqrt{8 C_{\eta}^2 \beta + 4B_\Theta^2}.\label{eq:Gamma^2-online}
\end{equation}
where $\Sigma_{h,t}^{\theta_h} = \sum_{i=1}^{t-1} \Cov_{s_h^i}^{\pi_h^i, \theta_h} \allowbreak [\phi_h^{\pi_h^i}(s_h^i, b_h)]$ is the covariate matrix and $\Sigma_{s_h}^{\pi,\theta}=\Cov_{s_h}^{\pi,\theta} \allowbreak [\phi_h^{\pi_h}(s_h, b_h)]$. We summary the \textbf{M}aximal \textbf{L}ikelihood \textbf{E}stimation with \textbf{O}ptimistic \textbf{V}alue \textbf{I}teration (MLE-OVI) algorithm as the following.
\begin{algorithm}[H]
    \begin{algorithmic}[1]
    \Require {$\eta, \cD$}
    \State Initialize $\cD=\emptyset$.
    \For{$t=1,\dots, T$}
    \State Initialize $\hat W_{H+1}^t=0$
    \For{$h=H, H-1,\dots, 1$}.
    \State Obtain kernel $\Lambda_h^t\leftarrow \sum_{i=1}^{t-1} \phi_h(s_h^i, a_h^i, b_h^i) \phi_h(s_h^i, a_h^i, b_h^i)^\top + I$. 
    \State Solve the ridge regression for $\hat \omega_h^t\leftarrow \rbr{\Lambda_h^t}^{-1} \rbr{\sum_{i=1}^{t-1} \phi_h(s_h^i, a_h^i, b_h^i) \bigrbr{u_h^i + \hat W_{h+1}^t (s_{h+1}^i)}} $.
    \State Update $\hat U_h^t(\cdot,\cdot,\cdot)\leftarrow \phi_h(\cdot,\cdot,\cdot)^\top \hat \omega_h^t + \Gamma_h^{(1, t)}(\cdot,\cdot,\cdot)$ and truncate to $[0, H-h+1]$.
    \State Compute $\hat W_h^t(s_h)$ and $\hat\pi_h^t(s_h)$ as the optimal value and optimal solution to S4 \eqref{eq:scheme-4} or S5 \eqref{eq:scheme-5}.
    \EndFor
    \State Announce $\hat\pi^t$ and collect a trajectory $\tau^t = \{(s_h^t, a_h^t, b_h^t, \hat\pi_h^t)\}_{h\in[H]}$. 
    \State $\cD\leftarrow \cD\cup \{\tau^t\}$. 
    \EndFor
    \end{algorithmic}
    \caption{Online MLE-OVI for Myopic Follower under Linear Markov Game}
    \label{alg:MLE-OVI}
\end{algorithm}

We also provide theoretical guarantee for the MLE-OVI algorithm.
\begin{theorem}[{Regret for MLE-OVI}]\label{thm:Online-ML}
    We choose 
    $$\Gamma^{(1, t)}_h(\cdot,\cdot,\cdot) = C_1 d H \allowbreak \sqrt{\log(2d H T^2/\delta)}\cdot \sqrt{\phi_h(\cdot,\cdot,\cdot)^\top (\Lambda_h^t)^{-1}\phi_h(\cdot,\cdot,\cdot)}$$ for some universal constant $C_1>0$ and $\beta = C_2 d\log(HT(1+\eta T^2)\delta^{-1})$ for some universal constant $C_2>0$. 
    For Scheme 4, 
\begin{align*}
    \Reg(T)\lesssim d H^2 \sqrt{d T} + \eta C_\eta H^2 d \sqrt{T} + \exp(4\eta B_A) (\eta C_\eta)^3 H^2 d^2 \log T, 
\end{align*}
and for Scheme 5, 
\begin{align*}
    \Reg(T)\lesssim d H^2 \sqrt{d T} + \exp(4\eta B_A)\eta C_\eta H^2 d \sqrt{T} + \exp(8\eta B_A) (\eta C_\eta)^3 H^2 d^2 \log T. 
\end{align*} 
    \begin{proof}
        See \Cref{sec:proof-Online-ML} for a detailed proof.
    \end{proof}
\end{theorem}

We observe from \Cref{thm:Online-ML} that the the optimistic value iteration methods proposed by both Scheme 4 and Scheme 5 achieve sublinear online regret. 
Scheme 5 suffers from an additional exponential term in the second term, which corresponds to the first order quantal response error. 
What happens here resembles the offline suboptimality bound for Scheme 3 in \Cref{rmk:MLE-PVI-dist-shift} and the reason is quite similar given that we directly use the MLE estimator $\hat\theta^t_{h, \MLE}$ for Scheme 5 instead of exploiting the confidence set for $\theta$.

\ifmain
\subsection{Online Learning with General Function Class}\label{sec:online-myopic-general}
\fi 
\ifneurips
\subsubsection{Online Learning with General Function Class}\label{sec:online-myopic-general}
\fi 
We develop an online learning algorithm very similar to the one given in \Cref{sec:Offline-MG} for the offline general function class, despite that we use optimism in leader's policy learning.
However, for the leader's Bellman loss which is defined by \eqref{eq:myopic-offline-general-Bellman loss} in the offline case, we use a slightly different version that directly incorporates the rule of optimism, 
\begin{align}
    &\!\!\! \ell_h^t(U_h', U_{h+1}, \theta_{h+1}) \nend
    &\!\!\!\quad = \sum_{i=1}^{t-1}  \rbr{U_h'(s_h^i, a_h^i, b_h^i) \!-\! u_h^i - \!\!\!\! \max_{\pi_{h+1}(s_{h+1})\in\sA}\inp[\big]{U_{h+1}(s_{h+1}^i, \cdot, \cdot)}{\pi_{h+1}\otimes \nu_{h+1}^{\pi, \theta}(\cdot, \cdot\given s_{h+1}^i)}_{\cA\times\cB}}^2.\label{eq:online-MG-bellman loss}
\end{align}
Here, there is a slight abuse of notation and we distinguish this loss from its correspondence in the offline case by the superscript $t$. 
We would like to remark that this loss function assembles the one used in GOLF of \citet{jin2021bellman}.
We build the confidence set for both the environment model and the behavior model as
\begin{align}
    &\CI_{\cU, \Theta}^t(\beta) \nend
    &\quad= \cbr{
    (U,\theta)\in\cU^{\otimes H}\times\Theta:
    \rbr{ \ds
        \cL_h^t(\theta_h)-\inf_{\theta_h'\in\Theta_h}\cL_h^t(\theta_h') \le \beta 
    \atop \ds
        \ell_h^t(U_h, U_{h+1}, \theta_{h+1}) - \inf_{U'\in\cU_h} \ell_h^t(U', U_{h+1}, \theta_{h+1})\le H^2\beta}, 
    \forall h\in[H]}. \label{eq:myopic-online-general-confset}
\end{align}
Following the principle of optimism, we can output a pair of optimistic model parameter, 
\begin{align}
    (\hat U^t, \hat\theta^t)=\argmax_{\pi_1\in\Pi_1 \atop (U, \theta)\in\CI_{\cU,\Theta}^t(\beta)} \EE_{s_1\sim\rho_0} \sbr{\inp[\big]{U_1(s_1, \cdot, \cdot)}{\pi_1\otimes\nu_1^{\pi, \theta}(\cdot,\cdot\given s_1)}_{\cA\times\cB}}.\label{eq:online-MG-opt-parameter}
\end{align}
Specifically, for any given state $s_h$, an optimistic policy is then given greedily by 
\begin{align}
    \hat\pi^t(s_h)=\argmax_{\pi_{h}(s_h)\in\sA}\inp[\big]{U_{h}(s_{h}, \cdot, \cdot)}{\pi_{h}\otimes \nu_{h}^{\pi, \theta}(\cdot, \cdot\given s_{h})}_{\cA\times\cB}. \label{eq:online-MG-hat pi}
\end{align}
We summarize the above \textbf{M}aximum \textbf{L}ikelihood \textbf{E}stimation with \textbf{G}lobal \textbf{O}ptimism based on \textbf{L}ocal \textbf{F}itting (MLE-GOLF) algorithm as the following.
\begin{algorithm}[H]
    \begin{algorithmic}[1]
    \Require {$\eta$}
    \State Initiate $\cD =\emptyset$.
    \For{$t=1,\dots,T$}
    \State Construct confidence set $\CI_{\cU, \Theta}^t(\beta)$ by \eqref{eq:myopic-online-general-confset}. 
    \State Solve for $\hat U^t, \hat\theta^t$ by \eqref{eq:online-MG-opt-parameter}.
    \State Solve for the greedy policy $\hat\pi^t$ by \eqref{eq:online-MG-hat pi}. 
    \State Deploy $\hat\pi^t$ and collect a trajectory $\tau^t = \{(s_h^t, a_h^t, b_h^t, \hat\pi_h^t)\}_{h\in[H]}$. 
    \State $\cD\leftarrow \cD\cup \{\tau^t\}$. 
    \EndFor
    \end{algorithmic}
    \caption{Online MLE-GOLF for Myopic Follower under General Function Approximation}
    \label{alg:MLE-GOLF}
\end{algorithm}

We identify two function classes whose complexities determine the online learning hardness. The first one is the Bellman residuals $\cG_L:\cS\times\cA\times\cB\rightarrow \RR$ defined as 
$$\cG_L=\{U_h-\TT_h^{*,\theta} U_{h+1}, U\in\cU, \theta\in\Theta, h\in[H]\}, $$
where $\TT_h^{*,\theta}:\cF(\cS\times\cA\times\cB)\rightarrow \cF(\cS\times\cA\times\cB)$ is the Bellman optimality operator defined in \eqref{eq:define optimistic Bellman opt}.
The second one is the QRE defined in \eqref{eq:QRE} where we define the class of QREs $\cG_F:\cS\times\cB\rightarrow \RR$  as
$$
\cG_F =  \{\Upsilon_h^{\pi}(r_h^{\theta} - r_h), \pi\in\Pi, \theta\in\Theta, h\in[H]\}, 
$$
where recall the operator $\Upsilon_h^\pi:\cF(\cS\times\cA\times\cB)\rightarrow \cF(\cS\times\cB)$ defined as 
\begin{align*}
    \rbr{\Upsilon_h^\pi f}(s_h, b_h) = \dotp{\pi_h(\cdot\given s_h, b_h)}{f(s_h, \cdot, b_h)} - \dotp{\pi_h\otimes \nu_h^{\pi}(\cdot,\cdot\given s_h)}{f(s_h,\cdot,\cdot)}.
\end{align*}
In the sequel, we let $\dim(\cG_L)=\dim_\E(\cG_L, T^{-1/2})$ and $\dim(\cG_F)=\dim_\E(\cG_F, T^{-1/2})$ be the eluder dimensions for these two function classes. \footnote{See \Cref{sec:eluder dimension} for definition of the eluder dimension.} 

\begin{theorem}[{Regret for MLE-GOLF}]\label{thm:Online-MG}
    Suppose that the following conditions hold for model class $\Theta$ and function class $\cU$:
\begin{itemize}
    \item[(i)] (\textit{Realizability}) There exists $\theta^*\in\Theta$ such that $r_h^{\theta^*}=r_h$ for any $h\in[H]$. For $\theta^*\in\Theta$, there exists $U\in\cU$ such that $U_h = U_h^{*}$ for any $h\in[H]$;
    \item[(ii)] (\textit{Completeness}) For any $U\in\cU$ and $\theta\in\Theta$, there exists $U'\in\cU$ such that $U'=\TT_h^{*,\theta} U$ for any $h\in[H]$. 
\end{itemize}
    By choosing $\beta\ge c \max\cbr{ \log(HT\cN_\infty(\cZ, T^{-1})\delta^{-1}), \log(HT \cN_\rho(\Theta, T^{-1})/\delta)}$ for some universal constant $c>0$, where $\cN_\rho(\cZ,\epsilon)$ and $\cN_\rho(\Theta, \epsilon)$ are the maximal (over $h\in[H]$) $\epsilon$-covering number defined in \eqref{eq:cN-cZ}, \eqref{eq:cN-Theta-myopic} for the joint function class $\cZ_h = \Theta_{h+1}\times\cU^2$ and  $\Theta_h$, respectively. 
    We have for the online algorithm that 
    \begin{align*}
        \Reg(T)\lesssim H^2 \sqrt{\dim(\cG_L)\beta T} + H^2 \eta C_\eta \sqrt{\dim(\cG_F)\beta T} + H^2(\eta C_\eta)^3\exp(4\eta B_A) \beta \log T, 
    \end{align*}
    where $\lesssim$ only hides universal constants.
    \begin{proof}
        See \Cref{sec:proof-Online-MG} for a detailed proof.
    \end{proof}
\end{theorem}

\Cref{thm:Online-MG} characterizes the online learning complexity in terms of the eluder dimensions of function classes $\cG_L$ and $\cG_F$.
In particular, we do not suffer from any $\exp(2\eta B_A)$ coefficient in the $\sqrt{T}$ term thanks to optimism, though the $\log(T)$ term still has $C^{(3)}=\cO(\eta B_A\exp(2\eta B_A))$. 
We remark that when applied to the linear function approximation, we can reproduce the result in \Cref{thm:Online-ML} with $\dim(\cG_L) \lesssim d$ and $\dim(\cG_F) \lesssim d$. 
\ifmain
\section{Extension to Learning with Farsighted Follower} \label{sec:farsighted}
\fi
\ifneurips
\subsection{Extension to Learning with Farsighted Follower} \label{sec:farsighted}
\fi
In this section, we explore the possibility of learning the behavior model for a farsighted follower in both the offline and online settings.
We extend our previous techniques to  online learning the QSE with nonmyopic follower. 
To ease our presentation, we use $M=(r^M, u^M, P^M)\in\cM$ to denote the model and all definitions previously introduced with $\theta$ can be naturally extended to this larger model class $M$. In particular, we let $M^*$ denote the true model. We suppose that the follower's reward in our model (which contains the true model) satisfies a linear constraint $\la x(\cdot), r_h^M(s_h, a_h, \cdot)\ra_{\cB} = \varsigma$ for all $(s_h, a_h)\in\cS\times\cA$ and $h\in[H]$, where $x:\cB\rightarrow\RR$ is a known function and $\varsigma\in\RR$ is a known constant. In particular, such a constraint rules out a free dimension in the follower's reward and is introduced to ensure that $r$ can be uniquely identified in the quantal response model. 
For instance, the linear constraint can be that the reward averaged over $\cB$ is $0.5$. We remark that such an assumption is not without loss of generality for a farsighted follower.


\paragraph{Offline Algorithm.}
Although we have gained success in learning the behavior model for the myopic case with both general function approximation and linear function approximation, previous results does not generalize to nonmyopic follower easily. A unique challenge that we face in the farsighted follower case is that the follower's choice model has a long term dependency, which means that the uncertainty in the transition kernel also enters the estimation error of the follower's behavior model. Another challenge is that the follower's choice is a joint effect of the follower's future expected total utilities, and it takes additional efforts to decompose it stepwise so as to utilize the knowledge for planning.

We propose a model based method for jointly learning the environment model and the behavior model. Recall from \Cref{sec:markov_game_def} the model is given by $M=\cbr{P_h^M, r_h^M, u_h^M}_{h\in[H]}$. We consider the following negative generalized likelihood 
\begin{align*}
    \cL_h(M) = - \sum_{i=1}^T \rbr{\eta A_h^{\pi^i, M}(s_h^i, b_h^i) - \log P_h^M(s_{h+1}^i\given s_h^i, a_h^i, b_h^i) - (u_h^i - u_h^M(s_h^i, a_h^i, b_h^i))^2}, 
\end{align*}
where the first term $\eta A_h^{\pi^i, M}$ comes from the follower's quantal response under model $M$ and policy $\pi^i$, and the second term is the likelihood of the transition kernel. 
Note that computing this $A_h^{\pi^i, M}$ actually requires computing the best response for step $h+1,\dots,H$ and update the follower's Q- and V-functions according to \eqref{eq:qv_pi_qr}. 
Using this generalized likelihood, one can directly obtain a confidence set for model $M$ and do planning using pessimism, 
\begin{align*}
    \hat\pi &= \argmax_{\pi\in\Pi} \min_{M\in\CI_M(\beta)} J(\pi, M), \quad \text{where} \\
    \confset_\cM(\beta) &= \cbr{M\in\cM\given \cL_h (M)\le \inf_{M'\in\cM} \cL_h(M') + \beta, \quad \forall h\in[H]}, 
\end{align*}
and
$
    J(\pi, M)
$ is just the total reward for the leader evaluated under an estimated model $M$.
We have the following \textbf{P}essimistic \textbf{M}aximum \textbf{L}ikelihood \textbf{E}stimation algorithm for the offline setting. 
\begin{algorithm}[H]
    \begin{algorithmic}[1]
    \Require {$\eta, \cD$}
    \State Build confidence set $\confset_\cM(\beta) = \cbr{M\in\cM: \cL_h(M)-\inf_{M'\in\cM}\cL_h(M')\le \beta, \forall h\in[H]}$.
    \State Find the pessimistic policy $\hat\pi = \argmax_{\pi\in\Pi}\min_{ M\in\confset_\cM(\beta)} J(\pi, M)$. 
    \Ensure {$\hat\pi$}
    \end{algorithmic}
    \caption{Pessimistic MLE (PMLE) for Offline Learning with Farsighted Follower}
    \label{alg:real-PMLE}
\end{algorithm}
\begin{theorem}[Suboptimality for PMLE]\label{thm:PMLE}
We let $\beta\ge  \allowbreak 9\log(3e^2H \cN_\rho(\cM,T^{-1})\delta^{-1})$, where $\cN_\rho(\cM,\epsilon)$ is the minimal size of an $\epsilon$-optimistic covering net of $\cM$, which is defined in \Cref{lem:MLE}. 
Suppose $\la x, r_h^M(s_h, a_h, \cdot)\ra_\cB = \varsigma$ for all $M\in\cM$, $(s_h, a_h)\in\cS\times\cA$ and $h\in[H]$, and we let $\kappa = \nbr{x}_\infty/|\la x, \ind\ra_\cB|$.
Then with probability at least $1-\delta$, we have for the PMLE algorithm that 
\begin{align*}
    \subopt(\hat\pi) \lesssim C_1^{\pi^*} H^2 \sqrt{ \beta T^{-1}} + C_2^{\pi^*} \eta C_\eta H^{5/2}\sqrt{ \eff_H(\gamma)\beta T^{-1/2}} + C_3^{\pi^*}\cdot C^{(2)} L^{(2)} H \beta T^{-1}, 
\end{align*}
where 
\[C_1^{\pi^*} = \max_{M\in\cM, h\in[H]} \sqrt{\frac{\EE\Bigsbr{\Bigrbr{\bigrbr{U_h^{\pi^*, M}-\bigrbr{u_h+P_h W_{h+1}^{\pi^*, M}}}(s_h, a_h, b_h)}^2 }}{T^{-1}\sum_{i=1}^T \EE^i\Bigsbr{\Bigrbr{\bigrbr{U_h^{\pi^*, M}-\bigrbr{u_h+P_h W_{h+1}^{\pi^*, M}}}(s_h, a_h, b_h)}^2 }}}, \]
\[C_2^{\pi^*} =\max_{M\in\cM, h\in[H]}\sqrt\frac{\EE \Bigsbr{\rbr{\rbr{\EE_{s_h, b_h} -\EE_{s_h}}\sbr{\sum_{l=h}^H \gamma^{l-h}\rbr{r_l^M - r_l + \gamma \bigrbr{P_l^M - P_l} V_{l+1}^{\pi^*, M}}(s_l,a_l,  b_l)}}^2}}{T^{-1}\sum_{i=1}^T \EE^i \Bigsbr{\rbr{\rbr{\EE_{s_h, b_h}^i -\EE_{s_h}^i}\sbr{\sum_{l=h}^H \gamma^{l-h}\rbr{ r_l^M - r_l + \gamma \bigrbr{P_l^M - P_l} V_{l+1}^{\pi^*, M}}(s_l,a_l,  b_l)}}^2}}, \]
\[C_3^{\pi^*} = \max_{h\in[H], M\in\cM}{\frac{\EE\sbr{ \rbr{\EE_{s_h, b_h}\sbr{\bigrbr{r_h^{M} - r_h + \gamma \bigrbr{P_h^M - P_h}  V_{h+1}^{\pi^*, M}}(s_h, a_h, b_h)}}^2}}{T^{-1}\sum_{i=1}^T \EE^i\sbr{ \rbr {\EE_{s_h, b_h}^i\sbr{{\bigrbr{r_h^{M} - r_h + \gamma \bigrbr{P_h^M - P_h}  V_{h+1}^{\pi^*, M}}(s_h,a_h, b_h)}}}^2}}},\]
where the expectation $\EE$ without superscript is taken with respect to $(\pi^*, \nu^{\pi^*})$, and we have constants 
\begin{align*}
    L^{(2)} &= c H^2 \eff_H(c_2)^2 \kappa^2 \exp\rbr{8\eta B_A} (\eta^{-1}+B_A)^2, \\
    C^{(2)}  & =  2  \eta^2 H^2\cdot \exp\orbr{6\eta B_A} \cdot(1+4 \eff_H(\gamma)) \cdot \rbr{\eff_H(\exp(2\eta B_A)\gamma)}^2.
\end{align*}
\end{theorem}
\begin{proof}
    See \Cref{sec:proof-PMLE} for a detailed proof.
\end{proof}

Compared to \Cref{thm:Offline-MG}, the suboptimality for a farsighted follower becomes more complicated in the sense that concentrability coefficient $C_2^{\pi^*}$ for the follower's quantal response error also depend on estimation errors for both the follower's reward $r_h$ and the transition $P_h$ along the optimal trajectory.
In addition, the second order quantal response error, i.e., the last term in the suboptimality, suffers from a $\exp(\cO(\eta H))$ coefficient, which occurs as a result of the quantal response model and the coefficient also grows exponentially with respect to the inverse temperature $\eta$.
Such a result suggests for small data size $T$, the second order QRE terms dominates. 

\paragraph{Online Algorithm.}
Similar to the offline setting, 
the negative generalized-likelihood for the model $M=\cbr{r_h^M, P_h^M, u_h^M}_{h\in[H]}$ at step $h$ and time $t$ is given by
\begin{align}
    \cL_h^t(M)& = -\sum_{i=1}^{t-1} \bigg(\eta A_{h}^{\pi^i, M}(s_h^i, b_h^i) + \log P_h^M(s_{h+1}^i\given s_h^i, a_h^i, b_h^i)   - \rbr{u_{h}^i - u_{h}^M(s_h^i, a_h^i, b_h^i)}^2\bigg),  \label{eq:MLE}
\end{align}
The OMLE algorithm plays the greedy policy with respect to the leader's most favorable model in the $\beta$-superlevel set of the above log-likelihood $\cL^t(\cdot)$, 
\begin{align}
    \pi^t=\argmax_{\pi\in\Pi, M\in\confset_\cM^t(\beta)} J(\pi, M), \quad \st \quad \confset_\cM^t(\beta) = \cbr{M\in\cM\given \cL_h^t (M)\le \inf_{M'\in\cM} \cL_h^t(M') + \beta, \forall h\in[H]}, \label{eq:OMLE}
\end{align}
where
$
    J(\pi, M) 
$
is just the total reward for the leader evaluated under an estimated model $M$.
We have the following \textbf{O}ptimistic \textbf{M}aximum \textbf{L}ikelihood \textbf{E}stimation algorithm for the online setting. 
\begin{algorithm}[H]
    \begin{algorithmic}[1]
    \Require {$\eta, T$}
    \State Initiate $\cD=\emptyset$.
    \For{$t=1,\dots,T$}
    \State Build confidence set $\confset^t_\cM(\beta) = \cbr{M\in\cM: \cL_h^t(M)-\inf_{M'\in\cM}\cL_h^t(M')\le \beta, \forall h\in[H]}$.
    \State Find the optimistic policy $\pi^t = \argmax_{\pi\in\Pi, M\in\confset_\cM^t(\beta)} J(\pi, M)$. 
    \State Announce $\pi^t$ and observe a trajectory $\tau^t = \ocbr{\orbr{s_h^t, a_h^t, b_h^t, \pi_h^t}}_{h\in[H]}$.
    \State Update $\cD \leftarrow \cD \cup \{\tau^t\}$.
    \EndFor
    \end{algorithmic}
    \caption{Optimistic MLE (OMLE) for Online Farsighted Follower}
    \label{alg:OMLE}
\end{algorithm}

In the sequel, we define three types of errors the corresponding (distributional) eluder dimensions for characterizing the online learning complexities.
The first is the leader's Bellman residuals and the class is defined as 
$\cG_L=\{\EE^{\pi}[(U_h^{*, M} - u_h -  W_{h+1}^{*, M})(s_h,a_h,b_h)], M\in\cM, h\in[H]\}$, where we define $U^{*, M}, W^{*, M}$ as the leader's optimistic U- and W-functions under $\pi^{M} = \argmax_{\pi\in\Pi}J(\pi, M)$ and $M$.
The second function class is an analogy to the $\QRE$ in \eqref{eq:QRE} for farsighted follower defined as
\begin{align}\label{eq:GF1-neurips}
     \cG_F^1 = \cbr{(\EE_{s_h, b_h}^{\pi} - \EE_{s_h}^{\pi})\sbr{\sum_{l=h}^H \gamma^{l-h}\bigrbr{r_l^M- r_l + \gamma (P_l^{M} - P_h) V_{l+1}^{*, M}}(s_l, a_l, b_l)}}, 
\end{align}
for all $(\pi, M, h)\in(\Pi, \cM, [H])$ and $(s_h, b_h)\in\cS\times\cB$. Here, we let $\EE_{s_h, b_h}^\pi[\cdot] = \EE^{\pi, \nu^\pi}[\cdot\given s_h, b_h]$ and $V^{*, M}$ is the follower's optimistic V-function under $\pi^M$ and $M$. In particular, when we take $\gamma=0$ for myopic follower, \eqref{eq:GF1-neurips} reduces to the $\QRE$ defined in \eqref{eq:QRE}.
The last function class is unique for the farsighted case, which captures the second order term in the QRE as the follower's Bellman error, 
\ifmain
\begin{align*}
    \cG_F^2=\cbr{\EE_{s_h, b_h}^{\pi} \sbr{\bigrbr{r_h^M- r_h + \gamma (P_h^{M} - P_h) V_{l+1}^{*, M}}(s_h, a_h, b_h)}},
\end{align*}
\fi
for all $(\pi, M, h)\in(\Pi, \cM, [H])$ and $(s_h, b_h)\in\cS\times\cB$.
We denote by $\dim(\cG_L)$, $\dim(\cG_F^1)$, and $\dim(\cG_F^2)$ their eluder dimensions\footnote{See \Cref{sec:eluder farsighted} for definitions.} with properly selected parameters. 

\begin{theorem}[{Regret for OMLE}]\label{thm:OMLE-farsighted}
    We let $\beta\ge  \allowbreak 9\log(3e^2H T\cN_\rho(\cM,T^{-1})\delta^{-1})$, where $\cN_\rho(\cM,\epsilon)$ is the minimal size of an $\epsilon$-optimistic covering net of $\cM$, which is defined in \Cref{lem:MLE}. 
    Suppose $\la x, r_h^M(s_h, a_h, \cdot)\ra_\cB = \varsigma$ for all $M\in\cM$, $(s_h, a_h)\in\cS\times\cA$ and $h\in[H]$, and we let $\kappa = \nbr{x}_\infty/|\la x, \ind\ra_\cB|$.
    Then with probability at least $1-\delta$, we have for the OMLE algorithm that 
    \begin{align*}
        \Reg(T) \le H^2\sqrt{\dim(\cG_L) \beta T} + \eta C_\eta  H^2\eff_H(\gamma) \sqrt{\dim(\cG_F^1) \beta T} + H^2 C^{(2)} L^{(2)} \dim(\cG_F^2)\beta \log(T),
    \end{align*}
    where 
    $L^{(2)} = c H^2 \eff_H(c_2)^2 \kappa^2 \exp\rbr{8\eta B_A}  C_\eta^2$ for some universal constant $c>0$ and for $c_2 = \gamma(2\exp(2\eta B_A)+\kappa\exp(4\eta B_A))$, and $C^{(2)} = 2  \eta^2 H^2  \exp\orbr{6\eta B_A}  (1+4 \eff_H(\gamma)) \cdot \rbr{\eff_H(\exp(2\eta B_A)\gamma)}^2$.
    \begin{proof}
        See \Cref{sec:proof-farsighted MDP} for a detailed proof.
    \end{proof}
\end{theorem}

We remark that our theorem handles a wide range of MDP classes, e.g., linear Markov game in \Cref{def:linear MDP}, linear matrix MDP \citep{zhou2021provably}, or linear mixture MDP \citep{chen2022unified},
where we have
$\dim(\cG_L) = \dim(\cG_F^2) \lesssim d$ and $\dim(\cG_F^1)\lesssim Hd$. 
See \Cref{sec:eluder farsighted} for more details.
We also note that the first order terms (with $\sqrt T$ regret) only depends polynomially on $\eta, H$ and the exponential effect lies in the $\log(T)$ term, which indicates that we can handle a follower with more rationality. Additionally, in the best case we have $\kappa = |\cB|^{-1}$ by taking $x=\ind$ and how to relax this $\la x(\cdot), r_h^M(s_h, a_h, \cdot)\ra_{\cB} = \varsigma$ constraint leaves for future work.


\section{Conclusion}

In this work, we study the problem of learning a quantal Stackelberg equilibrium (QSE) in a leader-follower Markov game, where the follower always adopts the quantal response policy against the leader. 
Moreover, in this game,  the leader cannot observe the follower's rewards, and has to infer how the follower reacts to her announced policy from follower's actions. 
We propose sample efficient algorithms for both the online and offline settings where the follower can be either myopic or  farsighted. 
Besides, we allow the state space to be very large by incorporating   function approximators in value function or transition model estimation. 
Our algorithms are based on a combination of (i) the principle of pessimism/optimism in the face of uncertainty, (ii) model-free or model-based RL for leader's problem, and (iii) maximum likelihood estimation for follower's quantal response mapping. 
We prove that our algorithms are sample efficient. 
Moreover, when specialized to the case with a myopic follower and linear function approximation, our algorithms are also computationally efficient. 
Our theoretical analysis features a novel performance difference lemma that is tailored to the bilevel structure of learning QSE. 
We believe such a result would be useful in other multi-agent RL problems involving inferring the behaviors of other agents.

\newpage
\bibliographystyle{ims}
\bibliography{reference}

\newpage 
\appendix

\newpage

\section{Additional Background and Notation}\label{sec:app-notations}
We stick to the following notations and their shorthands in the appendix.
\begin{table}[H]
\centering
{
    \setlength\doublerulesep{1pt}

\ifmain\begin{tabular}{p{3cm} p{11cm}}\fi
\ifneurips\begin{tabular}{p{4cm} p{8cm}}\fi
    \toprule[2pt]\midrule[0.5pt]
    Notations & Interpretations \\ \toprule[1.5pt]
    $u_h, r_h, P_h$ & the leader's and the follower's rewards and the transition kernel for the MDP.\\\midrule
    $M$, $M^*$, $\cM$ & a model for the leader's and follower's rewards and also the transition kernel. $M^*$ is the true model. The  model class is $\cM$.\\\midrule
    $\theta$, $\Theta$ & We use $\theta$ to denote the part of the model (follower's reward if myopic, and follower's reward and transition kernel if nonmyopic) that determines the follower's quantal response mapping $\pi\rightarrow\nu^\pi$. The model class is $\Theta$.
    \\\midrule
    $\pi, \nu^{\pi, M}, \nu^{\pi,\theta}$ & $\pi$ is the leader's policy, $\nu^{\pi,M}$ is the follower's response under $(\pi, M)$, $\nu^{\pi, \theta}$ has the same meaning as $\nu^{\pi,M}$.
    \\\midrule
    $J(\pi,M)$ & the total utility collected by the leader as the follower responses is $\nu^{\pi, M}$.\\\midrule
    $Q_h^{\pi, M}, V_h^{\pi, M}, A_h^{\pi, M}$, $U_h^{\pi,M}, W_h^{\pi,M}$ & (Q, V, A)-functions for the follower and the (U, W) functions for the leader under $(\pi, \nu^{\pi, M}, P^M, r^M, u^M)$ at step $h$.\\\midrule
    $\EE^{\pi, M}, \EE$ & the expectation $\EE^{\pi, M}$ is taken with respect to the trajectory generated by $(\pi,\nu^{\pi, M}, T^M)$, $\EE$ is a always a short hand for $\EE^{\pi, M^*}$.\\\midrule
    $\EE_{s_h, b_h}, \EE_{s_h}$ & $\EE_{s_h, b_h}$ is a short hand of $\EE^{\pi, M^*}\sbr{\cdot\given s_h, b_h}$ and $\EE_{s_h}$ is a short hand of $\EE^{\pi, M^*}\sbr{\cdot\given s_h}$.\\\midrule
    $r_h^\pi(s_h, b_h)$ & $r_h^\pi = \inp[]{r_h(s_h, \cdot, b_h)}{\pi_h(\cdot\given s_h, b_h)}_\cA$ 
    \\\midrule
    $ P_h^\pi(s_{h+1}\given s_h, b_h)$ & $P_h^\pi = \inp[]{P_h(s_{h+1}\given s_h, \cdot, b_h)}{\pi_h(\cdot\given s_h, b_h)}_\cA$.\\\midrule
    $\pi^*$, $\pi^{*,M}$ & $\pi^*$ is the best policy for the leader, and $\pi^{*, M}$ is the best policy for the leader under model $M$, i.e. $\pi^{*,M} = \argmax_{\pi\in\Pi} J(\pi, M)$,\\\midrule
    $\Upsilon_h^\pi$ & the operator for the QRE defined in \eqref{eq:Upsilon}
    \\\midrule
    $T_h^{\pi,\nu}$, $T_h^{\pi}$, $T_h^{\pi, \theta}$, $T_h^{*}$, $T_h^{*, \theta}$ & $T_h^{\pi,\nu}$, $T_h^{\pi}$, $T_h^{\pi, \theta}$ are integral operators defined in \eqref{eq:operator_T_pi_nu}, \eqref{eq:operator_T_pi}, \eqref{eq:operator_T_pi_theta}, respectively. $T_h^*$ and $T_h^{*,\theta}$ are optimality integral operators defined in \eqref{eq:greedy} and \eqref{eq:optimistic-operator}, respectively.
    \\\midrule
    $\TT_h^{\pi}$, $\TT_h^{\pi,\theta}$, $\TT_h^*$, $\TT_h^{*,\theta}$ & 
    These are Bellman (optimality) operators defined in \eqref{eq:TT_pi}, \eqref{eq:TT_pi_theta}, \eqref{eq:bellman_opt_oper}, \eqref{eq:TT_*_theta}, respectively.\\\midrule
    $B_A$ & $B_A$ upper bounds the follower's Q-, V-, and A-functions and is specified in \eqref{eq:define_BA}.\\

    \bottomrule[2pt]
\end{tabular}
}
\caption{Table for notations.}\label{tab:notation}
\end{table}

In the following, we introduce additional background knowledge that will be useful in the proofs. 
Let $\Theta$ be the set of model parameters that determines the follower's quantal response model. 
In particular, when the follower is myopic, i.e., $\gamma = 0$, 
$\theta$ is the parameter of the follower's reward function.
In this case, we assume that there exists some $\theta^* \in \Theta$ such that $r_h = r_h^{\theta^*}$ for all $h \in [H]$. 
In addition, when the follower is farsighted with $\gamma >0$, we parameterize both the follower's reward function and transition kernel and write $\{ r^{\theta}, P^{\theta} \}_{\theta \in \Theta}$. 
In this case, we assume that there exists $\theta^* \in \Theta$ such that 
$r_h = r_h^{\theta^*}$ and $P_h = P_h^{\theta^*}$ for all $h \in [ H ] $. 
We note that each $\theta$ uniquely specifies a quantal response mapping according to \eqref{eq:quantal_response_policy}.
We let $\nu^{\pi, \theta} $ denote the quantal response of a  leader's policy $\pi$ under the model with parameter $\theta$. 

Moreover, recall that we denote the leader's state- and action-value function by $W$ and $U$ respectively. 
To simplify the notation, 
for any leader's policy $\pi$ and follower's policy $\nu$,  we define integral operators $ T_h^{\pi, \nu }, T_h ^{\pi} \colon  \cF(\cS\times\cA \times \cB)\rightarrow \cF(\cS ) $ by letting 
\#
\bigl ( T_h^{\pi , \nu }   U_{h} \bigr)(s_{h}) = \la U_{h}(s_{h}, \cdot, \cdot), \pi_{h}\otimes \nu_{h}  (\cdot, \cdot\given s_{h})\ra _{\cA\times \cB}, \label{eq:operator_T_pi_nu}\\
\bigl ( T_h^{\pi }   U_{h} \bigr)(s_{h}) = \la U_{h}(s_{h}, \cdot, \cdot), \pi_{h}\otimes \nu_{h}^{\pi} (\cdot, \cdot\given s_{h})\ra _{\cA\times \cB} ,\label{eq:operator_T_pi} 
\#
where $\nu^{\pi} $ is the quantal response of $\pi$ under the true model. 
Then the quantal  Bellman operator $\TT^{\pi}_h$ defined in \eqref{eq:bellman_operator_leader}  can be equivalently written as 
\begin{align}\label{eq:TT_pi}
\bigl( \TT_h^{\pi } f \bigr)  (s_h, a_h, b_h) = u_h(s_h, a_h, b_h) + \EE_{s_{h+1}\sim P_h(\cdot\given s_h, a_h, b_h)} \sbr{\bigl ( T_h^{\pi} f\bigr) (s_{h+1})}.
\end{align}
Similarly, for any $\theta \in \Theta$, we define $T^{\pi, \theta}_h $ by letting 
\begin{align}\label{eq:operator_T_pi_theta}
\big ( T_h^{\pi, \theta} U_{h} \bigr) (s_{h}) = \la U_{h}(s_{h}, \cdot, \cdot), \pi_{h}\otimes \nu_{h}^{\pi,\theta}(\cdot, \cdot\given s_{h})
\ra _{\cA\times \cB} . 
\end{align}
We similarly define a quantal Bellman operator $\TT_h^{\theta}$ by letting 
\begin{align}\label{eq:TT_pi_theta}
    \bigl( \TT_h^{\pi,\theta} f \bigr)  (s_h, a_h, b_h) = u_h(s_h, a_h, b_h) + \EE_{s_{h+1}\sim P_h(\cdot\given s_h, a_h, b_h)} \bigl[ \bigl ( T_h^{\pi,\theta} f\bigr) (s_{h+1})\bigr] .
\end{align}

{\noindent \bf Bellman Optimality  Equation in Myopic Case.}
Furthermore, recall that we let $\Pi = \{ \Pi_h \}_{h\in [ H]} $ denote the class of leader's policies. 
Specifically
 for the case where the follower is 
 myopic, we define an operator 
$T_h^*  \colon \cF(\cS\times\cA\times\cB)\rightarrow \cF(\cS)$ as 
\begin{align}
   \bigl( T_{h}^*  f\bigr) (s_h) = \max_{\pi_h\in\Pi_h} \dotp{f(s_h,\cdot,\cdot)}{\pi_h\otimes \nu^{\pi}_h(\cdot, \cdot\given s_h)}_{\cA\times \cB }. \label{eq:greedy}
\end{align}
In other words, \eqref{eq:greedy} can be regarded as finding the  ``greedy'' policy of the leader, assuming leader's reward function is $f$ and the follower is myopic. 
Based on $T_h^*$, we define the Bellman optimality operator $\TT_h^*\colon \cF(\cS\times\cA\times\cB)\rightarrow \cF(\cS\times\cA\times\cB)
$ 
by letting 
\begin{align}
    \bigl ( \TT_h^{*} f \bigr )  (s_h, a_h, b_h) = u_h(s_h, a_h, b_h) + \EE_{s_{h+1}\sim P_h(\cdot\given s_h, a_h, b_h)} \bigl[ \bigl ( T_{h+1}^{*} f\bigr) (s_{h+1}) \bigr] ,\label{eq:bellman_opt_oper}
\end{align}
Note that we let $\pi^*$  denote the optimal policy of the leader.
Let $W^*$ and $U^*$ denote $W^{\pi^*}$ and $U^{\pi^*}$ respectively, which are the value functions of the leader at QSE. 
Using the Bellman operator defined in \eqref{eq:bellman_opt_oper}, we obtain the \emph{Bellman optimality equation} for the leader: 
\#\label{eq:bellman_opt_eqn}
U_h^* (s,a,b) = \bigl ( \TT_h^* U_{h+1}^* \bigr )(s,a,b), \qquad 
W_h^*(s) =  \bigl( T_{h}^*  U_h^* \bigr) (s),\qquad \forall (s,a,b,h) ,
\# 
where we stick to the convention that $U_{H+1} ^* = \mathbf{0}$. 

The key message conveyed in the  Bellman equation \eqref{eq:bellman_opt_eqn}  is that, when the model is known and the follower is myopic, the optimal policy of the follower can be computed via dynamic programming, which is similar to the case of an MDP.  
However, such a benign property cannot be extended to the farsighted case where $\gamma >0 $. 
The main reason for such dichotomy is that, in the myopic case, for each time step $h$, the quantal response $\nu_h^{\pi}$ only depends on $\pi$ through $\pi_h$. 
Thus, whenIn this case, 
we can regard the leader's problem with an auxiliary MDP with the action space at step $h$ being $\Pi_h $. 
The reward function $\tilde r_h$ and the transition kernel $\tilde P_h$ of such an auxiliary MDP are given by  
\$
\tilde r_h (s_h, \pi_h) & =  \la u_h(s_h, \cdot, \cdot), \pi_h \otimes \nu_h^{\pi}  (\cdot, \cdot \given s_h) \ra _{\cS \times \cA }, \notag \\
\tilde P_h (s_{h+1} \given s_h , \pi_h ) & =  \la  (s_{h+1}  \given s_h, \cdot, \cdot ) , \pi_h \otimes \nu_h^{\pi}  (\cdot, \cdot \given s_h) \ra _{\cS \times \cA },  
\$
where $\pi_h \in \Pi_h$ is an action,   and $u_h$ is the reward of the leader and $P_h$ is the transition kernel of the leader-follower Markov game. 
The optimal policy $\pi^*$ of the leader is exactly the optimal policy of the auxiliary MDP, and thus can be found via dynamic programming. 
In contrast, for a farsighted follower,
for each timestep $h$, 
the quantal response policy $\nu_h^{\pi}$ 
depends on $\pi$ through $\{ \pi_\ell  \}_{\ell \geq h }.$
As a result, the quantal response is a complicated mapping of the leader's policy, which prohibits dynamic programming.\footnote{The  \emph{feedback Stackelberg equilibrium} for leader-follower games with farsighted followers can be solved via dynamic programming. See, e.g., \cite{bacsar1998dynamic} for details. Our notion of Stackelberg equilibrium corresponds to the global Stackelberg equilibrium \citep{bacsar1998dynamic}, which does not admits a dynamic programming formulation in general.} 

Finally, we can also define  the Bellman optimality operator when the follower is critic and has reward $r^{\theta} $. 
In this case, similar to $T_h^*$ defined in \eqref{eq:greedy}, 
we define an operator $T_h^{* , \theta}$   by letting 
\begin{align} \label{eq:optimistic-operator}
\bigl(  T_{h}^{*, \theta } f\bigr) (s_h) = \max_{\pi_h\in\Pi_h} \big \la f(s_h,\cdot,\cdot), \pi_h\otimes \nu^{\pi, \theta}_h (\cdot, \cdot\given s_h)\bigr \ra _{\cA\times \cB },
\end{align}
where $\nu^{\theta} _h $ is the quantal response based on  reward $r^{\theta}_h $.  
Based on $T_{h}^{*, \theta} $, we define  
the  Bellman optimality operator for the leader 
$\TT_h^{*, \theta}:\cF(\cS\times\cA\times\cB)\rightarrow \cF(\cS\times\cA\times\cB)$ by letting 
\begin{align}\label{eq:TT_*_theta}
    \bigl ( \TT_h^{*, \theta} f \bigr )  (s_h, a_h, b_h) = u_h(s_h, a_h, b_h) + \EE_{s_{h+1}\sim P_h(\cdot\given s_h, a_h, b_h)} \bigl[ \bigl ( T_{h+1}^{*, \theta}f\bigr) (s_{h+1}) \bigr] ,
\end{align}
The Bellman optimality equation corresponding to $\TT^{*, \theta}$ can be established similar to \eqref{eq:bellman_opt_eqn}.

\vspace{5pt} 
{\noindent \bf Additional Notation.} 
In the sequel, we let $B_A$ be the global upper bound for the follower's advantage function. We derive an explicit form for $B_A$ in the sequel.
Specifically, for the follower's true advantage functions, we have $\nbr{A_h^\pi}_\infty = \nbr{Q_h^\pi - V_h^\pi}_\infty$.
Here, an upper bound for $V_h^\pi$ can be derived by
\begin{align*}
    \abr{V_h^\pi(s_h)}  = \biggabr{\eta^{-1} \log\biggrbr{\sum_{b_h\in\cB} \exp\rbr{\eta Q_h^\pi(s_h, b_h)}} } \le \eta^{-1} \log\abr{\cB} + \nbr{Q_h^\pi(s_h,\cdot)}_\infty,
\end{align*}
where $\abr{\cB}$ is the number of follower's actions if $\cB$ is discrete, and $\abr{\cB}$ is the length of $\cB$ on the real line for continuous action case. 
Moreover, in the continuous action case, we can always normalize $\cB$ to the unit interval on the real line, i.e., $\cB=[0,1]$, which helps us get rid of the $\eta^{-1}\log|\cB|$ term.  
In the sequel, we just keep this term in the upper bound and remind the readers what $\abr{\cB}$ stands for here.
Therefore, we have that 
\begin{align*}
    \nbr{Q_h^\pi}_\infty \le \nbr{r_h^\pi}_\infty  + \gamma \nbr{V_{h+1}^\pi}_\infty\le \rbr{\gamma \eta^{-1} \log |\cB| + 1} + \gamma \nbr{Q_{h+1}^\pi}_\infty.
\end{align*}
By a recursive argument, we have for the Q function that 
\begin{align}
    \nbr{Q_h^\pi}_\infty \le \eff_{H-h+1}(\gamma)\cdot \rbr{\gamma \eta^{-1}\log\abr{\cB}+1} \eqdef B_Q, \quad \forall h\in[H], \pi\in\Pi, \label{eq:def B_Q}
\end{align}
where we define $\eff_h(\gamma) = (1-\gamma^h)/(1-\gamma)$ as the effective horizon for the follower truncated for $h$ steps and $B_Q$ as the upper bound for the follower's Q function.
Therefore, for the follower's advantage function, we have 
\begin{align*}
    \nbr{A_h^\pi}_\infty &= \nbr{Q_h^\pi - V_h^\pi}_\infty \nend
    &= \max_{(s_h, b_h)\in\cS\times\cB}\Bigabr{Q_h^\pi(s_h, b_h) -  \max_{\nu'\in\Delta(\cB)}\rbr{\dotp{Q_h^\pi(s_h,\cdot)}{\nu'(\cdot)} + \eta^{-1}\cH(\nu')}
    }\nend
    & \le \nbr{Q_h^\pi}_\infty + \eta^{-1} \log\abr{\cB}\nend
    &\le \rbr{1+\eff_{H-h+1}(\gamma) }\cdot  \rbr{\eta^{-1}\log\abr{\cB}+1}. 
\end{align*}
Therefore, it suffices to set 
\#\label{eq:define_BA}
B_A = \rbr{1+\eff_{H}(\gamma) }\cdot  \bigrbr{\eta^{-1}\log\abr{\cB}+1}.
\# 
Note that this $B_A$ also boundes $Q_h$ and $V_h$ by our derivation. Hence, we also denote by $B_A$ the upper bounds for the Q and the V functions for the follower.
\vspace{5pt}

\subsection{Function Classes and Covering Number}\label{sec:covering number}
We define several function class with their corresponding covering numbers that will be used in establishing our learning guarantees.

\paragraph{General Model Class $\cM$.}

we consider $\cM$ to be the model class where each $M\in\cM$ uniquely specifies 
the  environment of a Markov game for both the leader and the follower. 
Specifically, for any $M\in\cM$, with slight abuse of notation, we write  $M=(u^M, r^M, P^M)$, which  specifies the leader's and the follower's reward as well as the transition kernel. We consider the covering number of $\cM$ with respect to the following distance:
\begin{align*}
    \varrho(M,\tilde M)\defeq \max_{h\in[H], \atop 
    (s_h, a_h, b_h)\in\cS\times\cA\times\cB}
    \cbr{\bignbr{u_h- \tilde u_h}_\infty, \bignbr{r_h- \tilde r_h}_\infty, \bignbr{P_h(\cdot\given s_h, a_h, b_h)-\tilde P_h(\cdot\given s_h, a_h, b_h)}_1}, 
\end{align*}
where $(r, u, P)$ is a short hand of $(r^M, u^M, P^M)$ and $(\tilde r, \tilde u, \tilde P)$ is a short hand of $(r^{\tilde M}, u^{\tilde M}, P^{\tilde M})$.
With  this definition, we have the following proposition that bounds the error betweeen policies and value functions in terms of $\varrho (\cdot, \cdot)$.
\begin{proposition} \label{prop:error-prop}
    For any two $M,\tilde M\in\cM$ such that $\rho(M, \tilde M)\le \epsilon$, and two policy $\pi, \tilde\pi$ such that $\bignbr{\pi_h - \tilde\pi_h}_1\le \epsilon$, 
    we let $(U, W, Q, V,   \nu)$ be the value functions and quantal response associated with $\pi$  under $M$,  , and let $(\tilde U, \tilde W, \tilde Q, \tilde V,   \tilde \nu)$ be corresponding terms associated with $\tilde \pi$ under $\tilde M$. 
    Then,  
    for all  $h\in[H]$ and $ (s_h, a_h, b_h)\in\cS\times\cA\times\cB$, we have 
    \begin{align*}
    \bignbr{Q_h - \tilde Q_h}_\infty &\le 2\epsilon\cdot (1+\gamma B_A) \cdot \eff_{H}(\gamma), \nend
    \bignbr{U_h - \tilde U_h}_\infty &\le \epsilon H \cdot \bigrbr{4\eta H (1+\gamma B_A) \cdot \eff_H(\gamma) + 1 + 2H}, \nend
    \bignbr{V_h - \tilde V_h }_\infty 
    &\le \bignbr{Q_h - \tilde Q_h}_\infty \le 2\epsilon\cdot (1+\gamma B_A) \cdot \eff_{H}(\gamma), \nend
    \bignbr{W_h - \tilde W_h}_\infty &\le \epsilon (H+1) \cdot \bigrbr{4\eta H (1+\gamma B_A)\cdot \eff_H(\gamma) + 1 + 2H}. 
    \end{align*}
    Meanwhile, 
   for the quantal response, 
    \begin{align*}
        D_\H^2(\tilde \nu_h(\cdot\given s_h), \nu_h(\cdot\given s_h)) \le \nbr{\nu_h(\cdot\given s_h) - \tilde \nu_h(\cdot\given s_h)}_1 \le 8 \eta \epsilon \cdot (1+\gamma B_A)  \cdot \eff_{H}(\gamma).
    \end{align*}
\end{proposition}
\begin{proof}
    See Appendix \ref{proof:prop:error-prop} for a detailed proof.
\end{proof}
Therefore, we see it suffices to control $\varrho(\cdot, \cdot)$ in order for the value functions for both the follower and the leader to be under control. Moreover, the quantal response is also under control.
Therefore, we define a new distance $\rho$ for $\cM$ as 
\begin{equation}\label{eq:rho-cM}
    \rho(M,\tilde M)\defeq 6 \!\!\!\!\!\!\!\!\!\!\!\!\max_{\pi\in\Pi, h\in[H], \atop 
    (s_h, a_h, b_h)\in\cS\times\cA\times\cB}\!\!
    \cbr{\begin{aligned}
        &\bignbr{u_h- \tilde u_h}_\infty, \bignbr{r_h- \tilde r_h}_\infty, D_\H\orbr{P_h(\cdot\given s_h, a_h, b_h), \tilde P_h(\cdot\given s_h, a_h, b_h)}\\
        & \bignbr{Q_h^\pi - \tilde Q_h^\pi}_\infty, \bignbr{U_h^\pi - \tilde U_h^\pi}_\infty, D_\H(\tilde \nu_h^\pi(\cdot\given s_h), \nu_h^\pi(\cdot\given s_h)) 
    \end{aligned}}
\end{equation}
And we denote by $\cN(\cM)=\cN_\rho(\cM, T^{-1})$ the covering number for model class $\cM$. Note that $\cN_\rho(\cM,\epsilon)$ is related to $\cN_\varrho(\cM, \epsilon)$ only by a change of $\epsilon$ according to \Cref{prop:error-prop} where we just take $\pi=\tilde\pi$.

\paragraph{Response Model Class $\Theta$.}
Let $\Theta$ be the set of model parameters that determines the follower's quantal response model. 
In particular, when the follower is myopic, i.e., $\gamma = 0$, 
$\theta$ is the parameter of the follower's reward function.
In this case, we assume that there exists some $\theta^* \in \Theta$ such that $r_h = r_h^{\theta^*}$ for all $h \in [H]$. 
In addition, when the follower is farsighted with $\gamma >0$, we parameterize both the follower's reward function and transition kernel and write $\{ r^{\theta}, P^{\theta} \}_{\theta \in \Theta}$. 
In this case, we assume that there exists $\theta^* \in \Theta$ such that 
$r_h = r_h^{\theta^*}$ and $P_h = P_h^{\theta^*}$ for all $h \in [ H ] $. 
We note that each $\theta$ uniquely specifies a quantal response mapping according to \eqref{eq:quantal_response_policy}.
We consider distance $\rho$ for model $\Theta$ as
\begin{equation}\label{eq:rho-Theta}
    \rho(\theta,\tilde\theta)\defeq \max_{\pi\in\Pi, h\in[H], \atop 
    (s_h, a_h, b_h)\in\cS\times\cA\times\cB}
    \cbr{\begin{aligned}
        &\bignbr{r_h- \tilde r_h}_\infty, \bignbr{P_h(\cdot\given s_h, a_h, b_h) - \tilde P_h(\cdot\given s_h, a_h, b_h)}_1\\
        & \bignbr{Q_h^\pi - \tilde Q_h^\pi}_\infty, D_\H(\tilde \nu_h^\pi(\cdot\given s_h), \nu_h^\pi(\cdot\given s_h)) 
    \end{aligned}}, 
\end{equation}
where $(r, P, Q, \nu)$ is given under $\theta$ and $(\tilde r, \tilde P, \tilde Q, \tilde \nu)$ is given under $\tilde \theta$.
We note that $\theta$ is just a subclass of $\cM$. We denote by $\cN_\rho(\Theta, \epsilon)$ the covering number of $\Theta$ with respect to this $\rho$.

Notably, when the follower is myopic, we just need to cover each $\Theta_h$, where $\Theta_h$ only contains the parameters for $r_h^\theta$, with respect to the following distance 
\begin{align}\label{eq:rho-Theta_h}
    \rho(\theta_h,\tilde\theta_h) \defeq \max_{s_h\in\cS} \cbr{(1+\eta)\|r_h-\tilde r_h\|_\infty, D_\H(\tilde \nu_h^\pi(\cdot\given s_h), \nu_h^\pi(\cdot\given s_h)) }.
\end{align}
Here, the additional $(1+\eta)$ term is needed by \Cref{lem:MLE-formal}.
A covering number for $\Theta_h$ can be thus denoted by $\cN_\rho(\Theta_h, \epsilon)$.
Only for the myopic case, we denote by 
\begin{align}\label{eq:cN-Theta-myopic}
    \cN_\rho(\Theta, \epsilon) = \max_{h\in[H]} \cN_\rho(\Theta_h, \epsilon).
\end{align}
For the nonmyopic case, this covering number $\cN_\rho(\Theta, \epsilon)$ is with respect to the distance defined in \eqref{eq:rho-Theta}.

We calculate this covering number in \eqref{eq:cN-Theta-myopic} for a given step $h\in[H]$ where the follower's reward parameter space is $\Theta_h$ in the linear myopic case.
Specifically, we don't need to consider the transition for myopic follower, and the Q function is simply the reward function. Therefore, we just need to bound $\|r_h-\tilde r_h\|_\infty\le \min \cbr{\epsilon^2 /(8 \eta), \epsilon /(1+\eta)} $, and by \Cref{prop:error-prop}, we have $D_\H(\tilde \nu_h^\pi(\cdot\given s_h), \nu_h^\pi(\cdot\given s_h)) <\epsilon$. Therefore, a covering number for $\Theta_h$ is given by 
\begin{align}\label{eq:cN-Theta_h}
    \log \cN_\rho(\Theta_h, \epsilon) \le d \log\rbr{1+\frac{2 B_\Theta B_\phi}{\min \cbr{\epsilon^2 /(8 \eta), \epsilon /(1+\eta)}}} \lesssim d\log\rbr{1+ \eta /\epsilon^2 + (1+\eta)/\epsilon}.
\end{align}
where $B_\Theta$ bounds the $\Theta_h$ class in the 2 norm and $B_\phi$ bounds the feature mapping $\phi$ in the 2 norm for each $(s_h, a_h, b_h)$. The covering number for $\Theta$ is just $H\cN_\rho(\Theta_h, \epsilon)$.

We remark that $\cM$ is strictly larger than $\Theta$ in the sense that $\cM$ also contains the leader's reward.
We introduce $\cM$ for model-based learning for the leader in \Cref{sec:farsighted}, and $\Theta$ is used for learning the quantal response mapping via model-based maximum likelihood estimation as we have discussed before.

\paragraph{Leader's Value Function Class  $\cU$.}
For both online and offline learning the leader's value function with general function approximation and myopic follower, we introduce function class $\cU:\cS\times\cA\times\cB\rightarrow \RR$, which we assume is uniformly bounded by $H$ and the following completeness and realizability assumption holds. 
\begin{condition}\label{cond:real-comp}
We say that $\cU$ satisfies the realizability and the completeness conditions if the followings hold, 
\begin{itemize}
    \item[(i)] (\textit{Realizability}) There exists $\theta^*\in\Theta$ such that $r_h^{\theta^*}=r_h$ for any $h\in[H]$. For any $\pi\in\Pi, \theta\in\Theta$, there exists $U\in\cU$ such that $U_h = \TT_h^{\pi,\theta} U_{h+1}$ for any $h\in[H]$;
    \item[(ii)] (\textit{Completeness}) For any $U\in\cU$, $\pi\in\Pi$, $\theta\in\Theta$, and $h\in[H]$, there exists $U'\in\cU$ such that $U'=\TT_h^{\pi,\theta} U_{h+1}$. 
\end{itemize}
\end{condition}
We consider the covering number of $\cU$ with respect to $\rho(U, \tilde U) = \onbr{U-\tilde U}_\infty$. Specifically, we denote by $\cN(\cU)=\cN_\rho(\cU, T^{-1})$ the covering number of $\cU$. In the following, we characterize two joint function classes including $\cU$ and $\Theta$ that will be used under the general function approximation setting with myopic follower.

\paragraph{Joint Class $\cY$.}
For studying the offline case with myopic follower using general function approximation, we have a joint class $\cY_h = \Theta_{h+1}\times \Pi_{h+1}\times \cU^2$, where $\Theta_h$ only contains the follower's reward at step $h$. 
This class is used for studying the suboptimality of the MLE-BCP in \Cref{alg:MLE-BCP}, where $\cU^2$ is used to approximate the leader's value function $U_h^{\pi, \theta}, U_{h+1}^{\pi,\theta}$, $\Theta_{h+1}$ is class for the follower's reward at step $h+1$, and $\Pi_{h+1}$ is the leader's policy class at step $h+1$.
We consider the following distance, 
  \begin{equation}\label{eq:rho-cY}
    \rho(y, \tilde y) = \max_{s_{h+1}\in\cS}\cbr{\begin{aligned}
        &\bignbr{U_h-\tilde U_h}_\infty,  \bignbr{U_{h+1}-\tilde U_{h+1}}_\infty, \\
        &\sup_{s_{h+1}\in\cS}\bignbr{(\pi_{h+1}\otimes \nu_{h+1}^{\pi, \theta}-\tilde \pi_{h+1}\otimes \nu_{h+1}^{\tilde\pi, \tilde\theta})(\cdot, \cdot\given s_{h+1})}_1
    \end{aligned}}.
\end{equation}
We let $\cN(\cY)=\cN_\rho(\cY, T^{-1})$. Define $\cN_\rho(\Pi, \epsilon)$ as the $\epsilon$-covering number of the policy class $\Pi$ with respect to distance $\rho(\pi_h, \tilde\pi_h) = \max_{(s_h, b_h)\in\cS\times\cB} \nbr{(\pi_h - \tilde\pi_h)(\cdot\given s_h, b_h)}_1$.
It is easy to see that $\cN_\rho(\cY_h, \epsilon)\le \cN_\rho(\cU,\epsilon)\cdot \cN_\rho(\Pi_{h+1}, \epsilon')\cdot\cN_\rho(\Theta_{h+1}, \epsilon') $ following the result in \Cref{prop:error-prop}, where for the myopic follower, 
$$\epsilon' = \frac{\epsilon}{8\eta (1+\gamma B_A)\eff_H(\gamma)} = \frac{\epsilon}{8\eta}.$$
In particular, we let 
\begin{align}\label{eq:cN-cY}
    \cN_\rho(\cY, \epsilon) = \max_{h\in[H]} \cN_\rho(\cY_h, \epsilon).
\end{align}

\paragraph{Joint Class $\cZ$. }
For the online myopic setting with general function approximation $\cU$, we consider a joint function class $\cZ_h=\cU^2\times\Theta_{h+1}$. 
This class is used for studying the regret of the MLE-GOLF in \Cref{alg:MLE-GOLF}, where $\cU^2$ is used to approximate the leader's optimistic value function $U_h^{*,\theta}, U_{h+1}^{*,\theta}$, $\Theta_{h+1}$ is class for the follower's reward at step $h+1$.
We consider the following distance for $z, \tilde z\in\cZ_h$, 
    \begin{align}\label{eq:rho-cZ}
        &\rho\orbr{z, \tilde z}  = \max_{h\in[H]}\cbr{\bignbr{U_h-\tilde U_h}_\infty, \bignbr{ T_{h+1}^{*,\theta} U_{h+1} (\cdot) -  T_{h+1}^{*,\tilde\theta} \tilde U_{h+1} (\cdot)}_\infty }, 
    \end{align} 
    where the optimistic operator $T_{h+1}^{*, \theta}$ is defined in \eqref{eq:optimistic-operator}.
    We denote by $\cN(\cZ_h)=\cN_\rho(\cZ_h, T^{-1})$ the covering number of the smallest $T^{-1}$-covering net for $\cZ_h$ with respect to this distance $\rho(\cdot,\cdot)$. One should notice that for any $\bignbr{U_{h+1} - \tilde U_{h+1}}_\infty\le \epsilon$, $\onbr{r_{h+1}^\theta - r_{h+1}^{\tilde \theta}}_\infty \le \epsilon$, we have
    \$
    &\bignbr{ T_{h+1}^{*,\theta} U_{h+1} (\cdot) -  T_{h+1}^{*,\tilde\theta} \tilde U_{h+1} (\cdot)}_\infty \nend
    &\quad \le \max_{\pi\in\Pi}\cbr{\bignbr{T_{h+1}^{\pi,\theta} U_{h+1} (\cdot) -  T_{h+1}^{\pi,\tilde\theta} \tilde U_{h+1} (\cdot)}_\infty, \bignbr{T_{h+1}^{\tilde\pi,\theta} U_{h+1} (\cdot) -  T_{h+1}^{\tilde\pi,\tilde\theta} \tilde U_{h+1} (\cdot)}_\infty}\nend
    &\quad \le \max_{\pi\in\Pi} \bignbr{W_{h+1}^\pi - \tilde W_{h+1}^\pi}_\infty \nend
    &\quad \le \epsilon (H+1) \rbr{4\eta H + 1 + 2H}
    \$
    for the myopic case, 
    where we let $\pi$ be the optimal policy under $\theta, U$ and $\tilde\pi$ be the optimal policy under $\tilde \theta, \tilde U$. Here, the last inequality comes from \Cref{prop:error-prop} with $\lambda=0$.
    Therefore, $\cN_\rho(\cZ_h, \epsilon)\le \cN_\rho(\cU, \epsilon) \cdot \cN_\rho(\cU, \epsilon') \cdot \cN_\rho(\Theta_{h+1}, \epsilon')$, where 
    $$\epsilon' = \frac{\epsilon}{(H+1)(4\eta H +1 + 2H)}. $$

This definition will be used in \Cref{lem:MLE} for the farsighted follower case where we directly incorporate MLE algorithm using $\cM$.
In particular, we let
\begin{align}\label{eq:cN-cZ}
    \cN_\rho(\cZ, \epsilon) = \max_{h\in[H]} \cN_\rho(\cZ_h, \epsilon). 
\end{align}


\subsection{Eluder Dimension}\label{sec:eluder dimension}
we present the definition of (distributional) eluder dimension 
that will be useful for our online learning purpose.

\begin{definition}[Eluder Dimension]\label{def:Eluder dimension}
    Let $\cG$ be a function class defined on $\cX$.
	The  eluder dimension $\dim_\E(\cG,\epsilon)$ is the length of the longest sequence $\{x_1, \ldots, x_n\} \subset \cX$ such that there exists $\epsilon'\ge\epsilon$ where for all $m \in [n]$, there exists $g_m\in\cG$ such that
    \begin{align*}
        \sum_{i=1}^{m-1} \rbr{g_m(x_i)}^2 \le \epsilon', \quad \abr{g_m(x_m)} > \epsilon'.
    \end{align*}
\end{definition}
We can similarly define an eluder dimension for signed measure.
\begin{definition}[Eluder Dimension for Signed Measures]
    \label{def:DE}
    Let $\cG$ be a function class defined on measurable space $\cX$, and $\sP$ be a family of signed measures over $\cX$. Suppose that any $g\in\cG$ is integrable with respect to any $\rho\in\sP$.
	The eluder dimension for signed measures with respect to $\cG$ and $\sP$ is denoted by $\dim_\DE(\cG,\sP,\epsilon)$, which is the length of the longest sequence $\{\rho_1, \ldots, \rho_n\} \subset \sP$ satisfying the following condition:  
    there exists $\epsilon'\ge\epsilon$ where for any $m\in[n]$, there exists $g_m\in\cG$ such that  
    \begin{align*}
        \sum_{i=1}^{m-1} \rbr{\int_\cX g_m \rd \rho_i}^2 \le \epsilon', \quad \abr{\int_\cX  g_m \rd\rho_m} > \epsilon'.
    \end{align*}
\end{definition}

In the sequel, we will denote by $\dim(\cG)=\dim_\DE(\cG, \sP, T^{-1/2})$ the eluder dimension with respect to the class of signed measures $\sP$. Note that the standard distributional eluder dimension is a special case for the eluder dimension for signed measures.

\subsubsection{Eluder Dimensions in Myopic Case}\label{sec:eluder myopic}
We discuss the function classes and their corresponding eluder dimensions (with respect to signed measures) that characterize the hardness of leader's exploration problem in the face of a myopic  follower. 
Such a complexity measure will be used for analyzing the regret of the MLE-GOLF algorithm.
Detailed proofs of the regret upper bound in terms of  these eluder dimensions  are available in \Cref{sec:proof-Online-MG}.
\paragraph{Eluder Dimension for Leader's Bellman Error.}
As is shown in the proof \Cref{sec:proof-Online-MG}, the leader's Bellman error we are dealing with is 
\begin{align*}
    \sum_{t=1}^T \EE^{\hat\pi^t} [(\hat U_h^t - \TT_h^{*, \hat\theta^t} \hat U_{h+1}^t)(s_h, a_h, b_h)], 
\end{align*}
where $\hat U^t, \hat\theta^t$ are just the optimistic estimators obtained at episode $t$.
We define a class of functions that corresponds to this error
\begin{align*}
    \cG_L=\cbr{g:\Pi\times\cU^2\times\Theta\rightarrow \RR: g=\EE^\pi [U_h-\TT_h^{*,\theta} U_{h+1}]}. 
\end{align*}
For the leader's Bellman error, we have the following configurations:
\begin{itemize}[leftmargin=20pt]
    \item[(i)] Define function class $\cG_{h,L} = \bigcbr{g:\cS\times\cA\times\cB\rightarrow \RR\given  g = U_h-\TT_h^{*, \theta} U_{h+1}, \exists U\in\cU, \theta\in\Theta}$. 
    Moreover, consider sequence $\{g_h^i = \hat U_h^i - \TT_h^{*, \hat\theta^i} \hat U_{h+1}^i\}_{i\in[T]}$, where $\hat\theta^i, \hat U^i$ are the estimated $\theta$, $U$ at episode $i$.  
    It is obvious that $g_h^i\in\cG_{h,L}$;
    \item[(ii)] Define the class of probability measures as 
    $$\sP_{h,L} = \cbr{\rho\in\Delta(\cS\times\cA\times\cB)\given\rho(\cdot)=\PP^\pi((s_h, a_h, b_h)=\cdot), \pi\in\Pi}.$$
    Moreover, consider sequence $\{\rho^i(\cdot) = \PP^{\hat\pi^i}((s_h, a_h, b_h)=\cdot)\}_{i\in[T]}$, where $\hat\pi^i$ is the optimistic policy used at episode $i$;
\end{itemize}
Under these definitions, we let low rank property for $\cG_L$ that $\dim(\cG_L) = \max_{h\in[H]}\dim_\DE\rbr{\cG_{h, L}, \sP_{h, L}, T^{-1/2}}$.
This configuration works because for the chosen sequences, we have
\begin{align*}
    \EE_{\rho_h^i}[g_h^t] = \EE^{i} \sbr{(\hat U_h^t - \TT_h^{*, \hat\theta^t} \hat U_{h+1}^t)(s_h,a_h,b_h)}.
\end{align*}
If $i=t$, this will be the leader's Bellman error we aim to bound.
Moreover, the online guarantee will be
$$\sum_{i=1}^{t-1} \rbr{\EE_{\rho^i}\sbr{g_h^t}}^2 = \sum_{i=1}^{t-1}\EE^{i}[(\hat U_{h}^{t} - \TT_{h}^{*, \hat\theta^t} \hat U_{h + 1}^t)^2]\lesssim H^2\beta$$ for any $t\in[T]$, which enables us to use \Cref{lem:de-regret}.
In particular, for the 
linear Markov Game in \Cref{def:linear MDP}, we have $\dim(\cG_L)\lesssim d$ as is discussed in \citet{jin2021bellman}.

\paragraph{Eluder Dimension for Follower's Quantal Response Error.}
As is shown in the proof of \Cref{sec:proof-Online-MG}, 
the follower's quantal response error is characterized by the following term, 
\begin{align*}
    \sum_{t=1}^T \bigrbr{\Upsilon_h^{\hat\pi^t} (r_h^{\hat\theta^t}-r_h)}(s_h^t,b_h^t)= \sum_{t=1}^T\QRE(s_h^t, b_h^t;\hat\theta^t, \hat\pi^t).
\end{align*}
We define 
\begin{align*}
    \cG_F = \cbr{g:\cS\times\cB\times\Pi\times\Theta\rightarrow \RR: g=\Upsilon_h^{\pi}(r_h^\theta - r_h)(s_h, b_h)},
\end{align*}
where we  recall the linear operator $\Upsilon_h^\pi:\cF(\cS\times\cA\times\cB)\rightarrow \cF(\cS\times\cB)$ defined as 
\begin{align*}
    \rbr{\Upsilon_h^\pi f}(s_h, b_h) = \dotp{\pi_h(\cdot\given s_h, b_h)}{f(s_h, \cdot, b_h)} - \dotp{\pi_h\otimes \nu_h^{\pi}(\cdot,\cdot\given s_h)}{f(s_h,\cdot,\cdot)}. 
\end{align*}
Here, $(\Upsilon_h ^{\pi} f ) (s_h,b_h) $ quantifies  how far $b_h$ is from being the quantal response of $\pi$ at state $s_h$, measured in terms of  $f$. One can also think of $(\Upsilon_h^\pi f)(s_h, b_h)$ as the \say{advantage} of the reward induced by action $b_h$ compared to the reward induced by the quantal response. 
We have the following configurations for the follower's quantal response error:
\begin{itemize}[leftmargin=20pt]
    \item[(i)] We define function class on $\cS\times\cA\times\cB$ as,
    \begin{align*}
    \cG_{h, F} = \cbr{g:\cS\times\cA\times\cB\rightarrow \RR: \exists \theta\in\Theta, g= {r_h^\theta - r_h}}.
    \end{align*}
    In addition, we consider a sequence $\{g_h^i = r_h^{\hat\theta^i} - r_h\}_{i\in[T]}$. 
    \item[(ii)] We define a class of signed measures on $\cS\times\cA\times\cB$ as
    \begin{align*}
        \sP_{h, F} &= \{\rho(\cdot) = \PP^\pi(a_h=\cdot\given s_h, b_h)\delta_{(s_h, b_h)}(\cdot) 
        - \PP^\pi((a_h, b_h)=\cdot\given s_h)\delta_{(s_h)}(\cdot)\nend
        &\qqquad \biggiven \pi\in\Pi, (s_h, b_h)\in\cS\times\cB\}, 
    \end{align*}
    where $\delta_{(s_h, b_h)}(\cdot)$ is the measure that assigns measure $1$ to state-action pair $(s_h, b_h)$. 
    In addition, we consider a sequence 
    $$\cbr{\rho_h^i(\cdot) = \PP^{\hat\pi^i}(a_h=\cdot\given s_h, b_h)\delta_{(s_h, b_h)}(\cdot) 
    - \PP^{\hat\pi^i}((a_h, b_h)=\cdot\given s_h)\delta_{(s_h)}(\cdot)}_{i\in[T]}.$$
\end{itemize}
Under these definitions, we let $\dim(\cG_F) = \max_{h\in[H]}\dim_\DE\rbr{\cG_{h, F}, \sP_{h, F}, T^{-1/2}}$.
For simplicity, we denote by $g_h^t(s_h^i, b_h^i, \pi^i)$ the integral of $g_h^t$ with respect to the signed measure $\rho_h^i$ since $\rho_h^i$ is uniquely determined by $(s_h^i, b_h^i, \pi^i)$. 
    It is easy to check that $$g_h^t(s_h^i, b_h^i, \pi^i) = \bigrbr{\Upsilon_h^{\pi^i}(r_h^{\hat\theta^t}-r_h)}(s_h^i, b_h^i) = \QRE(s_h^i, b_h^i; \hat\theta^t, \pi^i).$$
This configurations work because when we take $t=i$, $g_h^t(s_h^t, b_h^t, \pi^t)$ will be exactly the quantal response error, and the online guarantee for these two sequences are 
    \begin{align*}
        \sum_{i=1}^{t-1} (g_h^t(s_h^i, b_h^i, \pi^i))^2 \lesssim C_\eta^2\beta, 
    \end{align*}
by our MLE guarantee, which enables us to use \Cref{lem:de-regret}.
Moreover, we remark that under the linear Markov game setting, i.e., $r_h(s_h, a_h,b_h) = \la \phi_h(s_h, a_h,b_h), \theta\ra_\cB$, we simply have $g_h^t(s_h^i, b_h^i, \pi^i) = \la (\Upsilon_h^{\pi^i}\phi_h)(s_h^i, b_h^i), \hat\theta_h^{t} - \theta_h^*\ra_\cB$, which has a bilinear form and we simply have for the eluder dimension $\dim(\cG_{F})\lesssim d$.

\subsubsection{Eluder Dimensions in Farsighted Case}
\label{sec:eluder farsighted}
We discuss the function classes and their corresponding eluder dimensions (with respect to signed measures) that characterize the hardness of leader's exploration problem in the face of a farsighted follower. 
Such a complexity measure will be used for analyzing the regret of the OMLE algorithm. 
Detailed proof using the eluder dimensions of these classes are available in \Cref{sec:proof-farsighted MDP}.

\paragraph{Eluder Dimension for the Leader's Bellman Error.}
As is shown in \Cref{sec:proof-farsighted MDP}, the leader's Bellman error is 
\begin{align*}
    \sum_{t=1}^T \EE^{\pi^t} \sbr{\bigrbr{\tilde U_h^t - u_h}(s_h, a_h, b_h) - \tilde W_{h+1}^t(s_{h+1})}, 
\end{align*}
where $\tilde U^t=U^{\pi^t, M^t}$, $\tilde W^t=W^{\pi^t, M^t}$, and $M^t,\pi^t$ are the optimistic model estimator and optimistic policy obtained at episode $t$. 
We define the following error class
\begin{align*}
    \cG_L = \cbr{\cM\times\Pi\rightarrow\RR: \EE^\pi\sbr{\bigrbr{U_h^{\pi^M, M} - u_h}(s_h, a_h, b_h) - W_{h+1}^{\pi^M,M}(s_{h+1})}}, 
\end{align*}
where $\pi^M=\argmax_{\pi\in\Pi} J(\pi, M)$ is the optimistic policy corresponding to model $M$.
To formally define the eluder dimension for $\cG_L$, we consider the following configurations for step $h\in[H]$.
\begin{itemize}[leftmargin=20pt]
    \item[(i)] Define function class $\cG_{h,L}$ as
    \begin{align*}
        \cG_{h, L} &= \Big\{g:\cS\times\cA\times\cB\rightarrow \RR \Biggiven g={\bigrbr{U_h^{\pi^{\tilde M},\tilde M} - u - P_h W_{h+1}^{\pi^{\tilde M},\tilde M}}(s_h, a_h, b_h)}, \exists \tilde M\in\cM\Big\}, 
    \end{align*}
    where we define $\pi^M = \argmax_{\pi\in\Pi}J(\pi, M)$.
    Specifically, the expectation is taken under $\pi$ and the true model. 
    Consider a sequence of function $\{g_h^i = (\tilde U_h^{i} - u - P_h \tilde W_{h+1}^i)\}_{i\in[T]}$. We have $g_h^i\in\cG_{h, L}$ since we define $\tilde U^i = U^{\pi^i, M^i}$ and we have by the optimism in the algorithm that $\pi^i = \pi^{M^i}$. The same also holds for $\tilde W^i$.
    \item[(ii)] Define a class of probability measures over $\cS\times\cA\times\cB$ as $$\sP_{h, L}=\{\PP^\pi((s_h, a_h, b_h)=\cdot), \forall \pi\in\Pi\}.$$
    Consider a sequence of probability measures $\{\rho_h^i(\cdot)=\PP^{\pi^i}((s_h, a_h, b_h)=\cdot)\}_{i\in[T]}$, where $\pi^i$ is the policy used at episode $i$.
    \item[(iii)] Under these two sequences, we denote by $g_h^t(\pi^i) = \EE_{\rho_h^i}[g_h^t]$ for simplicity. 
    We have 
    $$g_h^t(\pi^i) =\EE^{\pi^i}\bigsbr{{\bigrbr{\tilde U_h^t - u - P_h \tilde W_{h+1}^t}(s_h,a_h,b_h)}}, $$
    which should be bounded by $3H$.
\end{itemize}
We thus denote by $\dim(\cG_L) = \max_{h\in[H]}\dim_\DE(\cG_{h, L}, \sP_{h, L}, T^{-1/2})$. This is actually the distributional eluder dimension of $\cG_{h, L}$, which satisfies $\dim(\cG_L)\lesssim d$ for the linear Markov game defined in \Cref{def:linear MDP} because for any $g\in\cG_{h, L}$ and $\rho\in\sP_{h, L}$, the integral of $g$ with respect to $\rho$ always admits a bilinear form $\EE_\rho[g]= \la \EE_\rho[\phi_h(s_h, a_h, b_h)], \zeta^g\ra$. 


\paragraph{Eluder Dimension for the  First-Order Error of the Follower's Quantal Response.}
As we will show in the proof \Cref{sec:proof-farsighted MDP}, the first order term in the quantal response error is just
\begin{align*}
    \sum_{h=1}^H\sum_{t=1}^T |\tilde\Delta_h^{1, t}(s_h^t, b_h^t)|,
\end{align*}
where $\tilde\Delta_h^{1, t}$ is defined as
\begin{align*}
    \tilde \Delta^{(1, t)}_h(s_h, b_h) &=  \rbr{\EE_{s_h, b_h}^t -\EE_{s_h}^t}\Biggsbr{\sum_{l=h}^H \gamma^{l-h}{\rbr{\tilde Q_l^t - r_l^{\pi^t} - \gamma P_l^{\pi^t} \tilde V_{l+1}^t}(s_l, b_l)}}, 
\end{align*} 
Here, $\tilde r_l^t = r_l^{M^t}$, $\tilde P_l^t = P_l^{M^t}$, $\tilde V^t = V^{\pi^t, M^t}$, and $\tilde Q^t=Q^{\pi^t, M^t}$. $M^t$ is the estimated model and $\pi^t$ is the optimistic policy at episode $t$. In particular, $\pi^t = \argmax_{\pi\in\Pi}J(\pi, M^t) \eqdef \pi^{M^t}$.
We define 
\begin{align*}
    \cG_F^1 &= \bigg\{\Pi\times\cM\times\cS\times\cB\rightarrow \RR: \nend
    &\qqquad \ts
    (\EE_{s_h, b_h}^\pi - \EE_{s_h}^\pi)\Bigsbr{\sum_{l=h}^H \gamma^{l-h}\bigrbr{r_l^M- r_l + \gamma (P_l^{M} - P_l) V_{l+1}^{\pi^M, M}}(s_l, a_l, b_l)}, \forall h\in[H]\bigg\} . 
\end{align*}
We consider the following configurations for the definition of the eluder dimension of $\cG_F^1$. 
\begin{itemize}[leftmargin=20pt]
    \item[(i)] Define function class $\cG_{h, F}^1$ as 
    \begin{align*}
        \cG_{h, F}^1 &= \Bigg\{g:(\cS\times\cA\times\cB)^{H-h+1}\rightarrow \RR \bigggiven \exists M\in\cM, \nend
        &\qqquad g((s_l, a_l, b_l)_{l=h}^H) = {\sum_{l=h}^H \gamma^{l-h}\bigrbr{r_l^M- r_l + \gamma (P_l^{M} - P_l) V_{l+1}^{\pi^M, M}}(s_l, a_l, b_l)}
        \Bigg\},
    \end{align*}
    where we remind the readers that $\pi^M = \argmax_{\pi\in\Pi}J(\pi, M)$ only depends on $M$. Consider 
    sequences 
    $$\cbr{g_h^t=\sum_{l=h}^H \gamma^{l-h}\orbr{\tilde r_l^t- r_l + \gamma (\tilde P_l^t - P_l) \tilde V_{l+1}^t}}_{t\in[T]}. $$
    It is clear that $g_h^t\in\cG_{h, F}^1$ since $\tilde V_h^t = V_h^{\pi^t, M^t}$ and $\pi^t = \argmax_{\pi\in\Pi}J(\pi, M^t) = \pi^{M^t}$.
    
    \item [(ii)] Define a class of signed measures over $(\cS\times\cA\times\cB)^{H-h+1}$ as 
    \begin{equation*}
        \sP_{h, F}^1 = \cbr{\begin{aligned}
            &\PP^\pi(((s_l, a_l, b_l)_{l=h+1}^H , a_h)=\cdot \given s_h, b_h)\delta_{(s_h, b_h)}(\cdot) \nend
            &\quad - \PP^\pi(((s_l, a_l, b_l)_{l=h+1}^H , a_h, b_h)=\cdot \given s_h)\delta_{(s_h)}(\cdot) 
        \end{aligned}
        \bigggiven \pi\in\Pi, (s_h, b_h)\in\cS\times\cB}, 
    \end{equation*}
    where $\delta_{s_h, b_h}$ is the measure that puts measure $1$ on a single state-action pair $(s_h, b_h)$, and the conditional is well defined by the Markov property. Also, consider the following sequence, 
    \begin{align*}
        \Big\{\rho_h^t(\cdot)&=\PP^{\pi^t}(((s_l, a_l, b_l)_{l=h+1}^H , a_h)=\cdot \given s_h^t, b_h^t)\delta_{(s_h^t, b_h^t)}(\cdot) \nend
        &\qquad - \PP^{\pi^t}(((s_l, a_l, b_l)_{l=h+1}^H , a_h, b_h)=\cdot \given s_h^t)\delta_{(s_h^t)}(\cdot)\Big\}_{t\in[T]},
    \end{align*}
    and we also have $\rho_h^t\in\sP_{h, F}^1$.
    \item   [(iii)] 
    In particular, we define $g_h^t(s_h^i, b_h^i, \pi^i)$ as the integral of $g_h^t$ with respect to $\rho_h^i$, which is given by
    $$g_h^t(s_h^i, b_h^i, \pi^i)= \rbr{\EE_{s_h^i, b_h^i}^{\pi^i}-\EE_{s_h^i}^{\pi^i}} \sbr{\sum_{l=h}^H \gamma^{l-h}\Bigrbr{\tilde r_l^t- r_l + \gamma (\tilde P_l^t - P_l) \tilde V_{l+1}^t}(s_l, a_l, b_l)},$$
    Note that the sequence of signed measures is uniquely determined by $\{(s_h^t, b_h^t, \pi^t)\}_{t\in[T]}$. Moreover, we have $g_h^t(s_h^i, b_h^i, \pi^i)$ bounded by $\eff_H(\gamma)(2\nbr{r_h}_\infty + 2\nbr{V_{h+1}}_\infty) \le 4B_A \eff_H(\gamma)$, where the definition of $B_A$ is available in \eqref{eq:define_BA};  
\end{itemize}
We define the maximal eluder dimension of $\cG_{h, F}^1$ with respect to $\sP_{h,F}^1$ as $$\dim(\cG_F^1) =\max_{h\in[H]} \dim_\DE(\cG_{h, F}^1,\sP_{h, F}^1, T^{-1/2}). $$
We remark that the linear Markov game have eluder dimension $\dim(\cG_{h, F}^2)\lesssim Hd$.
Note that for each $l\in\{h, \dots, H\}$ in the expression of $g_h^t(s_h^i, b_h^i, \pi^i)$, we have 
\begin{align*}
    &\rbr{\EE_{s_h^i, b_h^i}^{\pi^i}-\EE_{s_h^i}^{\pi^i}} \sbr{\sum_{l=h}^H \gamma^{l-h}\Bigrbr{\tilde r_l^t- r_l + \gamma (\tilde P_l^t - P_l) \tilde V_{l+1}^t}(s_l, a_l, b_l)}\nend
    &\quad  = \sum_{l=h}^H \gamma^{l-h} \Bigdotp{\bigrbr{\EE_{s_h^i, b_h^i}^{\pi^i}-\EE_{s_h^i}^{\pi^i}}\sbr{\phi_l(s_l,a_l,b_l)}}{\hat\theta_h^t - \theta_h + \sum_{s_{l+1}}(\hat\mu_l^t - \mu_l)(s_{l+1})\tilde V_{l+1}^t(s_{l+1})}, 
\end{align*}
which admits a bilinear form. 
Since we can stack $H$ vectors in the summation together for both sides of bilinear form, which creates a bilinear form with dimension $Hd$, we claim that $\dim(\cG_{F}^1)\lesssim Hd$ in this case.

\paragraph{Eluder Dimension for the Second-Order Error of the Follower's Quantal Response. }
According to \Cref{sec:proof-farsighted MDP}, the second order term in the quantal response error is just 
\begin{align*}
    \sum_{t=1}^T \EE^t\sbr{ \rbr{\rbr{\tilde Q_h^t - r_h^{\pi^t} - \gamma P_h^{\pi^t} \tilde V_{h+1}^t}(s_h, b_h)}^2}.
\end{align*}
We denote by $\cG_{F}^2$ the class of functions corresponding to this second-order QRE, 
\begin{align*}
    \cG_F^2=\bigcbr{\EE_{s_h, b_h}^{\pi} \osbr{\orbr{r_h^M- r_h + \gamma (P_h^{M} - P_h) V_{l+1}^{*, M}}(s_h, a_h, b_h)}}, 
\end{align*}
for all $\pi\in\Pi, M\in\cM, (s_h, a_h, b_h)\in\cS\times\cA\times\cB$ and $h\in[H]$.
We consider the following configurations.
\begin{itemize}[leftmargin =20pt]
    \item[(i)] We take the same function class $\cG_F^2$ as
    \begin{align*}
        \cG_{h, F}^2 &= \Bigg\{g:\cS\times\cA\times\cB\rightarrow \RR: \exists M\in\cM, h\in[H] \nend
        &\qqquad g(s_h, b_h, a_h) = {\bigrbr{r_h^M- r_h + \gamma (P_h^{M} - P_h) V_{l+1}^{\pi^M, M}}(s_h, a_h, b_h)}
        \Bigg\},
    \end{align*}
    where $\cG_F^2$ is bounded by $4B_A$. 
    \item[(ii)] We define a class of probability measures on $\cS\times\cA\times\cB$ as $$\sP_{h, F}^2 = \cbr{\PP^\pi(a_h=\cdot\given s_h, b_h)\delta_{(s_h, b_h)}(\cdot)\given \pi\in\Pi, (s_h,b_h)\in\cS\times\cB}, $$
    where $\delta_{(s_h,b_h)}(\cdot)$ is the measure that assigns $1$ to the state-action pair $(s_h, b_h)$. 
    \item[(iii)] We take a sequence of functions $\{g_h^t\}_{t\in[T]}$ as $\{g_h^t = \tilde r_h^{t}- r_h + \gamma (\tilde P_h^t - P_h) \tilde V_{l+1}^{t}\}_{t\in[T]}$, and take a sequence of probability measures as $\{\rho_h^t(\cdot) = \PP^{\pi^t}(a_h=\cdot\given s_h^t, b_h^t)\delta_{(s_h^t, b_h^t)}(\cdot)\}_{t\in[T]}$, where we define $\tilde r_h^t = r_h^{M^t}$ and $\tilde P_h^t = P_h^{M^t}$.
    One can check that $g_h^t\in\cG_{h,F}^1$ since $\tilde V_h^t = V_h^{\pi^t, M^t}$ and we have $\pi^t = \argmax_{\pi\in\Pi}J(\pi, M^t) = \pi^{M^t}$. In addition, we define $g_h^t(s_h^i, b_h^i, \pi^i)$ as the integral of $g_h^t$ with respect to $\rho_h^i$, which is given by
    \begin{align*}
        g_h^t (s_h^i, b_h^i, \pi^i) = \EE_{s_h^i, b_h^i}^{\pi^i} \sbr{\bigrbr{\tilde r_h^{t}- r_h + \gamma (\tilde P_h^t - P_h) \tilde V_{l+1}^{t}}(s_h, a_h, b_h)},
    \end{align*}
    Note that the sequence of probability measures is uniquely determined by $\{(s_h^t, b_h^t, \pi^t)\}_{t\in[T]}$.
\end{itemize}
We let $\dim(\cG_F^2)=\max_{h\in[H]}\dim_\DE(\cG_{h,F}^2,\sP_{h, F}^2, T^{-1/2})$ be the eluder dimension.
Similar to the previous discussion, we also have a bilinear for $g_h^t(s_h^i, b_h^i, \pi^i)$ under the linear Markov game. Hence, $\dim(\cG_{h, F}^2)\lesssim d$.

We remark that 
for the linear matrix MDP \citep{zhou2021provably}, we just have $P(s_{h+1}|s_h, a_h, b_h) = \psi_h(s_h, a_h, b_h)^\top W_h \varphi_h(s')$ for some unknown $\RR^{d\times d}$ matrix $W$. 
By viewing $W_h\varphi_h(s')=\mu_h(s')$, our discussions hold for the linear matrix MDP as well.
For the linear mixture MDP \citep{chen2022unified}, the rewards and the transition kernel share the same parameter $\theta_h$, but the transition kernel is given by $P_h(s_{h+1}\given s_h,a_h, b_h)=\la \psi_h(s_{h+1}, s_h, a_h, b_h), \theta_h\ra$. 
We remark that we can instead take 
$$g_h^t(\pi^i) =\EE^{\pi^i}\bigsbr{{\bigrbr{(u_h^{M^t} - u_h) + (P_h^{M^t} - P_h) \tilde W_{h+1}^i}(s_h,a_h,b_h)}}, $$ 
for the leader's Bellman error, 
$$g_h^t(s_h^i, b_h^i, \pi^i)= \rbr{\EE_{s_h^i, b_h^i}^{\pi^i}-\EE_{s_h^i}^{\pi^i}} \sbr{\sum_{l=h}^H \gamma^{l-h}\Bigrbr{\tilde r_l^t- r_l + \gamma (\tilde P_l^t - P_l) \tilde V_{l+1}^i}(s_l, a_l, b_l)},$$
for the first order quantal response error, and 
\begin{align*}
    g_h^t (s_h^i, b_h^i, \pi^i) = \EE_{s_h^i, b_h^i}^{\pi^i} \sbr{\bigrbr{\tilde r_h^{t}- r_h + \gamma (\tilde P_h^t - P_h) \tilde V_{l+1}^{i}}(s_h, a_h, b_h)},
\end{align*}
for the second order quantal response error, these errors still admit a low rank factorization. In fact, the online guarantees of these three errors are already implied by \eqref{eq:OnN-guarantee-Bellmanerror}, \eqref{eq:OnN-1st-g-sq}, and \eqref{eq:OnN-guarantee-2ndQRE}, where a change of the V- or the W-functions does not matter because we have strong guarantee on the TV distance for the transition kernel, and the change of V- or W-functions only introduces a $\cO(\beta)$ terms in the online guarantee which is caused by the squared TV distance for the transition kernel.


{\ifneurips 
\input{neurips/algorithm_detail_neurips.tex}
\fi}


\section{More Details of Technical Ingredients}\label{sec:app-major-tech}

In this section, we summarize and provide  proofs of the  important techniques used for analyzing the QSE. The following is a table of addition constants used in this section. 
\begin{table}[h!]
    \centering
    {\setlength\doublerulesep{1pt}
    \begin{tabular}{p{2cm}|p{10.5cm}}
        \toprule[2pt]\midrule[0.5pt]
        Notations & Interpretations \\ \toprule[1.5pt]
        $C^{(0)}$   & $C^{(0)}=2\eta H$ \\\midrule
        $C^{(1)}$ & $C^{(1)}=\eta^2 H   \bigrbr{1+ 4 \eff_H(\gamma)}\cdot \exp(2\eta B_A)$ \\\midrule
        $C^{(2)}$ & $
           C^{(2)}   =  2  \eta^2 H^2\cdot \exp\orbr{6\eta B_A} \cdot(1+4 \eff_H(\gamma)) \cdot \rbr{\eff_H(\exp(2\eta B_A)\gamma)}^2$ 
        \\\midrule
        $C^{(3)}$ & $C^{(3)}={\eta^2 \exp(2\eta B_A)}\bigrbr{2+\eta B_A \cdot  \exp\rbr{2\eta B_A}}/2$
        \\\midrule
        $L^{(1)}$ & $L^{(1)} = 6(\eta^{-1}+2 B_A)\cdot {\eff_H\rbr{\gamma}}$
        \\\midrule 
        $L^{(2)}$ & $L^{(2)} = c H^2 \eff_H(c_2)^2 \kappa^2 \exp\rbr{8\eta B_A} (\eta^{-1}+B_A)^2$, where 
        $c_2 = \gamma(2\exp(2\eta B_A)+\kappa\exp(4\eta B_A))$, and $c$ is a universal constant.\\
        \bottomrule[2pt]
    \end{tabular}
    }
    \caption{Constants used for \Cref{sec:app-major-tech}}\label{tab:}
\end{table}

\subsection{Performance Difference Lemma for for QSE}\label{sec:app-subopt-decompose}

In this subsection, we further elaborate on the performance difference lemma introduced in \S\ref{sec:subopt decomposition}. 
For generality, we consider the farsighted case.  
In the following,  
we consider a fixed policy $\pi$ and let its quantal response under the true model be $\nu^{\pi}$. 
Let $\tilde \nu$ be an estimate of $\nu^{\pi}$ and let $\tilde U$ and $\tilde W$ be any estimates of  $U^{\pi} $ and $  W^{\pi}$, which are defined respectively in    \eqref{eq:U_function} and \eqref{eq:W_function}.
We note that $\tilde W$ and $\tilde U$ not necessarily satisfy $
\tilde W_h(s ) = \la \tilde U_h (s, \cdot , \cdot ) , \pi _h \otimes \tilde \nu_h^{\pi} (\cdot , \cdot \given s) \ra . 
$.
We present a slightly more general version of   the performance difference lemma, which directly implies \eqref{eq:performance diff-1}.


\begin{lemma}[Performance Difference]\label{lem:subopt-decomposition} 
For any fixed policy $\pi$, 
let $\tilde \nu$ be an estimate of the quantal response $\nu^{\pi}$ and let $\tilde U$ and $\tilde W$ be estimates of $U^\pi$ and $W^{\pi}$ respectively. 
Based on $\tilde U$ and $\tilde W$, we can estimate $J(\pi)$ defined in \eqref{eq:J} by 
\$
\EE_{s_1\sim\rho_0}\bigl [ \orbr{T_1^{\pi,\tilde\nu}\tilde U_1} (s_1 )\bigr ]  \qquad \textrm{and} \qquad \EE_{s_1\sim\rho_0}\bigl [ \tilde W_1(s_1)\bigr]  , 
\$
where we operator $T_h^{\pi,\tilde\nu}$ is defined in \eqref{eq:operator_T_pi_nu}.
The error or these estimators can be bounded as follows: 
\begin{align}
    &\EE_{s_1\sim\rho_0}\bigsbr{ \orbr{T_1^{\pi,\tilde\nu}\tilde U_1}(s_1 )} - J(\pi ) \nend
    &\quad\le\underbrace{{\sum_{h=1}^H \EE\Bigsbr{\orbr{\tilde U_h - u_h}(s_h, a_h, b_h)-  \orbr{T_{h+1}^{\pi,\tilde\nu} \tilde U_{h+1}}(s_{h+1})}}}_{\dr \text{Leader's Bellman error}}+ \underbrace{\sum_{h=1}^H   H  \cdot \EE \bigsbr{ \nbr{\rbr{\tilde \nu_h-\nu_h^{\pi }}(\cdot\given s_h)}_1} }_{\dr \text{Quantal response error}}\label{eq:perform-diff-general}, \\
    & \EE_{s_1\sim\rho_0}\bigsbr{\tilde W_1(s_1)}  - J(\pi)\nend
    &\quad=\underbrace{{\sum_{h=1}^H \EE\bigsbr{\orbr{\tilde U_h - u_h}(s_h, a_h, b_h)- \tilde W_{h+1}(s_{h+1})}}}_{\dr \text{Leader's Bellman error}} + \underbrace{\sum_{h=1}^H \EE\bigsbr{\tilde W_h(s_h) -  \orbr{T_{h}^{\pi,\tilde\nu} \tilde U_{h}} (s_{h})}}_{\ds\text{Value mismatch error}}\nend
    &\qqquad+\underbrace{\sum_{h=1}^H  H  \cdot  \EE  \bigsbr{ \nbr{\rbr{\tilde \nu_h-\nu_h^{\pi }}(\cdot\given s_h)}_1} }_{\dr \text{Quantal response error}}, 
    \label{eq:perform-diff-linear} 
\end{align}
where  the expectation is taken with respect to the randomness of the trajectory generated by $(\pi, \nu^{\pi})$ on the true model $M^*$.

\end{lemma}

Here, we provide two ways to estimate $J(\pi)$.
The first inequality \eqref{eq:perform-diff-general} is useful for  setting with general function approximation, and the second inequality \eqref{eq:perform-diff-linear} becomes handy particularly in the   setting with linear function approximation, where we  additionally introduced an algorithm with penalty/bonus terms. 

This lemma shows that the estimation error of $\EE_{s_1\sim\rho_0}\bigl [ \orbr{T_1^{\pi,\tilde\nu}\tilde U_1} (s_1 )\bigr ] $ can be decomposed into a sum of three terms --- leader's Bellman error, the error of the estimated quantal response, and an additional term that measures the mismatch between $\tilde W$ and $\tilde U$, which is included here for generality and will be used in the analysis of both \Cref{alg:PMLE} and \Cref{alg:MLE-OVI} where some penalties/bonuses are included in the estimation of $\hat W$.

Furthermore,  by the definition of 
$T_h^{\pi,\tilde\nu}$ in \eqref{eq:operator_T_pi_nu},
 $
 \tilde W_h(s ) = \la \tilde U_h (s, \cdot , \cdot ) , \pi _h \otimes \tilde \nu_h^{\pi} (\cdot , \cdot \given s) \ra
 $
 is equivalent to $\tilde W_ h = T_h^{\pi, \tilde \nu} \tilde U_h$. 
Thus, 
\eqref{eq:perform-diff-linear} directly implies \eqref{eq:performance diff-1}
as a special case.


\begin{proof}
By the definitions of $U^{\pi}$ and $W^{\pi}$ in \eqref{eq:U_function} and \eqref{eq:W_function}, $J(\pi)$ can be written as 
\$
J(\pi) = \EE_{s_1 \sim \rho_0 } [ W_1 ^{\pi} (s_1) ] = \EE_{s_1 \sim \rho_0 } [  \orbr{T_1 ^{\pi, \nu}U_1 ^{\pi} } (s_1) ],
\$ 
where we write $\nu = \nu^{\pi}$ to simplify the notation.
Recall that we define the quantal Bellman operator $\TT_h^{\pi}$ in \eqref{eq:bellman_operator_leader}, whose fixed point is $U^{\pi}$.
Then, by direct calculation, we have 
\#\label{eq:subopt-decompose-equality-0}
&\EE_{s_1\sim\rho_0}\bigsbr{ \orbr{T_1^{\pi,\tilde\nu}\tilde U_1}(s_1 )} - J(\pi )    \nend 
& \quad  =    \EE_{s_1\sim\rho_0}\bigsbr{\orbr{T_1^{\pi,\tilde\nu} - T_1^{\pi,\nu}} \tilde U_1(s_1)}  + \EE_{s_1\sim\rho_0}\bigsbr{\orbr{  T_1^{\pi,\nu}  \tilde U_1 - T_1^{\pi,\nu}  U_1^{\pi} } (s_1)}. 
\#
Furthermore, using the Bellman equation of $U^{\pi}$, we have 
\#\label{eq:subopt-decompose-equality-01}
& \EE_{s_1\sim\rho_0}\bigsbr{\orbr{  T_1^{\pi,\nu}  \tilde U_1 - T_1^{\pi,\nu}  U_1^{\pi} } (s_1)} 
= \EE \bigsbr{  \tilde U_1(s_1, s_1, a_1) - U_1 ^{\pi} (s_1, a_1, b_1)}\nend 
& \quad  = \EE \bigsbr{  \orbr{\tilde U_1 -u_1} (s_1, s_1, a_1) - T_{2}^{\pi,\nu} \tilde U_{2}(s_2 ) }  +  \EE\bigsbr{ T_{2}^{\pi,\nu} \tilde U_{2} (s_2) - T_{2}^{\pi,\nu}U_{2}^{\pi}  (s_2) }. 
\# 
Here the second equality follows from the Bellman equation, and the expectation is taken with respect to the randomness of the trajectory generated by $(\pi, \nu^{\pi})$ on the true model $M^*$. . 
Furthermore, by replacing $T_{2}^{\pi,\nu} \tilde U_{2}$ by $T_{2}^{\pi, \tilde \nu} \tilde U_{2}$ in \eqref{eq:subopt-decompose-equality-01}, and combining \eqref{eq:subopt-decompose-equality-0}, we obtain that 
\# 
&\EE_{s_1\sim\rho_0}\bigsbr{ \orbr{T_1^{\pi,\tilde\nu}\tilde U_1}(s_1 )} - J(\pi )    \nend 
     & \quad 
     =  {\sum_{h=1}^H \EE\bigsbr{\bigrbr{\tilde U_h - u_h}(s_h, a_h, b_h)-  \orbr{T_{h+1}^{\pi,\tilde\nu} \tilde U_{h+1}}(s_{h+1})}} + \sum_{h=1}^H \EE\bigsbr{\orbr{ T_h^{\pi,\tilde\nu}-  T_h^{\pi,\nu}} \tilde U_h(s_h)} ,\label{eq:subopt-decompose-equality-1}
\# 
where we apply recursion over all $h\in [H]$. 
Finally, note that $\tilde U_h $ is bounded by $H$ in the $\ell_{\infty}$-norm. Using H\"older's inequality, we have 
\#
\EE\bigsbr{\orbr{ T_h^{\pi,\tilde\nu}-  T_h^{\pi,\nu}} \tilde U_h(s_h)} \leq H \cdot \EE \bigsbr{ \nbr{\rbr{\tilde \nu_h-\nu_h}(\cdot\given s_h)}_1}.\label{eq:subopt-decompose-equality-12}
\#
Combining \eqref{eq:subopt-decompose-equality-1} and \eqref{eq:subopt-decompose-equality-12}, we establish \eqref{eq:perform-diff-general}.

It remains to prove  \eqref{eq:perform-diff-linear}. 
To this end, it suffices to incorporate the value mismatch error in \eqref{eq:perform-diff-linear} into \eqref{eq:perform-diff-general}.
Specifically,  for any $h \in [H]$, we have 
\#\label{eq:subopt-decompose-equality-2}
& \EE\bigsbr{\bigrbr{\tilde U_h - u_h}(s_h, a_h, b_h)-  T_{h+1}^{\pi,\tilde\nu} \tilde U_{h+1}(s_{h+1})} \notag \\
& \quad    ={\sum_{h=1}^H \EE\bigsbr{\bigrbr{\tilde U_h - u_h}(s_h, a_h, b_h)- \tilde W_{h+1}(s_{h+1})}} + \sum_{h=2}^H \EE\bigsbr{\tilde W_h(s_h) - T_h^{\pi,\tilde\nu}\tilde U_h(s_h)}. 
\#
Combining \eqref{eq:perform-diff-general} and \eqref{eq:subopt-decompose-equality-2}
we have 
\$
& \EE_{s_1\sim\rho_0}\bigsbr{ \tilde W_1(s_1)} - J(\pi) \notag \\
& \quad = \EE\bigsbr{\tilde W_1(s_1) - \orbr{T_1^{\pi,\tilde\nu}\tilde U_1}(s_1)} +  \EE\bigsbr{ \orbr{T_1^{\pi,\tilde\nu}\tilde U_1} (s_1)} - J(\pi) 
\notag \\
&  \quad\le {\sum_{h=1}^H \EE\bigsbr{\bigrbr{\tilde U_h - u_h}(s_h, a_h, b_h)- \tilde W_{h+1}(s_{h+1})}} + \sum_{h=1}^H \EE\bigsbr{\tilde W_h(s_h) - \orbr{T_h^{\pi,\tilde\nu}\tilde U_h }(s_h)}\nend
&\qqquad+ \sum_{h=1}^H  H \cdot  \EE \bigsbr{ \nbr{\rbr{\tilde \nu_h-\nu_h}(\cdot\given s_h)}_1} ,
\$ 
which gives us \eqref{eq:perform-diff-linear}. 
    Hence, we complete the proof.
\end{proof}

Recall that we estimate the quantal response mapping via   model-based maximum likelihood  estimation. 
In particular, we approximate the true quantal response policy $\nu^{\pi}$ within class $\{ \nu^{\pi, \theta}\}_{\theta \in \Theta}$, where 
$\nu^{\pi, \theta}$ can be written as 
\$
\nu_h^{\pi, \theta} (b_h\given s_h) &= \exp\bigl ( \eta \cdot A_h^{\pi, \theta} (s_h, b_h) \bigr), \quad\text{where}\quad A_h^{\pi, \theta} (s_h, b_h) = Q_h^{\pi, \theta} (s_h, b_h) - V_{   h}^{\pi, \theta} (s_h) ,
\$
Here $A^{\pi, \theta}$, $Q^{\pi, \theta}$, and $V^{\pi, \theta}$ are the advantage function, and value functions corresponding to policy $\pi$, under the model $\{ r^{\theta} , P^{\theta}\}$. 
In the following, we present a lemma which relates the error of quantal response mapping in \eqref{eq:perform-diff-general} and \eqref{eq:perform-diff-linear} to  estimation errors of the follower's value functions.
To simplify the notation, 
we consider a fixed policy $\pi$, and define  
$r_h^\pi(s_h,b_h) =\inp{r_h(s_h, \cdot, b_h)}{\pi(\cdot\given s_h, b_h)}_{\cA}$ and $P_h^\pi(s_{h+1}\given s_h,b_h) =\inp{P_h(s_{h+1}\given s_h, \cdot, b_h)}{\pi(\cdot\given s_h, b_h)}_{\cA} $.
To simplify the notation, 
in the rest of subsection, 
we let $\EE=\EE^{\pi,M^*}$ and $\EE_{z} [\cdot]=\EE^{\pi,M^*}[\cdot\given z]$ for any variable $z$.



\begin{lemma}[Response Model Error]
    \label{lem:performance diff}
    We consider a fixed policy $\pi$
and   let $\tilde Q$    be an estimate  of $Q^{\pi}$. 
We define a V-function  $\tilde V$ and an advantage function $\tilde A$ by letting 
\#\label{eq:tilde_functions}
\tilde V_h (s) = \frac{1}{\eta} \cdot \log \bigg(  \sum_{b \in \cB} \exp( \eta \cdot \tilde Q_h(s, b) ) \biggr), \qquad \textrm{and}\qquad \tilde A_h(s,a) = \tilde Q_h (s,b) - \tilde V_h (s). 
\# 
Furthermore, we define a follower's policy $\tilde \nu$ by letting $\tilde \nu_h(b \given s) = \exp( \eta\cdot \tilde A_h(s,b))$. 
Then the difference between $\tilde \nu$ and $\nu^{\pi}$ can be bounded by 
\begin{align}
    &\sum_{h=1}^H  H \cdot  \EE \bigsbr{\nbr{\rbr{\tilde \nu_h-\nu_h^{\pi} }(\cdot\given s_h)}_1} \nend
    &\quad\le C^{(0)} \cdot \sum_{h=1}^H \underbrace{\EE\bigsbr{\bigabr{\tilde \Delta^{(1)}_h(s_h, b_h)}}}_{\ds\text{1st-order error}}  + C^{(1)} \cdot 
    \sum_{h=1}^H \underbrace{\EE\bigsbr{ \bigabr{(\tilde A_h  - A_h^{\pi})(s_h, b_h)}^2}}_{\ds\text{2nd-order error}} \label{eq:taylor-myopic}\\
    &\quad\le 
    C^{(0)} \cdot 
    \sum_{h=1}^H \underbrace{\EE\bigsbr{\bigabr{\tilde \Delta^{(1)}_h(s_h, b_h)}}}_{\ds\text{1st-order error}}  + C^{(2)} \cdot 
    \max_{h\in [H]} \underbrace{\EE\bigsbr{ \bigabr{\orbr{\tilde Q_h - r_h^\pi - \gamma P_h^\pi \tilde V_{h+1}}(s_h, b_h)}^2}}_{\ds\text{2nd-order error}},\label{eq:taylor-farsighted} 
\end{align}
where $\tilde \Delta^{(1)}_h(s_h, b_h)$ is defined as
\begin{align*}
    \tilde \Delta^{(1)}_h(s_h, b_h) &=  \rbr{\EE_{s_h, b_h} -\EE_{s_h}}\Biggsbr{\sum_{l=h}^H \gamma^{l-h}\underbrace{\rbr{\tilde Q_l - r_l^\pi - \gamma P_l^\pi \tilde V_{l+1}}(s_l, b_l)}_{\ds\text{Follower's Bellman error}}}. 
\end{align*}
Furthermore, the constants $C^{(0)}$, $C^{(1)}$, and $C^{(2)}$ are given by
\#\label{eq:define_constants}
\begin{split}
    C^{(0)}&=2\eta H , \qquad   C^{(1)}=\eta^2 H   \bigrbr{1+ 4 \eff_H(\gamma)}\cdot \exp(2\eta B_A), \\
C^{(2)}  & =  2  \eta^2 H^2\cdot \exp\orbr{6\eta B_A} \cdot(1+4 \eff_H(\gamma)) \cdot \rbr{\eff_H(\exp(2\eta B_A)\gamma)}^2,
\end{split}
\#
where $B_A$ defined in \eqref{eq:define_BA} is an upper bound on the magnitude of the advantage function, and we define $\eff_H(x) = (1-x^H)/(1-x)$ as the \say{effective}  horizon with respect to $x$.

\end{lemma}
\begin{proof}
    See \Cref{sec:proof-performance diff} for a detailed proof.
\end{proof}

If we view the quantal response as a functional of the advantage function (as shown in \eqref{eq:quantal_response_policy}),  this lemma plays the role of Taylor expansion of the quantal response into the first- and second-order errors. 
In particular, the first-order error corresponds to the Bellman error of the follower's problem, and the second-order term is mean-squared error of the advantage function (as in \eqref{eq:taylor-myopic}) or the sum of squares of the Bellman error (as in \eqref{eq:taylor-farsighted}). 
Here we establish two versions of upper bounds because \eqref{eq:taylor-myopic} is handy for the myopic case while the second \eqref{eq:taylor-farsighted} is more useful for the farsighted case.
 
In the case with a  myopic follower,
the $Q$-function is reduced to the reward function $r$ of the follower. 
We have the following corollary. 
\begin{corollary}[Response Model Error for Myopic Case]\label{cor:response-diff-myopic}
    Let $r$ be the true reward function and let $\tilde r$ be an estimated reward. 
Let $\pi \colon \cS \times \cB \rightarrow \Delta(\cA )$ be a fixed policy. 
Let $\nu$ and $\tilde \nu$ be the quantal response function based on $r$ and $\tilde r$, respectively, i.e.,
\$
\nu(b\given s) \propto \exp \big (  \eta\cdot  r^{\pi} (s,b)\big ) ,\qquad \nu(b\given s) \propto \exp\big ( \eta \cdot \tilde r^{\pi} (s,b)\big ).
\$
Here we define $r^{\pi} $ by letting $r^{\pi} (s,b) = \la r(s,\cdot, b), \pi(\cdot \given s, b)\ra$, and define $\tilde r^{\pi}$ similarly. 
Then for any state $s\in \cS$, we have 
    \begin{align}
        D_\TV\rbr{\nu(\cdot\given s), \tilde\nu(\cdot\given s)}
        &\le  \eta  \EE_s\sbr{\abr{(\tilde r^\pi(s, b) - r^\pi(s, b)) - \EE_s\bigsbr{\tilde r^\pi(s, b) - r^\pi(s, b)}}} \nend
        &\qquad + C^{(3)}\EE_s\bigsbr{\rbr{\rbr{\tilde r^\pi(s, b) -r^\pi(s, b)} - \EE_s\sbr{\tilde r^\pi(s, b) -r^\pi(s, b)}}^2}, \label{eq:TV-for-myopic}  
    \end{align}
    where $C^{(3)}={\eta^2 \exp(2\eta B_A)}\bigrbr{2+\eta B_A \cdot  \exp\rbr{2\eta B_A}}/2$ and $B_A$  defined in \eqref{eq:define_BA} is $2 + 2 \log |\cB| / \eta  $. 
    Here the expectation $\EE_s$ is only taken with respect to $b \sim \nu(\cdot \given s)$ when conditioned on this $s$.
    \begin{proof}
        See \Cref{sec:proof-response-diff-myopic} for a detailed proof.
    \end{proof}
\end{corollary}
In the following, we present the result for a special case where 
the follower is myopic with a linear reward function. 
In this case, we write 
  $Q^{\pi, \theta}(s, b) = r^{\pi, \theta}(s, b) = \inp[]{\phi^\pi(s, b)}{\theta}$ for some $\RR^d$ kernel $\phi:\cS\times\cA\times\cB\rightarrow\RR^d$ with $\phi^\pi(s, b)=\inp{\phi(s, \cdot, b)}{\pi(\cdot\given s, b)}_\cA$ and parameter $\theta\in\RR^d$.
  The quantal response policy is given by $\nu^{\pi, \theta} (b \given s) \propto \exp(\eta \cdot \la \phi^{\pi}(s,b), \theta \ra)$. 
In particular, we show that the 1st- and the 2nd-order QRE decomposition in \eqref{eq:QRE-decompose} of \Cref{sec:learning QSE} is just a direct result of \Cref{cor:response-diff-myopic}. Recall by definition of the $\QRE$ in \eqref{eq:QRE}, 
\begin{align*}
    \QRE(s_h, b_h;\tilde\theta,\pi) &=  (\Upsilon_h^{\pi}(\tilde r_h - r_h))\orbr{s_h, b_h}\nend
    &= \dotp{\pi_h(\cdot\given s_h, b_h)}{(\tilde r_h - r_h)(s_h, \cdot, b_h)} - \dotp{\pi_h\otimes \nu_h^{\pi}(\cdot,\cdot\given s_h)}{(\tilde r_h - r_h)(s_h,\cdot,\cdot)}\nend
    & = \sbr{{(\tilde r_h^\pi(s_h, b_h) - r_h^\pi(s_h, b_h)) - \EE_{s_h}\bigsbr{\tilde r_h^\pi(s_h, b_h) - r_h^\pi(s_h, b_h)}}}. 
\end{align*}
We plug in the linear representation of the follower's reward $r_h^\pi(s_h, b_h) = \la\phi_h^\pi(s_h, b_h), \theta_h^*\ra$, which implies that 
\begin{align*}
    \EE_{s_h}\QRE(s_h, b_h;\tilde\theta, \pi)^2 = \Cov_{s_h}^{\pi,\theta^*} \sbr{(\tilde r_h^\pi - r_h^\pi)(s_h, b_h)} = \onbr{\tilde \theta_h -\theta_h}_{\Sigma_{s_h}^{\pi,\theta^*}}^2,
\end{align*}
where $\Sigma_{s_h}^{\pi, \theta^*} = \Cov_{s_h}^{\pi, \theta^*}[\phi_h^\pi(s_h, b_h)]$. Moreover, For the first term on the right hand side of \eqref{eq:TV-for-myopic}, 
\begin{align*}
    &\EE_{s_h}\sbr{\abr{(\tilde r^\pi(s_h, b_h) - r^\pi(s_h, b_h)) - \EE_s\bigsbr{\tilde r^\pi(s_h, b_h) - r^\pi(s_h, b_h)}}} \nend
    &\quad \le \sqrt{\EE_{s_h}\bigsbr{\rbr{\rbr{\tilde r^\pi(s_h, b_h) -r^\pi(s_h, b_h)} - \EE_{s_h}\sbr{\tilde r^\pi(s_h, b_h) -r^\pi(s_h, b_h)}}^2}} \nend
    &\quad = \EE_{s_h}\QRE(s_h, b_h;\tilde\theta, \pi)^2,
\end{align*}
where the inequality holds by the Cauchy-Schwarz inequality. The second term follows similarly.
We further generalize the above argument to the following corollary. 

\begin{corollary}[Response Model Error for Linear and  Myopic Case]\label{lem:response diff-myopic-linear}
    Under the setting of \Cref{cor:response-diff-myopic}, we assume the reward function of the myopic follower is a linear function of $\phi$. 
    Let $\theta^*$ be the parameter of the true reward function and let 
    $\tilde \theta \in \Theta$ be another parameter. 
    Let $\pi \in \Pi$ be any policy and let $s \in \cS$ be any state. 
   We define a matrix  $\Sigma_s^{\pi,\theta}$  as $ \Cov_s^{\pi,\theta}[\phi^\pi(s, b)], $ i.e.,
   \#\label{eq:define_sigma_s}
   \Sigma_s^{\pi,\theta} = \EE_s \bigl [ (\phi^\pi(s, b) - \EE_s [\phi^\pi(s, b)] \bigr )(\phi^\pi(s, b) - \EE_s [\phi^\pi(s, b)] \bigr )^\top \bigr ],
   \#
   where the expectation is taken with respect to $\nu^{\pi,\theta}(\cdot \given s)$.
    Then we have 
    \begin{align}
        &D_\TV\rbr{\nu^{\pi,\theta^*}(\cdot\given s) , \nu^{\pi,\tilde\theta}(\cdot\given s)} \nend
        &\quad \le \min\Bigcbr{f\Bigrbr{\sqrt{\trace\bigrbr{{\Psi}^{\dagger} \Sigma_{ s}^{\pi, \tilde\theta}}}\cdot \bignbr{\theta^*- \tilde\theta}_{{\Psi}}}, f\Bigrbr{\sqrt{\trace\bigrbr{{\Psi}^{\dagger} \Sigma_{ s}^{\pi, \theta^*}}}\cdot \bignbr{\theta^*- \tilde\theta}_{{\Psi}}}},
    \end{align}
    where $\Psi\in \SSS_+^{d}$ can be any fixed nonnegative definite matrix, $\Psi^{\dagger}$ is the pseudo-inverse of $\Psi$,  and the univariate function $f$ is defined as
    $
        f(x) = \eta x + C^{(3)} x^2 
    $ with 
    $$C^{(3)}={\eta^2 \exp(2\eta B_A)}\bigrbr{2+\eta B_A \cdot  \exp\rbr{2\eta B_A}}/2.$$
\end{corollary}

\begin{proof}
    See \Cref{sec:proof-response-diff-myopic_linear} for a detailed proof. 
\end{proof}
    
We first show that the uncertainty quantifier $\Gamma^{(2)}_h(s_h;\pi, \theta_h)$ used in \eqref{eq:Gamma^2} of \S\ref{sec:offline-ML} {\neurips and \eqref{eq:Gamma^2-neurips}} can be derived directly from \Cref{lem:response diff-myopic-linear}. 
Recall by definition, 
\begin{align*}
    \Gamma^{(2)}_h(s_h;\pi_h , \theta_h) = 2 H(\eta  \xi(s_h;\pi,\theta_h)  + C^{(3)}  \xi(s_h;\pi,\theta_h)^2 ), 
\end{align*}
where $\xi(s_h;\pi,\theta_h)^2$ is defined as
\begin{equation*}
    \xi(s_h;\pi, \theta_h)^ 2 =  \trace\bigrbr{\bigrbr{T\Sigma_{h,\cD}^{\theta} + I_d}^\dagger \Sigma_{s_h}^{\pi , \theta}}  \cdot  \bigl( 2 (\eta^{-1}+B_A)^2 \beta + 4B_\Theta^2 \bigr).
\end{equation*}
One can easily check that for any $\theta_h\in\CI_{h,\Theta}(\beta)$ where $\CI_{h,\Theta}(\beta)$ is the offline confidence set, 
\begin{align*}
    \Gamma^{(2)}_h(s_h;\pi_h,\theta_h) 
    &= 2 H \cdot f\Bigrbr{\sqrt{\trace\bigrbr{\bigrbr{T\Sigma_{h,\cD}^{\theta}+I_d}^\dagger \Sigma_{s_h}^{\pi,\theta}}} (2C_\eta^2 \beta + 4 B_\Theta^2)}\nend
    &\ge 2 H \cdot f\Bigrbr{\sqrt{\trace\bigrbr{\bigrbr{T\Sigma_{h,\cD}^{\theta}+I_d}^\dagger \Sigma_{s_h}^{\pi,\theta}}} (T\nbr{\theta_h -\theta_h^*}_{\Sigma_{h,\cD}^{\theta}}^2 + \nbr{\theta_h -\theta_h^*}_{I_d}^2)}\nend
    & \ge 2 H D_\TV\bigrbr{\nu_h^{\pi,\theta^*}(\cdot\given s_h), \nu_h^{\pi,\theta}(\cdot\given s_h)}, 
\end{align*}
where we use definition $C_\eta^2 = \eta^{-1}+B_A$ in the first equality, and in the first inequality, $2C_\eta^2\beta T^{-1} \ge \nbr{\theta_h-\theta_h^*}_{\Sigma_{h,\cD}^\theta}^2$ holds by \Cref{cor:formal-MLE confset-linear myopic} if $\CI_{h,\Theta}(\beta)$ is a valid confidence set, and $4B_\Theta^2\ge \nbr{\theta_h-\theta_h^*}_{I_d}^2$ holds by noting that $\nbr{\theta_h}^2\le B_\Theta$ for any $\theta_h\in\Theta_h$. 
The last inequality holds just by \Cref{lem:response diff-myopic-linear} where we plug in $\Psi = T\Sigma_{h,\cD}^{\theta}+I_d$.

\subsection{Learning Quantal Response via MLE}\label{sec:app-MLE}
In \S\ref{sec:MLE for behavior model}, we introduce how to learn the follower's  quantal response model from the follower's feedbacks via maximum likelihood estimation. 
In the following, we provide a slightly stronger lemma that implies  \Cref{sec:proof-MLE-general} in \S\ref{sec:MLE for behavior model}. 
In the following, we follow the notation used in \S\ref{sec:MLE for behavior model}, where 
we let $\theta = \{ \theta_h \}_{h \in [H]} \in \Theta $ denote the parameters of the follower's response model,
where $\Theta = \Theta_1 \times \ldots \times \Theta_H$ is the set of all parameters. 
Here we assume $\theta_h \in \Theta_h$ for all $h\in [H]$.
In specific, each $\theta$ is associated with a model, denoted by  ${ r^{\theta}, P^{\theta}}$, 
which is   shorthand notation for 
\$
\{ r_1^{\theta_1}, P_1^{\theta_1},  \ldots, r_H ^{\theta_H}, P_H^{\theta_H}\}. 
\$ 
Moreover, when the follower is myopic, $\theta$ only parameterize a reward function $r^{\theta}$. 
We assume that the parametric model of the follower captures the true model. That is, there exists $\theta^* \in \Theta $ such that 
$\{r^{\theta}, P^{\theta^*}\}$ is the true environment.


For any $\theta \in \Theta$ and any policy $\pi$ of the leader, we let $\nu^{\pi, \theta}$, $A^{\pi, \theta}$, $Q^{\pi, \theta } $ and $V^{\pi, \theta }$ denote the quantal response of $\pi$, advantage function, and Q- and V-functions under model $\{ r^{\theta}, P^{\theta} \}$, which are defined according to  \eqref{eq:quantal_response_policy}--{\main\eqref{eq:v_pi_qr}\fi}{\neurips \eqref{eq:qv_pi_qr}\fi}.
 The quantal response   under the true model is $  \nu^{\pi, \theta^*}$.
Thus, given a  (possibly adaptive) dataset $\cD = \{(s_h ^i, a_h ^i, b_h ^i, \pi_h ^i)\}_{i\in[t-1], h\in [H]}$,
the negative log-likelihood at step $h$ is given by
\#\label{eq:loglikelihood-1}
\cL_h^t  (\theta  )  & = -  \sum_{i = 1}^{t-1}  
\log \nu^{\pi^i, \theta} (b_h^i \given s_h^i) 
  = -  \sum_{i=1}^{t-1}   \eta  \cdot A_h^{\pi^i , \theta}(s^i,  b^i), 
\# 
 where   $\nu^{\pi^i, \theta} (b_h^i \given s_h^i) $ is the probability of observing the follower's action $b_h^i$ when the model parameter is $\theta$, state is $s_h^i$, and the leader announces $\pi^i$, and the second equality in \eqref{eq:loglikelihood-1} is due to \eqref{eq:quantal_response_policy}.   
 Here we assume the data $\cD$ satisfies the compliance property, i.e., $\PP^{\pi^t}_\cD(b_h^t\given s_h^t, (s_j^t, a_j^t, b_j^t, u_j^t)_{j\in[h-1]}, \tau^{1:t-1})  = \nu_h^{\pi^t}(b_h^t\given s_h^t), \forall h\in[H], t\in[T]$. 
 Such a property is satisfied 
by the online setting and the offline setting where behavior policies are adaptive. 
 Therefore, the MLE estimator of $\theta^*$ can be obtained by minimizing $\cL_h^t(\theta)$ over $\Theta$. 
Moreover, based on $\cL_h$, we can construct a confidence set for $\theta^*$: 
\begin{align}
    \confset_{h,\Theta}^t(\beta)=\cbr{\theta\in\Theta: \cL_h^t(\theta)\le \inf_{\theta'\in\Theta}\cL_h^t(\theta') + \beta}, \label{eq:behavior_model_confset-1}
\end{align}
where  $\cL_h^t(\theta) = \sum_{i=1}^{t-1} \eta A_h^{\pi^i, \theta}(s_h^i, b_h^i)$.

We remark that $\cL_h^t(\theta)$ \eqref{eq:loglikelihood-1} is a function of $\theta = \{ \theta_h \}_{h\in [H]}$. 
The reason is that the follower's quantal response is obtained by solving the optimal policy of an entropy-regularized MDP, whose optimal policy depends on the reward and transitions across all $H$ steps. 
But if the follower is myopic, then $\cL_h^t(\theta)$ only depends on $\theta_h$. In this case, we can regard $\cL_h^t$ as a function on $\Theta_h$, and replace $\Theta$ in \eqref{eq:behavior_model_confset-1} by $\Theta_h$.  We will leverage such observation in \Cref{sec:learning QSE}.

\begin{lemma}[Confidence Set]\label{lem:MLE-formal}
We define a distance $\rho  $ on $\Theta$ by letting 
\begin{align}\label{eq:rho for MLE}
 \rho(\theta, \tilde \theta) \defeq \max_{\pi\in\Pi, s_h\in\cS, h\in[H]} \cbr{D_\H\orbr{\nu_h^{\pi, \theta}(\cdot\given s_h), \nu_h^{\pi, \tilde \theta}(\cdot\given s_h)}, (1+\eta) \cdot\bignbr{Q_h^{\pi, \theta} - Q_h^{\pi, \tilde\theta}}_\infty}, 
\end{align}
where $B_A$ upper bounds the follower's A-, Q-, and V-functions and is specified in \eqref{eq:define_BA}. 
Let $\cN_\rho(\Theta,\epsilon)$ be the $\epsilon$-covering number of $\Theta$ with respect to the distance $\rho$.
That is, $\cN_\rho(\Theta,\epsilon)$ is the smallest $N \geq 1$ with the following property: there exists $\{ \theta^i\}_{i\in [N]} \subseteq \Theta$ such that, for any $\theta \in \Theta$, there exists $\theta^i$ such that $\rho(\theta, \theta^i) \leq \epsilon$.
For any $\delta \in (0,1)$, we set $\beta \ge  2\log(e^3 H\cdot \cN_\rho(\Theta, T^{-1})/\delta)$.  
Then  with probability at least $1-\delta$,   the following properties hold for $\confset_{h, \Theta}^t (\beta)$ defined in \eqref{eq:behavior_model_confset-1}:
    \begin{itemize}
    \item [(i)]  (Validity) $\theta^*\in\confset_{h, \Theta}^t(\beta)$; 
    \item [(ii)] (Accuracy) For any $\theta\in\Theta, t\in[T], h\in[H]$, it holds that 
    \#
        &
        \sum_{i=1}^{t-1} D_\H^2\bigrbr{\nu_h^{\pi^i, \theta}(\cdot\given s_h^i), \nu_h^{\pi^i, \theta^*}(\cdot\given s_h^i)} \le  \frac {1}{2}\rbr{\cL_h^t(\theta) - \cL_h^t(\theta^*)+ \beta}, \label{eq:MLE-guarantee-hellinger-1}\\
        &
        \sum_{i=1}^{t-1} \EE^{\pi^i }D_\H^2\bigrbr{\nu_h^{\pi^i, \theta}, \nu_h^{\pi^i, \theta^*}} \le  \frac {1}{2}\rbr{\cL_h^t(\theta) - \cL_h^t(\theta^*)+ \beta}\label{eq:MLE-guarantee-hellinger-2}.
 \#
 \end{itemize} 
 Furthermore, the above two inequalities ensure respectively that for 
   $\forall \theta'\in\bigcbr{\theta^*, \theta}, \theta\in\Theta, \forall h\in[H]$,
    \begin{align}
        &\sum_{i=1}^{t-1} {\Var_{s_h^i}^{\pi^i, \theta'} \bigsbr{Q_h^{\pi^i, \theta}(s_h, b_h) - Q_h^{\pi^i, \theta^*}(s_h, b_h)}} \le 4 C_\eta^2 \rbr{\cL_h^t(\theta) - \cL_h^t(\theta^*)+ \beta}, \label{eq:MLE_guarantee_Q-1} \\
        & \sum_{i=1}^{t-1} \EE^{\pi^i }{\Var_{s_h}^{\pi^i, \theta'} \bigsbr{Q_h^{\pi^i, \theta}(s_h, b_h) - Q_h^{\pi^i, \theta^*}(s_h, b_h)}} \le  4 C_\eta^2 \rbr{\cL_h^t(\theta) - \cL_h^t(\theta^*)+ \beta},\label{eq:MLE_guarantee_Q}
    \end{align}
    where $\Var_s^{\pi, \theta}[Z] = \Var^{\pi, \theta}[Z\given s] = \EE^{\pi,\theta}[(Z - \EE^{\pi, \theta}[Z\given s])^2\given s]$, $C_\eta =\eta^{-1}+B_A$. Moreover, for any $\theta\in\CI_\Theta(\beta), h\in[H]$, and a given  $t\in[T]$, we have with probability at least $1-\delta$ that
\begin{align}
&\sum_{i=1}^{t-1} 
    \rbr{\rbr{Q_h^{\pi^i, \theta} - Q_h^{\pi^i, \theta^*}}(s_h^i, b_h^i)
    -\EE_{s_h^i}^{\pi^i,\theta^*}\sbr{ \bigrbr{Q_h^{\pi^i, \theta} - Q_h^{\pi^i, \theta^*}}(s_h, b_h)}}^2 
    \le  \cO(C_\eta^2 \beta) .  \label{eq:MLE-guarantee-Q-3}
\end{align}
 
    \begin{proof}
        See \Cref{sec:proof-MLE-general} for a detailed proof.
    \end{proof}
\end{lemma}
We remark that if $\theta\in\CI_\Theta^t(\beta)$, the guarantees in the accuracy results \eqref{eq:MLE_guarantee_Q} and \eqref{eq:MLE_guarantee_Q-1} are just $8C_\eta^2 \beta$ since $\cL_h^t(\theta)-\cL_h^t(\theta^*)\le \cL_h^t(\theta) - \inf_{\theta'\in\Theta}\cL_h^t(\theta') \le \beta$ for all $h\in[H]$.
Moreover, the covering number defined here is a special case of a more general version given by \eqref{eq:rho-Theta} and also in the myopic case it is given by \eqref{eq:rho-Theta_h}. 
Next, we comment on the myopic case, where it suffices to consider a covering for $\Theta_h$ (which only contains the parameters for the follower's reward function at step $h$) instead of the whole $H$-step class $\Theta$.
\begin{remark}[Myopic case for \Cref{lem:MLE-formal}]\label{rmk:MLE-formal-myopic}
    If the follower is myopic, it suffices to use the distance $\rho$ for class $\Theta_h$ defined in \eqref{eq:rho-Theta_h}, and let $\cN_\rho(\Theta, \epsilon) = \max_{h\in[H]} \cN_\rho(\Theta_h, \epsilon)$. The conclusion in \Cref{lem:MLE-formal} still applies.
\end{remark}

Next, we give a proof on \Cref{eq:bandit-ub-1} 
, which borrows \Cref{rmk:MLE-formal-myopic} and the covering number of $\Theta_h$ given in \eqref{eq:cN-Theta_h}.
\begin{remark}[Formal statement of \eqref{eq:bandit-ub-1}
    ]\label{cor:formal-MLE confset-linear myopic}
    Consider the myopic and linear case, Suppose that $\beta\ge C d\log(H(1+\eta T^2 + (1+\eta)T)\delta^{-1})$ for some universal constant $C>0$,  \eqref{eq:MLE_guarantee_Q-1} further implies that for all $h\in[H]$
        \begin{align}
            \max\cbr{\bignbr{\hat\theta_h-\theta_h^*}_{\Sigma_{h, t}^{\theta^*}}^2, \bignbr{\hat\theta_h-\theta_h^*}_{\Sigma_{h, t}^{\hat\theta}}^2, 
            \EE^{\pi^i, M^*}\bignbr{\hat\theta_h-\theta_h^*}_{\Sigma_{h, t}^{\theta^*}}^2, 
            \EE^{\pi^i, M^*}\bignbr{\hat\theta_h-\theta_h^*}_{\Sigma_{h, t}^{\hat\theta}}^2 } \le 8 C_{\eta}^2 \beta,
            \label{eq:app-bandit-ub-1}
        \end{align}
        where $C_\eta =\eta^{-1}+B_A$, $B_A$ is specified in \eqref{eq:define_BA}, and $\Sigma_{h, t}^{\theta}$ is a data-dependent covariance matrix defined as 
        \begin{align}\label{eq:app-cov matrix}
            \Sigma_{h, t}^{\theta}= \sum_{i=1}^{t-1}
            {\Cov_{s_h^i}^{\pi^i, \theta}\bigsbr{\phi^{\pi^i}(s_h, b_h)}}, 
        \end{align}
    where $\Cov_{s_h}^{\pi, \theta}[\phi^{\pi}(s_h, b_h)]$ represents the covariance matrix of the feature $\phi^\pi$ with respect to $\nu_h^{\pi,\theta}(\cdot\given s_h)$.
\end{remark}

We next consider a special case where each trajectory in the offline data is independently collected. 
We have the following lemma for the MLE guarantee with independent dataset.

\begin{lemma}[{Confidence in Hellinger Distance with Independent Data}]
    \label{lem:MLE-indep-data}
    Suppose the conditions in \Cref{lem:MLE-formal} hold.
    Suppose each trajectory in the offline dataset $\cD=\{\tau^t\}_{t\in[T]}$ is independently collected. 
    For the confidence set $\CI_\Theta(\beta)$ defined in \Cref{eq:behavior_model_confset-1} with $\beta$ properly chosen according to \Cref{lem:MLE-formal},
    with probability at least $1-2\delta$, it holds that (i) $\theta^*\in\CI_\Theta(\beta)$; (ii) for any $\theta\in\CI_\Theta(\beta)$ , $h\in[H]$, 
    \begin{align*}
        \sum_{i=1}^T\EE_\cD\sbr{D_\H^2\rbr{\nu_h^{\pi^i, \theta}(\cdot\given s_h^i), \nu_h^{\pi^i, \theta^*}(\cdot\given s_h^i)}}  \le \cO(\beta),
    \end{align*}
    where the expectation is taken for  the randomness in both the trajectories and in the leader's policy choices.
    \begin{proof}
        See \ref{sec:proof-MLE-indep-data} for a detailed proof.
    \end{proof}
\end{lemma}
Up to now, we have obtain all the guarantees we  need from the MLE of the follower's quantal response.

\subsection{Learning Leader's Value Function}\label{sec:app-value function}
In this section, we study the problem of learning the leader's value function for both the offline and the online setting. 

\paragraph{Learning Leader's Value Function in Offline Setting.}
for each $\pi$ and estimated follower's response model $\theta$. 
We only focus on myopic follower in this section. Recall the Bellman loss we defined in \Cref{sec:offline-myopic}, which is defined as
\begin{align*}
    &\ell_h(U_h', U_{h+1}, \theta, \pi) = \sum_{i=1}^T \rbr{U_h'(s_h^i, a_h^i, b_h^i) - u_h^i -   T^{\pi, \theta}_h U_{h+1}(s_{h+1}^i)}^2.
\end{align*}
where we define $ T_h^{\pi, \theta}U_{h+1}(s_{h+1}) = \inp[]{U_{h+1}(s_{h+1}, \cdot, \cdot)}{\pi_{h+1}\otimes \nu_{h+1}^{\pi, \theta}(\cdot, \cdot\given s_{h+1})}$. 
We aim to characterize the confidence set 
\begin{align*}
    \CI_{\cU}^{\pi,\theta}(\beta) = \cbr{U\in\cU: \ell_h(U_h, U_{h+1},\theta, \pi) - \inf_{U_h'\in\cU} \ell_h(U_h', U_{h+1}, \theta,\pi)\le \beta, \forall h\in[H]}.
\end{align*}
Recall the definition of the Bellman operator for the leader in \eqref{eq:bellman_operator_leader}. Similar to this definition, we define $\TT_h^{\pi,\theta}:\sF(\cS\times\cA\times\cB)\rightarrow \sF(\cS\times\cA\times\cB)$ as 
\begin{align*}
    \rbr{\TT_h^{\pi,\theta} f} (s_h, a_h, b_h) = u_h(s_h, a_h, b_h) + \EE_{s_{h+1}\sim P_h(\cdot\given s_h, a_h, b_h)} \sbr{\rbr{ T_{h+1}^{\pi,\theta}f}(s_{h+1})},
\end{align*}
and we add $\theta$ to the superscription to remind ourselves that the expectation within $\TT_h^{\pi,\theta}$ is taken with respect to the follower's quantal behavior guided by both $\pi$ and $\theta$.
In the sequel, we denote by $U^{\pi,\theta}$ the leader's U function defined similar to \eqref{eq:U_function} but with respect to policy $\pi$ and the response model $\theta$, 
\begin{align*}
    U_h^{\pi,\theta}(s_h, a_h, b_h) & = u_h(s_h, a_h, b_h) + \rbr{ P_h \circ T_{h+1}^{\pi,\theta} \circ U_{h+1}^{\pi,\theta}}(s_h, a_h, b_h) = \TT_h^{\pi,\theta} U_{h+1}^{\pi,\theta}(s_h, a_h, b_h).
\end{align*}
We clarify that $\theta$  only contains the estimated reward for myopic follower. 
Now, we present the following corollary on the validity and accuracy of the confidence set $\CI_\cU^{\pi,\theta}(\beta)$. 
\begin{lemma}[Confidence Set $\CI_{\cU}^{\pi,\theta}(\beta)$]\label{lem:leader-bellman-loss}
    Suppose that each trajectory in the data is independently collected and the function class $\cU$ satisfies the realizability and the completeness assumption given by \Cref{thm:Offline-MG}. Suppose we have 
    $\beta \ge 
    {110 H^2\cdot\log(H \cN_\rho(\cY, T^{-1})\delta^{-1}) }  + (45 H^2 + 60 H )
  $, where the covering number is defined by \eqref{eq:cN-cY}.
  Here, we have a joint class $\cY_h = \Theta_{h+1}\times \Pi_{h+1}\times \cU^2$ and the $\epsilon$-covering number $\cN_\rho(\cY_h,\epsilon)$ is with respect to the distance $\rho$ specified in \eqref{eq:rho-cY}.
Then for any $h\in[H], \theta\in\Theta, \pi\in\Pi$, we have with probability at least $1-\delta$: (i) $U^{\pi, \theta}\in \CI_\cU^{\pi,\theta}(\beta)$; (ii) for any $\tilde U\in\CI_\cU^{\pi,\theta}(\beta)$, $\EE_\cD[\|\tilde U_{h} - \TT_{h}^{\pi,\theta}\tilde U_{h + 1}\|^2]\le 4\beta T^{-1}$, where $\EE_\cD$ is the expectation taken with respect to the data generating distribution.


    \begin{proof}
        See \Cref{sec:proof-leader-bellman-loss} for a detailed proof.
    \end{proof}
\end{lemma}

       

\paragraph{Learning Leader's Value Function in Online Setting.}
In this subsection, we provide gurantee for online learning the leader's value function. The analysis in this subsection maily follows \citet{jin2021bellman}.
Recall the online Bellman loss given by \eqref{eq:online-MG-bellman loss} in \Cref{sec:myopic-online}, 
\begin{align*}
    &\ell_h^t(U_h', U_{h+1}, \theta_{h+1}) \nend
    &\quad = \sum_{i=1}^{t-1}  \rbr{U_h'(s_h^i, a_h^i, b_h^i) - u_h^i -  \max_{\pi_{h+1}\in\sA}\inp[\Big]{U_{h+1}(s_{h+1}^i, \cdot, \cdot)}{\pi_{h+1}\otimes \nu_{h+1}^{\pi, \theta}(\cdot, \cdot\given s_{h+1}^i)}}^2.
\end{align*}
We define confidence set for each $\theta\in\Theta, t\in[T]$ as 
\begin{align*}
    \CI_\cU^{t, \theta}(\beta) = \cbr{U\in\cU: \ell_h^t(U_h, U_{h+1}, \theta_{h+1}, \pi) - \inf_{U'\in\cU_h} \ell_h^t(U', U_{h+1}, \theta_{h+1}, \pi)\le \beta, \forall h\in[H]}.
\end{align*}
In the sequel, we define $U^{*, \theta}$ as the optimal value function if the follower's true response model is $\theta$. 
Specifically, $U^{*, \theta}$ satisfies
\begin{align*}
    U_h^{*, \theta}(s_h, a_h, b_h) = u_h(s_h, a_h, b_h) + \bigrbr{\bigrbr{P_h\circ  T_{h+1}^{*, \theta}} U_{h+1}^{*, \theta}} (s_h, a_h, b_h), 
\end{align*}
where $ T_h^{*,\theta}:\cF(\cS\times\cA\times\cB)\rightarrow \cF(\cS)$ is the policy optimization operator defined as
\begin{align*}
     T_{h}^{*,\theta} f(s_h) = \max_{\pi_h\in\sA} \dotp{f(s_h,\cdot,\cdot)}{\pi_h\otimes \nu^{\pi, \theta}(\cdot, \cdot\given s_h)}.
\end{align*}
With respect to $ T^{*,\theta}$, we define the optimistic Bellman operator for the leader 
$\TT_h^{*,\theta}:\sF(\cS\times\cA\times\cB)\rightarrow \sF(\cS\times\cA\times\cB)$ as 
\begin{align*}
    \big(\TT_h^{*,\theta} f \bigr) 
    (s_h, a_h, b_h) = u_h(s_h, a_h, b_h) + \EE_{s_{h+1}\sim P_h(\cdot\given s_h, a_h, b_h)} \bigsbr{\bigrbr{ T_{h+1}^{*,\theta}f}(s_{h+1})}.
\end{align*}
We have the following guarantee on the confidence set.
\begin{lemma}[\textrm{Online guarantee for the confidence set of the leader's value function}]\label{lem:CI-U-online}
    For the online setting and the function class $\cU$ that satisfies the realizability and the completeness assumption given by \Cref{thm:Online-MG}. We consider a joint function class $\cZ_h=\cU^2\times\Theta_{h+1}$ and denote by $\cN_\rho(\cZ_{h}, \epsilon)$ the covering number of the smallest $\epsilon$-covering net for $\cZ_h$ with respect to this distance $\rho$ defined in \eqref{eq:rho-cZ}.
    By selecting $\beta \ge \epsilon_S=c H^2 \allowbreak \log(HT\cN_\rho(\cZ, T^{-1})\delta^{-1}) + (45 H^2 + 60 B_u)$ for some universal constant $c$ where the covering number is defined by \eqref{eq:cN-cZ}, for any $t\in[T], h\in[H], \theta\in\Theta$, we have with probability at least $1-\delta$: (i) $U^{*, \theta}\in \CI_\cU^{t,\theta}(\beta)$; (ii) for any $\tilde U\in\CI_\cU^{t,\theta}(\beta)$, $\sum_{i=1}^{t-1}\EE^{\pi^i}[\orbr{\orbr{\tilde U_{h} - \TT_{h}^{*,\theta}\tilde U_{h + 1}}(s_h, a_h, b_h)}^2]\le 4\beta$ and $\sum_{i=1}^{t-1} \orbr{\orbr{\tilde U_{h} - \TT_{h}^{*, \theta}\tilde U_{h + 1}}(s_h^i, a_h^i, b_h^i)}^2 \le 4 \beta$.
    \begin{proof}
        See \Cref{sec:proof-CI-U-online} for a detailed proof.
    \end{proof}
\end{lemma}

\subsection{Putting Everything Together:  Bounding Leader's Suboptimality}\label{sec:app-connection}

Here, we give a summary of the results presented in this section and show how the results in the previous parts are connected with each other for obtaining a guarantee of the suboptimality.

\vspace{5pt} 
{\noindent \bf Controlling Leader's Bellman Error.}
In \Cref{sec:app-subopt-decompose}, we study the suboptimality decomposition for the leader. From \Cref{lem:subopt-decomposition}, we learn that the suboptimality comprises two major terms, namely the leader's Bellman error and the follower's response error. The leader's Bellman error is simply given by 
\begin{align*}
    \text{Leader's Bellman error} = \sum_{h=1}^H \EE\sbr{\bigrbr{\tilde U_h - u_h}(s_h, a_h, b_h)-  T_{h+1}^{\pi,\tilde\nu} \tilde U_{h+1}(s_{h+1})} = \sum_{h=1}^H \EE\sbr{\tilde U_h - \TT_h^{\pi,\tilde\theta} \tilde U_{h+1}}, 
\end{align*}
where a list of definitions for $\TT_h$ and $ T_h$ can be found in \Cref{sec:app-notations}. For the offline setting with independent collected data, we use the guarantee from \Cref{lem:leader-bellman-loss} that $\EE_\cD[\|\tilde U_{h} - \TT_{h}^{\pi,\theta} \tilde U_{h + 1}\|^2]\le 4\beta T^{-1}$ if $\tilde U$ if properly chosen from the confidence set $\CI_\cU^{\pi,\theta}(\beta)$. 
For the online setting, we employ  \Cref{lem:CI-U-online} to show that $\sum_{i=1}^{t-1}\EE^{\pi^i}[\|\tilde U_{h} - \TT_{h}^{\theta}\tilde U_{h + 1}\|^2]\le 4\beta$ if $\tilde U$ is properly chosen such that  $\tilde U\in\CI_\cU^{t,\theta}(\beta)$. 
Moreover, in the online setting, $\pi$ is just the optimistic policy and $\TT_h^{\pi,\theta} = \TT_h^{\theta}$. Hence, the leader's Bellman error and the value function guarantee matches and we can control the leader's Bellman error by a distribution shift argument leveraging concentrability coefficients in the offline setting or via the eluder dimension of a proper function class that captures the complexity of this Bellman error in the online setting.

\vspace{5pt} 
{\noindent \bf Controlling Myopic Follower's Quantal Response Error.}
We first look at the easier setting where we aim to control a myopic follower's quantal response error, which is given by the TV distance betweem $\nu$ and $\tilde \nu$.
Recall from \Cref{cor:response-diff-myopic} that for a given state $s\in\cS$,
\begin{align*}
    \EE D_\TV\rbr{\nu(\cdot\given s), \tilde\nu(\cdot\given s)}
        &\le  \eta  \EE\sbr{\abr{(\tilde r^\pi(s, b) - r^\pi(s, b)) - \EE\bigsbr{\tilde r^\pi(s, b) - r^\pi(s, b)}}} \nend
        &\qquad + C^{(3)}\EE\sbr{\rbr{\rbr{\tilde r^\pi(s, b) -r^\pi(s, b)} - \EE\sbr{\tilde r^\pi(s, b) -r^\pi(s, b)}}^2}.
\end{align*}
If we look at the guarantee of MLE in \eqref{eq:MLE_guarantee_Q} of \Cref{lem:MLE-formal} for the myopic case, we directly have 
\begin{align*}
    \sum_{i=1}^{t-1} \EE^{\pi^i, M^*}{\Var_{s_h}^{\pi^i, \theta^*} \bigsbr{r^{\pi^i, \theta}(s, b) - r^{\pi^i, \theta^*}(s, b)}} \le  2 C_\eta^2 \beta, 
\end{align*}
for both the online and the offline cases, which gives control to both the first order and the second order terms in the TV distance upper bound.

\vspace{5pt} 
{\noindent \bf Controlling Farsighted Follower's Quantal Response Error.}
The last part is a more challenging case for a farsighted follower. Using the result in \Cref{lem:performance diff}, 
\begin{align*}
    \text{Quantal response error} \le C^{(0)}
    \sum_{h=1}^H \underbrace{\EE\sbr{\abr{\tilde \Delta^{(1)}_h(s_h, b_h)}}}_{\ds\text{1st-order error}}  + C^{(2)}
    \max_{h\in [H]} \underbrace{\EE\sbr{ \rbr{\tilde Q_h - r_h^\pi - \gamma P_h^\pi \tilde V_{h+1}}^2}}_{\ds\text{2nd-order error}}, 
\end{align*}
with $\tilde\Delta^{(1)}$ given by 
\begin{align*}
    \tilde \Delta^{(1)}_h(s_h, b_h) &=  \rbr{\EE_{s_h, b_h} -\EE_{s_h}}\Biggsbr{\sum_{l=h}^H \gamma^{l-h}\underbrace{\rbr{\tilde Q_l - r_l^\pi - \gamma P_l^\pi \tilde V_{l+1}}(s_l, b_l)}_{\ds\text{Follower's Bellman error}}}, 
\end{align*}
it is not straightforward to see how to bound these two terms by guarantee of the MLE in \Cref{lem:MLE-formal}. 
Fortunately,  we have the following two lemmas that  bound  the first-order error and the second-order error  separately.
\begin{lemma}[Bounding First-Order Error]\label{lem:1st-ub}
    For any $\pi\in\Pi$ and $(\tilde U, \tilde W, \tilde Q, \tilde V, \tilde A, \tilde \nu)$ satisfying the conditions in \Cref{lem:subopt-decomposition}, we have for all $h\in[H]$ that
    \begin{align*}
        \EE\sbr{\abr{\tilde \Delta^{(1)}_h(s_h, b_h)}} &\le   L^{(1)} \cdot
        \max_{h\in[H]} \EE \bigl [  D_\H(\nu_h(\cdot\given s_h),\tilde\nu_h(\cdot\given s_h)) \bigr ] , 
    \end{align*}
    where $L^{(1)} = 6(\eta^{-1}+2 B_A)\cdot {\eff_H\rbr{\gamma}}$ and  $\eff_H(\gamma) = \orbr{1-\gamma^H}/\orbr{1-\gamma}$ is the effective foresight of the follower. For the second order, we have
    \begin{align}\label{eq:1st-ub-2}
        &\bigrbr{\tilde \Delta_h^{(1)}(s_h, b_h)}^2  \nend
        &\quad \le 2 \rbr{\rbr{\EE_{s_h, b_h}-\EE_{s_h}} \bigsbr{\orbr{Q_h - \tilde Q_h}(s_h, b_h)}}^2 \\
        &\qqquad + 16 \gamma^2  \rbr{\eta^{-1} +2 B_A}^2\eff_H(\gamma) \sum_{l=h+1}^H \gamma^{l-h-1} {\rbr{\EE_{s_h}+\EE_{s_h, b_h}}\sbr{D_\H^2(\nu_l(\cdot\given s_l), \tilde\nu_l(\cdot\given s_l))}}. \notag 
    \end{align}
        \begin{proof}
            See \Cref{sec:proof-1st-ub} for a detailed proof.
        \end{proof}
    \end{lemma}
An important observation is that the first order term is bounded only by the follower's quantal response distance without invoking any transition model error, even for farsighted follower. This is because that the first order term only captures parts of the follower's response error. We next bound the second order term.
\begin{lemma}[Bounding Second-Order Error] \label{lem:2nd-ub}
    For any $\pi\in\Pi$ and $(\tilde U, \tilde W, \tilde Q, \tilde V, \tilde A, \tilde \nu)$ satisfying the conditions in \Cref{lem:subopt-decomposition}, we additionally assume $\tilde Q_h(s_h, b_h) = \tilde r_h^\pi(s_h, b_h) + (\tilde P_h^{\pi} \tilde V_{h+1})(s_h, b_h)$ for estimated reward $\tilde r$ and transition kernel $\tilde P$. Suppose that the follower's reward at each state satisfies a linear constraint $\dotp{x}{r_h(s_h, a_h, \cdot)} = \varsigma$ at every $(s_h,a_h)\in\cS\times\cA$ for some $x:\cB\rightarrow \RR$ such that $\inp{\ind}{x}\neq 0$ and $\varsigma\in\RR$. Define ratio $\kappa = \nbr{x}_\infty/|\inp{x}{\ind}|$. Then we have that
    \begin{align*}
        &\max_{h\in[H]}\EE\sbr{ \rbr{\rbr{\tilde Q_h - r_h^\pi - \gamma P_h^\pi \tilde V_{h+1}}(s_h, b_h)}^2} \nend
        &\quad \le L^{(2)} \max_{h\in[H]}\cbr{\EE D_\H^2(\nu_h(\cdot\given s_h),\tilde\nu_h(\cdot\given s_h))+\EE D_\TV^2(P_h^\pi(\cdot\given s_h, b_h),\tilde P_h^\pi(\cdot\given s_h, b_h))}, 
    \end{align*}
    where $L^{(2)} = c H^2 \eff_H(c_2)^2 \kappa^2 \exp\rbr{8\eta B_A} (\eta^{-1}+B_A)^2$ for some absolute constant $c>0$, and 
    $$c_2 = \gamma \rbr{2\exp\rbr{2\eta B_A}+\kappa\exp\rbr{4\eta B_A} }.$$
    \begin{proof}
        See \Cref{sec:proof-2nd-ub} for a detailed proof.
    \end{proof}
\end{lemma}

In the farsighted follower case, we will intensively turn to these two lemmas to control the first- and second-order terms both online and offline. Specifically, the first term in \eqref{eq:1st-ub-2} can be controlled by \eqref{eq:MLE-guarantee-Q-3} in \Cref{lem:MLE-formal}, while the second term is just the Hellinger distance, which can be controlled by our guarantee in \eqref{eq:MLE-guarantee-hellinger-1}. The argument for \Cref{lem:2nd-ub} is quite the same, while both the Hellinger distance of the quantal response and the TV distance of the transition kernel is controllable, as we will see in \Cref{lem:MLE}.











\section{Proofs for Offline Myopic Case}
We present proof for \Cref{sec:offline-myopic} in this part. 
In all these proofs, $\lesssim $ only hides universal constants.


\subsection{Proof of \Cref{thm:Offline-MG}}\label{sec:proof-offline-MG}
We give a proof for offline learning the best leader's policy with myopic follower and general function class.
The proof will be carried out in three steps. In the first step, we check the validity and the accuracy of the confidence built by our algorithm. 
In the second step, we decompose the suboptimality of the leader's total reward and relate the suboptimality to the leader's Bellman error and the follower's response error under the optimal policy $\pi^*$. 
In the last step, we relate the suboptimality to the guarantee of the confidence set via a distribution shift argument, which gives us the offline learning guarantee.

\paragraph{Step 1. Validity and Accuracy of the Confidence Set.}
We first recall the confidence set constructed in \eqref{eq:myopic-offline-general-confset},
\begin{align}
    &\CI_{\cU, \Theta}^\pi(\beta) \nend
    &\quad= \cbr{
    (U,\theta)\in\cU^{\otimes H}\times\Theta:
    \rbr{ \ds
        \cL_h(\theta_h)-\inf_{\theta'\in\Theta_h}\cL_h(\theta') \le \beta 
    \atop \ds
        \ell_h(U_h, U_{h+1}, \theta_{h+1}, \pi) - \inf_{U'\in\cU_h} \ell_h(U', U_{h+1}, \theta_{h+1}, \pi)\le H^2\beta}, 
    \forall h\in[H]}. \label{eq:myopic-offline-general-confset-1}
\end{align}
which contains constraints for both the follower's reward parameter $\theta$ and the leader's value function $U$.

To see that this confidence set is valid, 
we first look at the follower's side by invoking \Cref{lem:MLE-indep-data}. To do so, we need to check the conditions of \Cref{lem:MLE-formal} as the following: 
\begin{itemize}
    \item[(i)] the data compliance condition $\PP^{\pi^t}(b_h^t\given s_h^t, (s_j^t, a_j^t, b_j^t, u_j^t)_{j\in[h-1]}, \tau^{1:t-1})  = \nu_h^{\pi^t}(b_h^t\given s_h^t), \forall h\in[H], t\in[T]$ is satisfied for independently collected trajectories; 
    \item[(ii)] $\beta \ge  2\log(e^3H\cdot \cN_\rho(\Theta, T^{-1})/\delta)$ holds by our condition on $\beta$, where $\cN_\rho(\Theta,\epsilon)= \max_{h\in[H]} \cN_\rho(\Theta_h,\epsilon)$ and the distance $\rho$ is specified by \eqref{eq:rho-Theta}.
    \item[(iii)] The confidence set for $\theta$ in \eqref{eq:myopic-offline-general-confset-1} is exactly the same as \eqref{eq:behavior_model_confset-1}. 
\end{itemize}
Therefore, \Cref{lem:MLE-indep-data} combined with \Cref{rmk:MLE-formal-myopic} says it holds with probability at least $1-2\delta$ that (i) $\theta^*\in\CI_{\cU,\Theta}^\pi(\beta)$; (ii) for any $\theta\in\CI_{\cU,\Theta}^\pi(\beta)$ and $h\in[H]$ that
\begin{align*}
    \sum_{i=1}^T\EE_\cD\sbr{D_\H^2\rbr{\nu_h^{\pi^i, \theta}(\cdot\given s_h^i), \nu_h^{\pi^i, \theta^*}(\cdot\given s_h^i)}}  \lesssim\beta.
\end{align*}
Now, using the relationship between the Hellinger distance and the variance of the Q function given by \eqref{eq:MLE_guarantee_Q} in \Cref{lem:MLE-formal}, we conclude that for all $h\in[H], \theta\in\CI_{\cU, \Theta}^\pi(\beta)$, $\pi\in\Pi$,  
\begin{align}
    \sum_{i=1}^{T} \EE_\cD {\Var_{s_h}^{\pi^i, \theta^*} \bigsbr{r_h^{\pi^i, \theta}(s_h, b_h) - r_h^{\pi^i, \theta^*}(s_h, b_h)}} \lesssim (\eta^{-1}+ B_A)^{2} \beta,\label{eq:OffMG-variance-r-MLE-guarantee}
\end{align}
where we plug in $Q_h^{\pi, \theta}=r_h^{\pi,\theta}$ for myopic follower.
Therefore, we justify that $\CI_{\cU,\Theta}^\pi(\beta)$ is a valid confidence set for $\theta$ with probability at least $1-2\delta$.

We next look at the leader's side.
For the leader's value function, we invoke \Cref{lem:leader-bellman-loss} which says that under the realizability and completeness assumptions,  if $$H^2\beta \ge 
{110 H^2\cdot\log(H \cN_\rho(\cY, T^{-1})\delta^{-1}) }  + (45 H^2 + 60 H )$$ for the joint class $\cY_h=\cU^2\times\Pi_{h+1}\times\Theta_{h+1}$ where $\cN_\rho(\cdot)$ is defined in \eqref{eq:rho-cY}, we have with probability at least $1-\delta$ that for any $\theta\in\CI_{\cU,\Theta}^\pi(\beta)$, $\pi\in\Pi$ and $h\in[H]$: (i) $U^{\pi,\theta}\in\CI_{\cU,\Theta}^\pi(\beta)$; (ii) $\EE_\cD[( U_{h} - \TT_{h}^{\pi,\theta}U_{h + 1})^2]\le 4H^2 \beta T^{-1}$ for any $U\in\CI_{\cU,\Theta}^\pi(\beta)$. Note that $\EE_\cD$ takes expecation with respect to the data generating distribution.

In particular, the condition $H^2\beta \ge 
{110 H^2\cdot\log(H \cN_\rho(\cY, T^{-1})\delta^{-1}) }  + (45 H^2 + 60 H )$ is automatically satisfied by our choice of $\beta$.
Therefore, we conclude that with probabiility at least $1-3\delta$, we have that: 
(i) (validity) $(\theta^*, U^{\theta^*, \pi})\in\CI_{\cU,\Theta}^\pi(\beta)$ for all $\pi\in\Pi$; 
(ii) (accuracy) $\EE_\cD[(U_{h} - \TT_{h}^{\pi,\theta} U_{h + 1})^2]\lesssim H^2\beta T^{-1}$ and 
\eqref{eq:OffMG-variance-r-MLE-guarantee} for all $(U, \theta)\in\CI_{\cU,\Theta}^\pi(\beta), \pi\in\Pi, h\in [H]$. The following part is based on the sucess of the confidence set $\CI_{\cU, \Theta}^\pi(\beta)$.

\paragraph{Step 2. Suboptimality Decomposition.}
Recall the pessimistic policy optimization in \eqref{eq:offline-MG-pi^hat},
\begin{align}
    (\hat\pi, \hat U, \hat \theta)=\argmax_{\pi\in\Pi} \argmin_{(U, \theta)\in\CI_{\cU,\Theta}^\pi(\beta)} 
    \underbrace{\EE_{s_1\sim\rho_0} \sbr{\inp[\big]{U_1(s_1, \cdot, \cdot)}{\pi_1\otimes\nu_1^{\pi, \theta}(\cdot,\cdot\given s_1)}}}_{\ds J(\pi, U, \theta)}.\label{eq:offline-MG-pi^hat-1}
\end{align}
Define $(\tilde U, \tilde \theta) = \argmin_{(U, \theta)\in\CI_{\cU,\Theta}^{\pi^*}(\beta)} J(\pi^*, U, \theta)$ as the pessimistic estimators of $U$ and $\theta$ under $\pi^*$, respectively.
We can decompose the suboptimality as 
\begin{align*}
    \subopt(\hat\pi) &= J(\pi^*)-J(\hat\pi) \nend
    & = J(\pi^*)- J(\pi^*, \tilde U, \tilde \theta) + \underbrace{J(\pi^*, \tilde U, \tilde \theta) - J(\hat\pi, \hat U, \hat \theta)}_{\dr  (i)} +  \underbrace{J(\hat\pi, \hat U, \hat \theta) - J(\hat\pi)}_{\dr (ii)} \nend
    &\le J(\pi^*)- \EE_{s_1\sim\rho_0} \sbr{\inp[\big]{\tilde U_1(s_1, \cdot, \cdot)}{\pi_1^*\otimes\nu_1^{\pi^*, \tilde\theta}(\cdot,\cdot\given s_1)}_{\cA \times \cB }}
\end{align*}
where the inequality holds by noting that $(\textrm{i})\le 0$ and $(\textrm{ii})\le 0$.
Here, $(\textrm{i})\le 0$ simply by policy optimization in \eqref{eq:offline-MG-pi^hat-1}.
For $(\textrm{ii})$, we note that
$J(\hat\pi) = J(\hat\pi, U^{\hat\pi,\theta^*}, \theta^*)\ge J(\hat\pi, \hat U,\hat \theta)$ following from the validity of the confidence set $\CI_{\cU,\Theta}^\pi(\theta)$ that $(U^{\hat\pi, \theta^*}, \theta^*)\in \CI_{\cU,\Theta}^{\hat\pi}(\theta)$ and also that $\hat U, \hat\theta$ is the minimizer to $J(\hat\pi, \cdot, \cdot)$.
The remaining term is just the estimation error with respect to $\pi^*$. For simplicity, we let $\tilde\nu(b_h\given s_h)=\nu^{\pi^*, \tilde \theta}(b_h\given s_h)$ and $\tilde r^{\pi^*}(s_h,b_h) = r^{\pi^*,\tilde \theta}(s_h, b_h)$. We now invoke \Cref{lem:subopt-decomposition} on the suboptimality decomposition (applied with $\pi^*, \tilde U, \tilde \nu$) and \Cref{cor:response-diff-myopic} on the response model error, which gives us,
\begin{align}
    \subopt(\hat\pi)
    &\le {{\sum_{h=1}^H \EE\abr{\bigrbr{\tilde U_h - u_h}(s_h, a_h, b_h)-  T_{h+1}^{{\pi^*},\tilde\nu} \tilde U_{h+1}(s_{h+1})}}}
    + {\sum_{h=1}^H  H \EE \nbr{\rbr{\tilde \nu_h-\nu_h}(\cdot\given s_h)}_1}\nend
    &\le {\sum_{h=1}^H \underbrace{\EE\abr{\bigrbr{\tilde U_h - u_h}(s_h, a_h, b_h)-  T_{h+1}^{{\pi^*},\tilde\nu} \tilde U_{h+1}(s_{h+1})}}_{\dr (I_1)}}\nend
    &\qquad + \sum_{h=1}^H  2H \eta  \underbrace{\EE\sbr{\abr{(\tilde r_h^{\pi^*}(s_h, b_h) - r_h^{\pi^*}(s_h, b_h)) - \EE\bigsbr{\tilde r_h^{\pi^*}(s_h, b_h) - r_h^{\pi^*}(s_h, b_h)}}}}_{\dr  (I_2)} \label{eq:OffMG-subopt}\\
    &\qquad + \sum_{h=1}^H  2H C^{(3)} \underbrace{\EE\sbr{\rbr{\rbr{\tilde r_h^{\pi^*}(s_h, b_h) -r_h^{\pi^*}(s_h, b_h)} - \EE\sbr{\tilde r_h^{\pi^*}(s_h, b_h) -r_h^{\pi^*}(s_h, b_h)}}^2}}_{\dr (I_3)},   \notag
\end{align}
where the expectation is taken with respect to $\pi^*$ and the true model. 
\paragraph{Step 3. Handling the Distribution Shift.} 
For term $(\dr{I_1})$, we have by the accuracy result on the leader's side that
\begin{align}
    (\dr{I_1}) 
    &\le \sqrt{\EE\sbr{\rbr{\rbr{\tilde U_h -  \TT_h^{\pi^*,\tilde\theta} \tilde U_{h+1}} (s_h, a_h, b_h)}^2}}\nend
    & \lesssim \sqrt{H^2\beta T^{-1}} \cdot \frac{\sqrt{{\bignbr{{\tilde U_h  -  \TT_h^{\pi^*,\tilde\theta} \tilde U_{h+1}} }_{2, d^{\pi^*}}^2}}}{\sqrt{\|\tilde U_{h} - \TT_{h}^{\pi^*,\tilde\theta} \tilde U_{h + 1}\|_{2, \cD}^2}}\nend
    &\le \sqrt{H^2\beta T^{-1}} \cdot \max_{U\in\cU,\theta\in\Theta}
    \sqrt{\frac{{{\bignbr{{ U_h  -  \TT_h^{\pi^*,\theta}  U_{h+1}} }_{2, d^{\pi^*}}^2}}}{{{\bignbr{ U_{h} - \TT_{h}^{\pi^*,\theta}  U_{h + 1}}_{2,\cD}^2}}}}. \label{eq:OffMG-I1}
\end{align}
where we define $\nbr{\cdot}_{2, d^{\pi^*}}^2 = \EE^{\pi^*}\sbr{(\cdot)^2}$ and define $\nbr{\cdot}_{2, \cD}^2 = \EE_\cD\sbr{(\cdot)^2}$. Here, $d^{\pi^*}$ is the distribution of the trajectory generated under $\pi^*$, and $\EE_\cD$ is taken with respect to the data generating distribution, where the policy $\pi^i$ is also random. 
Here, the first inequality holds by noting that the leader's Bellman operator satisifies $\TT_h^{\pi^*,\tilde\theta} = u_h + P_h \circ T_{h+1}^{\pi^*, \tilde\nu} $ by definition and using the Cauchy-Schwartz inequality, and the first inequality holds by using $\EE_\cD[\|\tilde U_{h} - \TT_{h}^{\pi^*,\tilde \theta} \tilde U_{h + 1}\|^2]\lesssim H^2\beta T^{-1}$ for $(\tilde \theta, \tilde U)$ in the confidence set $\CI_{\cU,\Theta}^{\pi^*}(\beta)$. 
The last inequality holds by further taking a supreme over all possible $(\tilde U, \tilde\theta)$.
We remind the readers that the expectation in the numerator is taken with respect to $\pi^*$ and the true model. The supremum taken over both $\cU$ and $\Theta$ is reasonable because we treat $(\theta,U)$ as the model and need small Bellman error ensured from each choice of $(\theta, U)$.

For term $\dr (I_2)$ and $\dr (I_3)$, we first note that using the Cauchy Schwartz inequality, we have that $\dr (I_2)\le \sqrt{(\dr I_3)}$. Therefore, we just need to control $(\dr I_3)$. Then, we have by the accuracy of the confidence set on the follower's side that
\begin{align}
    {\dr (I_3)}
    &= \EE\sbr{\rbr{\rbr{\tilde r_h^{\pi^*}(s_h, b_h) -r_h^{\pi^*}(s_h, b_h)} - \EE\sbr{\tilde r_h^{\pi^*}(s_h, b_h) -r_h^{\pi^*}(s_h, b_h)}}^2} \nend
    &\lesssim {(\eta^{-1}+ B_A)^{2} \beta T^{-1}} \cdot {
        \frac{\EE\Var_{s_h}^{\pi^*,\theta^*}\sbr{\rbr{\tilde r_h^{\pi^*} - r_h^{\pi^*}}(s_h, b_h)}} 
        {\EE_\cD \Var_{s_h}^{\pi^i, \theta^*} \sbr{\rbr{\tilde r_h^{\pi^i} - r_h^{\pi^i}}(s_h, b_h)}}
    }\nend
    &\le {(\eta^{-1}+ B_A)^{2} \beta T^{-1}} \cdot \max_{\theta\in\Theta} {
        \frac{\EE\Var_{s_h}^{\pi^*,\theta^*}\sbr{\rbr{r_h^{\pi^*, \theta} - r_h^{\pi^*, \theta^*}}(s_h, b_h)}} 
        {\EE_\cD \Var_{s_h}^{\pi^i, \theta^*} \sbr{\rbr{r_h^{\pi^i, \theta} - r_h^{\pi^i, \theta^*}}(s_h, b_h)}}
    }. \label{eq:OffMG-I3}
\end{align}
where the first inequality holds by plugging in $\tilde \theta$ into the superscript of the follower's reward function in \eqref{eq:OffMG-variance-r-MLE-guarantee}, which still ensures the guarantee in \eqref{eq:OffMG-variance-r-MLE-guarantee} because $\tilde \theta\in\CI_{\cU,\Theta}^{\pi^*}(\beta)$, and the last inequality holds by taking a supremum over $\tilde\theta\in\Theta$. We remind the readers that $\Var_{s_h}^{\pi^*, \theta^*}$ is taken with respect to $\nu_h^{\pi^*, \theta^*}(\cdot\given s_h)$.
For the follower's side, recall the linear operator $\Upsilon_h^\pi:\cF(\cS\times\cA\times\cB)\rightarrow \cF(\cS\times\cB)$ defined as 
\begin{align*}
    \rbr{\Upsilon_h^\pi f}(s_h, b_h) = \dotp{\pi_h(\cdot\given s_h, b_h)}{f(s_h, \cdot, b_h)} - \dotp{\pi_h\otimes \nu_h^{\pi}(\cdot,\cdot\given s_h)}{f(s_h,\cdot,\cdot)}. 
\end{align*}
Here, $(\Upsilon_h ^{\pi} f ) (s_h,b_h) $ quantifies  how far $b_h$ is from being the quantal response of $\pi$ at state $s_h$, measured in terms of  $f$. One can also think of $(\Upsilon_h^\pi f)(s_h, b_h)$ as the \say{advantage} of the reward induced by action $b_h$ compared to the reward induced by the quantal response. 
We remark that the distribution shift in \eqref{eq:OffMG-I3} can be equivalently written as  
\begin{align*}
    \max_{\theta\in\Theta} 
        \frac{\EE\Var_{s_h}^{\pi^*,\theta^*}\sbr{\rbr{r_h^{\pi^*, \theta} - r_h^{\pi^*, \theta^*}}(s_h, b_h)}} 
        {\EE_\cD \Var_{s_h}^{\pi^i, \theta^*} \sbr{\rbr{r_h^{\pi^i, \theta} - r_h^{\pi^i, \theta^*}}(s_h, b_h)}} 
        &= \max_{\theta\in\Theta} \frac{\EE \sbr{\rbr{\rbr{\Upsilon_h^{\pi^*} (r_h^\theta - r_h^{\theta^*})}(s_h, b_h)}^2}}{\EE_\cD \sbr{\rbr{\rbr{\Upsilon_h^{\pi^i} (r_h^\theta - r_h^{\theta^*})}(s_h, b_h)}^2}}\nend
        & = \max_{\theta\in\Theta} \frac{{\bignbr{\Upsilon_h^{\pi^*} (r_h^\theta - r_h^{\theta^*})}_{2, d^{\pi^*}}^2}}{\bignbr{{\Upsilon_h^{\pi^i} (r_h^\theta - r_h^{\theta^*})}}_{2,\cD}^2}
\end{align*}


Now, plugging the results in \eqref{eq:OffMG-I1} and \eqref{eq:OffMG-I3} back into the suboptimality decomposition \eqref{eq:OffMG-subopt}, we have our result that
\begin{align*}
\subopt(\hat\pi) 
&\lesssim \max_{U\in\cU,\theta\in\Theta, h\in[H]}
    \sqrt{\frac{{{\bignbr{{ U_h  -  \TT_h^{\pi^*,\theta}  U_{h+1}} }_{2, d^{\pi^*}}^2}}}{{{\bignbr{ U_{h} - \TT_{h}^{\pi^*,\theta}  U_{h + 1}}_{2,\cD}^2}}}} 
\cdot H^2\sqrt{\beta T^{-1}}  \nend
&\qquad + \max_{\theta\in\Theta, h\in[H]}\sqrt{ \frac{{\bignbr{\Upsilon_h^{\pi^*} (r_h^\theta - r_h^{\theta^*})}_{2, d^{\pi^*}}^2}}{\bignbr{{\Upsilon_h^{\pi^i} (r_h^\theta - r_h^{\theta^*})}}_{2,\cD}^2}} \cdot  H^2 \sqrt{(1+\eta B_A)^2\beta T^{-1}}\nend
&\qquad  +   \max_{\theta\in\Theta, h\in[H]} {
    \frac{{\bignbr{\Upsilon_h^{\pi^*} (r_h^\theta - r_h^{\theta^*})}_{2, d^{\pi^*}}^2}}{\bignbr{{\Upsilon_h^{\pi^i} (r_h^\theta - r_h^{\theta^*})}}_{2,\cD}^2}
}\cdot  H^2 C^{(3)}{(\eta^{-1}+ B_A)^{2} \beta T^{-1}}.
\end{align*}
Moreover, note that 
\begin{align*}
    C^{(3)} (\eta^{-1}+B_A)^2 &\le \frac 1 2 \eta^2 \exp(2\eta B_A)(2+\eta B_A\exp(2\eta B_A)) (\eta^{-1}+B_A)^2\nend
    &\le \exp(4\eta B_A) (1+\eta B_A)^3\nend
    &\le \exp(4\eta B_A) (\eta C_\eta)^3,
\end{align*}
where $C_\eta = \eta^{-1}+B_A$. 
Therefore, we complete the proof of \Cref{thm:Offline-MG}.

\subsection{Proof of \Cref{thm:PMLE-VI-myopic}}\label{sec:proof-PMLE-VI-myopic} 
In this section, we give finite sample guarantee for \Cref{alg:PMLE}. 
During this proof, we let $\hat M = \{\hat U_h, \hat W_h\}_{h\in[H]}$ be the estimated value functions given by \Cref{alg:PMLE} and $M^{\pi} = \{U_h^{\pi}, W_h^{\pi}\}_{h\in[H]}$ be the value functions under policy $\pi$ and the true model. We also let $\hat\pi$ be the estimated policy.
We first summarize the three schemes used in \Cref{alg:PMLE}.
\begin{align}
    &\textbf{S1:}\quad  \hat\pi_h(s_h)  = \argmax_{\pi_h \in\sA} \min_{\theta_h\in\confset_{h, \Theta}(\beta)} \inp[\big]{\hat U_h(s_h, \cdot,\cdot)}{\pi_h \otimes\nu_h^{\pi_h , \theta_h}(\cdot,\cdot\given s_h)}_{\cA\times\cB},\label{eq:scheme-1-new}\\
    &\textbf{S2:}\quad  \hat\pi_h(s_h) = \argmax_{\pi_h \in\sA, \atop \theta_h\in\confset_{h, \Theta}(\beta)} \inp[\big]{\hat U_h(s_h, \cdot,\cdot)}{\pi_h \otimes\nu^{\pi_h , \theta_h}(\cdot,\cdot\given s_h)}_{\cA\times\cB} - \Gamma^{(2)}_h(s_h;\pi_h , \theta_h),\label{eq:scheme-2-new}\\
    & \textbf{S3:}\quad  \hat\pi_h(s_h) = \argmax_{\pi_h \in\sA}  \inp[\big]{\hat U_h(s_h, \cdot,\cdot)}{\pi_h \otimes\nu^{\pi_h , \hat\theta_{h,\MLE}}(\cdot,\cdot\given s_h)}_{\cA\times\cB} - \Gamma^{(2)}_h(s_h;\pi_h , \hat\theta_{h, \MLE}) .\label{eq:scheme-3-new}
\end{align}
Moreover, $\hat W_h(s_h)$ is the optimal value in these three schemes.
\paragraph{Step 1. Uncertainty Quantifications. }
We first invoke the following lemma that characterizes the uncertainty on the leader's side. 
\begin{lemma}[\textit{$\Gamma^{(1)}_h$ as a  $\delta$-uncertainty quantifier, Lemma 5.2 in \citet{jin2021pessimism}}] \label{lem:Gamma_1-pessi}
     Under the data compliance condition $\PP_\cD(u_h^i=u, s_{h+1}^i=s\given \tau^{i-1}, \{s_{h'}^i, \alpha_{h'}^i, a_{h'}^i, b_{h'}^i\}_{h'\in[h]}) =\PP(u_h=u, s_{h+1}=s\given s_h=s_h^i, a_h=a_h^i, b_h=b_h^i)$, in \Cref{alg:PMLE}, by setting $\Gamma^{(1)}_h(\cdot,\cdot,\cdot) \ge C_1 d H \sqrt{\log(2d H T/\delta)}\cdot \sqrt{\phi_h(\cdot,\cdot,\cdot)^\top \Lambda_h^{-1}\phi_h(\cdot,\cdot,\cdot)}$ for some universal constant $C_1>0$, we have with probability at least $1-\delta$ that the following holds,
\begin{align}
    \abr{\rbr{P_h\hat W_{h+1}  + u_h - \phi_h^\top \hat\omega_h}(s_h, a_h, b_h)} \le \Gamma^{(1)}_h(\cdot,\cdot,\cdot), \quad\forall (s_h,a_h,b_h)\in\cS\times\cA\times\cB, h\in[H].\label{eq:OffML-Bellman-error-guarantee}
\end{align}
\begin{proof}
    See \citet[Lemma 5.2]{jin2021pessimism} for a detailed proof.
\end{proof}
\end{lemma}
Here, \Cref{lem:Gamma_1-pessi} is only slightly different from Lemma 5.2 in \citet{jin2021pessimism} in terms of the feature map $\phi$ since we additionally have a follower's action. However, nothing much is changed, and we omit the proof here.
For the follower, we invoke \Cref{eq:bandit-ub-1} in \Cref{cor:formal-MLE confset-linear myopic}, and obtain 
\begin{align}
     \max\cbr{\bignbr{\hat\theta_h-\theta_h^*}_{\Sigma_{h, \cD}^{\theta_h^*}}^2, \bignbr{\hat\theta_h-\theta_h^*}_{\Sigma_{h, \cD}^{\hat\theta_h}}^2, 
    \EE^{\pi^i}\bignbr{\hat\theta_h-\theta_h^*}_{\Sigma_{h, \cD}^{\theta_h^*}}^2, 
    \EE^{\pi^i}\bignbr{\hat\theta_h-\theta_h^*}_{\Sigma_{h, \cD}^{\hat\theta_h}}^2 } \le 8C_{\eta}^2 \beta T^{-1}, \label{eq:OffML-MLE-guarantee}
\end{align}
for any $\hat\theta_h\in \confset_{h,\Theta}(\beta)=\{\theta_h\in\Theta:\cL_{h,\cD}(\theta_h)\ge \min_{\theta_h'\in\Theta_h}\cL_{h,\cD}(\theta_h') + \beta\}$. 
Note that \eqref{eq:OffML-MLE-guarantee} holds because $\beta\ge C d\log(H(1+\eta T^2 + (1+\eta)T)\delta^{-1})$ for some universal constant $C>0$, which satisfies the condition in \Cref{cor:formal-MLE confset-linear myopic}. 

It holds straightforward that the confidence set is valid and accurate. In the sequel, for ease of presentation, we take $\CI_{\Theta}(\beta)=\cbr{\{\theta_h\}_{h\in[H]}: \theta_h\in\cC_{h,\Theta}(\beta), \forall h\in[H]}$, which is just a combination of these $H$ independent confidence sets, and all the results still hold for $\CI_\Theta(\beta)$. In the sequel, if we say $\hat\theta_h\in\CI_\Theta(\beta)$, we actually mean $\hat\theta_h\in\CI_{h,\Theta}(\beta)$. The following part is based on the success of \eqref{eq:OffML-Bellman-error-guarantee} and \eqref{eq:OffML-MLE-guarantee}.

\paragraph{Step 2. Suboptimality Decompostion via Pessimism.}
We define $\hat J=\EE_{s_1\sim\rho_0}[\hat W_1(s_1)]$ where $\hat W_1$ is the estimated state value function given by \Cref{alg:PMLE}. By a standard regret decomposition, we have for the estimated policy $\hat\pi$,
\begin{align}
    \subopt(\hat\pi) = J(\pi^*) - J(\hat\pi) = {J(\pi^*) - \EE[\hat W_1(s_1)]} + \underbrace{\EE[\hat W_1(s_1)] - J(\hat\pi)}_{\dr (i)}.\label{eq:suboopt-1}
\end{align}
We invoke \Cref{lem:subopt-decomposition}
for $\hat\pi, \hat U, \hat W, \hat\nu$, where $\hat\nu=\nu^{\hat\pi, \hat\theta}$, 
and expand term (i) of the suboptimality in \eqref{eq:suboopt-1} into
\begin{align}
    \dr{(i)} &= \EE^{\hat\pi}[\hat W_1(s_1)] - J(\hat\pi) \nend
    & = \underbrace{{\sum_{h=1}^H \EE^{\hat\pi}\sbr{\bigrbr{\hat U_h - u_h}(s_h, a_h, b_h)- \hat W_{h+1}(s_{h+1})}}}_{\dr \text{Leader's Bellman error}} + \underbrace{\sum_{h=1}^H \EE^{\hat\pi}\sbr{\hat W_h(s_h) -  T_{h}^{\hat\pi,\hat\nu} \hat U_{h}(s_{h})}}_{\ds\text{Value mismatch error}}\nend
    &\qqquad+\underbrace{\sum_{h=1}^H  H \EE^{\hat\pi} \nbr{\rbr{\hat \nu_h-\nu_h}(\cdot\given s_h)}_1}_{\dr \text{Quantal response error}}.\label{eq:subopt-2}
\end{align}
We notify the readers that $\EE^{\hat\pi}$ is the expectation taken with respect to $\hat\pi$ and the true model, 
$\nu$ is the quantal response under policy $\hat\pi$ and the true model, 
and $(\hat W, \hat U, \hat \nu, \hat\pi)$ are the estimated value funcions, estimated response, and estimated policy given by \Cref{alg:PMLE}.
Specifically, we have by the form of these three schemes that $\hat\nu = \nu^{\hat\pi, \hat\theta_{s_h}}$ where we denote by $\hat\theta_{s_h}\in\CI_\Theta(\beta)$ the minimizer/maximizer corresponding to the policy optimization problem at each state $s_h$ in \eqref{eq:scheme-1-new} and \eqref{eq:scheme-2-new}, 
and with a little abuse of notation, we also let $\hat\theta_{s_h}$ denote the MLE estimator $\hat\theta_{h, \MLE}=\argmin_{\theta_h'\in\Theta_h} \cL_{h}^t(\theta_h')$ if we are talking about Scheme 3 in \eqref{eq:scheme-3-new}.
Next, we aim to show ${\dr (i)}\le 0$ by guarantee of pessimism in the value iteration step. Firstly, we have by \Cref{lem:Gamma_1-pessi} that the leader's Bellman error is always nonpositive. We just need to look at the value mismatch error and the quantal response error. For the quantal response error, we have by \Cref{lem:response diff-myopic-linear} that 
\begin{align}
    &2 H D_\TV\rbr{\nu_h(\cdot\given s_h), \hat\nu_h(\cdot\given s_h)} \nend
    &\quad \le 2 H \min_{\Psi\in\SSS^d}\cbr{
    f\rbr{\sqrt{\trace\rbr{{\Psi}^{\dagger} \Sigma_{ s_h}^{\hat\pi, \hat\theta_{s_h}}}}\cdot \bignbr{\theta^*- \hat\theta_{s_h}}_{{\Psi}}}, 
    f\rbr{\sqrt{\trace\rbr{{\Psi}^{\dagger} \Sigma_{ s_h}^{\hat\pi, \theta_h^*}}}\cdot \bignbr{\theta^*- \hat\theta_{s_h}}_{{\Psi}}}
    }\nend
    &\quad \le 
        \Gamma^{(2)}_h(s_h;\hat\pi(s_h), \hat\theta_{s_h}), \label{eq:OffML-TV-bonus}
\end{align}
where in the last inequality, we use the first term in the minimization. Specifically, we notice the definition $f(x) = \eta x + C^{(3)} x^2$ in \Cref{lem:response diff-myopic-linear} and the definition of $\Gamma^{(2)}_h(s_h;\alpha_h, \theta_h)$ in \eqref{eq:Gamma^2}, which are matching with each other, and use the fact that $\bignbr{\theta^*- \hat\theta_{s_h}}_{\Psi} \le \sqrt{8C_\eta^2 \beta+ 4 B_\Theta^2}$ if we plug in $\Psi = T\Sigma_{h, \cD}^{\hat\theta_{s_h}} + I_d$. 
The $8C_\eta^2 \beta$ term comes from the second guarantee in \eqref{eq:OffML-MLE-guarantee} where we use the property that $\hat\theta_{s_h}\in\CI_\Theta(\beta)$ and the $4B_\Theta^2$ term comes from the norm $\bignbr{\theta^*-\hat\theta_{s_h}}_{I_d} \le 2 B_\Theta$.
Therefore, for Scheme 2 and Scheme 3, we have by \eqref{eq:subopt-2} that
\begin{align*}
    {\dr (i)} &\le \sum_{h=1}^H \EE^{\hat\pi}\sbr{\hat W_h(s_h) -  T_{h}^{\hat\pi,\hat\nu} \hat U_{h}(s_{h})} + \sum_{h=1}^H  H \EE^{\hat\pi} \nbr{\rbr{\hat \nu_h-\nu_h}(\cdot\given s_h)}_1\nend
    &\le \sum_{h=1}^H \rbr{
        - \EE^{\hat\pi}\Gamma^{(2)}_h(s_h;\hat\pi(s_h), \hat\theta_{s_h}) + \EE^{\hat\pi}\Gamma^{(2)}_h(s_h;\hat\pi(s_h), \hat\theta_{s_h})
    } \le 0,
\end{align*}
where the first inequality holds by noting that the leader's Bellman error is nonpositive by \eqref{eq:OffML-Bellman-error-guarantee}, the first term in the second inequality comes directly from the definition of $\hat W_h(s_h)$, and the second term comes from \eqref{eq:OffML-TV-bonus}.
Moreover, for Scheme 1, we have by \eqref{eq:subopt-decompose-equality-2} in \Cref{lem:subopt-decomposition}  applied to $\hat\pi, \hat W, \hat U, \hat\nu$ that 
\begin{align*}
    {\dr (i)} &= {\sum_{h=1}^H \EE^{\hat\pi}\sbr{\bigrbr{\hat U_h - u_h}(s_h, a_h, b_h)- \hat W_{h+1}(s_{h+1})}} + \sum_{h=1}^H \EE^{\hat\pi}\sbr{\hat W_h(s_h) - T_h^{\hat \pi,\hat\nu}\hat U_h(s_h)}\nend
    &\qquad+  \sum_{h=1}^H \EE^{\hat\pi}\sbr{\rbr{ T_h^{\hat \pi,\hat\nu}-  T_h^{\hat \pi,\nu}} \hat U_h(s_h)}\le 0, 
\end{align*}
where we recall that $\nu$ is the quantal response under policy $\hat\pi$ and the true model. 
Similar to our previous discussion, by \eqref{eq:OffML-Bellman-error-guarantee} we have the first term nonpositive, and by definition of $\hat W$, we have the second term equal to $0$. 
For the third term, 
by our policy optimization in \eqref{eq:scheme-1-new}, we notice that $\nu=\nu^{\hat\pi, \theta^*}$ and $\hat\nu = \nu^{\hat\pi, \hat\theta_{s_h}}$, where $\hat\nu$ minimizes $T_h^{\hat\pi, \nu}\hat U_h$ when $\theta_h\in\CI_\Theta(\beta)$. Since the confidence set is valid, we conclude that the third term is also nonpositive. 

\paragraph{Step 3. Bounding Suboptimality.}
As a result of the previous discussions, we have for all these three schemes that 
\begin{align*}
    \subopt(\hat\pi) 
    &=   {J(\pi^*) - \EE[\hat W_1(s_1)]} + {\dr (i)}\nend
    & \le  {\sum_{h=1}^H \EE\sbr{\hat W_{h+1}(s_{h+1}) - \bigrbr{\hat U_h - u_h}(s_h, a_h, b_h) }}  
    + \sum_{h=1}^H \EE\sbr{T_h^{\pi^*,\nu^*}\hat U_h(s_h) - \hat W_h(s_h)}\nend
    & \le  {\sum_{h=1}^H 2\EE \Gamma^{(1)}_h(s_h, a_h, b_h)}  
    + \sum_{h=1}^H \underbrace{\EE\sbr{T_h^{\pi^*,\nu^*}\hat U_h(s_h) - \hat W_h^*(s_h)}}_{\dr (ii)},
\end{align*}
where the first inequality holds by ${\dr{(i)}}\le 0$ and using \eqref{eq:perform-diff-linear} in \Cref{lem:subopt-decomposition} for $\pi^*, \hat W, \hat U, \nu^*$. We remind the readers that $\EE$ is taken with respect to $\pi^*$ and define $\nu$ as the quantal response under $\pi^*$ and the true model.
The second inequality holds by \eqref{eq:OffML-Bellman-error-guarantee}.
Note that $\hat W$ is the optimizer corresponding to $\hat\pi$, which is still a mismatch for $\pi^*$ and we use it simply for the purpose of replacing the leader's Bellman error. 
However, we are also lucky to have $\hat W$ as the optimal value to all these schemes, which means that $\hat W_h(s_h)$ can be lower bounded by $\hat W_h^*(s_h)$, where $\hat W_h^*(s_h)$ is the optimal value of \eqref{eq:scheme-1-new}, \eqref{eq:scheme-2-new}, or \eqref{eq:scheme-3-new} when fixing $\alpha_h = \pi^*(s_h)$. 
Therefore, we just need to control (ii) then.
We do this separately for these three schemes. 

\subparagraph{Scheme 1:} For Scheme 1, we have by definition 
\begin{align*}
    {\dr (ii)} &= \EE\sbr{T_h^{\pi^*,\nu^*}\hat U_h(s_h) - \hat W_h^*(s_h)} \nend
    & = \EE\sbr{T_h^{\pi^*,\hat \nu^*_{s_h}}\hat U_h(s_h) - \hat W_h^*(s_h)} + \EE\sbr{\rbr{\rbr{T_h^{\pi^*,\nu^*} -T_h^{\pi^*,\hat \nu^*_{s_h}}} \hat U_h}(s_h)} \nend
    & \le 2 H \EE D_\TV\rbr{\nu_h^*(\cdot\given s_h), \hat \nu_{s_h}^*(\cdot)}, 
\end{align*}
where we define $\hat \nu^*_{s_h}\in\Delta(\cB)$ as the quantal response in \eqref{eq:scheme-1-new}, \eqref{eq:scheme-2-new}, or \eqref{eq:scheme-3-new} with respect to $\alpha_h =\pi^*(s_h)$ and $\theta_h = \hat\theta_{s_h}^*$ for $\hat\theta_{s_h} ^*$ as the optimizer to \eqref{eq:scheme-1-new}, \eqref{eq:scheme-2-new}, or \eqref{eq:scheme-3-new} under $\alpha_h =\pi^*(s_h)$. 
Here, the inequality holds by noting that the first term is equal to $0$, and upper bounding the second term with the TV bound.
Next, we invoke \Cref{lem:response diff-myopic-linear} as what we have done in \eqref{eq:OffML-TV-bonus}, and we use it for $\pi^*$ and plug in $\Psi = T \Sigma_{h, \cD}^{\theta_h^*} + I_d$,
\begin{align}
    &2 H D_\TV\rbr{\hat\nu_{s_h}^*(\cdot), \nu_h^*(\cdot\given s_h)} \nend
    &\quad \le 2 H \min_{\Psi\in\SSS^d}\cbr{
    f\rbr{\sqrt{\trace\rbr{{\Psi}^{\dagger} \Sigma_{ s_h}^{\pi^*, \hat\theta^*_{s_h}}}}\cdot \bignbr{\theta^*- \hat\theta_{s_h}}_{{\Psi}}}, 
    f\rbr{\sqrt{\trace\rbr{{\Psi}^{\dagger} \Sigma_{ s_h}^{\pi^*, \theta_h^*}}}\cdot \bignbr{\theta^*- \hat\theta_{s_h}}_{{\Psi}}}
    }\nend
    &\quad \le 
    \min\cbr{\Gamma^{(2)}_h(s_h;\pi^*(s_h), \hat \theta_{s_h}^*), \Gamma^{(2)}_h(s_h;\pi^*(s_h), \theta_h^*)}. \label{eq:OffML-TV-guarantee-3}
\end{align}
This time, we use the second term in the minimization rather than the first term that we have used in \eqref{eq:OffML-TV-bonus} for obtaining the last inequaltiy. We also have $\bignbr{\theta^*- \hat\theta_{s_h}}_{\Psi} \le \sqrt{8C_\eta^2 \beta+ 4 B_\Theta^2}$ by the first one in \eqref{eq:OffML-MLE-guarantee}, which relates the term within the $f$ function to $\Gamma^{(2)}_h$. Therefore, we conclude with ${\dr (ii)}\le \Gamma^{(2)}_h(s_h;\pi^*(s_h), \theta_h^*)$ for Scheme 1. 
Hence, we have for the suboptimality that
\begin{align*}
    \subopt(\hat\pi) \le \sum_{h=1}^H 2\EE \Gamma^{(1)}_h(s_h, a_h, b_h) + \EE \Gamma^{(2)}_h(s_h;\pi^*(s_h), \theta_h^*).
\end{align*}

\subparagraph{Scheme 2:} For Scheme 2, we notice 
\begin{align*}
\hat W_h^*(s_h) 
&= \max_{\theta_h\in\confset_{h, \Theta}(\beta)} \inp[\big]{\hat U_h(s_h, \cdot,\cdot)}{\pi^*\otimes\nu^{\pi^*, \theta_h}(\cdot,\cdot\given s_h)}_{\cA\times\cB} - \Gamma^{(2)}_h(s_h;\pi^*(s_h), \theta_h)\nend
&\ge T_h^{\pi^*, \nu^*} \hat U_h(s_h)- \Gamma^{(2)}_h(s_h;\pi^*(s_h), \theta_h^*),
\end{align*}
where the inequality holds by noting that $\theta_h^*\in\CI_\Theta(\beta)$ and $\nu^*=\nu^{\pi^*, \theta^*}$.
Therefore, we directly have ${\dr(ii)}\le \Gamma^{(2)}_h(s_h;\pi^*(s_h), \theta_h^*)$. Hence, we have for the suboptimality that
\begin{align*}
    \subopt(\hat\pi) \le \sum_{h=1}^H 2\EE \Gamma^{(1)}_h(s_h, a_h, b_h) + \EE\Gamma^{(2)}_h(s_h;\pi^*(s_h), \theta_h^*).
\end{align*}

{\noindent \bf Scheme 3:}
For Scheme 3, we have by definition
\begin{align*}
    \hat W_h^*(s_h) =  T_h^{\pi^*, \hat\nu_{s_h}^*} \hat U_h(s_h) - \Gamma^{(2)}_h(s_h;\alpha_h , \hat\theta_{s_h}^*), 
\end{align*}
where we remind the readers that $\hat\theta_{s_h}^* = \hat\theta_{h, \MLE}$ in the third scheme. Therefore, we have
\begin{align*}
    {\dr (ii)} 
    &= \EE\sbr{\rbr{T_h^{\pi^*, \nu^*} - T_h^{\pi^*, \hat\nu_{s_h}^*}}\hat U_h (s_h)} + \Gamma^{(2)}_h(s_h;\alpha_h , \hat\theta_{s_h}^*)\nend
    & \le 2 H D_\TV\rbr{\nu_h^*(\cdot\given s_h), \hat \nu_{s_h}^*(\cdot)} + \Gamma^{(2)}_h(s_h;\alpha_h , \hat\theta_{s_h}^*)\nend
    &\le 2 \Gamma^{(2)}_h(s_h;\pi^*(s_h), \hat \theta_{s_h}^*) =  2 \Gamma^{(2)}_h(s_h;\pi^*(s_h), \hat \theta_{h, \MLE}),
\end{align*}
where the last inequality holds by \eqref{eq:OffML-TV-guarantee-3}, but with $\hat\theta_{s_h}^*$ plugged in and using the bound for the first term in the minimization this time.
Hence, we have for the suboptimality that
\begin{align*}
    \subopt(\hat\pi) \le \sum_{h=1}^H 2\EE \Gamma^{(1)}_h(s_h, a_h, b_h) + 2 \EE\Gamma^{(2)}_h(s_h;\pi^*(s_h), \hat \theta_{h, \MLE}),
\end{align*}
which completes the proof of \Cref{thm:PMLE-VI-myopic}.

\subsection{Proof of \Cref{rmk:MLE-PVI-dist-shift}}\label{sec:proof-MLE-PVI-dist-shift}

\begin{proof}
We remark on (i) the distribution shift and (ii) the first order coefficient. 
For the distribution shift on the leader's side, 
suppose there exists some constant $c_1$ such that \begin{align}\Lambda_h\succeq I + c_1 T \EE^{\pi^*,\nu^{\pi^*}}[\phi_h\phi_h^\top],\label{eq:MLE-PVI-coverage-1}
\end{align}
    with high probability for all $h\in[H]$, then we have 
\begin{align}
\EE^{\pi^*,\nu^{\pi^*}}[\phi^\top (I+c_1 T \EE^{\pi^*,\nu^{\pi^*}}[\phi\phi^\top])^\dagger \phi] 
&\le (c_1 T)^{-1} \trace\orbr{(\EE^{\pi^*,\nu^{\pi^*}}[\phi\phi^\top]) (\EE^{\pi^*,\nu^{\pi^*}}[\phi\phi^\top])^\dagger}\nend 
&\le (c_1 T)^{-1}d, \label{eq:MLE-PVI-discuss-1}
\end{align}
which implies that 
$\sum_{h=1}^H\EE^{\pi^*, \nu^{\pi^*}} \osbr{\Gamma^{(1)}_h(s_h, a_h, b_h)}\le C_1  d^{3/2}H^2 (c_1 T)^{-1/2}$. 
For the follower's side, it is hard to directly put a guarantee with this $\hat\theta_{h,\MLE}$. However, we notice that
\begin{align*}
    \Sigma_{h,\cD}^{\theta_h} 
    &= T^{-1}\sum_{i=1}^T \phi_h^{\pi^i}(s_h^i, \cdot)^\top \rbr{\diag(\nu_h^{\pi^t, \theta_h}(\cdot\given s_h^i)) - \nu_h^{\pi^t, \theta_h}(\cdot\given s_h^i) \nu_h^{\pi^t, \theta_h}(\cdot\given s_h^i)^\top } \phi_h^{\pi^i}(s_h^i, \cdot)\nend
    &\succeq \exp(-2\eta B_A )T^{-1}\sum_{i=1}^T \phi_h^{\pi^i}(s_h^i, \cdot)^\top \rbr{\diag(u(\cdot)) - u(\cdot)u(\cdot)^\top} \phi_h^{\pi^i}(s_h^i, \cdot)\nend
    &=\exp(-2\eta B_A )\Sigma_{h,\cD}^{u} , 
\end{align*}
where the first equality holds by definition, the first inequality uses the conclusion in \Cref{prop:Hessian-ulb} where we plug in $u$ as the uniform distribution over $\cB$, and we just define $\Sigma_{h,\cD}^u$ by this identity.
If we plug in $\nu_h^{\pi^t,\theta_h}(\cdot\given s_h^i)$ instead of $u$ in the inequality, we also have
\begin{align*}
    \Sigma_{h,\cD}^{\theta_h}\succeq \exp(-2\eta B_A )\Sigma_{h,\cD}^{\theta_h^*}.
\end{align*}
Similarly, we also have $\Sigma_{s_h}^{\pi^*,\theta_h}\preceq \exp(2\eta B_A)\Sigma_{s_h}^{\pi^*,u}$ and $\Sigma_{s_h}^{\pi^*,\theta_h}\preceq \exp(2\eta B_A)\Sigma_{s_h}^{\pi^*,\theta_h^*}$ for any $\theta_h\in\Theta_h$ according to \Cref{prop:Hessian-ulb} and we define $\Sigma_{s_h}^{\pi^*,u}$ similarly.
Therefore we have for $\hat\theta_{h,\MLE}$ and $\theta_h^*$ that
\begin{align}\label{eq:shifting}
    \trace\rbr{\bigrbr{I + T\Sigma_{h,\cD}^{\hat\theta_{h,\MLE}}}^\dagger \Sigma_{s_h}^{\pi^*, \hat\theta_{h,\MLE}} } \le \exp(4\eta B_A)\trace\rbr{\bigrbr{I + T\Sigma_{h,\cD}^{\theta_h^*}}^\dagger \Sigma_{s_h}^{\pi^*, \theta_h^*} }. 
\end{align}
Recall that $\Sigma_{h,\cD}^{\theta_h^*}$ and $\Sigma_{s_h}^{\pi^*, \theta_h^*}$ has the following representations, 
\begin{align*}
    \Sigma_{h,\cD}^{\theta_h^*} &= T^{-1} \sum_{t=1}^T \EE_{s_h^i}^{\pi^i, \nu^{\pi^i}} \bigsbr{(\Upsilon_h^{\pi^i} \phi_h) (s_h^i, b_h) (\Upsilon_h^{\pi^i} \phi_h) (s_h^i, b_h)^\top}, \nend
    \Sigma_{s_h}^{\pi^*, \theta_h^*}& = \EE_{s_h}^{\pi^*, \nu^{\pi^*}} \bigsbr{(\Upsilon_h^{\pi^*} \phi_h) (s_h, b_h) (\Upsilon_h^{\pi^*} \phi_h) (s_h, b_h)^\top}.
\end{align*}
Under the condition that 
\begin{align*}
    &I + \sum_{t=1}^T \EE_{s_h^i}^{\pi^i, \nu^{\pi^i}} \bigsbr{(\Upsilon_h^{\pi^i} \phi_h) (s_h^i, b_h) (\Upsilon_h^{\pi^i} \phi_h) (s_h^i, b_h)^\top} \nend
    &\quad \succeq I + c_2T\EE^{\pi^*, \nu^{\pi^*}} \bigsbr{(\Upsilon_h^{\pi^*} \phi_h) (s_h, b_h) (\Upsilon_h^{\pi^*} \phi_h) (s_h, b_h)^\top} , 
\end{align*}
we have by the same argument in \eqref{eq:MLE-PVI-discuss-1} that 
\begin{align*}
    \trace\rbr{(I + T \Sigma_{h,\cD}^{\theta_h^*})^\dagger \Sigma_{s_h}^{\pi^*,\theta_h^*}} 
    &\le (c_2 T)^{-1}\trace\rbr{\bigrbr{\EE^{\pi^*, \nu^{\pi^*}}\bigsbr{(\Upsilon_h^{\pi^*} \phi_h) (s_h, b_h) (\Upsilon_h^{\pi^*} \phi_h) (s_h, b_h)^\top}}^\dagger \Sigma_{s_h}^{\pi^*,\theta_h^*}} \nend
    &\le (c_2 T)^{-1}\trace\rbr{
        \bigrbr{\EE^{\pi^*,\theta_h^*}
        \bigsbr{\Sigma_{s_h}^{\pi^*,\theta_h^*}}}^\dagger 
    \EE^{\pi^*,\theta_h^*}[\Sigma_{s_h}^{\pi^*,\theta_h^*}]}\nend
    &\le (c_2 T)^{-1} d.
\end{align*}
By definition of $\Gamma_h^{(2)}$, we have
\begin{align*}
    \Gamma^{(2)}_h(s_h;\pi , \theta_h^*) = 2 H(\eta \xi  + C^{(3)} \xi^2 ) 
    \end{align*}
and also 
\begin{align*}
    \xi = \sqrt{\trace\rbr{\bigrbr{T\Sigma_{h,\cD}^{\theta_h} + I_d}^\dagger \Sigma_{s_h}^{\pi , \theta}}} \cdot \sqrt{8 C_{\eta}^2 \beta + 4B_\Theta^2}, 
\end{align*}
Therefore, the first order term in $\sum_{h=1}^H \EE^{\pi^*,\nu^{\pi^*}}\Gamma^{(2)}_h(s_h;\pi, \theta_h^*)$ is bounded by
\begin{align*}
    &2 H^2 \eta \cdot\sqrt{8C_\eta^2 \beta+ 4B_\Theta^2}\cdot \sqrt{\trace\bigrbr{\bigrbr{T\Sigma_{h,\cD}^{\theta_h^*} + I}^\dagger \Sigma_{s_h}^{\pi,\theta_h^*}}}\nend
    &\quad \le 
    2 H^2 \eta \sqrt{d(8C_\eta^2 \beta+ 4B_\Theta^2)}\cdot  (c_2 T)^{-1/2}\nend
    &\quad \lesssim H^2 \eta C_\eta \sqrt{d\beta /(c_2 T)}, 
\end{align*}
and the second order term in $\sum_{h=1}^H \EE^{\pi^*,\nu^{\pi^*}}\Gamma^{(2)}_h(s_h;\pi, \theta_h^*)$ is bounded by 
\begin{align*}
    &2H^2 C^{(3)}\cdot  \rbr{8C_\eta^2 \beta+ 4B_\Theta^2}\cdot {\trace\bigrbr{\bigrbr{T\Sigma_{h,\cD}^{\theta_h^*} + I}^\dagger \Sigma_{s_h}^{\pi,\theta_h^*}}} \nend
    &\quad \le  2H^2 C^{(3)}\cdot  d \rbr{8C_\eta^2 \beta+ 4B_\Theta^2}\cdot (c_2 T)^{-1}\nend
    &\quad \lesssim H^2 d(\eta C_\eta)^3 \exp(4\eta B_A) \beta (c_2 T)^{-1}.
\end{align*}
Moreover, we plug in $\beta\sim d$ given by the covering number $\cN_\rho(\Theta,T^{-1})$ in \eqref{eq:cN-Theta_h}, $C_\eta = \eta^{-1}+B_A$, $B_\Theta =\cO(1)$, and $C^{(3)} = \rbr{2+\exp\rbr{2\eta B_A}\eta B_A} \eta^2 \exp\rbr{2\eta B_A}/{2}$, which gives us
\begin{align*}
    {\subopt(\hat\pi) 
    \lesssim \frac{d^{3/2}H^2} {\sqrt{c_1 T}} +  \eta C_\eta H^{2}d \cdot \sqrt{\frac{1}{c_2 T}} +  e^{4\eta B_A} (\eta C_\eta)^3 H^2  d^2 \cdot  \frac{1}{c_2 T}, }
\end{align*}
for Scheme 1 and Scheme 2. 
For Scheme 3, we have by our previous result in \eqref{eq:shifting} that
\begin{align*}
    { \subopt(\hat\pi) 
    \lesssim \frac{d^{3/2}H^2} {\sqrt{c_1 T}} + e^{2\eta B_A} \eta C_\eta H^{2}d \cdot \sqrt{\frac{1}{c_2 T}} +  e^{8\eta B_A} (\eta C_\eta)^3 H^2  d^2 \cdot  \frac{1}{c_2 T}.}
\end{align*}
Therefore, we conclude the proof.
\end{proof}
\section{Proofs for Online Myopic Case}
We present proof for \Cref{sec:myopic-online} in this part.
Before we dive into the proof,

\subsection{Proof of \Cref{thm:Online-MG}}\label{sec:proof-Online-MG}
We give a proof on \Cref{thm:Online-MG} in this subsection.

\paragraph{Step 1. Validity and Accuracy for the Confidence Set.}
Recall the confidence set we constructed in \eqref{eq:myopic-online-general-confset}
\begin{align}
    &\CI_{\cU, \Theta}^t(\beta) \nend
    &\quad= \cbr{
    (U,\theta)\in\cU\times\Theta:
    \rbr{ \ds
        \cL_h^t(\theta_h)-\inf_{\theta_h'\in\Theta_h}\cL_h^t(\theta_h') \le \beta 
    \atop \ds
        \ell_h^t(U_h, U_{h+1}, \theta_{h+1}) - \inf_{U'\in\cU_h} \ell_h^t(U', U_{h+1}, \theta_{h+1})\le H^2\beta}, 
    \forall h\in[H]}. \label{eq:OnMG-confset}
\end{align}
For the follower's side, note that the data compliance condition is automatically satisfied by the online interaction process. 
Since $\beta\ge C\log(H T\cN_\rho(\Theta, T^{-1})\delta^{-1})$, we have by \eqref{eq:MLE-guarantee-Q-3} in \Cref{lem:MLE-formal} that with probability at least $1-\delta$, for all $\theta\in \CI_{\cU,\Theta}^t(\beta), h\in[H], t\in[T]$, 
\begin{align}
    &\sum_{i=1}^{t-1} \rbr{ \rbr{r_h^{\hat\pi^i, \theta} - r_h^{\hat\pi^i, \theta^*}}(s_h^i, b_h^i) - \EE_{s_h^i}^{\hat\pi^i, \theta'}\sbr{\rbr{r_h^{\hat\pi^i, \theta} - r_h^{\hat\pi^i, \theta^*}}(s_h, b_h)}}^2 
    \lesssim C_\eta^2 \beta
    \label{eq:OnMG-MLE-guarantee}
\end{align}
where we recall the definition of $\cN_\rho(\Theta, \epsilon)$ in \eqref{eq:cN-Theta-myopic}. 
Recall that $\hat\pi^i$ is the policy used in episode $i$. On the leader's side, noticing that $H^2\beta \gtrsim H^2\log(HT\cN_\rho(\cZ, T^{-1})\delta^{-1})$, we have by \Cref{lem:CI-U-online} that with probability at least $1-\delta$ for all $t\in[T]$: 
\begin{itemize}
    \item[(i)] (Validity) $U^{*, \theta}\in \CI_{\cU, \Theta}^t(\beta)$ for any $\theta\in\CI_{\cU,\Theta}^t(\beta)$; 
    \item[(ii)] (Accuracy) for any $(\theta, U)\in\CI_{\cU, \Theta}^t(\beta)$, $\sum_{i=1}^{t-1}\EE^{\hat\pi^i}[(U_{h} - \TT_{h}^{*, \theta} U_{h + 1})^2]\lesssim H^2\beta$ and $\sum_{i=1}^{t-1}|(U_{h} - \TT_{h}^{*,\theta} U_{h + 1})(s_h^i, a_h^i, b_h^i)|^2\lesssim H^2\beta$ for all $h\in[H]$,
\end{itemize}
where we remind the readers that $U^{*, \theta}$ is defined as
\begin{align*}
    U_h^{*, \theta}(s_h, a_h, b_h) = u_h(s_h, a_h, b_h) + \bigrbr{\bigrbr{P_h\circ  T_{h+1}^{*, \theta}} U_{h+1}^{*, \theta}} (s_h, a_h, b_h), 
\end{align*}
where $ T_h^{*,\theta}:\cF(\cS\times\cA\times\cB)\rightarrow \cF(\cS)$ is the one-step optimistic integral operator defined as
\begin{align*}
     T_{h}^{*, \theta} f(s_h) = \max_{\alpha_h\in\cA} \dotp{f(s_h,\cdot,\cdot)}{\alpha_h\otimes \nu^{\alpha_h, \theta}(\cdot, \cdot\given s_h)}.
\end{align*}
With respect to $ T^{*, \theta}$, we define the optimistic Bellman operator for the leader 
$\TT_h^{*,\theta}:\sF(\cS\times\cA\times\cB)\rightarrow \sF(\cS\times\cA\times\cB)$ as 
\begin{align*}
    \bigrbr{\TT_h^{*, \theta} f} (s_h, a_h, b_h) = u_h(s_h, a_h, b_h) + \EE_{s_{h+1}\sim P_h(\cdot\given s_h, a_h, b_h)} \bigsbr{\bigrbr{ T_{h+1}^{*, \theta}f}(s_{h+1})}.
\end{align*}
Combining these results, we have for $\CI_{\cU,\Theta}^t(\beta)$ that with probability at least $1-2\delta $ and for all $t\in[T]$ that,
\begin{itemize}
    \item[(i)] (Validity) $(\theta^*, U^{*})\in \CI_{\cU, \Theta}^t(\beta)$; 
    \item[(ii)] (Accuracy) for any $(\theta, U)\in\CI_{\cU, \Theta}^t(\beta)$, we have \eqref{eq:OnMG-MLE-guarantee} and  $\sum_{i=1}^{t-1}\EE^{\hat\pi^i}[\|U_{h} - \TT_{h}^{*,\theta} U_{h + 1}\|^2]\lesssim H^2\beta$, $\sum_{i=1}^{t-1}\|(U_{h} - \TT_{h}^{*, \theta} U_{h + 1})(s_h^i, a_h^i, b_h^i)\|^2\lesssim H^2\beta$ hold for all $h\in[H]$.
\end{itemize}
The following proof is based on the success of $\CI_{\cU, \Theta}^t(\beta)$.

\paragraph{Step 2. Suboptimality Decomposition via Optimism.}
Recall the optimisitic policy optimization in \eqref{eq:online-MG-opt-parameter} that 
\begin{align}
    (\hat U^t, \hat\theta^t)=\argmax_{(U, \theta)\in\CI_{\cU,\Theta}^t(\beta)} \underbrace{\max_{\pi_1\in\Pi}\EE_{s_1\sim\rho_0} \sbr{{\inp[\big]{U_1(s_1, \cdot, \cdot)}{\pi_1\otimes\nu_1^{\pi, \theta}(\cdot,\cdot\given s_1)}}}_{\cA\times \cB } }_{\ds J(U, \theta)}, \label{eq:OnMG-U theta hat}
\end{align}
and also the optimistic policy in \eqref{eq:online-MG-hat pi} that 
\begin{align}
    \hat\pi^t(s_h)=\argmax_{\pi(s_h)\in\sA}\inp[]{U_{h}(s_{h}, \cdot, \cdot)}{\pi_{h}\otimes \nu_{h}^{\pi, \theta}(\cdot, \cdot\given s_{h})}. \label{eq:OnMG-pi hat}
\end{align}
We have for the suboptimality that
\begin{align*}
    \subopt(\hat\pi^t) = J(\pi^*) - J(\hat U^t, \hat\theta^t) + J(\hat U^t, \hat\theta^t) - J(\hat\pi^t) \le J(\hat U^t, \hat\theta^t) - J(\hat\pi^t)
\end{align*}
For the first inequality, we notice by the validity of $\CI_{\cU,\Theta}^t(\beta)$ that $(U^{\theta^*}, \theta^*)\in \CI_{\cU,\Theta}^t(\beta)$ and and by optimism $J(\hat U^t, \hat \theta^t) \ge J(U^{\theta^*}, \theta^*) = J(\pi^*)$. 
we then invoke \eqref{eq:subopt-decompose-equality-1} of \Cref{lem:subopt-decomposition} with respect to $\hat\pi^t, \hat U^t, \hat\nu^t$ where $\hat\nu^t$ is the quantal response with respect to $\hat\pi^t$ and $\hat\theta^t$,
\begin{align*}
    \subopt(\hat\pi^t) &\le J(\hat U^t, \hat\theta^t) - J(\hat\pi^t) \nend
    &\le {\sum_{h=1}^H \EE^t\Bigsbr{ {\bigrbr{\hat U_h^t - u_h}(s_h, a_h, b_h)-  T_{h+1}^{\hat\pi^t,\hat\nu^t} \hat U_{h+1}^t(s_{h+1})}}} 
    + \sum_{h=1}^H 2 H \EE^t D_\TV\rbr{\hat\nu_h^t(\cdot\given s_h), \nu_h^t(\cdot\given s_h)}\nend
    &\le  {\sum_{h=1}^H \EE^t\Bigsbr{\underbrace{\bigrbr{\hat U_h^t - u_h}(s_h, a_h, b_h)-  T_{h+1}^{*, \hat\theta^t} \hat U_{h+1}^t(s_{h+1})}_{\ds \hat U_h^t - \TT_h^{*, \hat\theta^t}\hat U_{h+1}^t }}} \nend
    &\qquad + \sum_{h=1}^H 2 H \eta \EE^t\Bigsbr{\bigabr{ \underbrace{(\hat r_h^t(s_h, b_h) - r_h^t(s_h, b_h)) - \EE^t\bigsbr{\hat r_h^t(s_h, b_h) - r_h^t(s_h, b_h)}}_{\ds \Upsilon_h^{\hat\pi^t} (r_h^{\hat\theta^t} - r_h) }} } \nend
    &\qquad + \sum_{h=1}^H 2 H C^{(3)}\EE^t\sbr{\rbr{\rbr{\hat r_h^t(s_h, b_h) -r_h^t(s_h, b_h)} - \EE^t \sbr{\hat r_h^t(s_h, b_h) -r_h^t(s_h, b_h)}}^2},
\end{align*}
where $\nu^t$ is the actual quantal response under $\hat\pi^t$ and the true model, $r^t(s_h, b_h) = \inp[]{\hat\pi^t(\cdot\given s_h, b_h)}{r_h(s_h, \cdot, \cdot)}$ is the integral reward for the follower under $\hat\pi^t$ and the true model, and $\hat r^t$ is the alternative of $r^t$ under the same policy $\hat\pi^t$ but the estimated model $\hat\theta^t$. We also denote by $\EE^t$ the expectation taken with respect to $\hat\pi^t$ and the true model.
Here, for the last inequality, we replace $T_{h+1}^{\hat\pi^t, \hat\nu^t}$ by the optimistic integral operator $T_{h+1}^{*, \hat\theta^t}$ by definition of $\hat\pi^t$ in the first term, and the remaining hold by using \Cref{cor:response-diff-myopic} for the TV distance.

\paragraph{Controlling Regret by eluder Dimension.}
For the leader's side, we have the following configurations:
\begin{itemize}[leftmargin=20pt]
    \item Define function class $\cG_{h,L} = \bigcbr{g:\cS\times\cA\times\cB\rightarrow \RR\given  g = U_h-\TT_h^{*, \theta} U_{h+1}, \exists U\in\cU, \theta\in\Theta}$. 
    Moreover, consider sequence $\{g_h^i = \hat U_h^i - \TT_h^{*, \hat\theta^i} \hat U_{h+1}^i\}_{i\in[T]}$, it is obvious that $g_h^i\in\cG_{h,L}$;
    \item Define the class of probability measures as 
    $$\sP_{h,L} = \cbr{\rho\in\Delta(\cS\times\cA\times\cB)\given\rho(\cdot)=\PP^\pi((s_h, a_h, b_h)=\cdot), \pi\in\Pi}.$$
    Moreover, consider sequence $\{\rho^i(\cdot) = \PP^{\hat\pi^i}((s_h, a_h, b_h)=\cdot)\}_{i\in[T]}$;
    \item For the chosen sequences, we have
    \begin{align*}
        \EE_{\rho_h^i}[g_h^t] = \EE^{i} \sbr{(\hat U_h^t - \TT_h^{*, \hat\theta^t} \hat U_{h+1}^t)(s_h,a_h,b_h)}.
    \end{align*}
    If $i=t$, this corresponds to the leader's Bellman error we aim to bound.
    Moreover, the online guarantee is 
    $$\sum_{i=1}^{t-1} \rbr{\EE_{\rho^i}\sbr{g_h^t}}^2 = \sum_{i=1}^{t-1}\EE^{i}[(\hat U_{h}^{t} - \TT_{h}^{*, \hat\theta^t} \hat U_{h + 1}^t)^2]\lesssim H^2\beta$$ for any $t\in[T]$ by the guarantee of $\CI_{\cU, \Theta}^t(\beta)$ since $(\hat U^t, \hat\theta^t)\in\CI_{\cU,\Theta}^t(\beta)$. Moreover, $|\EE_{\rho^i}[g_h^t]|\le 2H$ globally.
\end{itemize}
Under these conditions, we let $\dim(\cG_L) = \max_{h\in[H]}\dim_\DE\rbr{\cG_{h, L}, \sP_{h, L}, T^{-1/2}}$ and we have by the first order cumulative error in \Cref{lem:de-regret} that
\begin{align*}
    &\sum_{i=1}^T \abr{\EE^t\sbr{\bigrbr{\hat U_h^t - u_h}(s_h, a_h, b_h)-  T_{h+1}^{\hat\pi^t,\hat\nu^t} \hat U_{h+1}^t(s_{h+1})}} \nend
    & \quad \lesssim \sqrt{\dim(\cG_L) H^2\beta T} + \min\cbr{T, \dim(\cG_L)} H + \sqrt T.
\end{align*}
For the follower's side, recall the linear operator $\Upsilon_h^\pi:\cF(\cS\times\cA\times\cB)\rightarrow \cF(\cS\times\cB)$ defined as 
\begin{align*}
    \rbr{\Upsilon_h^\pi f}(s_h, b_h) = \dotp{\pi_h(\cdot\given s_h, b_h)}{f(s_h, \cdot, b_h)} - \dotp{\pi_h\otimes \nu_h^{\pi}(\cdot,\cdot\given s_h)}{f(s_h,\cdot,\cdot)}. 
\end{align*}
Here, $(\Upsilon_h ^{\pi} f ) (s_h,b_h) $ quantifies  how far $b_h$ is from being the quantal response of $\pi$ at state $s_h$, measured in terms of  $f$. One can also think of $(\Upsilon_h^\pi f)(s_h, b_h)$ as the \say{advantage} of the reward induced by action $b_h$ compared to the reward induced by the quantal response. 
We have the following configurations for the follower's quantal response error:
\begin{itemize}[leftmargin=20pt]
    \item We define function class on $\cS\times\cA\times\cB$ as,
    \begin{align*}
    \cG_{h, F} = \cbr{g:\cS\times\cA\times\cB\rightarrow \RR: \exists \theta\in\Theta, g= {r_h^\theta - r_h}}.
    \end{align*}
    In addition, we consider a sequence $\{g_h^i = r_h^{\hat\theta^i} - r_h\}_{i\in[T]}$. 
    \item We define a class of signed measures on $\cS\times\cA\times\cB$ as
    \begin{align*}
        \sP_{h, F} &= \{\rho(\cdot) = \PP^\pi(a_h=\cdot\given s_h, b_h)\delta_{(s_h, b_h)}(\cdot) 
        - \PP^\pi((a_h, b_h)=\cdot\given s_h)\delta_{(s_h)}(\cdot)\nend
        &\qqquad \biggiven \pi\in\Pi, (s_h, b_h)\in\cS\times\cB\}, 
    \end{align*}
    where $\delta_{(s_h, b_h)}(\cdot)$ is the measure that assigns measure $1$ to state-action pair $(s_h, b_h)$. 
    In addition, we consider a sequence 
    $$\cbr{\rho_h^i(\cdot) = \PP^{\hat\pi^i}(a_h=\cdot\given s_h, b_h)\delta_{(s_h, b_h)}(\cdot) 
    - \PP^{\hat\pi^i}((a_h, b_h)=\cdot\given s_h)\delta_{(s_h)}(\cdot)}_{i\in[T]}$$
    \item 
    For simplicity, we denote by $g_h^t(s_h^i, b_h^i, \pi^i)$ the integral of $g_h^t$ with respect to the signed measure $\rho_h^i$ since $\rho_h^i$ is uniquely determined by $(s_h^i, b_h^i, \pi^i)$. 
    It is easy to check that $$g_h^t(s_h^i, b_h^i, \pi^i) = \bigrbr{\Upsilon_h^{\pi^i}(r_h^{\hat\theta^t}-r_h)}(s_h^i, b_h^i) = \QRE(s_h^i, b_h^i; \hat\theta^t, \pi^i).$$
    We have $g_h^t(s_h^i, b_h^i, \pi^i)f$ globally bounded by $4$ and the
    online guarantee for these two sequences are 
    \begin{align*}
        \sum_{i=1}^{t-1} (g_h^t(s_h^i, b_h^i, \pi^i))^2 \lesssim C_\eta^2\beta, 
    \end{align*}
    which holds because $\hat\theta^t\in\CI_{\cU, \Theta}^t(\beta)$ and by \eqref{eq:OnMG-MLE-guarantee}.
\end{itemize}
We let $\dim(\cG_F) = \max_{h\in[H]}\dim_\DE(\cG_{h,F},\sP_{h, F}, T^{-1/2})$ and by using \eqref{eq:DE error-1st order}, we have
\begin{align*}
    &\sum_{t=1}^T \EE^t\sbr{\abr{(\hat r_h^t(s_h, b_h) - r_h^t(s_h, b_h)) - \EE^t\bigsbr{\hat r_h^t(s_h, b_h) - r_h^t(s_h, b_h)}}}\nend
    &\quad \lesssim  2\sum_{t=1}^T \abr{\Upsilon_h^{\hat\pi^t} \rbr{r_h^{\hat\theta^t} - r_h}(s_h^t, b_h^t) } + \log\rbr{\delta^{-1} H\cN_\rho(\Theta, T^{-1})} + 1\nend
    &\quad \lesssim  \sqrt{\dim(\cG_F) (C_\eta^2 + 32)\beta T} + \min\cbr{T, \dim(\cG_F)} +  \beta, 
\end{align*}
for all $h\in[H]$. Here, the first inequality uses a standard martingale concentration in \Cref{cor:martigale concentration} and note that the approximation error of a $T^{-1}$- covering net $\Theta_{T^{-1}}$ is $\cO(1)$. Moreover, we have for the second order term that 
\begin{align*}
    &\sum_{t=1}^T \EE^t\sbr{\rbr{\rbr{\hat r_h^t(s_h, b_h) -r_h^t(s_h, b_h)} - \EE^t \sbr{\hat r_h^t(s_h, b_h) -r_h^t(s_h, b_h)}}^2}\nend
    &\quad \lesssim 2\sum_{t=1}^T \rbr{\Upsilon_h^{\hat\pi^t} \rbr{r_h^{\hat\theta^t} - r_h}(s_h^t, b_h^t)}^2 + \log\rbr{\delta^{-1} H\cN_\rho(\Theta, T^{-1})} +  1\nend
    & \quad \lesssim \dim(\cG_F) C_\eta^2 \beta \log(T) + \min\cbr{T, \dim(\cG_F)} + \beta.
\end{align*}
In summary, for the Leader's Bellman error
\begin{align*}
    \text{LBE} \lesssim H^2 \sqrt{\dim(\cG_L) \beta T},
\end{align*}
for the follower's first-order QRE, 
\begin{align*}
    \text{1st-QRE}  \lesssim H^2 \eta \sqrt{\dim(\cG_F)C_\eta^2\beta T} \lesssim H^2 \eta C_\eta \sqrt{\dim(\cG_F) \beta T},
\end{align*}
and for the follower's second-order QRE,
\begin{align*}
    \text{2nd-QRE}
    \lesssim H^2 C^{(3)} \dim(\cG_F) C_\eta^2 \beta \log(T) \lesssim H^2 (\eta C_\eta)^3  \exp(4\eta B_A) \beta \log T
\end{align*}
Summing them up gives us the result to \Cref{thm:Online-MG}.


\subsection{Proof of \Cref{thm:Online-ML}}\label{sec:proof-Online-ML}
We give a proof for online learning with myopic follower and linear function approximation.

\paragraph{Step 1. Uncertainty Quantification.}
For the leader's side, we are able to invoke the same guarantee in \Cref{lem:Gamma_1-pessi} for all $t\in[T]$ with a union bound, which requires a $\beta$ that additionally has a $\log(T)$ dependency. The reason that \Cref{lem:Gamma_1-pessi} can still be used here is that the data compliance condition is still satisfied in the online learning process. Therefore, with probability at least $1-\delta$, for any $t\in[T], h\in[H]$, 
\begin{align}
    \abr{\rbr{P_h\hat W_{h+1}^t  + u_h - \phi_h^\top \hat\omega_h^t}(s_h, a_h, b_h)} \le \Gamma^{(1, t)}_h(\cdot,\cdot,\cdot), \quad\forall (s_h,a_h,b_h)\in\cS\times\cA\times\cB.\label{eq:OnML-Bellman-error-guarantee}
\end{align}
For the follower's side, we invoke \eqref{eq:app-bandit-ub-1} in \Cref{cor:formal-MLE confset-linear myopic}, and obtain 
\begin{align}
    \max\cbr{\bignbr{\hat\theta_h-\theta_h^*}_{\Sigma_{h, t}^{\theta_h^*}}^2, \bignbr{\hat\theta_h-\theta_h^*}_{\Sigma_{h, t}^{\hat\theta_h}}^2, 
    \EE^{\pi^i}\bignbr{\hat\theta_h-\theta_h^*}_{\Sigma_{h, t}^{\theta_h^*}}^2, 
    \EE^{\pi^i}\bignbr{\hat\theta_h-\theta_h^*}_{\Sigma_{h, t}^{\hat\theta_h}}^2 } \le 8 C_{\eta}^2 \beta, \label{eq:OnML-MLE-guarantee}
\end{align}
for any $\hat\theta_h\in \confset_{h,\Theta}^t(\beta)=\{\theta_h\in\Theta:\cL_{h}^t(\theta_h)\le \min_{\theta_h'\in\Theta_h}\cL_{h}^t(\theta_h') + \beta\}$.
Here, we require $\beta \ge C d\log(HT(1+\eta T^2 +(1+\eta)T)\delta^{-1})$, which takes a union bound over $\Theta_h$, $h\in[H]$, and $t\in[T]$. 
It holds straightforward that the confidence set is valid and accurate. In the sequel, for ease of presentation, we take $\CI_{\Theta}^t(\beta)=\cbr{\{\theta_h\}_{h\in[H]}: \theta_h\in\cC_{h,\Theta}^t(\beta), \forall h\in[H]}$, which is just a combination of these $H$ independent confidence sets, and all the results still hold for $\CI_\Theta^t(\beta)$. In the sequel, if we say $\hat\theta_h\in\CI_\Theta^t(\beta)$, we actually means $\hat\theta_h\in\CI_{h,\Theta}^t(\beta)$. The following part is based on the success of \eqref{eq:OnML-Bellman-error-guarantee} and \eqref{eq:OnML-MLE-guarantee}.

\paragraph{Step 2. Suboptimality Decomposition: Optimism.}
Recall the two schemes in \eqref{eq:scheme-4} and \eqref{eq:scheme-5}, 
\begin{align}
    \textbf{S4:}\quad  \pi_h^t(s_h)  &= \argmax_{\pi_h(s_h) \in\sA
    \atop \theta_h\in\confset_{h, \Theta}^t(\beta)} \inp[\big]{\hat U_h^t(s_h, \cdot,\cdot)}{\pi_h \otimes\nu_h^{\pi_h , \theta_h}(\cdot,\cdot\given s_h)}_{\cA_h\times \cB_h}, \quad \forall s_h\in\cS_h,\label{eq:scheme-4-new}\\
    \textbf{S5:}\quad  \pi_h^t(s_h)  &= \argmax_{\pi_h(s_h) \in\sA}  \inp[\big]{\hat U_h^t(s_h, \cdot,\cdot)}{\pi_h \otimes\nu^{\pi_h , \hat\theta_{h,\MLE}^t}(\cdot,\cdot\given s_h)}_{\cA_h\times \cB_h} + \Gamma^{(2,t)}_h(s_h;\pi_h , \hat\theta_{h, \MLE}^t). \label{eq:scheme-5-new}
\end{align}
Note that $\hat W_h(s_h)$ is simply the optimal value. 
We have for the suboptimality that
\begin{align*}
    \subopt(\hat\pi^t) = J(\pi^*) - J(\hat\pi^t) = \underbrace{J(\pi^*) - \EE\sbr{\hat W_1^t(s_1)}}_{\dr (i)} + {\EE\sbr{\hat W_1^t(s_1)} - J(\hat\pi^t)}, 
\end{align*}
By the performance decomposition of (i) given  by \eqref{eq:subopt-decompose-equality-2} of \Cref{lem:subopt-decomposition} applied with $\pi^*, \hat W^t,\hat U^t, \hat \nu^{*, t}$, where $\hat\nu^{*, t}$ is the optimizer to \eqref{eq:scheme-4-new} or \eqref{eq:scheme-5-new} when fixing $\pi_h=\pi_h^*$. We have that
\begin{align}
    \dr{(i)} &= J(\pi^*) - \EE[\hat W_1^t(s_1)] \nend
    & ={{\sum_{h=1}^H \EE\sbr{-\bigrbr{\hat U_h^t - u_h}(s_h, a_h, b_h) + \hat W_{h+1}^t(s_{h+1})}}} + {\sum_{h=1}^H \EE\sbr{-\hat W_h^t(s_h) + T_{h}^{\pi^*,
    \hat\nu^{*,t}} \hat U_{h}^t(s_{h})}} \nend
    &\qquad + \sum_{h=1}^H \rbr{T_{h}^{\pi^*,
    \nu} - T_{h}^{\pi^*,
    \hat\nu^{*,t}}} \hat U_{h}^t(s_{h})\label{eq:OnML-subopt-1} \\
    &\le {{\sum_{h=1}^H \EE\sbr{-\bigrbr{\hat U_h^t - u_h}(s_h, a_h, b_h) + \hat W_{h+1}^t(s_{h+1})}}} + {\sum_{h=1}^H \EE\sbr{-\hat W_h^t(s_h) + T_{h}^{\pi^*,\hat\nu^{*,t}} \hat U_{h}(s_{h})}}\nend
    &\qqquad + {\sum_{h=1}^H  \EE D_\TV\rbr{\hat\nu_h^{*, t}(\cdot\given s_h), \nu_h(\cdot\given s_h)}}.\label{eq:OnML-subopt-2}
\end{align}
where we define $\nu= \nu^{\pi^*}$ and the expectation is taken with respect to the trajectory induced by $\pi^*$ and the true model. In both Scheme 4 and Scheme 5, we directly have the first term of \eqref{eq:OnML-subopt-1} and \eqref{eq:OnML-subopt-2} nonpositive since $ \hat U_h^t(s_h, a_h, b_h) = \phi_h(s_h,a_h,b_h)^\top \hat\omega_h^t + \Gamma^{(1,t)}_h(s_h, a_h, b_h)\ge P_h \hat W_{h+1}^t + u_h$ by \eqref{eq:OnML-Bellman-error-guarantee}. 
For Scheme 4, we have that the second term in \eqref{eq:OnML-subopt-1} is also nonpositive since $\theta^*\in\CI_\Theta^t(\beta)$, and $\hat W_h^t(s_h)= \max_{\pi_h(s_h)\in\sA, \theta\in\CI_\Theta(\beta)} T_h^{\theta, \pi} \hat U_h^t(s_h) \ge \max_{\theta\in\CI_\Theta^t(\beta)}T_h^{\theta, \pi^*} \hat U_h^t(s_h) = T_h^{\pi^*, \hat\nu^{*,t}}$. Moreover, for the third term, we have it also nonpositive since $\hat\nu^{*,t}$ is the maximizer under $\pi^*$ and the fact that $\theta^*\in\CI_\Theta(\beta)$.
Therefore, we have \eqref{eq:OnML-subopt-1} nonpositive for Scheme 4. 

For Scheme 5, we just check the second term in \eqref{eq:OnML-subopt-2},
\begin{align*}
    &{-\hat W_h^t(s_h) + T_{h}^{\pi^*,\hat\nu^{*,t}} \hat U_{h}(s_{h})}\nend
    &\quad \le  { -
    T_h^{\pi^*, \hat\nu_h^{*, t}} \hat U_h^t(s_h)
     - \Gamma^{(2,t)}_h(s_h;\pi_h^*(s_h) , \hat\theta_{h, \MLE}^t) + T_h^{\pi^*, \hat\nu_h^{*, t}} \hat U_h^t(s_h) }\nend
     &\quad = - \Gamma^{(2,t)}_h(s_h;\pi_h^*(s_h) , \hat\theta_{h, \MLE}^t),
\end{align*}
where the inequality holds by noting that $\hat W_h^t(s_h)$ is maximized over the leader's policy.
Therefore, we have \eqref{eq:OnML-subopt-2} upper bounded by for Scheme 5 that
\begin{align}
    {\dr (i)} 
    &\le \sum_{h=1}^H \EE \rbr{D_\TV\bigrbr{\hat\nu_h^{*, t}(\cdot\given s_h), \nu_h(\cdot\given s_h)} - \Gamma^{(2,t)}_h(s_h;\pi_h^*(s_h) , \hat\theta_{h, \MLE}^t)}. \label{eq:OnML-subopt-3}
\end{align}
where the last inequality holds by \Cref{lem:response diff-myopic-linear} that for both Scheme 4 and Scheme 5,
\begin{align}
    &2 H D_\TV\rbr{\nu_h^{*,t}(\cdot\given s_h), \nu_h(\cdot\given s_h)}\nend
    &\quad \le 2 H\min_{\Psi\in \SSS^d_+}\cbr{f\rbr{\sqrt{\trace\rbr{\Psi^\dagger \Sigma_{s_h}^{\pi^*, \tilde \theta}}} \nbr{\theta_h^* - \tilde \theta}_{\Psi}}, f\rbr{\sqrt{\trace\rbr{\Psi^\dagger \Sigma_{s_h}^{\pi^*, \theta_h^*}}} \nbr{\theta_h^* - \tilde \theta}_{\Psi}}}\nend
    &\quad \le 2 H f\rbr{\sqrt{\trace\rbr{\rbr{\Sigma_{h,t}^{\tilde \theta} + I_d}^\dagger \Sigma_{s_h}^{\pi^*, \tilde \theta}}} \nbr{\theta_h^* - \tilde \theta}_{\Sigma_{h,t}^{\tilde \theta} + I_d}} \nend
    &\quad \le \Gamma^{(2, t)}_h(s_h;\pi_h^*(s_h), \tilde \theta), \label{eq:OnML-TV bound-1}
\end{align}
where we define $\tilde\theta$ as the optimizer to \eqref{eq:scheme-4-new} or \eqref{eq:scheme-5-new} with $\pi^*$ plugged in.
Here, the second inequality holds by plugging $\Psi = \Sigma_{h,t}^{\hat\theta_{s_h}^{*, t}} + I_d$ and only keeping the first.
Recall that 
$\Sigma_{h,t}^{\theta_h} = \sum_{i=1}^{t-1} \Cov_{s_h^i}^{\pi^i, \theta_h} \allowbreak [\phi_h^{\pi^i}(s_h^i, b_h)]$ and  $\Gamma^{(2,t)}_h$ that $\Gamma_h^{(2,t)}(s_h;\pi_h(s_h),\theta_h)=2 H(\eta \xi  + C^{(3)} \xi^2 )
$ with 
\begin{align*}
    \xi = \sqrt{\trace\rbr{\rbr{\Sigma_{h, t}^{\theta_h} + I_d}^\dagger \Sigma_{s_h}^{\pi_h , \theta_h}}} \cdot \sqrt{8 C_{\eta}^2 \beta + 4B_\Theta^2}.
\end{align*}
Moreover, the last inequality of \eqref{eq:OnML-TV bound-1} holds by further noting that $\bignbr{\theta_h^* - \tilde\theta}_{\Sigma_{h, t}^{\tilde\theta}+I_d}^2 \le 8 C_\eta^2 \beta + 4 B_\Theta^2$ by guarantee of \eqref{eq:OnML-MLE-guarantee} and noting that $\tilde\theta\in\CI_\Theta^t(\beta)$ by definition for both Scheme 4 (explicitly in the optimization problem) and Scheme 5 (the MLE estimator is within the confidence set). 
Plugging \eqref{eq:OnML-TV bound-1} into \eqref{eq:OnML-subopt-3}, we conclude that ${\dr (i)}\le 0$
for Scheme 5 as well.

Summarizing previous results, we have by \eqref{eq:perform-diff-linear} in \Cref{lem:subopt-decomposition} that for $(\hat\pi^t, \hat W^t, \hat U^t, \hat \nu^t)$, where $\hat\nu^t$ is the optimizer to Scheme 4 or Scheme 5 under $\hat\pi^t$,
\begin{align}
    &\sum_{t=1}^T \subopt(\hat\pi^t) \nend
    &\quad\le 
    \sum_{t=1}^T {\EE\sbr{\hat W_1^t(s_1)} - J(\hat\pi^t)}\nend
    &\quad\le \sum_{t=1}^T {\sum_{h=1}^H \EE^t\sbr{\bigrbr{\hat U_h^t - u_h}(s_h, a_h, b_h)- \hat W_{h+1}^t(s_{h+1})}} 
    + \sum_{t=1}^T {\sum_{h=1}^H \EE^t\sbr{\hat W_h^t(s_h) -  T_{h}^{\hat\pi^t,\hat \nu^t} \hat U_h^t(s_{h})}}\nend
    &\qqquad+\sum_{t=1}^T {\sum_{h=1}^H  H \EE^t \nbr{\rbr{\hat \nu_h^t-\nu_h^t}(\cdot\given s_h)}_1}\nend
    &\quad\le \sum_{t=1}^T \sum_{h=1}^H 2 \EE^t \Gamma_h^{(1, t)}(s_h, a_h, b_h) 
    +  \underbrace{\sum_{t=1}^T {\sum_{h=1}^H \EE^t\sbr{\hat W_h^t(s_h) -  T_{h}^{\hat\pi^t,\hat \nu^t} \hat U_h^t(s_{h})}}}_{\dr (ii)}\nend
    &\qqquad+\sum_{t=1}^T {\sum_{h=1}^H  \Gamma_h^{(2, t)} (s_h;\hat\pi_h^t(s_h), \theta^*)} \label{eq:OnML-subopt-4}
\end{align}
where we define $\nu_h^t$ as the real quantal response under $\hat\pi^t$ and $\EE^t$ is taken with respect to policy $\hat\pi^t$ under the true model. Here, the first inequality holds by optimism, the last inequality holds by the guarantee in \eqref{eq:OnML-Bellman-error-guarantee} in the first term, and the last term comes from \eqref{eq:OnML-TV bound-1}, which says
\begin{align}
    &D_\TV\rbr{\nu_h^t(\cdot\given s_h), \hat\nu_h^t(\cdot\given s_h)}\nend
    &\quad \le \min_{\Psi\in \SSS^d_+}\cbr{f\rbr{\sqrt{\trace\rbr{\Psi^\dagger \Sigma_{s_h}^{\hat\pi^t, \hat \theta^t}}} \nbr{\theta_h^* - \hat \theta^t}_{\Psi}}, f\rbr{\sqrt{\trace\rbr{\Psi^\dagger \Sigma_{s_h}^{\hat\pi^t, \theta_h^*}}} \nbr{\theta_h^* - \hat \theta^t}_{\Psi}}}\nend
    &\quad \le f\rbr{\sqrt{\trace\rbr{\rbr{\Sigma_{h,t}^{\theta^*} + I_d}^\dagger \Sigma_{s_h}^{\hat\pi^t, \theta^*}}} \nbr{\theta_h^* - \hat \theta^t}_{\Sigma_{h,t}^{\theta^*} + I_d}} \nend
    &\quad \le \Gamma^{(2, t)}_h(s_h;\hat\pi_h^t(s_h), \theta^*), \label{eq:OnML-TV bound-2}
\end{align}
where we define $\hat\theta^t$ as the maximizer to \eqref{eq:scheme-4-new} or \eqref{eq:scheme-5-new} at state $s_h$, which also corresponds to $\hat\nu^t$ and $\hat\pi^t$ in the optimization problem. The first inequality follows from \Cref{lem:response diff-myopic-linear}, the second inequality holds by using the second term in the minimization and plug in $\Psi=\Sigma_{h, t}^{\theta^*} + I_d$, and the last inequality just follows from the definition of $\Gamma^{(2, t)}$ and using the fact that 
$\onbr{\theta_h^* - \hat \theta^t}_{\Sigma_{h,t}^{\theta^*} + I_d}^2 \le 8 C_\eta^2 \beta + 4 B_\Theta^2$ since $\hat\theta^t\in\CI_\Theta^2(\beta)$ and using the bound in \eqref{eq:OnML-MLE-guarantee}.
Furthermore, for Scheme 4, we have for term (ii) that ${\dr (ii)}=0$ by definition in \eqref{eq:scheme-4-new}. For Scheme 5, we have
\begin{align*}
    {\dr (ii)} 
    &= \sum_{t=1}^T {\sum_{h=1}^H \EE^t\sbr{\hat W_h^t(s_h) -  T_{h}^{\hat\pi^t,\hat \nu^t} \hat U_h^t(s_{h})}} \nend
    &= \sum_{t=1}^T \sum_{h=1}^H \EE^t \Gamma^{(2, t)}_h \rbr{s_h; \hat\pi_h^t(s_h), \hat\theta_{h,\MLE}} \nend
    &\le \exp\rbr{2\eta B_A}\sum_{t=1}^T \sum_{h=1}^H \EE^t \Gamma_h^{(2, t)}\rbr{s_h; \hat\pi_h^t(s_h), \theta_h^*}, 
\end{align*}
where the last inequaltiy holds by the following proposition. 

\begin{proposition}\label{eq:Gamma_2-ub}
For any $\theta,\tilde\theta\in\Theta$, $\pi\in\Pi, s_h\in\cS$, suppose $\onbr{A_h^{\pi, \theta}(s_h,\cdot)} \le B_A$ and $\onbr{\tilde A_h^{\pi, \theta}(s_h,\cdot)} \le B_A$. We then have that
\begin{align*}
    \Sigma_{s_h}^{\pi, \theta} \le \exp\rbr{2\eta B_A} \Sigma_{s_h}^{\pi, \tilde \theta}, 
\end{align*}
and also
\begin{align*}
    \Gamma_h^{(2,t)}(s_h;\pi_h(s_h),\theta_h) \le \exp\rbr{4\eta B_A} \Gamma_h^{(2,t)}(s_h;\pi_h(s_h),\tilde \theta_h).
\end{align*}
\begin{proof}
To see how we derive the upper bound, we note that by definition $\Gamma_h^{(2,t)}(s_h;\pi_h(s_h),\theta_h)=2 H(\eta \xi  + C^{(3)} \xi^2 )
$ with 
\begin{align*}
    \xi = \sqrt{\trace\rbr{\rbr{\Sigma_{h, t}^{\theta_h} + I_d}^\dagger \Sigma_{s_h}^{\pi_h , \theta_h}}} \cdot \sqrt{8 C_{\eta}^2 \beta + 4B_\Theta^2}.
\end{align*}
Furthermore, by definition of $\Sigma_{h, t}^{\theta_h}$, we have 
\begin{align*}
    \Sigma_{h,t}^{\theta_h} 
    &= \sum_{i=1}^{t-1} \Cov_{s_h^i}^{\pi^i, \theta_h} \sbr{\phi_h^{\pi^i}(s_h^i, b_h)} \nend
    &= \sum_{i=1}^{t-1}  \sum_{b_h', b_h''} \phi_h^{\pi^i}(s_h,b_h') \cdot \H_{s_h}^{\pi^i, \theta_h}(b_h', b_h'') \cdot \phi_h^{\pi^i}(s_h,b_h'')\nend
    &\le \exp\rbr{2\eta B_A } \sum_{i=1}^{t-1}  \sum_{b_h', b_h''} \phi_h^{\pi^i}(s_h,b_h') \cdot \H_{s_h}^{\pi^i, \tilde \theta_h}(b_h', b_h'') \cdot \phi_h^{\pi^i}(s_h,b_h'')\nend
    & = \exp\rbr{2\eta B_A} \Sigma_{h, t}^{\tilde \theta_h}, 
\end{align*}
where we define $\H_{s_h}^{\pi, \theta}(b_h', b_h'') = \diag(\nu_h^{\pi, \theta}(\cdot\given s_h)) - \nu_h^{\pi, \theta}(\cdot\given s_h) \nu_h^{\pi, \theta}(\cdot\given s_h)^\top$ as the Hessian matrix corresponding to $\nu_h^{\pi, \theta}$ at state $s_h$.
Here, the inequality holds by invoking \Cref{prop:Hessian-ulb} and 
a lower bound follows similarly by invoking the lower bound in \Cref{prop:Hessian-ulb}.
Also, the argument for $\Sigma_{s_h}^{\pi,\theta_h} = \Cov_{s_h}^{\pi,\theta_h}[\phi_h^\pi(s_h, b_h)]$ follows the same way. 
Therefore, we have for $\Gamma^{(2, t)}$ that 
\begin{align*}
    \Gamma_h^{(2,t)}(s_h;\pi_h(s_h),\theta_h)
    &=2 H\eta \sqrt{\trace\rbr{\rbr{\Sigma_{h, t}^{\theta_h} + I_d}^\dagger \Sigma_{s_h}^{\pi_h , \theta_h}}} \cdot \sqrt{8 C_{\eta}^2 \beta + 4B_\Theta^2}  \nend
    &\qquad + 2 H C^{(3)} \rbr{\sqrt{\trace\rbr{\rbr{\Sigma_{h, t}^{\theta_h} + I_d}^\dagger \Sigma_{s_h}^{\pi_h , \theta_h}}} \cdot \sqrt{8 C_{\eta}^2 \beta + 4B_\Theta^2}}^2\nend
    &\le 2 H\eta \exp\rbr{2\eta B_A}\sqrt{\trace\rbr{\rbr{\Sigma_{h, t}^{\tilde\theta_h} + I_d}^\dagger \Sigma_{s_h}^{\pi_h , \tilde\theta_h}}} \cdot \sqrt{8 C_{\eta}^2 \beta + 4B_\Theta^2}  \nend
    &\qquad + 2 H C^{(3)} \exp\rbr{4\eta B_A} \rbr{\sqrt{\trace\rbr{\rbr{\Sigma_{h, t}^{\tilde\theta_h} + I_d}^\dagger \Sigma_{s_h}^{\pi_h , \tilde\theta_h}}} \cdot \sqrt{8 C_{\eta}^2 \beta + 4B_\Theta^2}}^2\nend
    & \le \exp\rbr{4\eta B_A} \Gamma_h^{(2,t)}(s_h;\pi_h(s_h),\tilde\theta_h), 
\end{align*}
which completes the proof.
\end{proof}
\end{proposition}
Therefore, for Scheme 4, 
\begin{align}
    \sum_{t=1}^T \subopt(\hat\pi^t) 
    \le \sum_{t=1}^T \sum_{h=1}^H 2 \EE^t \Gamma_h^{(1, t)}(s_h, a_h, b_h) 
    +  \sum_{t=1}^T {\sum_{h=1}^H  \Gamma_h^{(2, t)} (s_h;\hat\pi_h^t(s_h), \theta^*)}, \label{eq:OnML-scheme-4-regret}
\end{align}
and for Scheme 5, 
\begin{align}
    \sum_{t=1}^T \subopt(\hat\pi^t) 
    \le \sum_{t=1}^T \sum_{h=1}^H 2 \EE^t \Gamma_h^{(1, t)}(s_h, a_h, b_h) 
    +  2\exp(4\eta B_A)\sum_{t=1}^T {\sum_{h=1}^H  \Gamma_h^{(2, t)} (s_h;\hat\pi_h^t(s_h), \theta^*)}. \label{eq:OnML-scheme-5-regret}
\end{align}

By \eqref{eq:OnML-scheme-4-regret} and \eqref{eq:OnML-scheme-5-regret}, we can see that it suffices to bound $\sum_{t=1}^T \EE^t \Gamma_h^{(1, t)}(s_h, a_h, b_h)$ and $\sum_{t=1}^T \EE^t \Gamma_h^{(2,t)}(s_h;\hat\pi_h^t(s_h),\theta_h^*)$.

\paragraph{Controlling the Suboptimality by Eluder Dimension.}
We first bound the leader's Bellman error $\sum_{t=1}^T \EE^t \Gamma_h^{(1, t)}(s_h, a_h, b_h)$.
Note that by linear approximation, 
\begin{align*}
    &\sum_{t=1}^T \EE^t \Gamma_h^{(1, t)}(s_h, a_h, b_h) \nend
    & \quad = \sum_{t=1}^T \EE^t {C_1 d H \allowbreak \sqrt{\log(2d H T^t/\delta)}\cdot \sqrt{\phi_h(s_h, a_h, b_h)^\top (\Lambda_h^t)^{-1}\phi_h(s_h, a_h, b_h)}}\nend
    &\quad \le \sum_{t=1}^T {C_1 d H \allowbreak \sqrt{\log(2d H T^2/\delta)} \cdot \sqrt{\phi_h(s_h^t, a_h^t, b_h^t)^\top (\Lambda_h^t)^{-1}\phi_h(s_h^t, a_h^t, b_h^t)}} \nend
    &\qqquad + C_1 d H \allowbreak \sqrt{\log(2d H T^2/\delta)} B_\phi \sqrt T \log\rbr{He\delta^{-1}} \nend
    &\quad \le {C_1 d H \allowbreak \sqrt{\log(2d H T^2/\delta)} \cdot B_\phi \sqrt{T d \log(T/d)}}  + C_1 d H \allowbreak \sqrt{\log(2d H T^2/\delta)} B_\phi \sqrt T \log\rbr{He\delta^{-1}},
\end{align*}
where the first inequality holds with probability at least $1-\delta$ for all $h\in[H]$ by a standard martingale concentration in \Cref{cor:martigale concentration} with $B_\phi = \nbr{\phi_h}_\infty$ and noting that $\Lambda_h^t\succeq I_d$, and the last inequality holds by definition $\Lambda_h^t = \sum_{i=1}^{t-1} \phi_h(s_h^i, a_h^i, b_h^i)\phi_h(s_h^i, a_h^i, b_h^i)^\top + I_d$, and using the elliptical potential lemma in \Cref{lem:elliptical potential}.

For the follower's response error, we just need to bound 
\begin{align*}
    &\sum_{t=1}^T \EE^t \Gamma_h^{(2,t)}(s_h;\hat\pi_h^t(s_h),\theta_h^*) \nend
    &\quad \lesssim \sum_{t=1}^T 4 H\eta \sqrt{\trace\rbr{\rbr{\Sigma_{h, t}^{\theta_h^*} + I_d}^\dagger  \Sigma_{s_h^t}^{\pi^t,\theta^*} }} \cdot \sqrt{8 C_{\eta}^2 \beta + 4B_\Theta^2}  \nend
    &\qqquad + \sum_{t=1}^T 4 H C^{(3)} \rbr{\sqrt{\trace\rbr{\rbr{\Sigma_{h, t}^{\theta_h^*} + I_d}^\dagger \Sigma_{s_h^t}^{\pi^t,\theta^*} }} \cdot \sqrt{8 C_{\eta}^2 \beta + 4B_\Theta^2}}^2 \nend
    &\qqquad + 2 H B_\phi \rbr{\eta \sqrt{8 C_\eta^2 \beta + 4 B_\Theta^2}}\sqrt T \log\rbr{2H\delta^{-1}}+ 8 H^2 B_\phi^2 \rbr{C^{(3)} \sqrt{8 C_\eta^2 \beta + 4 B_\Theta^2}}^2 \log(2H\delta^{-1}),
\end{align*}
where the inequality uses the martingale concentration for nonnegative processes, i.e., \eqref{eq:martingale-1} for the first order term in $\Gamma^{(2, t)}$ and \eqref{eq:martingale-2} in \Cref{cor:martigale concentration} for the second order term.
Next, we justify the use of the elliptical potential lemma in order to bound these terms. 
We remind the readers of the following definitions
\begin{gather*}
    \bigrbr{\Upsilon_h^{\pi,\theta} \phi_h}(s_h, b_h) = \phi_h^\pi(s_h, b_h) - \bigdotp{\phi_h^\pi(s_h, \cdot)}{\nu_h^{\pi, \theta}(\cdot\given s_h)}, \nend
    \Sigma_{s_h}^{\pi,\theta}  = \Cov_{s_h}^{\pi, \theta}\sbr{\phi_h^\pi(s_h, b_h)} = \Cov^{\pi,\theta}\bigsbr{\bigrbr{\Upsilon_h^{\pi,\theta}\phi_h}(s_h, b_h)\biggiven s_h}, \nend
    \Sigma_{h,t}^{\theta} = \sum_{i=1}^{t-1} \Cov_{s_h^i}^{\pi^i,\theta} \bigsbr{\bigrbr{\Upsilon_h^{\pi^i,\theta}\phi_h}(s_h, b_h)\biggiven s_h = s_h^i} = \sum_{i=1}^{t-1} \Sigma_{s_h^i}^{\pi^i,\theta}.
\end{gather*}
Critically, we see that $\Sigma_{h,t}^\theta$ is just a summation of $\Sigma_{s_h^i}^{\pi^i, \theta}$, which is a nonnegative definite $d$-dimensional matrix, and this gives us a self-normalized process. Therefore, we have from \Cref{lem:potential-matrix} the elliptical potential lemma for matrices that
\begin{align*}
    \sum_{t=1}^T \sqrt{\trace\rbr{\rbr{\Sigma_{h, t}^{\theta_h^*} + I_d}^\dagger  \Sigma_{s_h^t}^{\pi^t,\theta^*} }}  &\le \sqrt{C_0 d T\log\rbr{1+4 B_\phi^2 T/d}}, \nend
    \sum_{t=1}^T {\trace\rbr{\rbr{\Sigma_{h, t}^{\theta_h^*} + I_d}^\dagger  \Sigma_{s_h^t}^{\pi^t,\theta^*} }} & \le C_0 d \log\rbr{1+ 4 B_\phi^2T/d}.
\end{align*}
where $C_0=4 B_\phi^2 /(\log(1+B_\phi^2))$ and $4B_\phi^2$ upper bounds $\trace\bigrbr{\Sigma_{s_h}^{\pi,\theta^*}}$. 
To see why $B_\phi$ is a valid upper bound, we have by definition of $\Sigma_{s_h}^{\pi, \theta^*}$ that 
\begin{align*}
    \trace\rbr{\Sigma_{s_h}^{\pi, \theta^*}} &= \trace\rbr{\Cov^{\pi,\theta^*}\bigsbr{\bigrbr{\Upsilon_h^{\pi,\theta^*}\phi_h}(s_h, b_h)\biggiven s_h}}\nend
    &\le \max_{s_h, b_h, \pi}{\bignbr{\bigrbr{\Upsilon_h^{\pi,\theta^*}\phi_h}(s_h, b_h)}^2}\nend
    &\le 4B_\phi^2.
\end{align*}

We summarize the result here. For the leader's Bellman error in Scheme 4, we can bound it by 
\begin{align*}
    {\text{LBE}} &\lesssim {C_1 d H^2 \allowbreak \sqrt{\log(2d H T^2/\delta)} \cdot B_\phi \sqrt{T d \log(T/d)}}  + C_1 d H^2  \sqrt{\log(2d H T^2/\delta)} B_\phi \sqrt T \log\rbr{He\delta^{-1}}\nend
    &\lesssim d H^2 \sqrt{T d}.
\end{align*}
For the first-order term in the follower's quantal response, we can bound it by 
\begin{align*}
    \text{1st-QRE}&\lesssim 4 H^2\eta \sqrt{C_0 d T} \sqrt{8C_\eta^2 \beta + 4 B_\Theta^2} + 2 H^2 B_\phi\eta \sqrt{2C_\eta^2 \beta + 4 B_\Theta^2} \sqrt{T} \nend
    &\le \eta C_\eta H^{2} d \sqrt T, 
\end{align*}
where we use $\beta\lesssim d$. In Scheme 5, we just multiply an $\exp(4\eta B_A)$ factor. 
For the second-order term in Scheme 4, 
\begin{align*}
    \text{2nd-QRE}\lesssim H^2 C^{(3)} (C_\eta^2\beta + 4 B_\Theta^2) C_0 d^2 \log(T) \lesssim \exp(4\eta B_A) (1+\eta B_A)^3 H^{2} d^2 \log(T)
\end{align*}
In Scheme 5, we just multiply an $\exp(4\eta B_A)$ factor. 
Combining everything together, 
for Scheme 4, 
\begin{align*}
    \Reg(T)\lesssim d H^2 \sqrt{d T} + \eta C_\eta H^2 d \sqrt{T} + \exp(4\eta B_A) (1+\eta B_A)^3 H^2 d^2 \log T, 
\end{align*}
and for Scheme 5, 
\begin{align*}
    \Reg(T)\lesssim d H^2 \sqrt{d T} + \exp(4\eta B_A)\eta C_\eta H^2 d \sqrt{T} + \exp(8\eta B_A) (1+\eta B_A)^3 H^2 d^2 \log T. 
\end{align*}
We thus complete the proof of \Cref{thm:Online-ML}.


\section{Proofs for Farsighted Case}
Before we dive into the proof, we first present the following guarantee for the MLE method used in \S\Cref{sec:farsighted}.
\begin{lemma}[Guarantee of MLE]\label{lem:MLE}
    By choosing $\beta\ge  \allowbreak 9\log(3e^2H \cN_\rho(\cM,T^{-1})\delta^{-1})$, where $\cN_\rho(\cM,\epsilon)$ is the minimal size of an $\epsilon$-optimistic covering net of $\cM$. Here, an $\epsilon$-optimistic covering net $\cM_\epsilon\subset \cM$ is a finite subset such that for any $M\in\cM$, there exists $\tilde M\in\cM_\epsilon$ satisfying the following conditions:
    \begin{itemize}[leftmargin=20pt]
        \item[(i)] $D_\H\orbr{\nu_h^{\pi, M}(\cdot\given s_h), \nu_h^{\pi, \tilde M}(\cdot\given s_h)}\le \epsilon$, $D_\H\orbr{P_h^M(\cdot\given s_h, a_h, b_h), \allowbreak P_h^{\tilde M}(\cdot\given s_h, a_h, b_h)}\le \epsilon$, $\bigabr{(u_h^{M}-u_h^{\tilde M})(s_h,a_h,b_h)} \le \epsilon$, and $|(A_h^{\pi, M}-A_h^{\pi, \tilde M})(s_h, b_h)|\le \eta^{-1} \epsilon$ for all $\pi\in\Pi$, $h\in[H]$, $(s_h, a_h, b_h)\in\cS\times\cA\times\cB$;
        \item[(ii)] $P_h^{M}(s_{h+1}\given s_h, a_h, b_h)\le \exp(\epsilon)P_h^{\tilde M} (s_{h+1}\given s_h, a_h, b_h)$ for all $h\in[H]$ and $(s_{h+1}, s_h, a_h, b_h)\in\cS^2 \times \cA\times\cB$.
    \end{itemize}
The confidence set $\CI_\cM^t(\beta)$ satisfies the following with probability at least $1-\delta$: for any given $t\in[T]$, $M^*\in\confset_\cM^t(\beta)$, and it also holds for $\forall h\in[H], \forall M\in\confset^t(\beta)$ that
    \begin{align*}
        \sum_{i=1}^{t-1}D_{\RL, h}^2\rbr{M, M^*,\pi^i}
        \le 4\beta,\quad 
        \sum_{i=1}^{t-1}\hat D_{\RL, h,i}^2\rbr{M, M^*}
        \le 4\beta, 
    \end{align*}
 where 
    \begin{align*}
        D_{\RL,h}^2 (M,  M^*;\pi) &=   \EE^{\pi, M^*}D_\H^2\rbr{\nu_h^{\pi,M}, \nu_h^{\pi,M^*}} +
        \EE^{\pi, M^*} D_\H^2(P_h^{M}, P_h^{M^*}) +
        \EE^{\pi, M^*}\rbr{u_h^{M^*}-u_h^M}^2,\nend
        D_{\RL,h,i}^2(M,M^*) &= D_\H^2\orbr{\nu_h^{\pi^i, M}(\cdot\given s_h^i), \nu_h^{\pi^i, \tilde M}(\cdot\given s_h^i)} +
        D_\H^2\orbr{P_h^M(\cdot\given s_h^i, a_h^i, b_h^i), \allowbreak P_h^{\tilde M}(\cdot\given s_h^i, a_h^i, b_h^i)} \nend
        &\qquad + \bigrbr{(u_h^{M}-u_h^{\tilde M})(s_h^i,a_h^i,b_h^i)}^2,
    \end{align*}
    and $\EE^{\pi, M}$ is taken under policy $\pi$ and the model $M$.
    \begin{proof}
        See \Cref{sec:proof-MLE} for a detailed proof.
    \end{proof}
\end{lemma}
\Cref{lem:MLE} guarantees that the confidence set $\confset^t(\beta)$ is valid in the sense that $M^*\in\confset^t(\beta)$ and any $M\in\confset^t(\beta)$ has $D_\RL$ bounded by $\cO(\beta)$. 
For the optimistic covering net, we remark that constraints in the first conditions are discussed in \eqref{eq:rho-cM}. The second condition requires $P_h^M$ to be dominated above by $P_h^{\tilde M}$, which is needed to control the difference in the log-likelihood.

\subsection{Proof of \Cref{thm:PMLE} on PMLE wiith Farsighted Follower}
\label{sec:proof-PMLE}
In this subsection, we provide a formal proof to \Cref{thm:PMLE}.
The proof is carried out as the following.

\paragraph{Step 1. Offline Suboptimality Decomposition}
By \Cref{lem:MLE}, we have with high probability that $M^*\in\CI_\cM(\beta)$, and we have the following suboptimality decomposition,
\begin{align*}
    J(\pi^*) - J(\hat\pi) 
    &= J(\pi^*) - J(\pi^*, M^{\pi^*}) + J(\pi^*, M^{\pi^*}) - J(\hat\pi, M^{\hat\pi}) + J(\hat\pi, M^{\hat\pi}) - J(\hat\pi)\nend
    &\le J(\pi^*) - J(\pi^*, M^{\pi^*}) +J(\pi^*, M^{\pi^*}) - J(\hat\pi, M^{\hat\pi})\nend
    &\le J(\pi^*) - J(\pi^*, M^{\pi^*}) ,
\end{align*}
where we define $M^\pi = \argmin_{M\in\CI_\cM(\beta)}J(\pi, M)$ as the pessimistic estimated model for $\pi$. Here, the first inequality holds by noting that $J(\hat\pi) = J(\hat\pi, M^*) \ge J(\hat\pi, M^{\hat\pi})$ by the validity of the confidence set $\CI_\cM(\beta)$ and the definition of $M^{\hat\pi}$, and the  second inequality is a direct result of policy optimization. Now, we further decompose the suboptimality using \Cref{lem:subopt-decomposition}. We define $\tilde M = M^{\pi^*}$, and let $\tilde U, \tilde W$ be the follower's value functions under policy $\pi^*$ and the estimated model $\tilde M$. Let $\tilde \nu$ be the estimated quantal response under $\pi^*$ and $\tilde M$. We have for $\pi^*, \tilde U, \tilde W, \tilde\nu$ that 
\begin{align*}
    &J(\pi^* ) - J(\pi^*, \tilde M)\nend
    &\quad \le \sum_{h=1}^H \EE \sbr{\rbr{\tilde U_h - u_h}(s_h, a_h, b_h) -  \tilde W_{h+1}(s_{h+1})} + \sum_{h=1}^H 2 H  \EE \sbr{D_\TV\rbr{\tilde \nu_h(\cdot\given s_h), \nu_h(\cdot\given s_h)}}\nend
    &\quad \le \sum_{h=1}^H \underbrace{\EE \sbr{\rbr{\tilde U_h - u_h}(s_h, a_h, b_h) -  \tilde W_{h+1}(s_{h+1})}}_{\dr\text{Leader's Bellman error}} \nend
    &\qqquad +  
    C^{(0)} \cdot 
    \sum_{h=1}^H \underbrace{\EE\bigsbr{\bigabr{\tilde \Delta^{(1)}_h(s_h, b_h)}}}_{\ds\text{1st-order error}}  + C^{(2)} \cdot 
    \max_{h\in [H]} \underbrace{\EE\bigsbr{ \bigrbr{\orbr{\tilde Q_h - r_h^{\pi^*} - \gamma P_h^{\pi^*} \tilde V_{h+1}}(s_h, b_h)}^2}}_{\ds\text{2nd-order error}}
\end{align*}
where the expectation is taken under $\pi^*$ and the true model $M^*$, and we define $\tilde Q, \tilde V$ as the follower's value functions under policy $\pi^*$ and model $\tilde M$. 
Here, the first inequality holds by \eqref{eq:perform-diff-linear} where we notice that $J(\pi^*, \tilde M) = \EE[\tilde W_1(s_1)]$ and also that $\tilde W_h = T_h^{\pi^*, \tilde \nu}\tilde U_h$. The second inequality holds by using \Cref{lem:performance diff}. Moreover, the definition of $\tilde\Delta_H^{(1)}(s_h, b_h) $ is given by 
\begin{align*}
    \tilde \Delta^{(1)}_h(s_h, b_h) &=  \rbr{\EE_{s_h, b_h} -\EE_{s_h}}\Biggsbr{\sum_{l=h}^H \gamma^{l-h}\underbrace{\rbr{\tilde Q_l - r_l^{\pi^*} - \gamma P_l^{\pi^*} \tilde V_{l+1}}(s_l, b_l)}_{\ds\text{Follower's Bellman error}}}. 
\end{align*}
In the sequel, we separately bound these three terms.

\paragraph{Step 2. Bounding the Leader's Bellman Error.}
We present the gurantee we have for the Leader's Bellman error on the samples. Define $\EE^i=\EE^{\pi^i}$, which is the expectation taken under $\pi^i$ and the true model $M^*$. We have
\begin{align*}
    &\sum_{i=1}^T 
    \EE^i \sbr{\rbr{\bigrbr{\tilde U_h - u_h}(s_h, a_h, b_h) - \tilde W_{h+1}(s_{h+1})}^2}\nend
    &\quad =\sum_{i=1}^T \EE^i \sbr{\rbr{{\bigrbr{\tilde u_h - u_h}(s_h, a_h, b_h) + \bigrbr{\tilde P_h - P_h} \tilde W_{h+1}(s_{h+1})}}^2}\nend
    &\quad \le 2\sum_{i=1}^T \EE^i \sbr{\rbr{\rbr{\tilde u_h - u_h}(s_h,a_h,b_h)}^2} + 8 H^2 \sum_{i=1}^T \EE^i D_\TV^2\rbr{\tilde P_h(\cdot\given s_h, a_h, b_h), P_h(\cdot\given s_h, a_h, b_h)}\nend
    &\quad \lesssim  H^2 \beta, 
\end{align*}
where the first equality holds by the definition of $\tilde U$, and the first inequality holds by the Jensen's inequality, and the last inequality holds by the MLE guarantee in \Cref{lem:MLE}, where we hide some universal constants by \say{$\lesssim$}.
Therefore, we have that
\begin{align*}
    &\sum_{h=1}^H \EE \sbr{\rbr{\tilde U_h - u_h}(s_h, a_h, b_h) -  \tilde W_{h+1}(s_{h+1})}\nend
    &\quad \le \sum_{h=1}^H \sqrt{\EE \Bigsbr{\Bigrbr{\bigrbr{\tilde U_h - u_h}(s_h, a_h, b_h) -  \tilde W_{h+1}(s_{h+1})}^2}}\nend
    &\quad \lesssim H\sqrt{H^2 \beta} \cdot 
    \sqrt{\frac{\EE \Bigsbr{\Bigrbr{\bigrbr{\tilde U_h - u_h}(s_h, a_h, b_h) -  \tilde W_{h+1}(s_{h+1})}^2}}{\sum_{i=1}^T 
    \EE^i \Bigsbr{\rbr{\bigrbr{\tilde U_h - u_h}(s_h, a_h, b_h) - \tilde W_{h+1}(s_{h+1})}^2}}}\nend
    &\quad \le H^2\sqrt{ \beta} \cdot \max_{M\in\cM, h\in[H]} \sqrt{\frac{\EE\sbr{\rbr{\bigrbr{U_h^{\pi^*, M}-\bigrbr{u_h+P_h W_{h+1}^{\pi^*, M}}}(s_h, a_h, b_h)}^2 }}{\sum_{i=1}^T \EE^i\sbr{\rbr{\bigrbr{U_h^{\pi^*, M}-\bigrbr{u_h+P_h W_{h+1}^{\pi^*, M}}}(s_h, a_h, b_h)}^2 }}}.
\end{align*}

\paragraph{Step 3. Bound the first-order Error of the Follower's Response.}
We first study the following guarantee for the 1st order term over the offline samples,
\begin{align*}
    &\sum_{i=1}^T \EE^i \sbr{\rbr{\rbr{\EE_{s_h, b_h}^i -\EE_{s_h}^i}\sbr{\sum_{l=h}^H \gamma^{l-h}\rbr{\tilde r_l - r_l + \gamma \bigrbr{\tilde P_l - P_l} \tilde V_{l+1}}(s_l,a_l,  b_l)}}^2}\nend
    &\quad \lesssim \sum_{i=1}^T \EE^i \sbr{\rbr{\rbr{\EE_{s_h, b_h}^i -\EE_{s_h}^i}\sbr{\sum_{l=h}^H \gamma^{l-h}\rbr{\tilde r_l - r_l + \gamma \bigrbr{\tilde P_l - P_l} V_{l+1}^{\pi^i, \tilde M}}(s_l,a_l,  b_l)}}^2}\nend
    &\qqquad + \sum_{i=1}^T \EE^i \sbr{\rbr{\rbr{\EE_{s_h, b_h}^i -\EE_{s_h}^i}\sbr{\sum_{l=h}^H \gamma^{l-h}\rbr{\gamma \bigrbr{\tilde P_l - P_l} \bigrbr{\tilde V_{l+1} - V_{l+1}^{\pi^i,\tilde M}}}(s_l,a_l,  b_l)}}^2} \nend
    &\quad \lesssim \sum_{i=1}^T \EE^i \sbr{\rbr{\rbr{\EE_{s_h, b_h}^i -\EE_{s_h}^i}\Biggsbr{\sum_{l=h}^H \gamma^{l-h}\rbr{\tilde r_l - r_l + \gamma \bigrbr{\tilde P_l - P_l} V_{l+1}^{\pi^i, \tilde M} }(s_l, a_l, b_l)}}^2}\nend
    &\qqquad + B_A^2 \sum_{i=1}^T   \EE^i\rbr{\sum_{l=h}^H \gamma^{l-h}\rbr{\EE_{s_h,b_h}^i + \EE_{s_h}^i}  \sbr{D_\TV\rbr{\tilde P_l(\cdot\given s_l, a_l, b_l), P_l(\cdot\given s_l, a_l, b_l)}}}^2\nend
    &\quad \lesssim \sum_{i=1}^T \EE^i \biggsbr{\biggrbr{\underbrace{\rbr{\EE_{s_h, b_h}^i -\EE_{s_h}^i}\Biggsbr{\sum_{l=h}^H \gamma^{l-h}\rbr{\tilde r_l - r_l + \gamma \bigrbr{\tilde P_l - P_l} V_{l+1}^{\pi^i, \tilde M} }(s_l, a_l, b_l)}}_{\ds \tilde \Delta^{(1)}_{h, \pi^i, \tilde M}(s_h, b_h) }}^2}\nend
    &\qqquad + B_A^2 \sum_{i=1}^T   \eff_H(\gamma) \sum_{l=h}^H \gamma^{l-h} \EE^i D_\TV^2\rbr{\tilde P_l(\cdot\given s_l, a_l, b_l), P_l(\cdot\given s_l, a_l, b_l)}
\end{align*}
where the first inequality holds by the Jensen's inequality, where we add a $(\tilde P_l - P_l) V_{l+1}^{\pi^i, \tilde M}$ term and substract it, which gives us a separate $(\tilde P_l - P_l) (\tilde V_{l+1} - V_{l+1}^{\pi^i, \tilde M})$ term. 
In the second inquality, we upper bound $(\tilde P_l - P_l) (\tilde V_{l+1} - V_{l+1}^{\pi^i, \tilde M})$ by the TV distance between $\tilde P_l$ and $P_l$ multiplied by the infinity norm $\bignbr{\tilde V_{l+1} - V_{l+1}^{\pi^i, \tilde M}}_\infty$, which is bounded by $2 B_A$ by our argument in \Cref{sec:app-notations}. Since the TV distance is always nonnegative, we can safely flips the sign between $\EE_{s_h, b_h}^i - \EE_{s_h}^i $. The above steps give us the first inequality. 
The second inequality simply holds by using the Cauchy-Schwartz inequality where we move the square inside the expectation for the summation of the TV distance, and the $\eff_H(\gamma)$ is just a byproduct produced when applying the Cauchy-Schwartz inequality.

Next, we show how to control this $\tilde \Delta_{h,\pi^i, \tilde M}(s_h, b_h)$ term. We first notice that $\tilde\Delta^{(1)}_{h, \pi^i, \tilde M}(s_h, b_h)$ is nothing but just $\tilde\Delta^{(1)}_h(s_h, b_h)$ plugged in with $\pi^i$ as the policy $\pi$ and $\tilde Q^{\pi^i}=Q^{\pi^i, \tilde M}, \tilde V^{\pi^i} = V^{\pi^i, \tilde M}$ as the follower's value functions $\tilde Q, \tilde V$. 
Hence, we can invoke \Cref{lem:1st-ub} which says that 
\begin{align*}
    &\bigrbr{\tilde \Delta_{h, \pi^i,\tilde M}^{(1)}(s_h, b_h)}^2  \nend
        &\quad \le 2 \rbr{\rbr{\EE_{s_h, b_h}^i-\EE_{s_h}^i} \bigsbr{\orbr{Q_h^{\pi^i} - \tilde Q_h^{\pi^i}}(s_h, b_h)}}^2 \nend
        &\qqquad + 16 \gamma^2  \rbr{\eta^{-1} +2 B_A}^2\eff_H(\gamma) \sum_{l=h+1}^H \gamma^{l-h-1} {\rbr{\EE_{s_h}^i+\EE_{s_h, b_h}^i}\sbr{D_\H^2(\nu_l^{\pi^i}(\cdot\given s_l), \tilde\nu_l^{\pi^i}(\cdot\given s_l))}}
\end{align*}
where we define $\tilde \nu_l^{\pi^i} = \nu_l^{\pi^i, \tilde M}$ as the quantal response under $\pi^i$ and the estimated model $\tilde M$. This is true by our definition of $\tilde Q^{\pi^i}, \tilde V^{\pi^i}$ that they are the follower's value functions under $\pi^i$ and model $\tilde M$.
Therefore, we have that 
\begin{align*}
    &\sum_{i=1}^T \EE^i\bigrbr{\tilde \Delta_{h, \pi^i,\tilde M}^{(1)}(s_h, b_h)}^2  \nend
    &\quad \lesssim \sum_{i=1}^T \EE^i\rbr{\rbr{\EE_{s_h, b_h}^i-\EE_{s_h}^i} \bigsbr{\orbr{Q_h^{\pi^i} - \tilde Q_h^{\pi^i}}(s_h, b_h)}}^2 \nend
    &\qqquad + \gamma^2  \rbr{\eta^{-1} +2 B_A}^2\eff_H(\gamma) \sum_{i=1}^T \sum_{l=h+1}^H \gamma^{l-h-1} {\EE^i\sbr{D_\H^2(\nu_l^{\pi^i}(\cdot\given s_l), \tilde\nu_l^{\pi^i}(\cdot\given s_l))}} \nend
    &\quad\lesssim \rbr{(\eta^{-2}+B_A^2) + \gamma^2  \rbr{\eta^{-1} +2 B_A}^2\eff_H(\gamma) H} \beta, 
\end{align*}
where the last inequality holds by both \eqref{eq:MLE-guarantee-Q-3} in \Cref{lem:MLE-formal} for the Q difference term and the MLE guarantee in \Cref{lem:MLE} for the Hellinger term. Here, we upper bound $\gamma^{l-h-1}$ by $1$ and take a summation over $h\in[H]$. Therefore, we conclude that 
\begin{align*}
    &
    \sum_{i=1}^T \EE^i \sbr{\rbr{\rbr{\EE_{s_h, b_h}^i -\EE_{s_h}^i}\sbr{\sum_{l=h}^H \gamma^{l-h}\rbr{\tilde r_l - r_l + \gamma \bigrbr{\tilde P_l - P_l} \tilde V_{l+1}}(s_l,a_l,  b_l)}}^2} \nend
    &\quad\lesssim  
    \sum_{i=1}^T \EE^i\bigrbr{\tilde \Delta_{h, \pi^i,\tilde M}^{(1)}(s_h, b_h)}^2 \nend
    &\qqquad + 
    B_A^2 \sum_{i=1}^T   \eff_H(\gamma) \sum_{l=h}^H \gamma^{l-h} \EE^i D_\TV^2\rbr{\tilde P_l(\cdot\given s_l, a_l, b_l), P_l(\cdot\given s_l, a_l, b_l)}\nend
    &\quad \le  \rbr{(\eta^{-2}+B_A^2) + \gamma^2  \rbr{\eta^{-1} +2 B_A}^2\eff_H(\gamma) H} \beta  + B_A^2 \eff_H(\gamma) H \beta \nend
    &\quad \lesssim \rbr{\eta^{-2} + B_A^2 }\eff_H(\gamma)H \beta. 
\end{align*}
As a result, we have that
\begin{align*}
    &C^{(0)}\sum_{h=1}^H \EE\sbr{\abr{\tilde \Delta_h^{(1)}(s_h, b_h)}}\nend
    &\quad\le C^{(0)} \sum_{h=1}^H \sqrt{\EE\sbr{\tilde \Delta_h^{(1)}(s_h, b_h)^2}}\nend
    &\quad \lesssim C^{(0)} H \sqrt{ \rbr{\eta^{-2} + B_A^2 }\eff_H(\gamma)H \beta} \nend
    &\qqquad \cdot \max_{h\in[H]}\sqrt\frac{\EE \sbr{\rbr{\rbr{\EE_{s_h, b_h} -\EE_{s_h}}\sbr{\sum_{l=h}^H \gamma^{l-h}\rbr{\tilde r_l - r_l + \gamma \bigrbr{\tilde P_l - P_l} \tilde V_{l+1}}(s_l,a_l,  b_l)}}^2}}{\sum_{i=1}^T \EE^i \sbr{\rbr{\rbr{\EE_{s_h, b_h}^i -\EE_{s_h}^i}\sbr{\sum_{l=h}^H \gamma^{l-h}\rbr{\tilde r_l - r_l + \gamma \bigrbr{\tilde P_l - P_l} \tilde V_{l+1}}(s_l,a_l,  b_l)}}^2}}\nend
    &\quad \lesssim \rbr{1 + \eta B_A }H^2\sqrt{ H \eff_H(\gamma)\beta} \nend
    &\quad \cdot \max_{M\in\cM, h\in[H]}\sqrt\frac{\EE \sbr{\rbr{\rbr{\EE_{s_h, b_h} -\EE_{s_h}}\sbr{\sum_{l=h}^H \gamma^{l-h}\rbr{r_l^M - r_l + \gamma \bigrbr{P_l^M - P_l} V_{l+1}^{\pi^*, M}}(s_l,a_l,  b_l)}}^2}}{\sum_{i=1}^T \EE^i \sbr{\rbr{\rbr{\EE_{s_h, b_h}^i -\EE_{s_h}^i}\sbr{\sum_{l=h}^H \gamma^{l-h}\rbr{ r_l^M - r_l + \gamma \bigrbr{P_l^M - P_l} V_{l+1}^{\pi^*, M}}(s_l,a_l,  b_l)}}^2}}, 
\end{align*}
where we notice that $C^{(0)}=2\eta H$.

\paragraph{Step 4. Bound the second-order Error in the Follower's Response.}
The last thing to do is controlling the second order term. We first expand the second order term in terms of $r_h, P_h, V_{h+1}$ by definitions and have the following guarantee for the second order term over the samples, 
\begin{align*}
    &\sum_{i=1}^T \EE^i\sbr{ \rbr {\EE_{s_h, b_h}^i\sbr{{\bigrbr{\tilde r_h - r_h + \gamma \bigrbr{\tilde P_h - P_h} \tilde V_{h+1}}(s_h,a_h, b_h)}}}^2}\nend
    &\quad \lesssim \sum_{i=1}^T \EE^i\sbr{ \rbr {\EE_{s_h, b_h}^i\sbr{{\bigrbr{\tilde r_h - r_h + \gamma \bigrbr{\tilde P_h - P_h} V_{h+1}^{\pi^i, \tilde M}}(s_h,a_h, b_h)}}}^2} \nend
    &\qqquad + \sum_{i=1}^T \EE^i\sbr{ \rbr {\EE_{s_h, b_h}^i\sbr{{\gamma \bigrbr{\tilde P_h - P_h} \bigrbr{\tilde V_{h+1} - V_{h+1}^{\pi^i, \tilde M} }(s_h,a_h, b_h)}}}^2}\nend
    &\quad \lesssim \sum_{i=1}^T \EE^i\sbr{ \rbr {\EE_{s_h, b_h}^i\sbr{{\bigrbr{\tilde r_h - r_h + \gamma \bigrbr{\tilde P_h - P_h} V_{h+1}^{\pi^i, \tilde M}}(s_h,a_h, b_h)}}}^2} \nend
    &\qqquad + \gamma^2B_A^2 \sum_{i=1}^T \EE^i\sbr{ D_\TV^2\rbr{\tilde P_h(\cdot\given s_h, a_h, b_h), P_h(\cdot\given s_h, a_h, b_h)}}, 
\end{align*}
where the first inequality holds by the Jensen's inequality, the second inequality holds by upper bounding the difference in $(\tilde P-P)(\tilde V-V^{\pi^i, \tilde M})$ by the TV distance and the upper bound for the follower's V function as $B_A$.
We now invoke \Cref{lem:2nd-ub} for $(\pi^i, Q_h^{\pi^i,\tilde M}, V_{h+1}^{\pi^i, \tilde M})$, which gives us 
\begin{align}
    &\sum_{i=1}^T \max_{h\in[H]}\EE^i\sbr{ \rbr {\EE_{s_h, b_h}^i\sbr{{\bigrbr{\tilde r_h - r_h + \gamma \bigrbr{\tilde P_h - P_h} V_{h+1}^{\pi^i, \tilde M}}(s_h,a_h, b_h)}}}^2} \nend
    &\quad =\sum_{i=1}^T\max_{h\in[H]}\EE^i\sbr{ \rbr{\rbr{Q_h^{\pi^i,\tilde M} - r_h^{\pi^i} - \gamma P_h^{\pi^i}  V_{h+1}^{\pi^i, \tilde M}}(s_h, b_h)}^2} \nend
    &\quad \le \sum_{i=1}^T L^{(2)} \sum_{h=1}^H \cbr{\EE D_\H^2(\nu_h^{\pi^i}(\cdot\given s_h),\nu_h^{\pi^i, \tilde M}(\cdot\given s_h))+\EE D_\TV^2(P_h^{\pi^i}(\cdot\given s_h, b_h),P_h^{\pi^i, \tilde M}(\cdot\given s_h, b_h))}\nend
    &\quad \lesssim L^{(2)} H \beta .\label{eq:PMLE-1}
\end{align}
where in the first inequality, we additionally replace the maximum over $h\in[H]$ as a summation over $h\in[H]$, and the second inequlity holds by the MLE guarantee in \Cref{lem:MLE}. Hence, we conclude that 
\begin{align*}
    &\sum_{i=1}^T \EE^i\sbr{ \rbr {\EE_{s_h, b_h}^i\sbr{{\bigrbr{\tilde r_h - r_h + \gamma \bigrbr{\tilde P_h - P_h} \tilde V_{h+1}}(s_h,a_h, b_h)}}}^2}\nend
    &\quad \lesssim \sum_{i=1}^T \EE^i\sbr{ \rbr {\EE_{s_h, b_h}^i\sbr{{\bigrbr{\tilde r_h - r_h + \gamma \bigrbr{\tilde P_h - P_h} V_{h+1}^{\pi^i, \tilde M}}(s_h,a_h, b_h)}}}^2} \nend
    &\qqquad + \gamma^2B_A^2 \sum_{i=1}^T \EE^i\sbr{ D_\TV^2\rbr{\tilde P_h(\cdot\given s_h, a_h, b_h), P_h(\cdot\given s_h, a_h, b_h)}}\nend
    &\quad \lesssim L^{(2)} H \beta + \gamma^2 B_A^2 \beta, 
\end{align*}
where the last inequality holds by using \eqref{eq:PMLE-1} for the first term and the MLE guarantee in \Cref{lem:MLE} for the second term.
Now, we invoke \Cref{prop:de-regret-prop} and obtain 
\begin{align*}
    &C^{(2)} \cdot 
    \max_{h\in[H]}\EE\bigsbr{ \bigrbr{\orbr{\tilde Q_h - r_h^{\pi^*} - \gamma P_h^{\pi^*} \tilde V_{h+1}}(s_h, b_h)}^2} \nend
    &\quad = C^{(2)} \cdot \max_{h\in[H]}
    \EE\sbr{ \rbr{\EE_{s_h, b_h}\sbr{\bigrbr{\tilde r_h - r_h + \gamma \bigrbr{\tilde P_h - P_h} \tilde V_{h+1}}(s_h, a_h, b_h)}}^2} \nend
    &\quad \lesssim C^{(2)} \bigrbr{L^{(2)} H \beta + \gamma^2 B_A^2 \beta} \nend
    &\qqquad\cdot \max_{h\in[H]}\frac{\EE\sbr{ \rbr{\EE_{s_h, b_h}\sbr{\bigrbr{\tilde r_h - r_h + \gamma \bigrbr{\tilde P_h - P_h} \tilde V_{h+1}}(s_h, a_h, b_h)}}^2}}{\sum_{i=1}^T \EE^i\sbr{ \rbr {\EE_{s_h, b_h}^i\sbr{{\bigrbr{\tilde r_h - r_h + \gamma \bigrbr{\tilde P_h - P_h} \tilde V_{h+1}}(s_h,a_h, b_h)}}}^2}}\nend
    &\quad \lesssim C^{(2)} \bigrbr{L^{(2)} H \beta + \gamma^2 B_A^2 \beta} \nend
    &\qqquad\cdot \max_{h\in[H], M\in\cM}\frac{\EE\sbr{ \rbr{\EE_{s_h, b_h}\sbr{\bigrbr{r_h^{M} - r_h + \gamma \bigrbr{P_h^M - P_h}  V_{h+1}^{\pi^*, M}}(s_h, a_h, b_h)}}^2}}{\sum_{i=1}^T \EE^i\sbr{ \rbr {\EE_{s_h, b_h}^i\sbr{{\bigrbr{r_h^{M} - r_h + \gamma \bigrbr{P_h^M - P_h}  V_{h+1}^{\pi^*, M}}(s_h,a_h, b_h)}}}^2}}, 
\end{align*}
where the first inequality is just a distribution, and the last inequality takes a maximum over $\cM$. 

In summary, for the leader's Bellman error, 
\begin{align*}
    \text{LBE} \lesssim H^2 \sqrt{\beta} \cdot \max_{M\in\cM, h\in[H]} \sqrt{\frac{\EE\sbr{\rbr{\bigrbr{U_h^{\pi^*, M}-\bigrbr{u_h+P_h W_{h+1}^{\pi^*, M}}}(s_h, a_h, b_h)}^2 }}{\sum_{i=1}^T \EE^i\sbr{\rbr{\bigrbr{U_h^{\pi^*, M}-\bigrbr{u_h+P_h W_{h+1}^{\pi^*, M}}}(s_h, a_h, b_h)}^2 }}},
\end{align*}
for the first order term in the follower's quantal response error, 
\begin{align*}
    &\text{1st-QRE} \nend
    &\quad \lesssim  \eta C_\eta H^2\sqrt{ H \eff_H(\gamma)\beta} \nend
    &\qquad \cdot \max_{M\in\cM, h\in[H]}\sqrt\frac{\EE \sbr{\rbr{\rbr{\EE_{s_h, b_h} -\EE_{s_h}}\sbr{\sum_{l=h}^H \gamma^{l-h}\rbr{r_l^M - r_l + \gamma \bigrbr{P_l^M - P_l} V_{l+1}^{\pi^*, M}}(s_l,a_l,  b_l)}}^2}}{\sum_{i=1}^T \EE^i \sbr{\rbr{\rbr{\EE_{s_h, b_h}^i -\EE_{s_h}^i}\sbr{\sum_{l=h}^H \gamma^{l-h}\rbr{ r_l^M - r_l + \gamma \bigrbr{P_l^M - P_l} V_{l+1}^{\pi^*, M}}(s_l,a_l,  b_l)}}^2}}, 
\end{align*}
and for the second order term in the follower's quantal response error, 
\begin{align*}
    &\text{2nd-QRE}\nend
    &\quad \lesssim C^{(2)} L^{(2)} H \beta \nend
    &\qqquad\cdot \max_{h\in[H], M\in\cM}\frac{\EE\sbr{ \rbr{\EE_{s_h, b_h}\sbr{\bigrbr{r_h^{M} - r_h + \gamma \bigrbr{P_h^M - P_h}  V_{h+1}^{\pi^*, M}}(s_h, a_h, b_h)}}^2}}{\sum_{i=1}^T \EE^i\sbr{ \rbr {\EE_{s_h, b_h}^i\sbr{{\bigrbr{r_h^{M} - r_h + \gamma \bigrbr{P_h^M - P_h}  V_{h+1}^{\pi^*, M}}(s_h,a_h, b_h)}}}^2}}. 
\end{align*}
Hence, we complete the proof of \Cref{thm:PMLE}

\subsection{Proof of \Cref{thm:OMLE-farsighted} on OMLE with Farsighted Follower}\label{sec:proof-farsighted MDP}
In this section, we provide formal proofs for the result to Theorem \ref{thm:OMLE-farsighted}.
The proof relies on a novel decomposition of the online regret in the Taylor-series form and utilizes the techniques for analyzing the OMLE in \citep{chen2022unified, foster2021statistical,jin2021bellman}. Before diving into the proof, we present the following key lemma that provides guarantees for the confidence set $\confset_\cM^t(\beta)$ given by the algorithm.

The proof is carried out as the following. We recall that $\beta\ge  \allowbreak 9\log(3e^2TH \cN_\rho(\cM,T^{-1})\delta^{-1})$, where we additionally include an $\log T$ term to ensure a union bound over $t\in[T]$.

\paragraph{Step 1. Online Regret Decompostion. }
By \Cref{lem:MLE}, we have with high probability that $M^*\in\confset_\cM^t(\beta)$. Hence, we can upper bound the online regret by
\begin{align*}
    \Reg(T) = \sum_{t=1}^T J(\pi^*, M^*) - J(\pi^t, M^*) \le \sum_{t=1}^T J(\pi^t, M^t) - J(\pi^t, M^*), 
\end{align*}
where the inequality holds by additionally noting that OMLE produces the pair $(\pi^t, M^t)$ that maximizes $J$ within the confidence set $\confset^t(\beta)$.
The key in studying the regret in this MDP with strategic follower is decomposing the performance difference into both the follower's temporal difference and the follower's response difference. 
We invoke \Cref{lem:subopt-decomposition} with $\pi^t, \tilde U^t, \tilde W^t, \tilde \nu^t$, where $\tilde U^t, \tilde W^t, \tilde \nu^t$ are given under policy $\pi^t$ and the estimated model $\tilde M^t$, and they should satisfy $\tilde W_h^t(s_h)=(T_h^{\pi^t, \tilde \nu^t}\tilde U_h^t) (s_h)$. We additionally define $\nu^t = \nu^{\pi^t, M^*}$. We have that
\begin{align*}
    &J(\pi^t, \tilde M^t) - J(\pi^t, M^*)\nend
    &\quad \le \sum_{h=1}^H \EE^t \sbr{\rbr{\tilde U_h^t - u_h}(s_h, a_h, b_h) -  \tilde W_{h+1}^t(s_{h+1})} + \sum_{h=1}^H 2 H  \EE^t \sbr{D_\TV\rbr{\tilde \nu_h^t(\cdot\given s_h), \nu_h^t(\cdot\given s_h)}}\nend
    &\quad \le \sum_{h=1}^H \underbrace{\EE^t \sbr{\rbr{\tilde U_h^t - u_h}(s_h, a_h, b_h) - \tilde W_{h+1}^t(s_{h+1})}}_{\dr \text{Leader's Bellman error}} \nend
    &\qqquad +   C^{(0)}
    \sum_{h=1}^H \underbrace{\EE^t\sbr{\abr{\tilde \Delta^{(1,t)}_h(s_h, b_h)}}}_{\ds\text{1st-order error}}  + C^{(2)}
    \max_{h\in [H]} \underbrace{\EE^t\sbr{ \rbr{\rbr{\tilde Q_h^t - r_h^{\pi^t} - \gamma P_h^{\pi^t} \tilde V_{h+1}^t}(s_h, b_h)}^2}}_{\ds\text{2nd-order error}}
\end{align*}
where the expectation $\EE^t$ is taken with respect to $\pi^t$ and the true model $M^*$, 
$C^{(0)}=2\eta H$, 
$C^{(2)} = 2 H\eta^2 H (1+4 \eff_H(\gamma))\exp\rbr{6\eta B_A}\cdot \rbr{\eff_H(\exp(2\eta B_A)\gamma)}^2$ with $\eff_H(x) = (1-x^H)/(1-x)$ as the \say{effective}  horizon with respect to $x$, 
and $\tilde\Delta_j^{(1, t)}(s_h, b_h)$ is defined as
\begin{align*}
    \tilde \Delta^{(1, t)}_h(s_h, b_h) &=  \rbr{\EE_{s_h, b_h}^t -\EE_{s_h}^t}\Biggsbr{\sum_{l=h}^H \gamma^{l-h}\underbrace{\rbr{\tilde Q_l^t - r_l^{\pi^t} - \gamma P_l^{\pi^t} \tilde V_{l+1}^t}(s_l, b_l)}_{\ds\text{Follower's Bellman error}}}.
\end{align*}
Here, the second inequality comes from \Cref{lem:performance diff} and uses the definition that $\tilde V_h^t = \eta^{-1}\log \int \exp(\eta \tilde Q_h^t)$, $\tilde A_h^t = \tilde Q_h^t -\tilde V_h^t$,  and that $\tilde\nu_h^t=\exp\orbr{\eta \tilde A_h^t}$ under the alternative model $\tilde M^t$. In the sequel, we will bound these three terms separately.

\paragraph{Step 2. Bounding the Leader's Bellman Error.}
We first show that the leader's Bellman error is controllable when summed up for $T$ steps. Specifically, consider the following configurations for step $h\in[H]$,
\begin{itemize}[leftmargin=20pt]
    \item[(i)] Define function class $\cG_{h,L}$ as
    \begin{align*}
        \cG_L^h &= \Big\{g:\cS\times\cA\times\cB\rightarrow \RR \Biggiven g={\bigrbr{U_h^{\pi^{\tilde M},\tilde M} - u - P_h W_{h+1}^{\pi^{\tilde M},\tilde M}}(s_h, a_h, b_h)}, \exists \tilde M\in\cM\Big\}, 
    \end{align*}
    where we define $\pi^M = \argmax_{\pi\in\Pi}J(\pi, M)$.
    Specifically, the expectation is taken under $\pi$ and the true model. Consider a sequence of function $\{g_h^i = (\tilde U_h^{i} - u - P_h \tilde W_{h+1}^i)\}_{i\in[T]}$. We see directly that $g_h^i\in\cG_{h, L}$ since $\tilde U^i = U^{\pi^i, M^i}$ and we have by the optimism in the algorithm that $\pi^i = \pi^{M^i}$. The same also holds for $\tilde W^i$.
    \item[(ii)] Define a class of probability measures over $\cS\times\cA\times\cB$ as $$\sP_{h, L}=\{\PP^\pi((s_h, a_h, b_h)=\cdot), \forall \pi\in\Pi\}.$$
    Consider a sequence of probability measures $\{\rho_h^i(\cdot)=\PP^{\pi^i}((s_h, a_h, b_h)=\cdot)\}_{i\in[T]}$.
    \item[(iii)] Under this two sequences, we denote by $g_h^t(\pi^i) = \EE_{\rho_h^i}[g_h^t]$ for simplicity. 
    We have 
    $$g_h^t(\pi^i) =\EE^i\bigsbr{{\bigrbr{\tilde U_h^t - u - P_h \tilde W_{h+1}^t}(s_h,a_h,b_h)}}, $$
    which should be bounded by $3H$.
\end{itemize}
We denote by $\dim(\cG_L) = \max_{h\in[H]}\dim_\DE(\cG_{h, L}, \sP_{h, L}, T^{-1/2})$ in the sequel.
Our guarantee for the sequence $\{g_h^i\}_{i\in[T]}$ and $\{\pi^i\}_{i\in[T]}$ is 
\begin{align}\label{eq:OnN-guarantee-Bellmanerror}
    \sum_{i=1}^{t-1} \rbr{g_h^t(\pi^i)}^2 
    &= \sum_{i=1}^t \rbr{\EE^i\bigsbr{{\bigrbr{\tilde U_h^t - u - P_h \tilde W_{h+1}^t}(s_h,a_h,b_h)}}}^2 \nend
    &\le \sum_{i=1}^t \EE^i\sbr{\rbr{\bigrbr{\tilde U_h^t - u_h - P_h \tilde W_{h+1}^t}(s_h,a_h,b_h)}^2}\nend
    &\le \sum_{i=1}^t 2 \EE^i\sbr{\rbr{\rbr{\tilde u_h^t - u_h}(s_h, a_h, b_h)}^2} + 8 H^2 \EE^i D_\TV^2\rbr{P_h(\cdot\given s_h, a_h, b_h), \tilde P_h^t(\cdot\given s_h, a_h, b_h)}\nend
    &\le 8 H^2 \cdot 4 \beta,
\end{align}
where we define $\tilde u_h^t = u_h^{M^t}$ and $\tilde P_h^t = P_h^{M^t}$. The first ineqality holds by the Cauchy-Schwartz inequality ,  the second inequality holds by noting that $\tilde U_h^t = \tilde u_h^t + \tilde P_h^t \tilde W_h^t$, and the last inequality holds by invoking the guarantee in \Cref{lem:MLE}. We then have by the first order argument in \Cref{lem:de-regret} that 
\begin{align*}
    \sum_{i=1}^T \abr{g_h^t(\pi^t)} \le 2\sqrt{\dim\rbr{\cG_L} 32 H^2 \beta T} + 3 H\min\cbr{T, \dim\rbr{\cG_L}}  + \sqrt T, 
\end{align*}
which implies that the leader's Bellman error is upper bounded by $\cO( H^2 \sqrt{\dim\rbr{\cG_L} H \beta T})$.

\paragraph{Step 3. Bounding the first-order Term.}
we next show that the first-order term in the follower's response error is also under control. 
\begin{itemize}[leftmargin=20pt]
    \item[(i)] Define function class $\cG_{h, F}^1$ as 
    \begin{align*}
        \cG_{h, F}^1 &= \Bigg\{g:(\cS\times\cA\times\cB)^{H-h+1}\rightarrow \RR \bigggiven \exists M\in\cM, \nend
        &\qqquad g((s_l, a_l, b_l)_{l=h}^H) = {\sum_{l=h}^H \gamma^{l-h}\bigrbr{r_l^M- r_l + \gamma (P_l^{M} - P_l) V_{l+1}^{\pi^M, M}}(s_l, a_l, b_l)}
        \Bigg\},
    \end{align*}
    where we remind the readers that $\pi^M = \argmax_{\pi\in\Pi}J(\pi, M)$ only depends on $M$. Consider 
    sequences 
    $$\cbr{g_h^t=\sum_{l=h}^H \gamma^{l-h}\orbr{\tilde r_l^t- r_l + \gamma (\tilde P_l^t - P_l) \tilde V_{l+1}^t}}_{t\in[T]}, $$
    where we define $\tilde r_l^t = r_l^{M^t}$ and $\tilde P_l^t = P_l^{M^t}$.
    It is obvious that $g_h^t\in\cG_{h, F}^1$ since $\tilde V_h^t = V_h^{\pi^t, M^t}$ and $\pi^t = \argmax_{\pi\in\Pi}J(\pi, M^t) = \pi^{M^t}$.
    
    \item Define a class of signed measures over $(\cS\times\cA\times\cB)^{H-h+1}$ as 
    \begin{equation*}
        \sP_{h, F}^1 = \cbr{\begin{aligned}
            &\PP^\pi(((s_l, a_l, b_l)_{l=h+1}^H , a_h)=\cdot \given s_h, b_h)\delta_{(s_h, b_h)}(\cdot) \nend
            &\quad - \PP^\pi(((s_l, a_l, b_l)_{l=h+1}^H , a_h, b_h)=\cdot \given s_h)\delta_{(s_h)}(\cdot) 
        \end{aligned}
        \bigggiven \pi\in\Pi, (s_h, b_h)\in\cS\times\cB}, 
    \end{equation*}
    where $\delta_{s_h, b_h}$ is the measure that puts measure $1$ on a single state-action pair $(s_h, b_h)$, and the conditional is well defined by the Markov property. Also, consider the following sequence, 
    \begin{align*}
        \Big\{\rho_h^t(\cdot)&=\PP^{\pi^t}(((s_l, a_l, b_l)_{l=h+1}^H , a_h)=\cdot \given s_h^t, b_h^t)\delta_{(s_h^t, b_h^t)}(\cdot) \nend
        &\qquad - \PP^{\pi^t}(((s_l, a_l, b_l)_{l=h+1}^H , a_h, b_h)=\cdot \given s_h^t)\delta_{(s_h^t)}(\cdot)\Big\}_{t\in[T]},
    \end{align*}
    and we also have $\rho_h^t\in\sP_{h, F}^1$.
    \item  
    In particular, we define $g_h^t(s_h^i, b_h^i, \pi^i)$ as the integral of $g_h^t$ with respect to $\rho_h^i$, which is given by
    $$g_h^t(s_h^i, b_h^i, \pi^i)= \rbr{\EE_{s_h^i, b_h^i}^{\pi^i}-\EE_{s_h^i}^{\pi^i}} \sbr{\sum_{l=h}^H \gamma^{l-h}\Bigrbr{\tilde r_l^t- r_l + \gamma (\tilde P_l^t - P_l) \tilde V_{l+1}^t}(s_l, a_l, b_l)},$$
    Note that the sequence of signed measures is uniquely determined by $\{(s_h^t, b_h^t, \pi^t)\}_{t\in[T]}$. Moreover, we have $g_h^t(s_h^i, b_h^i, \pi^i)$ bounded by $\eff_H(\gamma)(2\nbr{r_h}_\infty + 2\nbr{V_{h+1}}_\infty) \le 4B_A \eff_H(\gamma)$, where we the definition of $B_A$ is available in \eqref{eq:define_BA};  
\end{itemize}
We define the maximal eluder dimension of $\cG_{h, F}^1$ with respect to $\sP_{h,F}^1$ as $$\dim(\cG_F^1) =\max_{h\in[H]} \dim_\DE(\cG_{h, F}^1,\sP_{h, F}^1, T^{-1/2}).$$
We first see what guarantee we have on the given sequences $\{g_h^t\}_{t\in[T]}$ and $\{(s_h^t, b_h^t, \pi^t)\}_{t\in[T]}$, 
\begin{align}
    &\sum_{i=1}^{t-1} \rbr{g_h^t(s_h^i, b_h^i, \pi^i)}^2 \nend
    & \quad =\sum_{i=1}^{t-1} \rbr{\rbr{\EE_{s_h^i, b_h^i}^{i}-\EE_{s_h^i}^{i}} \sbr{\sum_{l=h}^H \gamma^{l-h}\Bigrbr{\tilde r_l^t- r_l + \gamma (\tilde P_l^t - P_l) \tilde V_{l+1}^t}(s_l, a_l, b_l)}}^2 \nend
    & \quad \le 2\sum_{i=1}^{t-1} \rbr{\rbr{\EE_{s_h^i, b_h^i}^{i}-\EE_{s_h^i}^{i}} \sbr{\sum_{l=h}^H \gamma^{l-h}\Bigrbr{\tilde r_l^t- r_l + \gamma (\tilde P_l^t - P_l) V_{l+1}^{\pi^i, M^t}}(s_l, a_l, b_l)}}^2\nend
    &\qqquad + 2\sum_{i=1}^{t-1} \rbr{\rbr{\EE_{s_h^i, b_h^i}^{i}-\EE_{s_h^i}^{i}}\sbr{\sum_{l=h}^H \gamma^{l-h+1} (\tilde P_l^t - P_l) \bigrbr{V_{l+1}^{\pi^i, M^t} - \tilde V_{l+1}^t}(s_l, a_l,b_l)}}^2 \nend
    & \quad \le 2\sum_{i=1}^{t-1} \Biggrbr{\underbrace{
        \rbr{\EE_{s_h^i, b_h^i}^{i}-\EE_{s_h^i}^{i}} \sbr{\sum_{l=h}^H \gamma^{l-h}\Bigrbr{\tilde r_l^t- r_l + \gamma (\tilde P_l^t - P_l) V_{l+1}^{\pi^i, M^t}}(s_l, a_l, b_l)}
    }_{\ds\tilde\Delta^{(1)}_{h, \pi^i, M^t}(s_h, b_h)}}^2\nend
    &\qquad + \underbrace{C B_A^2 \eff_{H}(\gamma)\sum_{i=1}^{t-1} \rbr{\EE_{s_h^i, b_h^i}^i + \EE_{s_h^i}^i} \sbr{\sum_{l=h}^H \gamma^{l-h+1} D_\TV^2 \rbr{\tilde P_l^t(\cdot\given s_l, a_l,b_l), P_l(\cdot\given s_l, a_l,b_l)}}}_{\dr (i)}, \label{eq:OnN-1st-g-sq}
\end{align}
where the first inequaltiy holds by the Jensen's inequality, and the second inequality holds by upper bounding the difference $\tilde P_l^t - P_l$ by the TV distance. Here, we are able to use $B_A$ as the upper bound for the V functions by our discussion in \Cref{sec:app-notations}, and $C$ hides some universal constant.
Applying \Cref{lem:1st-ub} with $\tilde \Delta^{(1)}_h(s_h, b_h)$ replaced by
\begin{align*}
    \tilde \Delta^{(1)}_{h,\pi^i, M^t}(s_h, b_h) 
    &=  \rbr{\EE_{s_h, b_h}^{i} -\EE_{s_h}^{i}}\Biggsbr{\sum_{l=h}^H \gamma^{l-h} 
    {\rbr{Q_l^{\pi^i, M^t} - r_l^{\pi^i} - \gamma P_l^{\pi^i} V_{l+1}^{\pi^i, M^t}}(s_l, b_l)}}\nend
    & = \rbr{\EE_{s_h, b_h}^{i} -\EE_{s_h}^{i}}\Biggsbr{\sum_{l=h}^H \gamma^{l-h} 
    {\rbr{\tilde r_l^t - r_l + \gamma (\tilde P_l^t - P_l) V_{l+1}^{\pi^i, M^t}}(s_l, a_l, b_l)}},
\end{align*}
we obtain for all $t\in[T], h\in[H]$,
\begin{align}
    &\sum_{i=1}^{t-1}\bigrbr{\tilde \Delta_{h, \pi^i, M^t}^{(1)}(s_h^i, b_h^i)}^2  \nend
    &\quad \le 2 \sum_{i=1}^{t-1}\rbr{\rbr{\EE_{s_h^i, b_h^i}^i-\EE_{s_h^i}^i} \bigsbr{\orbr{Q_h^{\pi^i, M^t} -  Q_h^{\pi^i}}(s_h, b_h)}}^2 \nend
    &\qqquad + C \gamma^2  \rbr{\eta^{-1} +2 B_A}^2\eff_H(\gamma) \sum_{l=h+1}^H \gamma^{l-h-1} {\rbr{\EE_{s_h^i}^i+\EE_{s_h^i, b_h^i}^i}\sbr{D_\H^2(\nu_l^{\pi^i}(\cdot\given s_l), \nu_l^{\pi^i, M^t}(\cdot\given s_l))}}\nend
    &\quad\le  8 C_\eta^2 \beta  + 64 B_Q^2 \log\rbr{TH\cN_\infty(\cM, T^{-1})\delta^{-1}} + C \gamma^2  \rbr{\eta^{-1} +2 B_A}^2\eff_H(\gamma)^2  \beta \nend
    &\quad \le \cO\rbr{ (\eta^{-1}+2B_A)^2 \eff_H(\gamma)^2  \beta},\label{eq:OnN-1st-Delta1-sq}
\end{align}
where the second inequality holds by \eqref{eq:MLE-guarantee-Q-3} in \Cref{lem:MLE-formal}, and the covering number $\cN_\varrho(\cM, \epsilon)$ is with respect to the infinite norm of the Q function. Here, $\cO$ only hides universal constant independent of $H,\eta, T$. 
Meanwhile, for the TV distance term in \eqref{eq:OnN-1st-g-sq}, we have also by \Cref{lem:MLE} that 
\begin{align*}
    {\dr (i)}\le 8 B_V^2 \eff_{H}(\gamma){\sum_{l=h}^H \gamma^{l-h+1} 8\beta }\le \cO(B_A^2 \eff_H(\gamma)^2 \beta),
\end{align*}
where $\cO$ only hides some universal constants.
Hence, we conclude that 
\begin{align*}
    \sum_{i=1}^{t-1} \rbr{g_h^t(s_h^i, b_h^i, \pi^i)}^2  \le \cO\rbr{\rbr{\eta^{-1}+B_A}^2 \eff_H(\gamma)^2 \beta}.
\end{align*}
Now, for the first-order term, we have  
\begin{align*}
    &C^{(0)}
    \sum_{t=1}^T \sum_{h=1}^H{\EE^t\sbr{\abr{\tilde \Delta^{(1,t)}_h(s_h, b_h)}}}\nend
    &\quad \le 2 C^{(0)} \sum_{h=1}^H \underbrace{\sum_{t=1}^T \abr{\tilde\Delta^{(1, t)}_h(s_h^t, b_h^t)}}_{\ts \sum_{t=1}^T |g_h^t(s_h^t, b_h^t, \pi^t)|} + H C^{(0)} \cdot 4 \bignbr{\tilde \Delta^{(1)}_h}_\infty \log\rbr{H\cN_\rho(\cM, T^{-1})\delta^{-1}} \nend
    &\quad \le \cO\rbr{ HC^{(0)} \eff_H(\gamma)\sqrt{\dim(\cG_F^1)\rbr{\eta^{-1}+B_A}^2  \beta T }}\nend
    &\quad\le \cO\rbr{ H^2 \eff_H(\gamma)\rbr{1+\eta B_A} \sqrt{\dim(\cG_F^1)\beta T}},
\end{align*}
where the first inequality follows from a standard martingale concentration in \Cref{cor:martigale concentration}, and the second inequality holds by using the first order regret bound in \Cref{lem:de-regret}, and the last inequality holds by $C^{(0)}=2\eta H$.

\paragraph{Step 3. Bounding the Second-Order Term.}
Previously, we decompose the online regret and obtain a second-order term, which we referred to as (ii), 
\begin{align*}
    {\dr (ii)}\defeq C^{(2)}
    \sum_{t=1}^T \max_{h\in [H]} {\EE^t\sbr{ \rbr{\rbr{\tilde Q_h^t - r_h^{\pi^t} - \gamma P_h^{\pi^t} \tilde V_{h+1}^t}(s_h, b_h)}^2}}.
\end{align*}
For this term, we specify the function class to use for our purpose, 

\begin{itemize}[leftmargin =20pt]
    \item[(i)] We take the same function class $\cG_F^2$ as
    \begin{align*}
        \cG_{h, F}^2 &= \Bigg\{g:\cS\times\cA\times\cB\rightarrow \RR: \exists M\in\cM, h\in[H] \nend
        &\qqquad g(s_h, b_h, a_h) = {\bigrbr{r_h^M- r_h + \gamma (P_h^{M} - P_h) V_{l+1}^{\pi^M, M}}(s_h, a_h, b_h)}
        \Bigg\},
    \end{align*}
    where $\cG_F^2$ is bounded by $4B_A$. 
    \item[(ii)] We define a class of probability measures on $\cS\times\cA\times\cB$ as $$\sP_{h, F}^2 = \cbr{\PP^\pi(a_h=\cdot\given s_h, b_h)\delta_{(s_h, b_h)}(\cdot)\given \pi\in\Pi, (s_h,b_h)\in\cS\times\cB}, $$
    where $\delta_{(s_h,b_h)}(\cdot)$ is the measure that assigns $1$ to the state-action pair $(s_h, b_h)$. 
    \item[(iii)] We take a sequence of functions $\{g_h^t\}_{t\in[T]}$ as $\{g_h^t = \tilde r_h^{t}- r_h + \gamma (\tilde P_h^t - P_h) \tilde V_{l+1}^{t}\}_{t\in[T]}$, and take a sequence of probability measures as $\{\rho_h^t(\cdot) = \PP^{\pi^t}(a_h=\cdot\given s_h^t, b_h^t)\delta_{(s_h^t, b_h^t)}(\cdot)\}_{t\in[T]}$, where we define $\tilde r_h^t = r_h^{M^t}$ and $\tilde P_h^t = P_h^{M^t}$.
    One can check that $g_h^t\in\cG_{h,F}^1$ since $\tilde V_h^t = V_h^{\pi^t, M^t}$ and we have $\pi^t = \argmax_{\pi\in\Pi}J(\pi, M^t) = \pi^{M^t}$. In addition, we define $g_h^t(s_h^i, b_h^i, \pi^i)$ as the integral of $g_h^t$ with respect to $\rho_h^i$, which is given by
    \begin{align*}
        g_h^t (s_h^i, b_h^i, \pi^i) = \EE_{s_h^i, b_h^i}^{\pi^i} \sbr{\bigrbr{\tilde r_h^{t}- r_h + \gamma (\tilde P_h^t - P_h) \tilde V_{l+1}^{t}}(s_h, a_h, b_h)},
    \end{align*}
    Note that the sequence of probability measures is uniquely determined by $\{(s_h^t, b_h^t, \pi^t)\}_{t\in[T]}$.
\end{itemize}
We let $\dim(\cG_F^2)=\max_{h\in[H]}\dim_\DE(\cG_{h,F}^2,\sP_{h, F}^2, T^{-1/2})$ be the eluder dimension.
We next establish guarantee for $\sum_{i=1}^{t-1} (g_h^t(s_h^i, b_h^i, \pi^i))^2$.
\begin{align}\label{eq:OnN-guarantee-2ndQRE}
    &\sum_{i=1}^{t-1} (g_h^t(s_h^i, b_h^i, \pi^i))^2 \nend
    &\quad = \sum_{i=1}^{t-1} \rbr{\EE_{s_h^i, b_h^i}^{i} \sbr{\bigrbr{\tilde r_h^{t}- r_h + \gamma (\tilde P_h^t - P_h) \tilde V_{h+1}^{t}}(s_h, a_h, b_h)}}^2\nend
    &\quad  \le 2\sum_{i=1}^{t-1} \rbr{\EE_{s_h^i, b_h^i}^{i} \sbr{\bigrbr{\tilde r_h^{t}- r_h + \gamma (\tilde P_h^t - P_h) V_{h+1}^{\pi^i, M^t}}(s_h, a_h, b_h)}}^2 \nend
    &\qqquad + 2\cdot 16B_A^2 \sum_{i=1}^{t-1} \rbr{\EE_{s_h^i, b_h^i}^i\sbr{D_\TV^2\rbr{\tilde P_h^t(\cdot\given s_h, a_h, b_h), P_h(\cdot\given s_h, a_h, b_h)}}}\nend
    &\quad \lesssim 3 \underbrace{\sum_{i=1}^{t-1}\EE^i\rbr{\EE_{s_h, b_h}^{i} \sbr{\bigrbr{\tilde r_h^{t}- r_h + \gamma (\tilde P_h^t - P_h) V_{h+1}^{\pi^i, M^t}}(s_h, a_h, b_h)}}^2}_{\dr (iii)} + 32 B_A^2 \log\rbr{TH\cN_\rho(\cM,T^{-1}\delta^{-1})}\nend
    &\qqquad + 48 B_A^2 \underbrace{\sum_{i=1}^{t-1} \EE^i\sbr{D_\TV^2\rbr{\tilde P_h^t(\cdot\given s_h, a_h, b_h), P_h(\cdot\given s_h, a_h, b_h)}}}_{\dr (iv)} + 64 B_A^2 \log\rbr{TH\cN_\rho(\cM, T^{-1})}
\end{align}
where in the first inequality, we use $2B_A$ to upper bound $\|\tilde V_{h+1}^t - V_{h+1}^{\pi^i, M^t}\|_\infty$, and uses the Cauchy-Schwartz inequality to move the square into the expectation.
Here, the second inequality uses a standard martingale concentration result in \Cref{cor:martigale concentration} for both terms, and we invoke the same upper bound $B_A$ for the V functions.
Now, we invoke \Cref{lem:2nd-ub}, which says that term (iii) enjoys the following upper bound, 
\begin{align*}
    {\dr (iii)} &=\sum_{i=1}^{t-1}\EE^i\rbr{\EE_{s_h, b_h}^{i} \sbr{\bigrbr{\tilde r_h^{t}- r_h + \gamma (\tilde P_h^t - P_h) V_{h+1}^{\pi^i, M^t}}(s_h, a_h, b_h)}}^2 \nend
    & = \sum_{i=1}^{t-1}\EE^i\rbr{{\bigrbr{Q_h^{\pi^i, M^t} - r_h^{\pi^i} - \gamma P_h^{\pi^i} V_{h+1}^{\pi^i, M^t}}(s_h, a_h, b_h)}}^2\nend
    &\le L^{(2)} \sum_{i=1}^{t-1} \max_{h\in[H]} \cbr{\EE^i D_\H^2\bigrbr{\nu_h^{\pi^i, M^t}(\cdot\given s_h), \nu_h^{\pi^i}(\cdot\given s_h)}+ \EE^i D_\TV^2\bigrbr{P_h^{\pi^i, M^t}(\cdot\given s_h, b_h), P_h^{\pi^i}(\cdot\given s_h, b_h)}}\nend
    &\le L^{(2)} 4H \beta , 
\end{align*}
where the last inequality uses the guarantee in \Cref{lem:MLE}. To enable a direct use of the MLE guarantee, we replace the maximum by a sum over all $h\in[H]$ and upper bound the TV distance by the Hellinger distance. Here, $L^{(2)}$ is defined in \Cref{lem:2nd-ub}. For term (iv), we use the same guarantee in \Cref{lem:MLE} and the same bounding the TV distance by the Hellinger distance argument and obtain ${\dr (iv)}\le 4\beta$. Therefore, we conclude that 
\begin{align*}
    \sum_{i=1}^{t-1} (g_h^t(s_h^i, b_h^i, \pi^i))^2 \lesssim {\bigrbr{L^{(2)} H + B_A^2} \beta},
\end{align*}
where $\lesssim$ only hides some universal constant. As a result of the second order regret in \Cref{lem:de-regret}, we have 
\begin{align*}
    \sum_{t=1}^{T} (g_h^t(s_h^t, b_h^t, \pi^t))^2 
    &\lesssim {\dim(\cG_F^2) \bigrbr{L^{(2)} H + B_A^2} \beta  + \min\cbr{T, \dim(\cG_F^2)} 16 B_A^2+1}\nend
    &\lesssim {H \dim(\cG_F^2) \beta  L^{(2)}},
\end{align*}
where $\lesssim$ hides some universal constant. Therefore, we conclude that
\begin{align*}
    {\dr (ii)}
    &\defeq C^{(2)}
    \sum_{t=1}^T \max_{h\in [H]} {\EE^t\sbr{ \rbr{\rbr{\tilde Q_h^t - r_h^{\pi^t} - \gamma P_h^{\pi^t} \tilde V_{h+1}^t}(s_h, b_h)}^2}}\nend
    &\le C^{(2)} \sum_{t=1}^T \sum_{h=1}^H \EE^t \sbr{ \rbr{\rbr{\tilde Q_h^t - r_h^{\pi^t} - \gamma P_h^{\pi^t} \tilde V_{h+1}^t}(s_h, b_h)}^2} \nend
    & \le 2C^{(2)} \sum_{t=1}^T \sum_{h=1}^H \EE^t_{s_h^t, b_h^t} \sbr{ \rbr{\rbr{\tilde Q_h^t - r_h^{\pi^t} - \gamma P_h^{\pi^t} \tilde V_{h+1}^t}(s_h, b_h)}^2} \nend
    &\qquad + 4 H C^{(2)} 9 B_A^2 \log\rbr{H \cN_\rho(\cM,T^{-2})\delta^{-1}}. 
\end{align*}
where in the first inequality, we replace the maximum by the summation over $h\in[H]$ and in the second inequality, we invoke the martingale concentration in \Cref{cor:martigale concentration}.
Hence, we establish our bound for the second-order term as
\begin{align*}
    {\dr(ii)} 
    &\le 2 C^{(2)} \sum_{h=1}^H \sum_{t=1}^T \rbr{g_h^t(s_h^t, b_h^t, \pi^t)}^2 + 4 H C^{(2)} 9 B_A^2 \log\rbr{H \cN_\rho(\cM,T^{-2})\delta^{-1}}\nend
    &\lesssim {H^2 C^{(2)} \dim(\cG_F^2) \beta  L^{(2)}} +   { H C^{(2)} \beta}
\end{align*}
where $L^{(2)} = c H^2 \eff_H(c_2)^2 \kappa^2 \exp\rbr{8\eta B_A} C_\eta^2$ with 
$c_2 = \gamma(2\exp(2\eta B_A)+\kappa\exp(4\eta B_A))$, and $C^{(2)} = 2 H^2\eta^2 (1+4 \eff_H(\gamma))\exp\rbr{6\eta B_A}\cdot \rbr{\eff_H(\exp(2\eta B_A)\gamma)}^2$.

In summary, for the leader's Bellman error, we have
\begin{align*}
    \text{LBE}\lesssim H(2\sqrt{\dim\rbr{\cG_L} H^2 \beta T} + 3 H\min\cbr{T, \dim\rbr{\cG_L}}  + \sqrt T) \lesssim H^2 \sqrt{\dim\rbr{\cG_L} \beta T}, 
\end{align*}
for the first-order term in the follower's QRE,
\begin{align*}
    {\text{1st-QRE}}\lesssim  H^2 \eff_H(\gamma) \eta C_\eta \sqrt{\dim(\cG_F^1)\beta T}, 
\end{align*}
where $C_\eta = \eta^{-1}+B_A$,
and for the second-order term in the follower's QRE, 
\begin{align*}
\text{2nd-QRE}\lesssim H^2 C^{(2)} L^{(2)} \dim(\cG_F^2) \beta \log T
\end{align*}
which completes the proof for \Cref{thm:OMLE-farsighted}.

\subsection{Proof of \Cref{lem:MLE} on the Guarantee of MLE}\label{sec:proof-MLE}
The following proof mainly follows the proof of Theorem E.1 in \citet{chen2022unified}.
Recall the MLE function is defined as
\begin{align*}
    \cL^t_h(M)&\defeq -\sum_{i=1}^{t-1}  \bigg(\eta \rbr{Q_{h}^{\pi^i, M}(s_h^i, b_h^i)-V_{h}^{\pi^i, M}(s_h^i)} + \log P_h^M(s_{h+1}^i\given s_h^i, a_h^i, b_h^i)  \nend
    &\qquad - \rbr{u_{h}^i - u_{h}^M(s_h^i, a_h^i, b_h^i)}^2\bigg)\nend
    & = -\sum_{i=1}^{t-1} \rbr{\log \nu_h^{\pi^i, M}(b_h^i\given s_h^i) + \log P_h^M(s_{h+1}^i\given s_h^i, a_h^i, b_h^i) - \rbr{u_{h}^i - u_{h}^M(s_h^i, a_h^i, b_h^i)}^2},
\end{align*}
and the RL distance $D_{\RL, h}$ is defined as
\begin{align*}
    D_{\RL,h}^2 (M,  M^*;\pi) =   \EE^{\pi, M^*}D_\H^2\rbr{\nu_h^{\pi,M}, \nu_h^{\pi,M^*}} +
    \EE^{\pi, M^*} D_\H^2(P_h^{M}, P_h^{M^*}) +
    \EE^{\pi, M^*}\rbr{u_h^{M^*}-u_h^M}^2,
\end{align*}
We take an $\epsilon$-optimistic covering net of $\cM$ and denote the covering number as $\cN_\rho(\cM,\epsilon)$
By \Cref{lem:freeman-variation}, and take the filtration as $\sF_{i-1}=\sigma(\tau^{1:i-1})$, we have with probability at least $1-\delta/3H$, for all $t\in[T]$ and $M$ belonging to a $\epsilon$-covering net $\cM_\epsilon$  that
\begin{align}
    \frac{1}{2} \sum_{i=1}^{t-1} \log \frac{\nu_h^{\pi^i, M}}{\nu_h^{\pi^i, M^*}} 
    &\le \sum_{i=1}^{t-1} \log \EE^{\pi^i}\sbr{\exp\rbr{ \frac{1}{2}\cdot \log \frac{\nu_h^{\pi^i, M}}{\nu_h^{\pi^i, M^*}}}} + \log\rbr{\frac{3H\cN_\rho(\cM,\epsilon)}{\delta}}\nend
    & \le -\sum_{i=1}^{t-1} \EE^{\pi^i}\sbr{1 -  \sqrt{\frac{\nu_h^{\pi^i, M}}{\nu_h^{\pi^i, M^*}}}} + \log\rbr{\frac{3H\cN_\rho(\cM,\epsilon)}{\delta}}\nend
    &= -\frac{1}{2} \sum_{i=1}^{t-1} \EE^{\pi^i}D_\H^2\rbr{\nu_h^{\pi^i,M}(\cdot\given s_h), \nu_h^{\pi^i,M^*}(\cdot\given s_h)} + \log\rbr{\frac{3H\cN_\rho(\cM,\epsilon)}{\delta}}, \label{eq:MLE-nu}
\end{align}
where the first inequality is a direct result of \Cref{lem:freeman-variation}, the second inequality holds by using the inequality $\log(x)\le x-1$, and the last equality holds by the definition of the Hellinger distance.
Following the same argument, we have with probability at least $1-\delta/3H$ and for all $t\in[T]$, $M\in\cM_\epsilon$ that
\begin{align}
    \frac{1}{2}\sum_{i=1}^{t-1} \log \frac{P_h^M}{P_h^{M^*}} \le -\frac{1}{2} \sum_{i=1}^{t-1} \EE^{\pi^i} D_\H^2(P_h^{M}, P_h^{M^*}) + \log\rbr{\frac{3H\cN_\rho(\cM,\epsilon)}{\delta}}.\label{eq:MLE-T}
\end{align}
For the reward, we view $u_h^i$ as a random variable and have with the same argument that it holds with probability at least $1-\delta/3H$ that
\begin{align}
    &\frac 1 3 \sum_{i=1}^{t-1} \rbr{\rbr{u_h^i - u_h^{M^*}}^2 - \rbr{u_h^i - u_h^{M}}^2}\nend
    &\quad \le \sum_{i=1}^{t-1} \log \EE^{\pi^i}\sbr{\exp\rbr{\frac 1 3\rbr{\rbr{u_h^i - u_h^{M^*}}^2 - \rbr{u_h^i - u_h^{M}}^2}}}+\log\rbr{\frac{3H\cN_\rho(\cM,\epsilon)}{\delta}}\nend
    &\quad \le -\sum_{i=1}^{t-1} \frac 1 9 \EE^{\pi^i}\rbr{u_h^{M^*}-u_h^M}^2 + \log\rbr{\frac{3H\cN_\rho(\cM,\epsilon)}{\delta}},\label{eq:MLE-u}
\end{align}
where the second inequality holds by using the inequality $\EE[\exp(\lambda ((r-\EE r)^2-(r - \hat r)^2))]\le \exp\bigrbr{-\lambda(1-2\sigma^2\lambda)\rbr{\EE r - \hat r}^2}$ for any $\lambda\in\RR$, $\hat r\in\RR$ and $\sigma^2$-sub-Gaussian random variable $r$. Here, we notice that $\sup_{h\in[H]}\nbr{u_h}_\infty\le 1$ by our model assumption.
Suppose that for $M\in\cM$, the nearest point to $M$ in the $\epsilon$-covering net is $\tilde M$. 
Note that for the pair $(M, \tilde M)$, we have
\begin{align}\label{eq:D_RL-covering bound}
    &\sup_{\pi\in\Pi, h\in[H]}\abr{D_{\RL,h}^2(M, M^*;\pi) - D_{\RL,h}^2(\tilde M, M^*;\pi)}\nend
    &\quad \le 2\sup_{\pi\in\Pi, h\in[H]}\EE^{\pi, M^*}\sbr{D_\H\rbr{\nu_h^{\pi, M}, \nu_h^{\pi, \tilde M}} + D_\H\rbr{P_h^M, P_h^{\tilde M}} + \abr{u_h^{M}-u_h^{\tilde M}}}\nend
    &\quad \le 6\epsilon, 
\end{align}
where the inequality holds by noting that the Hellinger distance satisfies the triangle inequality and the fact that $D_\H(\cdot,\cdot)\le 1$ and $|u_h|\le 1$. 
Moreover, for the negative log-likelihood function, 
\begin{align}\label{eq:cL-covering bound}
    \cL_h^t(\tilde M) - \cL_h^t(M) &= \sum_{i=1}^{t-1} \rbr{\log \frac{\nu_h^{\pi^i, M}(b_h^i\given s_h^i)}{\nu_h^{\pi^i, \tilde M}(b_h^i\given s_h^i)} + \log \frac{P_h^{M}(s_{h+1}^i\given s_h^i, a_h^i, b_h^i)}{P_h^{\tilde M}(s_{h+1}^i\given s_h^i, a_h^i, b_h^i)}} \nend
    &\qquad + \sum_{i=1}^{t-1} \bigrbr{u_{h}^i - u_{h}^{\tilde M}(s_h^i, a_h^i, b_h^i)}^2 - \rbr{u_{h}^i - u_{h}^M(s_h^i, a_h^i, b_h^i)}^2\nend
    &\le \eta T \sup_{\pi\in\Pi, h\in H} \bignbr{A_h^{\pi,M} - A_h^{\pi,\tilde M}}_\infty + T \log(\exp(\epsilon)) + 2T \sup_{ h\in H} \bignbr{u_h^{M} - u_h^{\tilde M}}_\infty\nend
    &\le 4 T, 
\end{align}
where the first inequality holds by noting the optimistic covering condition $P_h^{M}(s_{h+1}\given s_h, a_h, b_h)\le \exp(\epsilon)P_h^{\tilde M} (s_{h+1}\given s_h, a_h, b_h)$.
Now, using \eqref{eq:MLE-nu}, \eqref{eq:MLE-T}, and \eqref{eq:MLE-u}, we conclude that with probability at least $1-\delta$, for all $M\in \cM$ and $h\in[H]$,
\begin{align*}
    &\sum_{i=1}^{t-1}D_{\RL, h}^2\rbr{M, M^*;\pi^i}\nend
    &\quad \le {\sum_{i=1}^{t-1} \EE^{\pi^i}D_\H^2\rbr{\nu_h^{\pi^i,\tilde M}, \nu_h^{\pi^i,M^*}} +
    \sum_{i=1}^{t-1} \EE^{\pi^i} D_\H^2(P_h^{\tilde M}, P_h^{M^*}) +
    \sum_{i=1}^{t-1} \EE^{\pi^i}\rbr{u_h^{M^*}-u_h^{\tilde M}}^2} + 6\epsilon T\nend
    &\quad \le -3 \rbr{ \sum_{i=1}^{t-1} \log \frac{\nu_h^{\pi^i, \tilde M}}{\nu_h^{\pi^i, M^*}} + 
    \sum_{i=1}^{t-1} \log \frac{P_h^{\tilde M}}{P_h^{M^*}} + 
    \sum_{i=1}^{t-1} \rbr{\rbr{u_h^i - u_h^{M^*}}^2 - \rbr{u_h^i - u_h^{\tilde M}}^2}} \nend
    &\qqquad + 9\log\rbr{\frac{3H\cN_\rho(\cM,\epsilon)}{\delta}}+6T\epsilon \nend
    &\quad\le 3 \rbr{\cL^t_h(\tilde M)- \cL^t_h(M^*)} + 9\log\rbr{\frac{3H\cN_\rho(\cM,\epsilon)}{\delta}} + 6T\epsilon, 
\end{align*}
where the first inequality holds by definition of $D_\RL$ and \eqref{eq:D_RL-covering bound}, the second inequality holds by taking a union bound over the success of \eqref{eq:MLE-nu}, \eqref{eq:MLE-T}, and \eqref{eq:MLE-u}, and the last inequality holds by definition of $\cL^t_h$.
Furthermore, we notice by \eqref{eq:cL-covering bound} that
\begin{align}\label{eq:MLE-D_RL}
    \sum_{i=1}^{t-1}D_{\RL, h}^2\rbr{M, M^*;\pi^i}
    &\le 3 \rbr{\cL^t_h(\tilde M) -  \cL^t_h(M) + \cL^t_h(M)- \cL^t_h(M^*)} + 9\log\rbr{\frac{3H\cN_\rho(\cM,\epsilon)}{\delta}} + 6T\epsilon\nend
    &\le 3 \rbr{\cL^t_h(M)- \cL^t_h(M^*)} + 9\log\rbr{\frac{3H\cN_\rho(\cM,\epsilon)}{\delta}} + 18T\epsilon.
\end{align}
Now, we replace $M$ by $\hat M_{h,\MLE} =\argmin_{M\in\cM} \cL^t_h(M)$, $\epsilon$ by $T^{-1}$ in \eqref{eq:MLE-D_RL} and obtain with probability $1-\delta$ for all $h\in[H], t\in[T]$ that 
\begin{align*}
    \cL^t_h(M^*) - \inf_{M\in\cM}\cL^t_h(M) &\le -\frac 1 3 \sum_{i=1}^{t-1}D_{\RL, h}^2\rbr{M, M^*;\pi^i} + 9\log\rbr{\frac{3H\cN_\rho(\cM,\epsilon)}{\delta}} + 18T\epsilon\nend
    &\le 9\log\rbr{\frac{3H\cN_\rho(\cM,T^{-1})}{\delta}} + 18 = \beta, 
\end{align*}
which shows that $M^*\in \confset^t(\beta)$ with high probability. On the other hand, we plug in any $M\in\confset_\cM^t(\beta)$ and obtain for any $h\in[H]$, $t\in[T]$, and $M\in\confset_\cM^t(\beta)$ that
\begin{align*}
    \sum_{i=1}^{t-1}D_{\RL, h}^2\rbr{M, M^*;\pi^i}
    &\le 3 \rbr{\cL^t_h(M)- \cL^t_h(M^*)} + 9\log\rbr{\frac{3e^2H\cN_\rho(\cM,\epsilon)}{\delta}} \le 4\beta,
\end{align*}
where the last inequality holds by definition of the confidence set. Hence, we complete our proof of \Cref{lem:MLE}.
Lastly, we can also take filtration $\sF_{i-1} = \sigma((s_h^j, a_h^j, b_h^j)_{j\in[i-1]}, s_h^i)$ for \eqref{eq:MLE-nu} and $\sF_{i-1} = \sigma((s_h^j, a_h^j, b_h^j)_{j\in[i]})$ for \eqref{eq:MLE-T} and \eqref{eq:MLE-u}, follows exactly the same steps, and obtain
\begin{align*}
    \sum_{i=1}^{t-1} \hat D_{\RL,h,i}^2(M,M^*) 
    \le 4 \beta,
\end{align*}
which finishes the proof of \Cref{lem:MLE}.

\section{Proofs of the Auxiliary Results in  \Cref{sec:app-major-tech}}
In this section, we provide proof for the lemmas introduced  in \Cref{sec:app-major-tech}. 


\subsection{Proof of \Cref{lem:performance diff}}\label{sec:proof-performance diff}

In the following, we prove \Cref{lem:performance diff}, which relates the estimation error of quantal response policy to a few estimation errors involving the follower's value functions. 
To simplify the notation, we let $\nu$, 
$Q, V, A$ denote $\nu$, $Q^{\pi}$, $V^{\pi}$, and $A^{\pi}$, respectively, which are quantities computed under the true model $M^*$.  
Note that we have  
$\nu_h(b\given s) = \exp(\eta \cdot A_h (s, b) )$ and 
$\tilde \nu_h (b \given s) = \exp(\eta \cdot \tilde A_h (s,b) )$.
By the upper bound in \eqref{eq:nu-tv-ub-0} of \Cref{lem:response diff},
we have 
\begin{align}
    \dr{(i)} &\defeq \sum_{h=1}^H  H \cdot  \EE \bigsbr{ \nbr{\tnu_h(\cdot\given s_h)-\nu_h(\cdot\given s_h)}_1} = \sum_{h=1}^H  2 H \cdot  \EE \bigsbr{  D_\TV \bigrbr{\nu(\cdot\given s_h), \tilde \nu(\cdot\given s_h)}} \nend
    &\le 2 \eta  H \cdot  \underbrace{\sum_{h=1}^H \EE\bigsbr{\bigabr{\orbr{A_h-\tilde A_h}(s_h, b_h)}}}_{\dr (ii)} \nend
    &\qqquad + \eta^2  H \cdot \sum_{h=1}^H \EE\Bigsbr{ \exp\bigrbr{\eta\bigabr{\orbr{A_h-\tilde A_h}(s_h, b_h)}}\cdot \bigabr{\orbr{A_h-\tilde A_h}(s_h, b_h)}^2}. \label{eq:(i)}
\end{align}
In the following, we let 
\$
    \tilde\Delta_h^{(1)} (s_h, b_h)&\defeq  \rbr{\EE_{s_h, b_h} - \EE_{s_h}}\sbr{\sum_{l=h}^H \gamma^{l-h}\bigrbr{\orbr{\tilde Q_l - r_l^\pi - \gamma P_l^\pi \tilde V_{l+1}}(s_l, b_l)}}, \\
    \tilde\Delta_h^{(2)}(s_h) &\defeq \EE_{s_h}\sbr{\sum_{l=h}^H \gamma^{l-h} \kl\infdivx[\big]{\nu_l(\cdot\given s_l)}{\tilde\nu_l(\cdot\given s_l)}}. 
\$
We note that we denote $\EE_{z} [\cdot]=\EE^{\pi,M^*}[\cdot\given z]$ for any variable $z$.
Here the expectations in $\tilde\Delta^{(1)}_h$ and $\tilde\Delta^{(2)}_h$ are taken with respect to the randomness of the trajectory generated by $\{ \pi, \nu^{\pi}\}$, given $s_h$ or $(s_h, b_h)$.
We can further bound   (ii) defined in \eqref{eq:(i)} by invoking \Cref{lem:AQV-func diff}, 
which implies that 
\$
\dr{(ii)}&= \sum_{h=1}^H \EE\sbr{\abr{\rbr{\EE_{s_h, b_h}-\EE_{s_h}} \bigsbr{\tilde\Delta_h^{(1)}(s_h, b_h) - \gamma\eta^{-1}\tilde\Delta_{h+1}^{(2)} (s_{h+1})} + \eta^{-1}\kl\infdivx[\big]{\nu_h(\cdot\given s_h)}{\tilde \nu_h(\cdot\given s_h)}}}.
\$
By the law of total expectation, 
we have 
\$
\EE \Bigsbr{\EE _{s_h, b_h} \bigsbr{   \tilde\Delta_{h+1}^{(2)} (s_{h+1})}} = \EE \Bigsbr{\EE _{s_h } \bigsbr{   \tilde\Delta_{h+1}^{(2)} (s_{h+1})}}  \geq 0.
\$
Besides, by the definition of $\tilde \Delta_h^{(2)}$, we have 
\$
\EE\bigsbr{\tilde\Delta_h^{(2)}(s_h)} = \EE\bigsbr{ \kl\infdivx[\big]{\nu_h(\cdot\given s_h)}{\tilde \nu_h(\cdot\given s_h)} + \gamma \cdot \tilde\Delta_{h+1}^{(2)} (s_{h+1}) }. 
\$ 
Thus, by triangle inequality, we have 
\begin{align}
    \dr{(ii)}  &\le \sum_{h=1}^H\EE\Bigsbr{\bigabr{\rbr{\EE_{s_h, b_h}-\EE_{s_h}}\bigsbr{\tilde\Delta_h^{(1)}(s_h, b_h)}}} + 2\eta^{-1} \sum_{h=1}^H\EE\bigsbr{\tilde\Delta_h^{(2)}(s_h)}, \label{eq:(ii)}
\end{align}
Furthermore, for the second term on the right-hand side of \eqref{eq:(ii)}, we  apply the inequality between KL divergence and the $\chi^2$ divergence to each term $\kl\infdivx[]{\nu_l(\cdot\given s_l)}{\tilde\nu_l(\cdot\given s_l)}$ and obtain that 
\begin{align}
\kl\infdivx[\big]{\nu_l(\cdot\given s_l)}{\tilde\nu_l(\cdot\given s_l)}
&\le \chi^2\infdivx[\big]{\nu_l(\cdot\given s_l)}{\tilde\nu_l(\cdot\given s_l)}\nend
&= \inp[\Bigg]{\nu_l(\cdot\given s_l)}{\biggrbr{\sqrt{\frac{\nu_l(\cdot\given s_l)}{\tilde\nu_l(\cdot\given s_l)}}-\sqrt{\frac{\tilde\nu_l(\cdot\given s_l)}{\nu_l(\cdot\given s_l)}}}^2}_{\cB}\nend
&\le \eta^2 \cdot \EE_{s_l}\sbr{\exp\bigrbr{\eta  \cdot \bigabr{\orbr{A_l-\tilde A_l}(s_l,b_l)}}\cdot \bigrbr{\orbr{A_l -\tilde A_l}(s_l, b_l)}^2}, \label{eq:kl-ub}
\end{align}
where last expectation is with respect to $b_ l \sim \nu_{l } (\cdot \given s_l)$. 
Here the inequality holds by noting that 
 $\sqrt{\nu_l/\tilde \nu_l}=\exp(\eta(A_l-\tilde A_l)/2)$ and the basic inequality $| \exp( x ) - \exp(y)| \leq \exp ( | x-y|) \cdot |x -y|$.
Plugging \eqref{eq:kl-ub} back into \eqref{eq:(ii)}, we obtain
\begin{align}
    \dr{(ii)} &\le \sum_{h=1}^H\EE\Bigsbr{\bigabr{\rbr{\EE_{s_h, b_h}-\EE_{s_h}}\bigsbr{\tilde\Delta_h^{(1)}(s_h, b_h)}}}  + 2\eta^{-1} \sum_{h=1}^H \EE\sbr{\sum_{l=h}^H \gamma^{l-h} \cdot \kl\infdivx[\big]{\nu_l(\cdot\given s_l)}{\tilde\nu_l(\cdot\given s_l)}}\nend
    &\le \sum_{h=1}^H\EE\sbr{\bigabr{\rbr{\EE_{s_h, b_h}-\EE_{s_h}}\bigsbr{\tilde\Delta_h^{(1)}(s_h, b_h)}}} \nend
    &\qquad + 2\eta \sum_{h=1}^H \sum_{l=h}^{H} \gamma^{l-h} \cdot  \EE\sbr{\exp\bigrbr{\eta  \cdot \bigabr{\orbr{A_l-\tilde A_l}(s_l,b_l)}}\cdot \bigabr{\orbr{A_l -\tilde A_l}(s_l, b_l)}^2}\nend
    &\le \sum_{h=1}^H\EE\sbr{\bigabr{\rbr{\EE_{s_h, b_h}-\EE_{s_h}}\bigsbr{\tilde\Delta_h^{(1)}(s_h, b_h)}}} \nend
    &\qquad + \frac{2\eta(1-\gamma^H)}{1-\gamma}\cdot \sum_{h=1}^H  \EE\sbr{\exp\bigrbr{\eta  \cdot \bigabr{\orbr{A_h-\tilde A_h}(s_h,b_h)}}\cdot \bigabr{\orbr{A_h -\tilde A_h}(s_h, b_h)}^2} .
     \label{eq:(ii)-2}
\end{align}
Recall that we  define $\eff_H(x) = (1-x^H)/(1-x)$ as the \say{effective}  horizon with respect to $x$.

Plugging \eqref{eq:(ii)-2} back into \eqref{eq:(i)}, we conclude that
\begin{align}
    \dr{(i)}&\le 2\eta  H \cdot \sum_{h=1}^H  \EE\sbr{\abr{\rbr{\EE_{s_h, b_h}-\EE_{s_h}}\bigsbr{\tilde\Delta_h^{(1)}(s_h, b_h)}}}   \nend
    &\qquad + \eta^2  H  \bigrbr{1+ 4  \cdot \eff_H(\gamma) } \cdot \sum_{h=1}^H  \EE\sbr{\exp\bigrbr{\eta  \cdot \bigabr{\orbr{A_h-\tilde A_h}(s_h,b_h)}}\cdot \bigabr{\orbr{A_h -\tilde A_h}(s_h, b_h)}^2}  . \label{eq:(ii)-21}
\end{align}
Note that we define  $C^{(1)}$ in \eqref{eq:define_constants}. 
Since $\oabr{\orbr{A_h-\tilde A_h}(s_h,b_h)} \leq 2 B_{A}$, 
by  \eqref{eq:(ii)-21}  and inequality 
\$
\exp\bigrbr{\eta  \cdot \bigabr{\orbr{A_h-\tilde A_h}(s_h,b_h)}}  \leq \exp(2\eta B_{A}), 
\$
we conclude the proof of \eqref{eq:taylor-myopic}. 

It remains to prove \eqref{eq:taylor-farsighted}. 
Notice that 
\#\label{eq:f_2-01}
\begin{split}
    V_h (s_h) = \max_{\nu' \in \Delta (\cB) }\bigl\{ \inp{\nu' }{Q_h (s_h, \cdot )}_{\cB } +\eta^{-1} \cH(\nu')\bigr\} , \\
    \tilde V_h (s_h) = \max_{\nu' \in \Delta (\cB) }\bigl\{ \inp{\nu' }{\tilde Q_h (s_h, \cdot )}_{\cB } +\eta^{-1} \cH(\nu')\bigr\} ,  
\end{split}
\#
where the maximizers are $\nu_h (\cdot \given s_h)$ and $\tilde \nu_h (\cdot \given s_h)$, respectively.
Then, by  \eqref{eq:f_2-01}
we have 
\#
& \bigabr{\orbr{V_h-\tilde V_h}(s_h)}   \notag \\
& \quad \le \max\Bigcbr{\inp[\big]{\nu_h(\cdot \given s_h) }  { \bigabr{Q_h(s_h, \cdot )-\tilde Q_h(s_h, \cdot )} }_{\cB} }, ~\Bigabr{\inp[\big]{\tilde\nu_h(\cdot \given s_h) }{\bigabr{ Q_h(s_h, \cdot )-\tilde Q_h(s_h, \cdot )} }_{\cB} } \notag\\
& \quad  =   \max\Bigcbr{  \EE_{s_h} \bigl [  \bigl | (Q_h - \tilde Q_h) (s_h, b_h ) \big | \bigr ] , ~   \EE_{s_h} \bigl [ \big |  (Q_h - \tilde Q_h)  (s_h, b_h )\bigr |  \cdot \tilde \nu_h (b_h \given s_h) / \nu_h (b_h \given s_h ) \bigr ]  }   \notag \\
& \quad  \leq    \exp(2 \eta B_A)  \cdot    \EE_{s_h} \bigl [ \big |  (Q_h - \tilde Q_h) (s_h, b_h )\big | \bigr ]   , \label{eq:f_2-1} 
\# 
where the expectation is taken with respect to $b_h \sim \nu_h (\cdot \given s_h)$. 
Here the first inequality is obtained from the optimality condition of \eqref{eq:f_2-01}, and the  last inequality holds because  $\nbr{\tilde\nu_h/\nu_h}_\infty \le \exp(2\eta B_A)$. 
Note that $\tilde A = \tilde Q - \tilde V$ and $A = Q - V$.
By triangle inequality, we have 
\begin{align}
    &\abr{\orbr{A_h-\tilde A_h}(s_h, b_h)}  
  \le  \bigabr{\orbr{Q_h-\tilde Q_h}(s_h, b_h)} + \exp\orbr{2\eta B_A} \cdot  \EE_{s_h} \bigl [  \big |  (Q_h - \tilde Q_h) (s_h, b_h )\big |  \bigr ]  .  \label{eq:f_2-11} 
\end{align}
%
%
%
Now for $\oabr{   (Q_h - \tilde Q_h) (s_h, b_h ) }$, by the Bellman equation $Q_h = r^{\pi}_h + \gamma P_h^{\pi} V_{h+1}$, we have  
\begin{align}
    &\bigabr{ \orbr{Q_h-\tilde Q_h}(s_h, b_h)} \nend
    &\quad \le  
    \bigabr{ 
        \orbr{\tilde Q_h - r_h^\pi - \gamma P_h^\pi \tilde V_{h+1}}(s_h, b_h)
    } 
         + \gamma  \cdot \bigabr{\bigrbr{P_h^\pi \orbr{V_{h+1}-\tilde V_{h+1}}}(s_h, b_h)}
         \nend
    &\quad \le \bigabr{\orbr{\tilde Q_h - r_h^\pi - \gamma P_h^\pi \tilde V_{h+1}}
    (s_h, b_h)} 
    + \gamma \cdot \exp\rbr{2\eta B_A}\EE_{s_h, b_h}\bigsbr{\oabr{Q_{h+1}^\pi-\tilde Q_{h+1}^\pi}},\label{eq:f_2-Q-ub}
\end{align}
where the first inequality holds by a standard decomposition and the second inequality is obtained by applying the  same upper bound for $V_{h}-\tilde V_h$ in \eqref{eq:f_2-1} to $V_{h+1} - \tilde V_{h+1}$. 
By recursion, we have 
\begin{align}
    \bigabr{\orbr{Q_h-\tilde Q_h}(s_h, b_h)} &\le  \sum_{l=h}^H \big (\gamma \cdot \exp(2\eta B_A  ) \big) ^{l-h} \cdot  \EE_{s_h, b_h}\bigsbr{\bigabr{\orbr{\tilde Q_l - r_l^\pi - \gamma P_l^\pi \tilde V_{l+1}}(s_l, b_l)}}.\label{eq:f_2-Q-telo}
\end{align}
Now, by the boundedness of $A_h$ and $\tilde A_h$, and \eqref{eq:f_2-11}, 
we have 
\$
& \EE\sbr{\exp\bigrbr{\eta  \cdot \bigabr{\orbr{A_h-\tilde A_h}(s_h,b_h)}}\cdot \bigabr{\orbr{A_h -\tilde A_h}(s_h, b_h)}^2} \notag \\
& \quad \leq \exp\bigrbr{2\eta B_A}\cdot \EE\bigsbr{ \bigabr{\orbr{A_h-\tilde A_h}(s_h, a_h)}^2}\nend
&\quad \le 2\exp\bigrbr{6\eta B_A}\cdot {\EE\bigsbr{ \bigabr{\orbr{\tilde Q_h - Q_h}(s_h, b_h)}^2}},
\$ 
where in the  last inequality we use the basic inequality $(a + b)^2 \leq 2 a ^2 + 2 b^2 $. 
Combining  with \eqref{eq:f_2-Q-telo}, we obtain that 
\begin{align*}
    &\EE\sbr{\exp\bigrbr{\eta  \cdot \bigabr{\orbr{A_h-\tilde A_h}(s_h,b_h)}}\cdot \bigabr{\orbr{A_h -\tilde A_h}(s_h, b_h)}^2} \nend
    &\quad \le 2\exp\bigrbr{6\eta B_A} \cdot \EE \biggsbr { \biggrbr{\sum_{l=h}^H (\gamma \cdot \exp(2\eta B_A  ) \big) ^{l-h} \cdot  \EE_{s_h, b_h}\bigsbr{\bigabr{\orbr{\tilde Q_l - r_l^\pi - \gamma P_l^\pi \tilde V_{l+1}}(s_l, b_l)}}}^2 } \nend
    &\quad \le 2\exp\orbr{6\eta B_A} \cdot \bigrbr{\eff_H(\exp\orbr{2\eta B_A}\gamma)}^2  \cdot \max_{l\in\{h, \dots, H\}}\EE\bigsbr{\bigabr{\orbr{\tilde Q_l - r_l^\pi - \gamma P_l^\pi \tilde V_{l+1}}(s_l, b_l)}^2}.
\end{align*}
Recall that we define  $C^{(2)} = 2 \eta^2  H^2 \cdot \exp\rbr{6\eta B_A}  \cdot \rbr{1+ 4 \eff_H(\gamma)} \cdot \rbr{\eff_H(\exp\cbr{2\eta B_A}\gamma)}^2$. By \eqref{eq:(ii)-21}, we establish \eqref{eq:perform-diff-linear}. 
Therefore, we 
  complete the proof of \Cref{lem:performance diff}.

\subsection{Proof of \Cref{cor:response-diff-myopic}}
\label{sec:proof-response-diff-myopic}

Since we consider any fixed state $s \in \cS$, in this proof, we omit $s$ 
to simplify the notation. 
To apply Lemma \ref{lem:performance diff}, 
we note that 
$Q$ and $\tilde Q$ in   Lemma \ref{lem:performance diff} becomes $r^{\pi}$ and $\tilde r^{\pi}$ in the myopic case. 
To make the proof consistent with that of Lemma \ref{lem:performance diff}, we use notation $\{ Q, \tilde Q, V, \tilde V, A, \tilde A\}$ in the sequel.

To begin with, 
we invoke \eqref{eq:nu-tv-ub-0} in \Cref{lem:response diff} and obtain that 
\begin{align*}
    D_\TV\orbr{\nu, \tilde \nu}
    &\le \eta \cdot \inp[\big]{\nu} {\oabr{\tilde A-A} + \frac \eta 2 \exp\bigrbr{\eta\oabr{\tilde A-A}} \cdot \orbr{\tilde A-A}^2}_{\cB } \nend
    &\le \eta \cdot \inp[\big]{\nu} {\oabr{\tilde A-A} + \frac \eta 2 \exp\orbr{2\eta B_A} \orbr{\tilde A-A}^2}_{\cB }\nend
    &= \eta \cdot \EE\bigsbr{\oabr{\tilde A-A}} + \frac {\eta^2} {2} \exp\bigrbr{2\eta B_A} \cdot \Bigrbr{\Var\orbr{\tilde A-A}+ \bigrbr{\EE\osbr{\tilde A-A}}^2},
\end{align*}
where the last equality holds by the  variance-mean decomposition. 
Here the expectation and variance are taken with respect to $   b \sim \nu(\cdot \given s ) $. 
Now, using \Cref{lem:AQV-func diff} to the myopic case, we have
\begin{align*}
    \EE\bigsbr{\oabr{\tilde A - A}} 
    &\le \EE\bigsbr{\bigabr{(\tilde Q - Q) - \EE  \osbr{\tilde Q - Q}}} + \eta^{-1}\kl\infdivx[]{\nu}  {\tilde\nu}.  
\end{align*}
For the variance term, we have
\begin{align*}
    \Var\orbr{\tilde A -A} & = \EE\bigsbr{\bigrbr{\orbr{\tilde A -A} - \EE\orbr{\tilde A -A}}^2}  = \EE\bigsbr{\bigrbr{\orbr{\tilde Q -Q} - \EE\osbr{\tilde Q -Q}}^2}, 
\end{align*}
where the last equality holds because $V $ and $\tilde V $ do not involve $b $. 
Furthermore, 
note that $$\eta\EE\osbr{A-\tilde A} = \kl\infdivx[]{\nu}{\tilde\nu} \leq 2 \eta B_A.$$ 
Thus, combining the inequalities above, we have 
  we have
\begin{align}
    & D_\TV(\nu,\tilde\nu)\nend 
    & \quad \le   \eta \cdot   \EE\bigsbr{\bigabr{(\tilde Q - Q) - \EE  \osbr{\tilde Q - Q}}} + \frac{\eta^2}{2} \cdot \exp\rbr{2\eta B_A} \cdot \EE\bigsbr{\bigrbr{\orbr{\tilde Q -Q} - \EE\osbr{\tilde Q -Q}}^2}\nend
    &\qquad + \kl\infdivx[]{\nu}{\tilde\nu} + \exp\rbr{2\eta B_A}/ 2 \cdot  \bigrbr{\kl\infdivx[]{\nu}{\tilde\nu}}^2\nend
    &\quad \le\eta \cdot   \EE\bigsbr{\bigabr{(\tilde Q - Q) - \EE  \osbr{\tilde Q - Q}}}  + \frac{\eta^2}{2} \cdot \exp\rbr{2\eta B_A} \cdot  \EE\bigsbr{\bigrbr{\orbr{\tilde Q -Q} - \EE\osbr{\tilde Q -Q}}^2}\nend
    &\qquad + \bigrbr{1 + \eta B_A \exp\rbr{2\eta B_A} }\cdot \kl\infdivx[]{\nu}{\tilde\nu}, \label{eq:TV-ub-taylor}
\end{align}
where the  last inequality holds by noting that $\kl\infdivx[]{\nu}{\tilde\nu}\le 2\eta B_A$.

In the following, we  handle the KL divergence term.
We calculate the derivative of $\eta^{-2}\kl\infdivx{\nu}{\tilde\nu}$ with respect to $\tilde Q$ and obtain
\begin{align*}
    \partial_{\tilde Q}\rbr{\eta^{-2}\kl\infdivx{\nu}{\tilde\nu}} = \eta^{-1}\partial_{\tilde Q} \bigrbr{\EE\osbr{A-\tilde A}} = \eta^{-1}{\bigrbr{\partial_{\tilde Q} \tilde V - \nu}} = \eta^{-1}\rbr{\tilde \nu - \nu},
\end{align*}
where $\ind$ denote the all one vector of length $|\cB|$ is $\cB$ is discrete.
Here the first equality follows from $\eta\EE\osbr{A-\tilde A} = \kl\infdivx[]{\nu}{\tilde\nu}$, the second equality holds because $\nu$ and $A$ do not depend on $\tilde Q$, and $\tilde A = \tilde Q - \tilde V$. 
Moreover, the last equality holds because 
$$\tilde V(s)  = \eta^{-1} \log \bigg(\sum_{b \in \cB} \exp \big( \eta \cdot  \tilde Q(s, b)\bigr) \biggr), $$
and also $\partial_{\tilde Q}\EE[\tilde Q] = \nu$.
We further take a second-order derivative and obtain
\begin{align*}
    \partial^2_{\tilde Q \tilde Q} \rbr{\eta^{-2}\kl\infdivx{\nu}{\tilde\nu}} = \eta^{-1}\partial_{\tilde Q} \tilde \nu = \diag(\tilde \nu) -\tilde \nu \tilde\nu^\top\eqdef \H, 
\end{align*}
where the last equality holds for the vector case. 
Note that the Hessian is upper and lower bounded by $\L$ where $\L = \diag(\nu )-\nu \nu^\top$, which is proved by  the following proposition. 
\begin{proposition}\label{prop:Hessian-ulb}
Let $\H = \diag(\tilde\nu) -\tilde\nu \tilde\nu^\top$ and $\L=\diag(\nu)-\nu \nu^\top$ where $\nu=\exp\orbr{\eta A}$ and $\tilde\nu=\exp\orbr{\eta \tilde A}$ are two quantal response over $\cB$ with $\nbr{A}_\infty\le B_A, \onbr{\tilde A}_\infty\le B_A$. Then   for any vector  $g\in \RR^{|\cB| }$, we have  
\begin{align}
    \exp\rbr{2 \eta B_A} \cdot g^\top \L g \ge x^\top \H x \ge \exp\rbr{-2 \eta B_A} \cdot g^\top \L g.\label{eq:Hessian ub lb}
\end{align}
\end{proposition}
\begin{proof}
Note that $\exp\rbr{-2 \eta B_A}\le  \tilde \nu(b) / \nu(b) \le\exp\rbr{ 2 \eta B_A}$ for any $b\in \cB$. 
Let $\EE^{\nu}$ and $\Var^{\nu}$ denote the expectation and variance under distribution $\nu$. 
Then we have 
\begin{align*}
    g^\top \L g & = \Var^\nu[g(b)]
   = \EE^\nu\bigsbr{\bigrbr{g(b) - \EE^\nu[g(b)]}^2},\notag \\
   g^\top \H g & = \Var^{\tilde \nu}[g(b)]
   = \EE^{\tilde \nu}\bigsbr{\bigrbr{g(b) - \EE^\nu[g(b)]}^2}.
\end{align*}
By direct computation, we have 
\begin{align*}
    & \exp\rbr{-\eta B_A} \cdot \EE^{\tilde\nu}\bigsbr{\bigrbr{g(b) - \EE^{\tilde\nu}[g(b)]}^2}
    \notag \\
    & \quad \le \exp\rbr{-\eta B_A}\cdot \EE^{\tilde\nu}\bigsbr{\bigrbr{g(b) - \EE^{\nu}[g(b)]}^2} 
     \le \EE^\nu\sbr{\rbr{g(b) - \EE^\nu[g(b)]}^2} 
\end{align*}
where the first inequality is true because changing $\EE^{\tilde\nu}[g(b)]$ to $\EE^{ \nu}[g(b)]$ incurs additional bias, and the second inequality is true because $\tilde \nu (b) / \nu(b)$ 
Similarly, we have 
\begin{align*}
     \EE^\nu\sbr{\rbr{g(b) - \EE^\nu[g(b)]}^2} 
    &\le \EE^{\nu}\sbr{\rbr{g(b) - \EE^{\tilde\nu}[g(b)]}^2}  
     \le  \exp\rbr{\eta B_A}  \cdot \EE^{\tilde\nu}\sbr{\rbr{g(b) - \EE^{\tilde\nu}[g(b)]}^2}.
\end{align*}
Therefore, we conclude that \eqref{eq:Hessian ub lb} holds. 
\end{proof}

Using the lower bound in \eqref{eq:Hessian ub lb}, we have for the KL divergence that
\begin{align*}
    \eta^{-2}\kl\infdivx[]{\nu}{\tilde\nu} &\le 1/2 \cdot (\tilde Q - Q)^\top  \H (\tilde Q - Q)  \le  \exp\rbr{2\eta B_A}/ 2 \cdot (\tilde Q - Q)^\top  \L (\tilde Q - Q) \nend
    &= \frac{\exp\rbr{2\eta B_A}}{2} \cdot (\tilde Q - Q)^\top  \bigrbr{\diag(\nu)-\nu\nu^\top} (\tilde Q - Q) , 
\end{align*}
where the first inequality holds by noting that the derivative of the KL-divergence at $\tilde\nu=\nu$ is zero, and we upper bound the KL-divergence  only by the second order term. Furthermore, the second inequality holds because  $\H\preceq \exp(2\eta B_A)\cdot \L$, which is proved  by \Cref{prop:Hessian-ulb}. 
Therefore, we conclude for \eqref{eq:TV-ub-taylor} that
\begin{align*}
    D_\TV \rbr{\nu, \tilde \nu} &\le \eta \cdot   \EE\bigsbr{\bigabr{(\tilde Q - Q) - \EE\osbr{\tilde Q - Q}}} + \frac{\eta^2}{2} \exp\rbr{2\eta B_A} \cdot \EE\bigsbr{\bigrbr{\orbr{\tilde Q -Q} - \EE\osbr{\tilde Q -Q}}^2}\nend
    &\qquad + \bigrbr{1 + \eta B_A \cdot \exp\rbr{2\eta B_A} }\cdot \kl\infdivx[]{\nu}{\tilde\nu}\nend
    &\le \eta \cdot   \EE\bigsbr{\bigabr{(\tilde Q - Q) - \EE\osbr{\tilde Q - Q}}} \nend
    &\qquad + \frac{\eta^2 \exp(2\eta B_A)}{2} \bigrbr{2+\eta B_A \cdot  \exp\rbr{2\eta B_A}} \cdot  \EE\bigsbr{\bigrbr{\orbr{\tilde Q -Q} - \EE\osbr{\tilde Q -Q}}^2}, 
\end{align*}
which finishes the proof of \Cref{cor:response-diff-myopic}.


\subsection{Proof of \Cref{lem:response diff-myopic-linear}}
\label{sec:proof-response-diff-myopic_linear}  
    \begin{proof}
Similar to the proof of \Cref{cor:response-diff-myopic}, we omit the state $s$ to simplify the notation. 
Note that $\tilde Q$ and $Q$ in \Cref{cor:response-diff-myopic}
correspond to $\tilde r^{\pi} = \la \phi^{\pi}, \tilde \theta\ra$ and $r^{\pi} = \la \phi^{\pi}, \theta^*$, respectively.  
Let $\EE_s $ denote the expectation taken with respect to $\nu^{\pi} (\cdot \given s)$.
Then we have 
\begin{align}
    &\EE_s\Bigsbr{\Bigrbr{\rbr{\tilde r^\pi(s, b) -r^\pi(s, b)} - \EE \sbr{\tilde r^\pi(s, b) -r^\pi(s, b)}}^2} = \bignbr{\theta^*-\tilde\theta}_{\Sigma_{ s}^{\pi, \theta}}^2  \nend
    &\quad 
    = 
    \trace\Bigrbr{ \Sigma_{ s}^{\pi, \theta^*} \cdot (\theta ^* - \tilde \theta)(\theta ^* - \tilde \theta)^\top }  
      \le {\trace\bigrbr{{\Psi}^{\dagger} \Sigma_{ s}^{\pi, \theta^*}}} \cdot \bignbr{\theta^*-\tilde\theta}_{\Psi}^2  , \label{eq:reward-covarance-bounded-by-Sigma}
\end{align}
where ${\Psi}^{\dagger}$ is the pseudo-inverse of $\Psi$.
Moreover, by Jensen's inequality,
we have 
\#
& \EE_s\Bigsbr{\Bigabr{\rbr{\tilde r^\pi(s, b) -r^\pi(s, b)} - \EE \sbr{\tilde r^\pi(s, b) -r^\pi(s, b)}} }  \label{eq:reward-abs-bounded-by-Sigma}  \\
&\quad  \leq \biggl(  \EE_s\Bigsbr{\Bigrbr{\rbr{\tilde r^\pi(s, b) -r^\pi(s, b)} - \EE \sbr{\tilde r^\pi(s, b) -r^\pi(s, b)}}^2} \biggr) \leq \sqrt{{\trace\bigrbr{{\Psi}^{\dagger} \Sigma_{ s}^{\pi, \theta^*}}}}  \cdot \bignbr{\theta^*-\tilde\theta}_{\Psi}.  \notag
\# 
Finally, combining \eqref{eq:reward-covarance-bounded-by-Sigma}, \eqref{eq:reward-abs-bounded-by-Sigma}, and 
  \Cref{cor:response-diff-myopic}, we conclude the proof. 
  We note that  a similar result also holds if we swap $\theta^*$ and $\tilde\theta$.
    \end{proof}

\subsection{Proof of \Cref{lem:MLE-formal}}
\label{sec:proof-MLE-general}
\begin{proof} 
We prove this lemma by leveraging 
 \Cref{lem:freeman-variation} with $X_t = (-\cL_{h,t}(\theta) + \cL_{h, t}(\theta^*))/2$ where $\cL_{h}^t(\theta)=-\sum_{i=1}^{t-1} \eta A_h^{\pi^i, \theta}(s_h^i,b_h^i)$. 
We choose filtration $\cF_{h, t-1}=\sigma(\tau^{1:t-1})$ where $\sigma(X)$ denotes the $\sigma$-algebra generated by $X$ and $\tau^{1:t-1}$ is just the history up to $t-1$.
Let $\cN_\rho(\Theta, \epsilon)$ be  the covering number for the $\epsilon$-covering net of $\Theta$  with respect to norm $\rho$ defined in \eqref{eq:rho for MLE}. 
Let $\Theta_\epsilon$ be  the $\epsilon$-covering net of $\Theta$. 
To simplify the notation, we define  $\iota =  \log\rbr{ H\cN_\rho(\Theta, \epsilon) / \delta}$.
Then, for all $\theta \in \Theta_{\epsilon}$, 
we have with probability $1-\delta$ for all $h\in[H], t\in[T]$ that
\begin{align}
    &\frac 1 2 \rbr{-\cL_h^t(\theta) + \cL_h^t(\theta^*)} \nend
    &\quad\le \sum_{i=1}^{t-1} \log\EE^{\pi^i   }\sbr{\sqrt{ \nu_h^{\pi^i,\theta}(\cdot\given s_h) \big / \nu_h^{\pi^i, \theta^*}(\cdot\given s_h)}}  + \log\rbr{ H\cN_\rho(\Theta, \epsilon) / \delta} \nend
    &\quad\le -  \sum_{i=1}^{t-1} \EE^{\pi^i }\Bigl [ D_\H^2\bigrbr{\nu_h^{\pi^i, \theta}(\cdot\given s_h ), \nu_h^{\pi^i, \theta^*}(\cdot\given s_h )} \Bigr ] + \log\rbr{ H\cN_\rho(\Theta, \epsilon) / \delta}, \label{eq:MLE-to-hellinger-1}
\end{align}
where the expectation is taken with respect to the true model. 
Here, the first inequality holds by applying  \Cref{lem:freeman-variation} and   taking a union bound over the  $\epsilon$-covering net. 
The second inequality holds by noting that $\log(x)\le x-1$ and by the definition of the Hellinger distance.

Meanwhile, by the definition of the distance $\rho$ in   \eqref{eq:rho for MLE},
for any $\theta, \tilde \theta \in \Theta$, we have 
\begin{align*}
&|D_\H^2(\nu_h^{\pi,\theta}, \nu_h^{\pi, \theta^*}) - D_\H^2(\nu_h^{\pi,\tilde\theta}, \nu_h^{\pi, \theta^*})|\nend
&\quad \le (D_\H^2(\nu_h^{\pi,\theta}, \nu_h^{\pi, \theta^*}) + D_\H^2(\nu_h^{\pi,\tilde\theta}, \nu_h^{\pi, \theta^*})) \cdot 
\bigabr{D_\H^2(\nu_h^{\pi,\theta}, \nu_h^{\pi, \theta^*}) - D_\H^2(\nu_h^{\pi,\tilde\theta}, \nu_h^{\pi, \theta^*})}\nend
&\quad \le 2 D_\H(\nu_h^{\pi,\tilde\theta}, \nu_h^{\pi,\theta})\nend
&\quad \le 2\rho(\theta, \tilde\theta), 
\end{align*} 
where  the second inequality holds by noting that the Hellinger distance does not exceed 1, and that the Hellinger distance satisfies the triangle inequality as a norm, and the last inequality holds by definition of $\rho$. 
Moreover, by noting that $\cL_h^t(\theta)=-\sum_{i=1}^t \eta A_h^{\pi^i, \theta}(s_h^i,b_h^i)$, we have
\begin{align*}
\bigl | \cL_h^t (\theta ) - \cL_h^t(\tilde \theta ) \bigr |
&\le \eta T \max_{i\in[t-1]}\bigabr{ A_h^{\pi^i, \theta}(s_h^i,b_h^i) - A_h^{\pi^i, \tilde\theta}(s_h^i,b_h^i)}\nend
&\le 2\eta T \max_{i\in[t-1]}\bignbr{ Q_h^{\pi^i, \theta}- Q_h^{\pi^i, \tilde\theta}}_\infty \nend
&\leq 2 T \cdot \rho(\theta, \tilde \theta),
\end{align*}
where the second inequality holds by noting that $|(V_h^{\pi, \theta} - V_h^{\pi, \tilde\theta})(s_h)| \le \onbr{Q_h^{\pi,\theta}-Q_h^{\pi,\tilde\theta}}_\infty$ by the same argument in \eqref{eq:f_2-01} and \eqref{eq:f_2-1}, and the last inequality holds by noting that $\eta\le (\gamma B_A + 1 + \eta)$.
Therefore, 
adding an extra term $3T\epsilon$ to the right-hand side of \eqref{eq:MLE-to-hellinger-1} extends the result to any $\theta\in\Theta$ by definition of the covering net $\Theta_\epsilon$.
We thus obtain for all $\theta\in\Theta, h\in[H], t\in[T]$ with probability $1-\delta$,
\begin{align}
    \frac 1 2 \rbr{-\cL_h^t(\theta) + \cL_h^t(\theta^*)} 
    &\le -  \sum_{i=1}^{t-1} \EE^{\pi^i, \theta^*}D_\H^2\rbr{\nu_h^{\pi^i, \theta}(\cdot\given s_h^i), \nu_h^{\pi^i, \theta^*}(\cdot\given s_h^i)} \nend
    &\qquad + \log\rbr{ H\cN_\rho(\Theta, \epsilon) / \delta} + 3T\epsilon. \label{eq:KL-2-D_H-online}
\end{align}
In the sequel, we take $\epsilon=T^{-1}$ and  let $\iota= \log\rbr{ H\cN_\rho(\Theta, T^{-1}) / \delta}+ 3$.
Now, we plug in $\hat\theta_{h,\MLE}=\argmin_{\theta'\in\Theta} \cL_h^t(\theta')$ in \eqref{eq:KL-2-D_H-online}
and obtain by the nonnegativity of the Hellinger distance that
\begin{align*}
    \cL_h^t(\theta^*) \le \inf_{\theta'\in\Theta} \cL_h^t(\theta' ) + 2\iota \le  \cL_h^t(\hat\theta_{h,\MLE}) + 2\iota  , 
\end{align*}
which guarantees that our confidence set is indeed valid by letting 
$$\beta \ge  2\iota = 2\log(e^3 H\cdot \cN_\rho(\Theta, T^{-1})/\delta). $$
Next, we show that our confidence set is also accurate.
For \eqref{eq:KL-2-D_H-online}, we have that
\begin{align}
    \sum_{i=1}^{t-1} \EE^{\pi^i, \theta^*}D_\H^2\rbr{\nu_h^{\pi^i, \theta}(\cdot\given s_h^i), \nu_h^{\pi^i, \theta^*}(\cdot\given s_h^i)} 
    &\le  \frac 1 2\rbr{\cL_h^t(\theta) - \cL_h^t(\theta^*)} +\iota\nend
    &\le \frac 1 2\rbr{\cL_h^t(\theta) - \cL_h^t(\theta^*)+\beta}, \label{eq:MLE-to-hellinger-2}
\end{align}
Now, if $\theta\in\CI_{h,\Theta}^t(\beta)$, it follows directly from \eqref{eq:MLE-to-hellinger-2} that
\begin{align*}
    \sum_{i=1}^{t-1} \EE^{\pi^i, \theta^*}D_\H^2\rbr{\nu_h^{\pi^i, \theta}(\cdot\given s_h^i), \nu_h^{\pi^i, \theta^*}(\cdot\given s_h^i)} \le \beta, 
\end{align*}
which shows that our confidence set is also valid and gives \eqref{eq:MLE-guarantee-hellinger-2}.
We next show how to derive the bound for the Q function. Invoking \Cref{lem:D_H-2-A^2}, we have that
\begin{align}
    8 D_\H^2\rbr{\nu_h^{\pi, \hat\theta}(\cdot\given s_h), \nu_h^{\pi, \theta^*}(\cdot\given s_h)} 
    &\ge \rbr{\frac{\eta}{1+\eta B_A}}^2\cdot \inp[\Big]{\nu_h^{\pi, \theta^*}}{ \orbr{A_h^{\pi, \hat\theta}-A_h^{\pi, \theta^*}}^2} \nend
    &\ge \rbr{\frac{\eta}{1+\eta B_A}}^2\cdot \EE_{s_h}^{\pi, \theta^*}\rbr{\bigrbr{\EE_{s_h, b_h}^{\pi, \theta^*} - \EE_{s_h}^{\pi, \theta^*}}\osbr{A_h^{\pi, \hat\theta}-A_h^{\pi, \theta^*}}}^2\nend
    &= \rbr{\frac{\eta}{1+\eta B_A}}^2\cdot \EE_{s_h}^{\pi, \theta^*}\Bigrbr{\orbr{\EE_{s_h, b_h}^{\pi, \theta^*} - \EE_{s_h}^{\pi, \theta^*}}\osbr{Q_h^{\pi, \hat\theta}-Q_h^{\pi, \theta^*}}}^2, \label{eq:hellinger-2-Q}
\end{align}
where the second inequality follows from the Jensen's inequality, and the last inequality holds by invoking \eqref{eq:A diff-1} in \Cref{lem:AQV-func diff}. One can also swap $\theta^*$ and $\hat\theta$ by the exchangeability of the Hellinger distance and obtain another version.  Note that $C_\eta = \eta^{-1}+B_A$. 
Plugging \eqref{eq:hellinger-2-Q} with $C_\eta$ into the previous accuracy guarantees gives \eqref{eq:MLE_guarantee_Q}.
Therefore, we have proved \eqref{eq:MLE-guarantee-hellinger-2} and \eqref{eq:MLE_guarantee_Q}. 
For deriving \eqref{eq:MLE-guarantee-hellinger-1} and \eqref{eq:MLE_guarantee_Q-1}, we just change our filtration to $\cF_{h, t-1} = \sigma((s_h^i, b_h^i)_{i\in[t-1]}, s_h^t)$ and everything follows. 

Lastly, we prove the guarantee in \eqref{eq:MLE-guarantee-Q-3}. 
We use \eqref{eq:MLE_guarantee_Q-1} with $\theta'$ replaced by $\theta^*$ and for all $\theta\in\CI_{h, \Theta}^t(\beta)$, 
\begin{align*}
    \sum_{i=1}^{t-1} {\Var_{s_h^i}^{\pi^i, \theta^*} \bigsbr{Q_h^{\pi^i, \theta}(s_h, b_h) - Q_h^{\pi^i, \theta^*}(s_h, b_h)}} \le 4 C_\eta^2 \rbr{\cL_h^t(\theta) - \cL_h^t(\theta^*)+ \beta} \le 8 C_\eta^2 \beta, 
\end{align*}
which is equivalent to saying that for all $h\in[H], \theta\in\CI_{h, \Theta}^t(\beta)$,
\begin{align}
    \sum_{i=1}^{t-1} 
    \EE_{s_h^i}^{\pi^i,\theta^*}\rbr{\rbr{Q_h^{\pi^i, \theta} - Q_h^{\pi^i, \theta*}}(s_h^i, b_h^i)
    -\EE_{s_h^i}^{\pi^i,\theta^*}\sbr{ \rbr{Q_h^{\pi^i, \theta} - Q_h^{\pi^i, \theta^*}}(s_h, b_h)}}^2 \le 8 C_\eta^2 \beta. \label{eq:MLE-Q-ub-1}
\end{align}
Recall the covering net $\Theta_\epsilon$ we constructed before. For all $\theta\in\Theta_\epsilon \cap \CI_{h,\Theta}^t(\beta), h\in[H]$, and a given $t\in[T]$, we have  by a standard martingale concentration in \Cref{cor:martigale concentration} that with probability at least $1-2\delta$,
\begin{align*}
    &\sum_{i=1}^{t-1} 
    \rbr{\rbr{Q_h^{\pi^i, \theta} - Q_h^{\pi^i, \theta^*}}(s_h^i, b_h^i)
    -\EE_{s_h^i}^{\pi^i,\theta^*}\sbr{ \rbr{Q_h^{\pi^i, \theta} - Q_h^{\pi^i, \theta^*}}(s_h, b_h)}}^2 
    \nend
    &\quad \le \frac 3 2\sum_{i=1}^{t-1}\EE_{s_h^i}^{\pi^i,\theta^*} 
    \rbr{\rbr{Q_h^{\pi^i, \theta} - Q_h^{\pi^i, \theta^*}}(s_h^i, b_h)
    -\EE_{s_h^i}^{\pi^i,\theta^*}\sbr{ \rbr{Q_h^{\pi^i, \theta} - Q_h^{\pi^i, \theta^*}}(s_h, b_h)}}^2 \nend
    &\qqquad + 32 B_A^2 \log\rbr{2H \cN_\rho(\Theta, \epsilon)\delta^{-1} }\nend
    &\quad \le 12 C_\eta^2 \beta  +32 B_A^2 \log\rbr{2H \cN_\rho(\Theta, \epsilon)\delta^{-1} }\le 28 C_\eta^2 \beta 
\end{align*}
where the second inequality follows from \eqref{eq:MLE-Q-ub-1}. 
Moreover, we note that
\begin{align*}
    &\rbr{\rbr{Q_h^{\pi^i, \theta} - Q_h^{\pi^i, \theta^*}}(s_h^i, b_h)
    -\EE_{s_h^i}^{\pi^i,\theta^*}\sbr{ \rbr{Q_h^{\pi^i, \theta} - Q_h^{\pi^i, \theta^*}}(s_h, b_h)}}^2 \nend
    &\qqquad - \rbr{\rbr{Q_h^{\pi^i, \tilde\theta} - Q_h^{\pi^i, \theta*}}(s_h^i, b_h)
    -\EE_{s_h^i}^{\pi^i,\theta^*}\sbr{ \rbr{Q_h^{\pi^i, \tilde\theta} - Q_h^{\pi^i, \theta^*}}(s_h, b_h)}}^2\nend
    &\quad \le 8 \cdot \max_{\pi\in\Pi,\theta\in\Theta}\bignbr{Q_h^{\pi,\theta}}_\infty
    \cdot 2 \bignbr{Q_h^{\pi^i, \theta} -  Q_h^{\pi^i, \tilde\theta}}_\infty \le 16 B_A \rho(\theta, \tilde\theta),
\end{align*}
where the last inequality holds by noting that $B_A$ upper bounds $\max_{\theta\in\Theta, \pi\in\Pi, h\in[H]}\onbr{Q_h^{\pi, \theta}}_\infty$ and using the definition of $\rho$. 
As a result, for all $\theta\in\CI_\Theta(\beta), h\in[H]$, and a given $t\in[T]$, we conclude that with probability at least $1-2\delta$,
\begin{align*}
    &\sum_{i=1}^{t-1} 
    \rbr{\rbr{Q_h^{\pi^i, \theta} - Q_h^{\pi^i }}(s_h^i, b_h^i)
    -\EE_{s_h^i}^{\pi^i,\theta^*}\sbr{ \rbr{Q_h^{\pi^i, \theta} - Q_h^{\pi^i, \theta^*}}(s_h, b_h)}}^2  \le 28 C_\eta^2 \beta  + 16 B_A.
\end{align*}
Replace $\delta$ by $\delta/2$, we have for all $\theta\in\CI_\Theta(\beta), h\in[H]$ and a given $t\in[T]$ with probability at least $1-\delta$ that 
\begin{align*}
    &\sum_{i=1}^{t-1} 
    \rbr{\rbr{Q_h^{\pi^i, \theta} - Q_h^{\pi^i }}(s_h^i, b_h^i)
    -\EE_{s_h^i}^{\pi^i,\theta^*}\sbr{ \rbr{Q_h^{\pi^i, \theta} - Q_h^{\pi^i, \theta^*}}(s_h, b_h)}}^2 
    \nend
    &\quad \le 28 C_\eta^2 \beta + 56 C_\eta^2\log 2 +16 B_A =  \cO(C_\eta^2 \beta).
\end{align*}
which 
finishes the proof of \Cref{lem:MLE-formal}.
\end{proof}

\subsection{Proof of \Cref{lem:MLE-indep-data}}\label{sec:proof-MLE-indep-data}

Here, we show the guarantee for the MLE with independently collected dataset.
Since each trajectory is independently collected, we are able to use the Bernstein inequality for indepedent random variables $Z_h^i = D_\H^2\orbr{\nu_h^{\pi^i, \theta}(\cdot\given s_h^i), \nu_h^{\pi^i, \theta^*}(\cdot\given s_h^i)}$, 
\#
    \abr{\frac 1 T \sum_{i=1}^T Z_h^i - \EE_\cD\sbr{Z_h^i}} &\le \sqrt{\frac{4\sum_{i=1}^T\Var[Z_h^i]\log(2\delta^{-1})}{T^2}} + \frac{4\log(2\delta^{-1})}{3T} \nend
    &\le \sqrt{\frac{4\sum_{i=1}^T \EE_\cD[Z_h^i]\log(2\delta^{-1})}{T^2}} + \frac{4\log(2\delta^{-1})}{3T}\nend
    &\le \frac{1}{2T}\sum_{i=1}^T \EE_\cD [Z_h^i] + \frac{2 \log(2\delta^{-1})}{T} + \frac{4\log(2\delta^{-1})}{3T},\notag
\#
where the second inequality holds by noting that $\Var[Z_h^i] \le \EE_\cD[(Z_h^i)^2]\le \EE_\cD[Z_h^i]$ by using the property that Hellinger distance is always upper bounded by $1$. We now conclude by further taking a union bound over $h\in[H]$ and $\theta\in\Theta$ that
\#
\frac{1}{T} \sum_{i=1}^T \EE_\cD[Z_h^i] &\le \frac{2}{T} \sum_{i=1}^T Z_h^i  + \frac{8\log(2H \cN_\rho(\Theta, \epsilon)\delta^{-1})}{T} + 6\epsilon\nend
&\le \frac{2}{T} \sum_{i=1}^T Z_h^i  + \frac{8\log(2eH \cN_\rho(\Theta, T^{-1})\delta^{-1})}{T},\notag
\#
where the last inequality holds by taking $\epsilon=T^{-1}$. Plug in the definition of $Z_h^i$, we have
\begin{align*}
    \frac 1 T\sum_{i=1}^T\EE_\cD\sbr{D_\H^2\rbr{\nu_h^{\pi^i, \theta}(\cdot\given s_h^i), \nu_h^{\pi^i, \theta^*}(\cdot\given s_h^i)}} &\le \frac{2}{T} \sum_{i=1}^T D_\H^2\rbr{\nu_h^{\pi^i, \theta}(\cdot\given s_h^i), \nu_h^{\pi^i, \theta^*}(\cdot\given s_h^i)}  \nend
    &\qquad + \frac{8\log(2eH \cN_\rho(\Theta, T^{-1})\delta^{-1})}{T}.
\end{align*}
Using \eqref{eq:MLE-guarantee-hellinger-1} in \Cref{lem:MLE-formal} for any $\theta\in\cC_{\Theta}(\beta)$, we give 
\begin{align}
    \sum_{i=1}^{T} D_\H^2\bigrbr{\nu_h^{\pi^i, \theta}(\cdot\given s_h^i), \nu_h^{\pi^i, \theta^*}(\cdot\given s_h^i)} 
        &\le  \frac 1 2\rbr{\cL_h(\theta) - \cL_h(\theta^*)} + \log\rbr{\frac{eH\cN_\rho(\Theta, T^{-1})}{\delta}} \nend
        &\le \frac 1 2\rbr{\cL_h(\theta) - \inf_{\theta'\in\Theta}\cL_h(\theta')} + \log\rbr{\frac{eH\cN_\rho(\Theta, T^{-1})}{\delta}}\nend
        &\le \frac 3 2 \beta, \label{eq:OffGM-nu-hellinger-1}
\end{align}
where the last inequality is just by definition of $\CI_\Theta(\beta)$ in \Cref{eq:behavior_model_confset-1}. Therefore, 
\begin{align*}
    \sum_{i=1}^T\EE_\cD\sbr{D_\H^2\rbr{\nu_h^{\pi^i, \theta}(\cdot\given s_h^i), \nu_h^{\pi^i, \theta^*}(\cdot\given s_h^i)}}  \le 3\beta + {8\log(2eH \cN_\rho(\Theta, T^{-1})\delta^{-1})} \le 11\beta, 
\end{align*}
with probability at least $1-2\delta$ for all $h\in[H]$ and $\theta\in\CI_\Theta(\beta)$. The validity guarantee is already shown in \Cref{lem:MLE-formal}.
We complete the proof of \Cref{lem:MLE-indep-data}.


\subsection{Proof of \Cref{lem:leader-bellman-loss}}\label{sec:proof-leader-bellman-loss}
In the following proof, we always consider the expectation to be taken with respect to the data generating distribution.
We first prove the following concentration result: for any $h\in[H]$ and any $y=(\theta_{h+1}, \pi_{h+1}, U_{h+1}, U_{h})\in \cY_h = \Theta_{h+1}\times \Pi_{h+1}\times \cU^2$, 
it holds with probability at least $1-\delta$ that 
      \begin{align}
        &\abr{T\EE[(U_{h} - \TT_{h}^{\pi,\theta}U_{h + 1})^2] - \ell_{h}(U_{h}, U_{h + 1}, \theta, \pi) + \ell_{h}(\TT_{h}^{\pi,\theta}U_{h + 1}, U_{h + 1}, \theta, \pi)} \notag\\
        &\qquad \le \epsilon_S + \frac{T}{2} \EE[(U_{h} - \TT_{h}^{\pi,\theta}U_{h + 1})^2].\label{eq:cY-confset-1}
      \end{align}
     where
    $
    {110 B_U^2\cdot\log(H \max_{h\in[H]}\cN_\rho(\cY_h, T^{-1})\delta^{-1}) } \cdot {T^{-1}} + (45 B_U^2 + 60 B_U )T^{-1}
    $ and $B_U=H$ is the upper bound for the function class $U$.

\paragraph{Concentration. }
Our proof is an adaptation of Lemma D.2 in \citep{lyu2022pessimism}, although we simplify a little bit by directly using the covering number for a joint class $\cY_h = \Pi_{h+1}\times\Theta_{h+1}\times\cU^2$. We take an $\epsilon$-covering net $\cY_\epsilon$ for $\cY$ with respect to distance $\rho$ specified by \eqref{eq:rho-cY}. 
We first use \Cref{lem:bernstein},  
where we take 
\begin{align*}
    {Z_h^i} &= \ell_{h}(U'_{h}, U_{h + 1}, \theta, \pi) - \ell_{h}(\TT_{h}^{\pi,\theta}U_{h + 1}, U_{h + 1}, \theta, \pi)\nend
    & = \rbr{U_h(s_h^i, a_h^i, b_h^i) - u_h^i -  T_{h+1}^{\pi,\theta} U_{h+1}(s_{h+1}^i)}^2  - \rbr{\TT_h^{\pi, \theta} U_{h+1}(s_h^i, a_h^i, b_h^i) - u_h^i - T_{h+1}^{\pi,\theta} U_{h+1}(s_{h+1}^i)}^2.
\end{align*}
Here, we recall the definition of $\TT_h^{\pi}$ given by \eqref{eq:bellman_operator_leader}.
One can verify that $|{Z_h^i}|\le 9B_U^2$ where $B_U$ bounds both the leader's reward and the value function class $\cU$.
We first calculate the expectation of $Z_h^i$, 
\begin{align*}
    \EE[{Z_h^i}] 
    &= \EE \bigg[\EE_{s_h^i, a_h^i, b_h^i} \Big[\rbr{U_h(s_h^i, a_h^i, b_h^i) - u_h^i -  T_{h+1}^{\pi,\theta} U_{h+1}(s_{h+1}^i)}^2  \nend
    &\qquad - \rbr{\TT_h^{\pi, \theta} U_{h+1}(s_h^i, a_h^i, b_h^i) - u_h^i - T_{h+1}^{\pi,\theta} U_{h+1}(s_{h+1}^i)}^2\Big]\bigg]\nend
    & = \EE\bigg[\Bigrbr{U_h(s_h^i, a_h^i, b_h^i)- \TT_h^{\pi, \theta} U_{h+1}(s_h^i, a_h^i, b_h^i)}\nend
    &\qquad \cdot \EE_{s_h^i, a_h^i, b_h^i}\sbr{U_h(s_h^i, a_h^i, b_h^i) + \TT_h^{\pi, \theta} U_{h+1}(s_h^i, a_h^i, b_h^i)- 2u_h^i - 2T_{h+1}^{\pi,\theta} U_{h+1}(s_{h+1}^i) }\bigg]\nend
    & = \EE\sbr{\Bigrbr{U_h(s_h^i, a_h^i, b_h^i)- \TT_h^{\pi, \theta} U_{h+1}(s_h^i, a_h^i, b_h^i)}^2}, 
\end{align*}
where $\EE_{x}[\cdot]$ is a short hand of $\EE[\cdot\given x]$ and the expectation is taken with respect to the data generating distribution. Here, the second equality holds by the law of total expectation, and the third equality holds by noting that $\EE_{s_h^i, a_h^i, b_h^i}\osbr{u_h^i  + T_h^{\pi,\theta} U_{h+1}(s_{h+1}^i)} = \TT_h^{\pi, \theta} U_{h+1} (s_h^i, a_h^i, b_h^i)$.
Next, we calculate the variance, 
\begin{align*}
    \Var[{Z_h^i}] &\le \EE[{Z_h^i}^2] \nend
    &\le \EE\bigg[\Bigrbr{U_h(s_h^i, a_h^i, b_h^i)- \TT_h^{\pi, \theta} U_{h+1}(s_h^i, a_h^i, b_h^i)}^2\nend
    &\qquad \cdot \EE_{s_h^i, a_h^i, b_h^i}\sbr{\rbr{U_h(s_h^i, a_h^i, b_h^i) + \TT_h^{\pi, \theta} U_{h+1}(s_h^i, a_h^i, b_h^i)- 2u_h^i - 2T_{h+1}^{\pi,\theta} U_{h+1}(s_{h+1}^i)}^2 } \bigg]\nend
    &\le 49 B_U^2 \EE\sbr{\Bigrbr{U_h(s_h^i, a_h^i, b_h^i)- \TT_h^{\pi, \theta} U_{h+1}(s_h^i, a_h^i, b_h^i)}^2} = 49 B_U^2 \EE[{Z_h^i}].
\end{align*}
Now, by \Cref{lem:bernstein}, we have for each $y\in\cY_\epsilon$ that
\begin{align*}
    \abr{\frac 1 T \sum_{i=1}^T {Z_h^i} - \EE\sbr{{Z_h^i}}} &\le \frac{1}{2T}\sum_{i=1}^T \EE [{Z_h^i}] + \frac{110 B_U^2\cdot\log(2\delta^{-1})}{T}.
\end{align*}
Now, we extend the result to $\cY$, where we notice that for any two $y, \tilde y\in\cY$ such that $\rho(y, \tilde y) \le \epsilon$, 
\begin{align*}
    &\bigrbr{U_h(s_h^i, a_h^i, b_h^i) - u_h^i -  T_{h+1}^{\pi,\theta} U_{h+1}(s_{h+1}^i)}^2  - \bigrbr{\TT_h^{\pi, \theta} U_{h+1}(s_h^i, a_h^i, b_h^i) - u_h^i - T_{h+1}^{\pi,\theta} U_{h+1}(s_{h+1}^i)}^2 \nend
    &\qquad - \rbr{\bigrbr{\tilde U_h(s_h^i, a_h^i, b_h^i) - u_h^i -  T_{h+1}^{\tilde\pi,\tilde\theta} \tilde U_{h+1}(s_{h+1}^i)}^2  - \bigrbr{\TT_h^{\tilde\pi, \tilde\theta} \tilde U_{h+1}(s_h^i, a_h^i, b_h^i) - u_h^i - T_{h+1}^{\tilde\pi,\tilde\theta} \tilde U_{h+1}(s_{h+1}^i)}^2}\nend
    &\quad \le 6 B_U \rbr{\onbr{U-\tilde U}_\infty + \bignbr{(T_{h+1}^{\pi,\theta}-T_{h+1}^{\tilde\pi,\tilde\theta} )\tilde U_{h+1}}_\infty + \bignbr{T_{h+1}^{\pi,\theta}(U_{h+1} - \tilde U_{h+1})}_\infty}\nend
    &\qquad + 6 B_U \cdot 2\rbr{\bignbr{(T_{h+1}^{\pi,\theta}-T_{h+1}^{\tilde\pi,\tilde\theta} )\tilde U_{h+1}}_\infty + \bignbr{T_{h+1}^{\pi,\theta}(U_{h+1} - \tilde U_{h+1})}}\nend
    &\quad\le 6 B_U(2\epsilon + B_U \epsilon ) + 12 B_U(B_U\epsilon + \epsilon)\nend
    &\quad\le (18 B_U^2 + 24 B_U)\epsilon, 
\end{align*}
where the second inequality follows from the definition of the covering net with respect to distance $\rho$ defined in \eqref{eq:rho-cY}. 
We obtain with probability at least $1-\delta$ that for any $y\in\cY_h$, $h\in[H]$, 
\begin{align*}
    \abr{\frac 1 T \sum_{i=1}^T {Z_h^i} - \EE\sbr{{Z_h^i}}} &\le \inf_{\epsilon>0} \frac{1}{2T}\sum_{i=1}^T \EE [{Z_h^i}] + \frac{110 B_U^2\cdot\log(H \cN_\rho(\cY_h, \epsilon)\delta^{-1})}{T} +2.5\cdot (18B_U^2 + 24 B_U)\epsilon\nend
    &\le \frac{1}{2T}\sum_{i=1}^T \EE [{Z_h^i}] + {110 B_U^2\cdot\log( H \cN_\rho(\cY_h, T^{-1})\delta^{-1}) } \cdot {T^{-1}} + (45 B_U^2 + 60 B_U )T^{-1}\nend
    & = \frac 1 T \rbr{\epsilon_S + \frac{T}{2} \EE[(U_{h} - \TT_{h}^{\pi,\theta}U_{h + 1})^2]}, 
\end{align*}
where $\epsilon_S = {110 B_U^2\cdot\log(H \max_{h\in[H]}\cN_\rho(\cY_h, T^{-1})\delta^{-1}) }  + (45 B_U^2 + 60 B_U )$, 
which proves our claim in \eqref{eq:cY-confset-1}.

\paragraph{Guarantee of the Confidence Set $\CI_{\cU}^{\pi,\theta}(\beta)$.}
We give a brief proof for the validity and the accuracy of the confidence set. For any $U_{h+1}\in\cU, \theta\in\Theta, \pi\in\Pi, h\in[H]$, on the one hand,
\begin{align}
    \ell_h(U_h, U_{h+1},\theta, \pi) - \inf_{U_h'\in\cU} \ell_h(U_h', U_{h+1}, \theta,\pi) \le \epsilon_S - \frac{T}{2} \EE[\orbr{U_{h} - \TT_{h}^{\pi,\theta}U_{h + 1}}^2], \label{eq:Bellman-loss-guarantee-1}
\end{align}
on the other hand,
\begin{align}
    \ell_h(U_h, U_{h+1},\theta, \pi) - \inf_{U_h'\in\cU} \ell_h(U_h', U_{h+1}, \theta,\pi) 
    &\ge \ell_h(U_h, U_{h+1},\theta, \pi) - \ell_h(\TT_h^{\pi, \theta} U_{h+1},  U_{h+1}, \theta,\pi) \nend
    &\ge - \epsilon_S  + \frac{T}{2} \EE[\orbr{U_{h} - \TT_{h}^{\pi,\theta}U_{h + 1}}^2], \label{eq:Bellman-loss-guarantee-2}
\end{align}
where the first inequality holds by the completeness assumption. 
For \eqref{eq:Bellman-loss-guarantee-1}, we plug in $U=U^{\pi,\theta}$ (realizability) and obtain 
$$\ell_h(U^{\pi,\theta}_h, U^{\pi,\theta}_{h+1},\theta, \pi) - \inf_{U_h'\in\cU} \ell_h(U_h', U^{\pi,\theta}_{h+1}, \theta,\pi) \le \epsilon_S.$$
Therefore, by having $\beta \ge \epsilon_S$, we have $U^{\pi,\theta}\in\CI_\cU^{\pi,\theta}(\beta)$. For the second one in \eqref{eq:Bellman-loss-guarantee-2}, we plug in any $U\in\CI_{\cU}^{\pi,\theta}(\beta)$ and obtain $\EE[\|U_{h} - \TT_{h}^{\pi,\theta}U_{h + 1}\|^2]\le T^{-1}\cdot (2\beta +2\epsilon_S)\le 4\beta T^{-1}$. Hence, we complete our proof of \Cref{lem:leader-bellman-loss}.

\subsection{Proof of \Cref{lem:CI-U-online}}\label{sec:proof-CI-U-online}
Our proof follows a similar scheme as in the proof of \Cref{lem:leader-bellman-loss}.

\paragraph{Concentration.}
We first take
\begin{align*}
    {Z_h^i} &= \ell_{h}(U'_{h}, U_{h + 1}, \theta_{h+1}, \pi) - \ell_{h}(\TT_{h}^{*,\theta}U_{h + 1}, U_{h + 1}, \theta_{h+1}, \pi)\nend
    & = \rbr{U_h(s_h^i, a_h^i, b_h^i) - u_h^i -  T_{h+1}^{*, \theta} U_{h+1}(s_{h+1}^i)}^2  - \rbr{\TT_h^{*, \theta} U_{h+1}(s_h^i, a_h^i, b_h^i) - u_h^i - T_{h+1}^{*, \theta} U_{h+1}(s_{h+1}^i)}^2.
\end{align*}
For the online setting, we have $(s_h^i, a_h^i, b_h^i, s_{h+1}^i)$ adapted to some filtration $\cF_h^i$. One choice of the filtration is $\cF_h^j=\sigma\rbr{\tau^{1:j-1}}$. Another choice of the filtration is $\cF_h^j=\sigma\rbr{(s_h^j, a_h^j, b_h^j), (s_h^i, a_h^i, b_h^i, s_{h+1}^i)_{i\in[j-1]}}$, where $\sigma(X)$ is the sigma-algebra of $X$.
They will both work for our proof.
We let $\EE^{i}_{x}[\cdot]$ be a short hand of $\EE^{i}[\cdot\given x, \cF_h^i]$ in the following proof.
We first calculate the expectation, 
\begin{align*}
    \EE^{i}[{Z_h^i}] 
    &= \EE^{i} \bigg[\EE^{i}_{s_h^i, a_h^i, b_h^i} \Big[\rbr{U_h(s_h^i, a_h^i, b_h^i) - u_h^i -  T_{h+1}^{*, \theta} U_{h+1}(s_{h+1}^i)}^2  \nend
    &\qquad - \rbr{\TT_h^{*, \theta} U_{h+1}(s_h^i, a_h^i, b_h^i) - u_h^i - T_{h+1}^{*, \theta} U_{h+1}(s_{h+1}^i)}^2\Big]\bigg]\nend
    & = \EE^{i}\bigg[\Bigrbr{U_h(s_h^i, a_h^i, b_h^i)- \TT_h^{*, \theta} U_{h+1}(s_h^i, a_h^i, b_h^i)}\nend
    &\qquad \cdot \EE^{i}_{s_h^i, a_h^i, b_h^i}\sbr{U_h(s_h^i, a_h^i, b_h^i) + \TT_h^{*, \theta} U_{h+1}(s_h^i, a_h^i, b_h^i)- 2u_h^i - 2T_{h+1}^{*, \theta} U_{h+1}(s_{h+1}^i) }\bigg]\nend
    & = \EE^{i}\sbr{\Bigrbr{U_h(s_h^i, a_h^i, b_h^i)- \TT_h^{*, \theta} U_{h+1}(s_h^i, a_h^i, b_h^i)}^2}, 
\end{align*}
where the second equality holds by the law of total expectation, and the third equality holds by noting that $\EE^{i}_{s_h^i, a_h^i, b_h^i}\osbr{u_h^i  + T_{h+1}^{*, \theta} U_{h+1}(s_{h+1}^i)} = \TT_h^{*, \theta} U_{h+1} (s_h^i, a_h^i, b_h^i)$.
Next, we calculate the variance, 
\begin{align*}
    \Var^{i}[{Z_h^i}^2] &\le \EE^{i}[{Z_h^i}^2]\nend
    &\le \EE^{i}\bigg[\Bigrbr{U_h(s_h^i, a_h^i, b_h^i)- \TT_h^{*, \theta} U_{h+1}(s_h^i, a_h^i, b_h^i)}^2\nend
    &\qquad \cdot \EE^{i}_{s_h^i, a_h^i, b_h^i}\sbr{\rbr{U_h(s_h^i, a_h^i, b_h^i) + \TT_h^{*, \theta} U_{h+1}(s_h^i, a_h^i, b_h^i)- 2u_h^i - 2T_{h+1}^{*, \theta} U_{h+1}(s_{h+1}^i)}^2 } \bigg]\nend
    &\le 49 B_U^2 \EE^{i}\sbr{\Bigrbr{U_h(s_h^i, a_h^i, b_h^i)- \TT_h^{*, \theta} U_{h+1}(s_h^i, a_h^i, b_h^i)}^2} = 49 B_U^2 \EE^{i}[{Z_h^i}].
\end{align*}
Also, one can verify that $|{Z_h^i}|\le 9B_U^2$ where $B_U$ bounds both the leader's reward and the value function class $\cU$.
We next take a $\epsilon$-covering of the class $\cZ_h=\cU^2\times\Theta_{h+1}$ with respect to the following distance defined in \eqref{eq:rho-cZ},
\begin{align*}
    &\rho\orbr{z, \tilde z}  = \max_{h\in[H]}\cbr{\bignbr{U_h-\tilde U_h}_\infty, \bignbr{ T_{h+1}^{*,\theta} U_{h+1} (\cdot) -  T_{h+1}^{*,\tilde\theta} \tilde U_{h+1} (\cdot)}_\infty }, 
\end{align*} 
We invoke the Freedman inequality \Cref{lem:freedman} for the martingale sequence $Z_h^i -\EE^i[Z_h^i]$, which says that for all $t\in [T], h\in[H], z\in\cZ_\epsilon$, it holds with probability at least $1-\delta$
\begin{align*}
    \sum_{i=1}^t \rbr{Z_h^i - \EE^i[Z_h^i]}
    &\le \frac{\lambda(e-2)}{9 B_U^2} \sum_{i=1}^t \EE^i\sbr{\rbr{Z_h^i - \EE^i[Z_h^i]}^2} + 9 \lambda^{-1} B_U^2\log(TH\cN_\rho(\cZ_h, \epsilon)\delta^{-1})\nend
    &\le \frac{\lambda(e-2) 49 }{9 } \sum_{i=1}^t \EE^i\sbr{Z_h^i} + 9 \lambda^{-1} B_U^2\log(TH\cN_\rho(\cZ_h, \epsilon)\delta^{-1}), \quad \forall \lambda\in(0, 1).
\end{align*}
Here, we plug in $\lambda = 9/98(e-2)$, $\epsilon=T^{-1}$ and notice that the above inequality also holds for $-Z_h^i + \EE^i[Z_h^i]$, which gives us for all $z\in\cZ_\epsilon, h\in[H], t\in[T]$ that
\begin{align}\label{eq:leader-Bellman-1}
    \abr{\sum_{i=1}^t \rbr{Z_h^i - \EE^i[Z_h^i]}} \le \frac 1 2 \sum_{i=1}^t \EE^i\sbr{Z_h^i} + c B_U^2\log(TH\cN_\rho(\cZ_h, T^{-1})\delta^{-1}), 
\end{align}
where we plug in $\epsilon=T^{-1}$ and $c=98(e-2)$ should be a universal constant. 
Next, we notice that
\begin{align*}
    &\bigrbr{U_h(s_h^i, a_h^i, b_h^i) - u_h^i -  T_{h+1}^{*,\theta} U_{h+1}(s_{h+1}^i)}^2  - \bigrbr{\TT_h^{*, \theta} U_{h+1}(s_h^i, a_h^i, b_h^i) - u_h^i - T_{h+1}^{*,\theta} U_{h+1}(s_{h+1}^i)}^2 \nend
    &\qquad - \rbr{\bigrbr{\tilde U_h(s_h^i, a_h^i, b_h^i) - u_h^i -  T_{h+1}^{*,\tilde\theta} \tilde U_{h+1}(s_{h+1}^i)}^2  - \bigrbr{\TT_h^{*, \tilde\theta} \tilde U_{h+1}(s_h^i, a_h^i, b_h^i) - u_h^i - T_{h+1}^{*,\tilde\theta} \tilde U_{h+1}(s_{h+1}^i)}^2}\nend
    &\quad \le 6 B_U \rbr{\onbr{U-\tilde U}_\infty + \bignbr{(T_{h+1}^{*,\theta}-T_{h+1}^{*,\tilde\theta} )\tilde U_{h+1}}_\infty + \bignbr{T_{h+1}^{*,\theta}(U_{h+1} - \tilde U_{h+1})}_\infty}\nend
    &\qquad + 6 B_U \cdot 2\rbr{\bignbr{(T_{h+1}^{*,\theta}-T_{h+1}^{*,\tilde\theta} )\tilde U_{h+1}}_\infty + \bignbr{T_{h+1}^{*,\theta}(U_{h+1} - \tilde U_{h+1})}}\nend
    &\quad\le 6 B_U(2\epsilon + B_U \epsilon ) + 12 B_U(B_U\epsilon + \epsilon) \le (18 B_U^2 + 24 B_U)\epsilon, 
\end{align*}
Now, we let $\epsilon_S=c B_U^2 \allowbreak \log(TH\cN_\rho(\cZ_h, T^{-1})\delta^{-1}) + (45 B_U^2 + 60 B_u)$ and extend the result in \eqref{eq:leader-Bellman-1} to the whole class $\cY$, 
\begin{align*}
    \abr{\sum_{i=1}^t \rbr{Z_h^i - \EE^i[Z_h^i]}} 
    &\le \frac 1 2 \sum_{i=1}^t \EE^i\sbr{Z_h^i} + c B_U^2\log(TH\cN_\rho(\cZ, T^{-1})\delta^{-1}) + 2.5\cdot(18 B_U^2 + 24 B_U)\nend
    &\le \frac 1 2 \sum_{i=1}^t \EE^i\sbr{Z_h^i} + \epsilon_S.
\end{align*}
the following argument follows exactly the same as \eqref{eq:Bellman-loss-guarantee-1} and \eqref{eq:Bellman-loss-guarantee-2}, 
where we use the realizability assumption that $U^{*,\theta}\in \cU$ and the completeness assumption that there exists $U'\in\cU$ such that $U'=\TT_h^{*, \theta} U$ for any $U\in\cU, \theta\in\Theta$. 
We finish our proof of \Cref{lem:CI-U-online}.

\subsection{Proof of \Cref{lem:1st-ub}}
\label{sec:proof-1st-ub}
Recall by definition, 
\begin{align*}
    \tilde \Delta^{(1)}_h(s_h, b_h) &=  \rbr{\EE_{s_h, b_h} -\EE_{s_h}}\Biggsbr{\sum_{l=h}^H \gamma^{l-h}\underbrace{\rbr{\tilde Q_l - r_l^\pi - \gamma P_l^\pi \tilde V_{l+1}}(s_l, b_l)}_{\ds\text{Follower's Bellman error}}}.
\end{align*}
In this section, we will bound $\EE\sbr{\abr{\tilde \Delta^{(1)}_h(s_h, b_h)}}$ by the KL distance in the following way.
We follow from the decomposition of the A-function in \Cref{lem:AQV-func diff},
\begin{align*}
    \abr{\tilde\Delta_h^{(1)}(s_h, b_h)} 
    &\le \bigabr{\bigrbr{A_h-\tilde A_h}(s_h, b_h)} + 2\eta^{-1} \Delta_h^{(2)}(s_h) \nend
    &\le \bigabr{\bigrbr{A_h-\tilde A_h}(s_h, b_h)} + 2\EE_{s_h}\sbr{\sum_{l=h}^H \gamma^{l-h} \inp[]{\nu_l(\cdot\given s_l)}{(A_l-\tilde A_l)(s_l, b_l)}}, 
\end{align*}
where we use the definition $\Delta_h^{(2)}(s_h)\defeq \EE_{s_h}\sbr{\sum_{l=h}^H \gamma^{l-h} \kl\infdivx[]{\nu_l}{\tilde\nu_l}}$ and the last inequality holds by noting that $\kl\infdivx[]{\nu}{\tilde\nu} = \eta\inp[]{\nu}{A-\tilde A}$.
Therefore, we conclude that
\begin{align}
    \EE\sbr{\abr{\tilde \Delta^{(1)}_h(s_h, b_h)}} 
    &\le  \EE\bigabr{\bigrbr{A_h-\tilde A_h}(s_h, b_h)} + 2\EE\sbr{\sum_{l=h}^H \gamma^{l-h} \inp[]{\nu_l(\cdot\given s_l)}{(A_l-\tilde A_l)(s_l, \cdot)}}
    \nend
    &
    \le 3 \sum_{l=h}^H \gamma^{l-h} \EE\bigabr{(A_l-\tilde A_l)(s_l, b_l)} \label{eq:Delta-A}
\end{align}
We now invoke the lower bound \eqref{eq:nu-tv-lb-1} in \Cref{lem:response diff} and obtain
\begin{align*}
    D_\TV(\nu_h, \tnu_h) &\ge \frac{1-\exp\rbr{-2\eta B_A}}{4 B_A} \cdot {\EE_{s_h}\bigabr{(\tA_h-A_h)(s_h, b_h)} } \nend
    &\ge  \frac{\eta}{2(1+ 2\eta B_A)} \cdot {\EE_{s_h}\bigabr{(\tA_h-A_h)(s_h, b_h)} }.
\end{align*}
Combining these results, we obtain
\begin{align*}
    \EE\sbr{\abr{\tilde \Delta^{(1)}_h(s_h, b_h)}} &\le 3\sum_{l=h}^H \gamma^{l-h} \EE\bigabr{(\tA_l-A_l)(s_l,b_l)} \nend
    &\le 3 \cdot \rbr{\frac{\eta}{2(1+ 2\eta B_A)}}^{-1} \sum_{l=h}^H \gamma^{l-h}\EE D_\TV(\nu_l(\cdot, s_l),\tilde\nu_l(\cdot, s_l)) \nend
    &\le 6(1+2\eta B_A)\cdot  {\frac{1-\gamma^H}{1-\gamma}}  \cdot \eta^{-1} \max_{h\in[H]} \EE D_\TV(\nu_h(\cdot, s_h),\tilde\nu_h(\cdot, s_h)),
\end{align*}
where the last inequality follows from from the fact that $(1-\exp(-x))/2x\ge 1/2(1+x)$.
Hence, we complete the proof of the first order of $\tilde \Delta^{(1)}$ in \Cref{lem:1st-ub}.

In the sequel, we will study how to upper bound $\bigrbr{\tilde \Delta_h^{(1)}(s_h, b_h)}^2$. We first have by \Cref{lem:AQV-func diff} that
\begin{align*}
    \bigrbr{\tilde \Delta_h^{(1)}(s_h, b_h)}^2 
    &= 2 \rbr{\rbr{\EE_{s_h, b_h}-\EE_{s_h}} \bigsbr{\orbr{A_h - \tilde A_h}(s_h, b_h)}}^2 + 2 \gamma^2 \eta^{-2} \rbr{\rbr{\EE_{s_h, b_h}-\EE_{s_h}}\bigsbr{\Delta_{h+1}^{(2)}(s_{h+1})}}^2\nend
    & \le 2 \rbr{\rbr{\EE_{s_h, b_h}-\EE_{s_h}} \bigsbr{\orbr{Q_h - \tilde Q_h}(s_h, b_h)}}^2 + 4 \gamma^2 \eta^{-2} \rbr{\EE_{s_h, b_h}\bigsbr{\Delta_{h+1}^{(2)}(s_{h+1})}}^2  \nend
    &\qquad + 4 \gamma^2 \eta^{-2} \rbr{\EE_{s_h}\bigsbr{\Delta_{h+1}^{(2)}(s_{h+1})}}^2,
\end{align*}
where the last inequality holds by using \eqref{eq:A diff-1} and note that $\eta^{-1}\kl\infdivx[]{\nu_h}{\tilde\nu_h} = \EE_{s_h}[A_h-\tilde A_h]$. By definition of $\Delta_h^{(2)}(s_h)$, we just focus on the second term and obtain 
\begin{align*}
    \rbr{\eta^{-1} \EE_{s_h, b_h}\bigsbr{\Delta_{h+1}^{(2)}(s_{h+1})} }^2
    &= \rbr{\eta^{-1}\EE_{s_h, b_h}\sbr{\sum_{l=h+1}^H \gamma^{l-h-1} \kl\infdivx[]{\nu_l(\cdot\given s_l)}{\tilde\nu_l(\cdot\given s_l)}}}^2\nend
    & = \rbr{\EE_{s_h, b_h}\sbr{\sum_{l=h+1}^H \gamma^{l-h-1} \inp[\big]{\nu_l(\cdot\given s_l)}{(A_l-\tilde A_l)(s_l, \cdot)}_\cB}}^2\nend
    &\le \eff_H(\gamma) \sum_{l=h+1}^H \gamma^{l-h-1}\rbr{\EE_{s_h, b_h}\sbr{\inp[\big]{\nu_l(\cdot\given s_l)}{\bigabr{(A_l-\tilde A_l)(s_l, \cdot)}}_\cB}}^2,
\end{align*}
where the last inequality follows from the Cauchy-Schwartz inequaltiy and we recall $\eff_H(\gamma) = (1-\gamma^H)/(1-\gamma)$.
We now invoke the lower bound \eqref{eq:nu-tv-lb-1} in \Cref{lem:response diff} and obtain
\begin{align*}
    D_\TV(\nu_h, \tnu_h) \ge \frac{\eta}{2(1+2\eta B_A)} \cdot \inp[\big]{\nu_h(\cdot\given s_h)}{\bigabr{
    \orbr{\tilde A - A}(s_h,\cdot)}}.
\end{align*}
Combining these results, we obtain
\begin{align*}
    &\rbr{\eta^{-1} \EE_{s_h, b_h}\bigsbr{\Delta_{h+1}^{(2)}(s_{h+1})} }^2 \nend
    &\quad \le 4\rbr{\eta^{-1} +2 B_A}^2\eff_H(\gamma) \sum_{l=h+1}^H \gamma^{l-h-1} {\EE_{s_h, b_h}\sbr{D_\H^2(\nu_l(\cdot\given s_l), \tilde\nu_l(\cdot\given s_l))}}, 
\end{align*}
where the inequality holds by using the Jensen's  inequality and move the expectation outside of the square. As a result, 
\begin{align*}
    &\bigrbr{\tilde \Delta_h^{(1)}(s_h, b_h)}^2  \nend
    &\quad \le 2 \rbr{\rbr{\EE_{s_h, b_h}-\EE_{s_h}} \bigsbr{\orbr{Q_h - \tilde Q_h}(s_h, b_h)}}^2 \nend
    &\qqquad + 16 \gamma^2  \rbr{\eta^{-1} +2 B_A}^2\eff_H(\gamma) \sum_{l=h+1}^H \gamma^{l-h-1} {\rbr{\EE_{s_h}+\EE_{s_h, b_h}}\sbr{D_\H^2(\nu_l(\cdot\given s_l), \tilde\nu_l(\cdot\given s_l))}}, 
\end{align*}
which completes our proof of \Cref{lem:1st-ub}.

\subsection{Proof of \Cref{lem:2nd-ub}}
\label{sec:proof-2nd-ub}
In this proof, we remind readers that $Q, A, r^\pi$ are functions from $\cS\times\cB$ to $\RR$, $V:\cS\times\RR$ and $P_h^\pi:\cS\times\cB\rightarrow\Delta(\cS)$. In the sequel, we will neglect the dependence on $s_h, b_h$ for simplicity.
The major part in this proof is to upper bound $\EE\osbr{\orbr{( P_h^\pi-\tilde P_h^{\pi})\tilde V_{h+1}}^2}$ and $\EE\osbr{\orbr{r_h^\pi-\tilde r_h^\pi}^2}$ by $D_{\RL, h}^2$ separately. 
Moreover, we use $B_A$ in \eqref{eq:define_BA} to bound the follower's Q- and A-function. 
We will leave out the dependence on $(s_h, b_h)$ most of the times in the following proof when it does not cause any confusion in the context.

Note that we only have guarantee for $D_\TV^2(\nu_h, \tilde\nu_h)$ by MLE, which cannot directly guarantee that the true utility is identifiable since a constant shift does not change the follower's behavior at all. For the reward to be identifiable, we need an additional linear constraint, namely $\inp{x}{r_h(s_h, a_h, \cdot)}=\varsigma$.
We start with the easier part with the transition kernel.
\begin{align*}
    \EE\osbr{\orbr{( P_h^\pi-\tilde P_h^{\pi})\tilde V_{h+1}}^2}\le 2^2 B_A^2 \EE\sbr{D_\TV^2( P_h^\pi, \tilde  P_h^\pi)}.
\end{align*}
For the follower's reward, we take a real number $\xi$ and have the following decomposition
\begin{align*}
    &\inf_{\xi\in\RR}\EE_{s_h}\abr{r_h^\pi-\tilde r_h^\pi - \xi}  \nend
    &\quad = \inf_{\xi\in\RR}\EE_{s_h}\bigabr{Q_h-\tilde Q_h - \xi - \gamma\bigrbr{ P_h^\pi-\tilde P_h^\pi}\tilde V_{h+1} - \gamma  P_h^\pi\bigrbr{V_{h+1}-\tilde V_{h+1}}}\nend
    &\quad \le \inf_{\xi\in\RR}\EE_{s_h}\bigabr{Q_h-\tilde Q_h - \xi} + \gamma \EE_{s_h}\bigabr{\bigrbr{ P_h^\pi-\tilde P_h^\pi}\tilde V_{h+1} } + \gamma\exp\rbr{2\eta B_A}\EE_{s_h}\bigabr{Q_{h+1} - \tilde Q_{h+1}}, \nend
    &\quad \le \EE_{s_h}\bigabr{A_h-\tilde A_h} + \gamma\EE_{s_h}\bigabr{\bigrbr{ P_h^\pi-\tilde P_h^\pi}\tilde V_{h+1} } + \gamma\exp\rbr{2\eta B_A}\EE_{s_h}\bigabr{Q_{h+1} - \tilde Q_{h+1}}
\end{align*}
where the first inequality holds by the same argument for $V-\tilde V$ in \eqref{eq:f_2-1}, and the second inequality holds simply by plugging $\xi = V_h(s_h) - \tilde V_h(s_h)$. Now, we can plug in the bound for $\EE_{s_h}\oabr{A_h-\tilde A_h}$ in \Cref{lem:response diff} and obtain
\begin{align}
    &\inf_{\xi\in\RR}\EE_{s_h}\abr{r_h^\pi-\tilde r_h^\pi - \xi} \nend
    &\quad \le \underbrace{2(\eta^{-1}+2B_A) D_\TV(\nu_h, \tilde\nu_h) + 2 \gamma B_A \EE_{s_h} D_\TV( P_h^\pi, \tilde P_h^\pi)}_{\ds \sD_h} + \gamma\exp\rbr{2\eta B_A}\EE_{s_h}\bigabr{Q_{h+1} - \tilde Q_{h+1}}.
    \label{eq:f_2-r-diff}
\end{align}
We next show what we can say about the utility when combining the guarantee of \eqref{eq:f_2-r-diff} with the linear constraint $\inp{x}{r_h(s_h, a_h, \cdot)}=\varsigma$. Specifically, we have the following lemma.
\begin{lemma}[Identification of the follower's utility]\label{lem:identification}
    Suppose for $r, \tilde r:\cB\rightarrow \RR$, for some distribution $\nu\in\Delta(\cB)$ such that $\nu>0$, we have $\inf_{\xi\in\RR}\inp{\nu}{\abr{r-\tilde r-\xi}}\le \varepsilon$ and $\inp{x}{r-\tilde r}=0$ hold at the same time for some $x:\cB\rightarrow \RR$ such that $\inp{\ind}{x}\neq 0$. We have
    \begin{align*}
        \inp{\nu}{\abr{r-\tilde r}}\le\rbr{1 + \nbr{\frac x \nu}_\infty \cdot \frac{1}{\abr{\inp{x}{\ind}}}} \epsilon
    \end{align*}
    \begin{proof}
        See \Cref{sec:proof-identification} for a detailed proof.
    \end{proof}
\end{lemma}
With \Cref{lem:identification}, we conclude with $\nbr{\nu}_\infty\ge \exp\rbr{-\eta B_A}$ and $\kappa = \nbr{x}_\infty/|\la x, \ind\ra|$ that
\begin{align*}
    \EE_{s_h}\abr{r_h^\pi-\tilde r_h^\pi} &\le \rbr{1+\exp\rbr{2\eta B_A} \kappa} \bigrbr{\sD_h + \gamma \exp\rbr{2\eta B_A}\cdot\EE_{s_h}\bigabr{Q_{h+1} - \tilde Q_{h+1}}}.
\end{align*}
On the other hand, for the Q-function, we have by \eqref{eq:f_2-Q-ub} that 
\begin{align*}
    \EE_{s_h}\bigabr{Q_h - \tilde Q_h}
    &\le \EE_{s_h}\bigabr{r_h^\pi-\tilde r_h^\pi + \gamma \bigrbr{ P_h^\pi-\tilde P_h^\pi}\tilde V_{h+1}} + \gamma \exp\rbr{2\eta B_A} {\EE_{s_h}\bigabr{Q_{h+1}-\tilde Q_{h+1}}}\nend
    &\le \underbrace{2\rbr{1+\exp\rbr{2\eta B_A} \kappa} }_{\ds c_1}\cdot \sD_h \nend
    &\qquad + \underbrace{\rbr{2+\exp\rbr{2\eta B_A} \kappa}\gamma \exp\rbr{2\eta B_A}}_{\ds c_2}\cdot\EE_{s_h}\bigabr{Q_{h+1} - \tilde Q_{h+1}}, 
\end{align*}
where in the last inequality, we directly upper bound $\EE_{s_h}|\gamma(P_h^\pi - \tilde P_h^\pi)\tilde V_{h+1}|$ by $\sD_h$. 
Therefore, we have by a recursive argument that 
\begin{align*}
    \EE_{s_h}\bigabr{Q_h -\tilde Q_h} \le \sum_{l=h}^H c_2^{l-h} c_1 \EE_{s_h}\sD_l.
\end{align*}
For now, we are able to deal with $(\EE_{s_h}\abr{r_h^\pi-\tilde r_h^\pi})^2$. However, note that what we actually want to get is the version with the square within the expectation $\EE_{s_h}$, i.e.,  $\EE\rbr{r_h^\pi-\tilde r_h^\pi}^2$. Therefore, we need a variance-mean decomposition,
\begin{align*}
    \EE\rbr{r_h^\pi-\tilde r_h^\pi}^2&= \EE\bigrbr{Q_h-\tilde Q_h - \gamma \bigrbr{ P_h^\pi -\tilde P_h^\pi}\tilde V_{h+1} -\gamma  P_h^\pi \bigrbr{V_{h+1}-\tilde V_{h+1}}}^2\nend
    &\le 4\EE\bigsbr{\bigrbr{A_h -\tilde A_h}^2 + \bigrbr{V_h-\tilde V_h}^2 + \gamma^2 \bigrbr{\bigrbr{ P_h^\pi -\tilde P_h^\pi}\tilde V_{h+1}}^2 +\gamma^2 \bigrbr{V_{h+1}-\tilde V_{h+1}}^2}\nend
    &\le 4\EE\bigsbr{\bigrbr{A_h -\tilde A_h}^2} + 16\gamma^2 B_A^2\EE\bigsbr{D_\TV( P_h^\pi,\tilde  P_h^\pi)^2}\nend
    &\qquad + 4\exp\rbr{4\eta B_A}\EE\bigsbr{\bigrbr{\EE_{s_h}\bigabr{Q_h-\tilde Q_h}}^2+ \bigrbr{\EE_{s_{h+1}}\bigabr{Q_{h+1}-\tilde Q_{h+1}}}^2}, 
\end{align*}
where in the first inequality, we use $Q=A+V$, and use the Jensen's inequality to derive the last term.
The last inequality holds by noting the upper bound for difference in the V-function used in \Cref{eq:f_2-1}.
We notice that the first term can be upper bounded by the squared Hellinger distance,
\begin{align*}
    D_\H^2\rbr{\nu_h,\tilde\nu_h} &= \dotp{\nu_h}{\rbr{1-\sqrt\frac{\tilde \nu_h}{\nu_h}}^2}_\cB\nend
    &= \Bigdotp{\nu_h}{\rbr{1-\exp\rbr{\frac \eta 2 (\tilde A_h - A_h)}}^2}_\cB\nend
    &\ge \rbr{\frac{1-\exp\rbr{-\eta B_A}}{2 B_A}}^2  \cdot \bigdotp{\nu_h}{\bigrbr{A_h-\tilde A_h}^2}_\cB \ge \rbr{\frac{\eta}{2}}^2  \cdot \bigdotp{\nu_h}{\bigrbr{A_h-\tilde A_h}^2}_\cB,
\end{align*}
where the first inequality holds by noting that $|1-\exp(x)|\ge (1-\exp(-B))|x|/B$ for any $|x|\le B$.
The last inequality uses the inequality $(1-\exp(-x))\ge x/(1+x)$ for all $x>0$. 
Therefore, we have the follower's squared reward difference bounded by
\begin{align*}
    &\EE\rbr{r_h^\pi -\tilde r_h^\pi}^2 \nend
    &\quad \le 16\eta^{-2} \EE D_\H^2(\nu_h,\tilde\nu_h) + 16\gamma^2 B_A^2\EE D_\TV^2( P_h^\pi,\tilde  P_h^\pi)\nend
    &\qqquad + 4 \exp\rbr{4\eta B_A} \EE\sbr{\rbr{\sum_{l=h}^H c_2^{l-h} c_1 \EE_{s_h}\sD_l}^2 + \rbr{\sum_{l=h+1}^H c_2^{l-h} c_1 \EE_{s_{h+1}}\sD_l}^2}\nend
    &\quad \le 16\eta^{-2} \EE D_\H^2(\nu_h,\tilde\nu_h) + 16\gamma^2 B_A^2\EE D_\TV^2( P_h^\pi,\tilde  P_h^\pi)\nend
    &\qqquad + 8 H \eff_H(c_2)^2 c_1^2 \exp\rbr{4\eta B_A} \max_{h\in[H]}\EE\sbr{{\sD_h}^2}\nend
    &\quad \le 16\eta^{-2} \EE D_\H^2(\nu_h,\tilde\nu_h) + 16\gamma^2 B_A^2\EE D_\TV^2( P_h^\pi,\tilde  P_h^\pi)\nend
    &\qqquad + 8 H \eff_H(c_2)^2 c_1^2\exp\rbr{4\eta B_A} \max_{h\in[H]}\EE\sbr{\rbr{{2(\eta^{-1}+2B_A) D_\TV(\nu_h, \tilde\nu_h) + 2 \gamma B_A \EE_{s_h} D_\TV( P_h^\pi, \tilde P_h^\pi)}}^2}, 
\end{align*}
where in the second inequality, we uses the Cauchy-Schwartz inequality that $\EE(\sum a_l x_l)^2\le \sum a_l \cdot \EE\sum a_l x_l^2 \le (\sum a_l)^2 \cdot \max_l \EE b_l^2$ for constant sequence $a_l>0$. 
In summary, we have
\begin{align*}
    \EE\rbr{r_h^\pi -\tilde r_h^\pi}^2
    &\le {32 H^2 \eff_H(c_2)^2 c_1^2\exp\rbr{4\eta B_A} \rbr{4(\eta^{-1}+2B_A)^2+4\gamma^2 B_A^2 }}  \nend
    & \qquad \cdot \max_{h\in[H]}\cbr{\EE D_\H^2(\nu_h,\tilde\nu_h)+\EE D_\TV^2(P_h^\pi,\tilde P_h^\pi)}\nend
    &\le \underbrace{640 H^2 \eff_H(c_2)^2 c_1^2\exp\rbr{4\eta B_A} (\eta^{-1}+B_A)^2}_{\ds c_3/4}  \nend
    & \qquad \cdot \max_{h\in[H]}\cbr{\EE D_\H^2(\nu_h,\tilde\nu_h)+\EE D_\TV^2(P_h^\pi,\tilde P_h^\pi)}
\end{align*}
Therefore, we conclude that
\begin{align*}
    \max_{h\in[H]}\EE\sbr{ \bigrbr{{\tilde Q_h - r_h^\pi - \gamma P_h^\pi \tilde V_{h+1}}}^2} 
    &\le 2 \max_{h\in[H]} \EE\sbr{\rbr{\tilde r_h^\pi - r_h^\pi}^2} + 2 \gamma^2\max_{h\in[H]} \EE\sbr{\bigrbr{\bigrbr{\tilde P_h^\pi - P_h^\pi}\tilde V_{h+1}}^2}\nend
    &\le c_3 \max_{h\in[H]}\cbr{\EE D_\H^2(\nu_h,\tilde\nu_h)+\EE D_\TV^2(P_h^\pi,\tilde P_h^\pi)}, 
\end{align*}
which completes our proof of \Cref{lem:2nd-ub}

\subsubsection{Proof of \Cref{lem:identification}}\label{sec:proof-identification}
For condition $\inf_{\xi\in\RR}\inp{\nu}{\abr{r-\tilde r-\xi}}\le \varepsilon$, we assume that the infimum is achieved at $\xi^*$. Let $r^* = r -\xi^*$ and we have
\begin{align*}
\abr{\inp{r^*-\tilde r}{x}} \le \inp{\abr{r^*-\tilde r}}{\abr{x}} \le \inp{\abr{r^*-\tilde r}}{\nu} \cdot \nbr{\frac{x}{\nu}}_\infty\le \varepsilon\nbr{\frac{x}{\nu}}_\infty,
\end{align*}
where the second inequality is just a distribution shift and the last inequality is given by the condition. Furthermore, for our target,
\begin{align*}
    \inp{\abr{r-\tilde r}}{\nu} \le \inp{\abr{r^*-\tilde r}}{\nu} + \abr{\xi^*} = \varepsilon + \abr{\inp{r^*-r}{x}} \cdot \frac{1}{\abr{\inp{x}{\ind}}},
\end{align*}
where the inequality follows from the triangle inequality and the equality holds by noting that $\inp{x}{\ind}\neq 0$ and $\inp{\abr{r^*-\tilde r}}{\nu}\le \varepsilon$. We bridge these two inequalities by noting that
\begin{align*}
    \abr{\inp{r^*-\tilde r}{x}} = \abr{\inp{r-\tilde r}{x} + \inp{r^*-r}{x} } = \abr{\inp{r^*-r}{x} },
\end{align*}
where the second inequality holds by noting that $\inp{r-\tilde r}{x}=0$. Combining these results and we have
\begin{align*}
    \inp{\abr{r-\tilde r}}{\nu} \le \epsilon + \abr{\inp{r^*-r}{x}} \cdot \frac{1}{\abr{\inp{x}{\ind}}} \le \rbr{1 + \nbr{\frac x \nu}_\infty \cdot \frac{1}{\abr{\inp{x}{\ind}}}} \epsilon, 
\end{align*}
which completes the proof on \Cref{lem:identification}.

\section{Auxiliary Results and Their Proofs}

In this section, 
we present and prove 
the auxiliary lemmas that helps theoretical analysis. 

\subsection{Auxiliary Results for Learning Quantal Response Mapping}

In this subsection, we introduce some technical lemmas used for learning follower's quantal response mapping. 
To simplify the notation, similar to the statement of \Cref{lem:performance diff}, we fix a 
policy $\pi$, and let $\tilde Q$ be an estimate of $Q^{\pi}$. 
Moreover, for Lemmas \ref{lem:response diff}--\ref{lem:KL-ub},  we consider a fixed state $s$ at a fixed timestep $h\in [H]$, and with slight abuse of notation, omit it in the value functions. 
In particular, for $\tilde V$ and $\tilde A$ defined in \eqref{eq:tilde_functions}, in the sequel, we write 
$\tilde Q = \tilde Q_h (s, \cdot )$, $\tilde V = \tilde V_h(s)$, and $\tilde A = \tilde A_h(s, \cdot)$. 
Furthermore, we write  $Q = Q^{\pi}_h (s, \cdot )$, $V = V^{\pi}_h (s )$, and $A = A^{\pi}_h (s, \cdot )$. 
Finally, 
we  write $\nu^{\pi} _h(\cdot \given s)$ and $\tilde \nu_h (\cdot \given s)$   as $\nu \in \Delta(\cB)$ and $\tilde \nu \in \Delta(\cB)$ respectively.
Using this notation, 
$Q, \tilde Q,$,$A $, and $\tilde A$ can all be regarded as vectors in $\RR^{|\cB|}$, $ V, \tilde V \in \RR$, and we have 
\#\label{eq:define_nu_vecs}
\nu(b) = \exp \big (\eta \cdot A(b) \bigr ), \qquad \tilde \nu (b) = \exp\big (\eta \cdot \tilde A(b) \bigr ), \forall b\in \cB. 
\#
The following lemma establishes upper and lower bounds on the total variation (TV) distance between $\nu $ and $\tilde \nu$.

\begin{lemma}[TV-Distance Between Quantal Responses]\label{lem:response diff}
Let $\nu, \tilde \nu \in \Delta(\cB)$ be defined in \eqref{eq:define_nu_vecs}, where $\tilde A = \tilde Q - \tilde V$ and $A = Q - V$. 
The TV distance  between $\nu$ and $\tilde\nu$ enjoys the following upper bound:  
\begin{align*}
    D_\TV\orbr{\nu, \tilde \nu}\le \eta \cdot \inp[\Big]{\nu} {\bigabr{\tilde Q-Q-\xi \cdot \ind } + \frac \eta 2 \exp\bigrbr{\eta\bigabr{\tilde Q-Q-\xi \cdot \ind}} \cdot \bigrbr{\tilde Q-Q-\xi \cdot \ind}^2}_{\cB }, \quad \forall \xi\in\RR.
\end{align*}
Here $\ind \in \RR^{|\cB|}$ is an all-one vector. 
In particular, setting $\xi = \tilde V - V$, we have 
\begin{align}\label{eq:nu-tv-ub-0}
    D_\TV\orbr{\nu, \tilde \nu}\le \eta \cdot \inp[\Big]{\nu} {\oabr{\tilde A - A  } + \frac \eta 2 \exp\bigrbr{\eta\oabr{\tilde A - A }}\cdot \orbr{\tilde A - A }^2}_{\cB }, \quad \forall \xi\in\RR.
\end{align}
Besides, 
the following  lower bounds hold: 
\begin{align}
    D_\TV & \orbr{\nu, \tilde \nu} \ge \frac{\eta}{2(1+2\eta B_A)} \cdot \inp[\big]{\nu}{\oabr{\tilde A - A}}_{\cB}, \label{eq:nu-tv-lb-1}\\
    D_\TV & \orbr{\nu, \tilde \nu}\ge \frac 1 2 \inp[\big]{\nu} { \eta \exp\bigrbr{- \eta \oabr{\tilde A- A}}\cdot \oabr{\tilde A-A}}_{\cB }.  \label{eq:nu-tv-lb-2}
\end{align}
Here $B_{A}$ is an upper bound on $\max \{ \| A \|_{\infty}, \| \tilde A \|_{\infty} \} $, which, for example, can be set as in \eqref{eq:define_BA}.
\end{lemma}


\begin{proof}
In this proof, to simplify the notation, we omit the subscript in $\la \cdot , \cdot \ra_{\cB}$ without causing confusion. 
We first prove  the lower bounds. 
Using the relation between TV distance and $\ell_1$-norm, 
we have 
\begin{align}
    D_\TV\bigrbr{\nu, \tilde \nu} = \frac{1}{2} \cdot \| \nu - \tilde \nu \|_1 = \frac{1}{2} \cdot \inp[\Big]{\nu}{\Bigabr{\ind -\frac{\tilde \nu}{\nu}}}  = \frac 1 2 \cdot \inp[\big]{\nu} {\bigabr{\ind -\exp\bigrbr{\eta\orbr{\tilde A - A }}}}. \label{eq:TV-1}
\end{align}
We note that $|\exp(x)-1|\ge (1-\exp(-B)) / B \cdot |x|$ holds for all $x \in [-B, B]$.  
Thus, for any $b \in \cB$, 
we have 
\#\label{eq:TV-11}
\Bigabr{1  -\exp\Bigrbr{\eta\bigrbr{\tilde A  (b) -A (b)  }}} \geq  \frac{1-\exp\rbr{-2 \eta B_A} } {2\eta B_A } \cdot \eta \cdot | \tilde A(b) - A(b) | .
\#
Hence, combining \eqref{eq:TV-1} and \eqref{eq:TV-11},  we have
\begin{align*}
    &\frac 1 2 \cdot \inp[\big]{\nu} {\bigabr{\ind -\exp\bigrbr{\eta\orbr{\tilde A - A }}}}
    \ge \frac \eta 2 \cdot \frac{1-\exp\bigrbr{-2 \eta B_A}}{2\eta B_A} \cdot \inp[]{\nu}{\oabr{\tilde A - A}} \geq \frac \eta 2 \cdot \frac{1}{1+2\eta B_A}\cdot \inp[]{\nu}{\oabr{\tilde A - A}}.
\end{align*}
where the last inequality holds by noting that 
$$
1 - \exp(-x ) = \frac{\exp(x)-1}{\exp(x)-1 + 1}\ge \frac{x}{1+x}, \quad \forall x>0.
$$
Therefore, we establish \eqref{eq:nu-tv-lb-1}.

Meanwhile,  we note that $ | \exp(x) - 1 | \geq \exp( -  | x| ) \cdot x $ holds for all $x \in \RR$, which indicates that
\begin{align*}
    &\frac 1 2 \cdot \inp[\big]{\nu} {\bigabr{\ind -\exp\bigrbr{\eta\orbr{\tilde A - A }}}}\ge 
    \frac 1 2 \inp[\big]{\nu} { \eta \exp\bigrbr{- \eta \oabr{\tilde A- A}}\cdot \oabr{\tilde A-A}}.
\end{align*}
Thus, we establish \eqref{eq:nu-tv-lb-2}.

It remains to establish an upper bound 
for the right hand side of \eqref{eq:TV-1}. Note that 
$A = Q - V$ and $\tilde A = \tilde Q - \tilde V$. 
Then we have 
\begin{align*}
    &\frac 1 2 \cdot \inp[\big]{\nu} {\bigabr{\ind -\exp\bigrbr{\eta\orbr{\tilde A - A }}}}  \nend
    &\quad = \frac 1 2 \cdot \inp[\big]{\nu} {\bigabr{\ind -\exp\bigrbr{\eta\orbr{\tilde Q - Q - \tilde V +  V }}}}  \nend
    &\quad \le \frac 1 2 \bigdotp{\nu}{\bigabr{\ind - \exp\bigrbr{\eta\bigrbr{\tilde Q - Q}}}} + \frac 1 2 \bigdotp{\nu}{\bigabr{\exp\bigrbr{\eta\bigrbr{\tilde Q - Q}} - \exp\bigrbr{\eta\bigrbr{\tilde Q - Q -\tilde V+V}}}}\nend
    &\quad = \frac 1 2 \cdot \inp[\big]{\nu} {\bigabr{\ind - \exp\bigrbr{\eta\orbr{\tilde Q- Q}}}} + \frac 1 2 \inp[\big]{\ind}{\bigabr{\exp\bigrbr{\eta\bigrbr{\tilde Q -V}}-\exp\bigrbr{\eta\bigrbr{\tilde Q -\tilde V}}}},
\end{align*}
where the inequality is just a split of terms and the last equality holds by the definition $\nu=\exp\rbr{\eta(Q-V)}$. Using the equality $\inp[]{\ind}{\exp(\eta\tilde Q)}=\exp(\eta \tilde V)$ by definition of $\tilde V$, and noting that $V$ and $\tilde V$ does not depend on $\cB$, we further obtain
\begin{align}
    &\frac 1 2 \cdot \inp[\big]{\nu} {\bigabr{\ind -\exp\bigrbr{\eta\orbr{\tilde A - A }}}}  \nend
    &\quad \le \frac 1 2 \cdot \inp[\big]{\nu} {\bigabr{\ind - \exp\bigrbr{\eta\orbr{\tilde Q- Q}}}} + \frac 1 2 \inp[\Big]{\exp\bigrbr{\tilde Q}}{\bigabr{\exp\bigrbr{\eta\bigrbr{-V}}-\exp\bigrbr{\eta\bigrbr{-\tilde V}}}} \nend
    &\quad = \frac 1 2 \cdot \inp[\big]{\nu} {\bigabr{1- \exp\bigrbr{\eta\bigrbr{\tilde Q- Q}}}} + \frac 1 2 \exp\bigrbr{\eta \tilde V}\cdot \bigabr{\exp\rbr{-\eta V} - \exp\bigrbr{-\eta\tilde V}}, \label{eq:TV-2}
\end{align}
Moreover, the second term on the right hand side of \eqref{eq:TV-2} can be bounded by the same trick,
\begin{align}
    \frac 1 2 \exp\bigrbr{\eta \tilde V}\cdot \bigabr{\exp\bigrbr{-\eta V} - \exp\bigrbr{-\eta\tilde V}}
    &= \frac 1 2 \exp\rbr{-\eta V}\cdot \bigabr{\exp\bigrbr{\eta \tilde V} - \exp\rbr{\eta V}}\nend
    &=\frac 1 2 \exp\bigrbr{-\eta V}\cdot \bigabr{\inp[\big]{\ind}{\exp\bigrbr{\eta \tilde Q}-\exp\bigrbr{\eta Q}}}\nend
    &= \frac 1 2 \bigabr{\inp[\big]{\nu}{\exp\bigrbr{\eta\tilde Q-\eta Q}-1}}, \label{eq:TV-3}
\end{align}
where the last equality uses the fact that $\nu = \exp(\eta Q -\eta V)$. 
Plugging \eqref{eq:TV-3} into \eqref{eq:TV-2} gives
\begin{align*}
    &\frac 1 2 \cdot \inp[\big]{\nu} {\bigabr{\ind -\exp\bigrbr{\eta\orbr{\tilde A - A }}}} \nend
    &\quad \le
    \inp[\big]{\nu} {\bigabr{\ind - \exp\bigrbr{\eta\orbr{\tilde Q- Q}}}}\nend
    &\quad \le\eta \cdot \inp[\big]{\nu} {\oabr{\tilde Q-Q} + \frac \eta 2 \cdot \exp\orbr{\eta\oabr{\tilde Q-Q}}\oabr{\tilde Q-Q}^2},
\end{align*}
where the last inequality holds by a Taylor expansion of $\exp(x)-1$ at $x=0$.
Moreover, note that adding a constant shift $\xi\ind$ to both $\tilde Q-Q$ does not change $\nu$ nor $\tilde \nu$. Hence, we finish the proof of the upper bound.
\end{proof}

The next lemma establishes a lower bound on the Hellinger distance between $\tilde \nu$ and $\nu$.

\begin{lemma}[Hellinger distance Between Quantal Responses]\label{lem:D_H-2-A^2}

    Let $\nu $ and $\tilde \nu$ be defined in \eqref{eq:define_nu_vecs}. 
    We let $D_\H (\cdot, \cdot)$ denote the Hellinger distance between probability distributions. 
    Then we  have
    \begin{align*}
        D_\H^2\rbr{\nu, \tilde\nu} \ge  \frac{\eta^2}{8(1+\eta B_A)^2} \cdot \inp[\big]{\nu}{ (\tilde A-A)^2}_{\cB}.
    \end{align*}
\end{lemma}
    \begin{proof}
        In this proof, to simplify the notation, we omit the subscript in $\la \cdot , \cdot \ra_{\cB}$.
         By the  definition of the Hellinger distance,  we have 
        \begin{align*}
            D_\H^2(\nu, \tilde\nu) &= \frac 1 2 \inp[\big]{\nu}{\bigrbr{1-\sqrt{ \tilde\nu / \nu}}^2} 
            =\frac 1 2 \inp[\big]{\nu}{\bigabr{\ind -\exp\bigrbr{\eta/2 \cdot \orbr{\tilde A-A} } } ^2}\nend
            &\ge \frac 1 8 \rbr{\frac{1-\exp\rbr{-\eta B_A}}{B_A}}^2\cdot \inp[\big]{\nu}{ (\tilde A-A)^2} \ge  \frac{\eta^2}{8(1+\eta B_A)^2} \cdot \inp[\big]{\nu}{ (\tilde A-A)^2}.
        \end{align*}
        where the first inequality follows from $|\exp(x)-1|\ge (1-\exp(-B))|x|/B$ for any $|x|\le B$, and the last inequality holds by noting that 
        $$
        1 - \exp(-x ) = \frac{\exp(x)-1}{\exp(x)-1 + 1}\ge \frac{x}{1+x}, \quad \forall x>0.
        $$
        Therefore, we conclude the proof. 
    \end{proof} 

    The next lemma establishes an upper bound on the KL-divergence between $\tilde \nu$ and $\nu$ in terms of difference of Q-functions.

\begin{lemma}[KL-Divergence Between Quantal Responses]\label{lem:KL-ub}
    Let $\nu $ and $\tilde \nu$ be defined in \eqref{eq:define_nu_vecs}. We have
    \begin{align*}
      \kl\infdivx[]{\nu}{\tilde\nu} \le \eta \cdot \inp[\big]{\nu - \tilde\nu}{Q-\tilde Q}_{\cB}.
    \end{align*}
\end{lemma}
    \begin{proof}
        To simplify the notation, we omit the subscript in $\la \cdot , \cdot \ra_{\cB}$.
        By the  definition of the KL divergence, we have 
        \begin{align*}
            \eta^{-1}\kl\infdivx[]{\nu}{\tilde\nu} &= \inp[\big]{\nu}{A-\tilde A}\nend
            & = \inp[\big]{\nu}{Q-\tilde Q} - \orbr{V-\tilde V}\nend
            & = \inp[\big]{\nu}{Q-\tilde Q} -\bigcbr{ \max_{\nu'}\inp{\nu'}{Q} + \eta^{-1} \cH(\nu')} + \bigcbr{ \max_{\nu'}\inp{\nu'}{\tilde Q} + \eta^{-1} \cH(\nu')},
        \end{align*}
where the second equality holds because $V$ and $\tilde V$ are real numbers, and the last equality follows from the optimality condition of the quantal response. 
Note that the optimal solutions to the optimization problems are $\nu$ and $\tilde \nu$, respectively. 
Thus, we have 
\$
-\bigcbr{ \max_{\nu'}\inp{\nu'}{Q} + \eta^{-1} \cH(\nu')} + \bigcbr{ \max_{\nu'}\inp{\nu'}{\tilde Q} + \eta^{-1} \cH(\nu')} \leq \la \tilde \nu, - Q + \tilde Q\ra, 
\$
where we replace $\tilde \nu'$ by $\tilde \nu$ in the first 
 maximization. 
 Therefore, combining the above equations above, we complete  the proof.
    \end{proof}

As shown in Lemma \ref{lem:response diff}, 
the estimation error  of the follower's A-function plays a central role in the analysis of quantal response. 
The next lemma  characterizes  $A_h-\tilde A_h$. 
To the following, we follow the same  notation as  in \Cref{lem:performance diff}. Besides, 
to simplify the presentation, we define  
\begin{equation*}
\begin{aligned}
    \Delta_h^{(1)} (s_h, b_h)&\defeq  \EE_{s_h, b_h}\sbr{\sum_{i=h}^H \gamma^{i-h}\bigrbr{r_i^\pi + \gamma P_{i}^\pi\tilde V_{i+1} - \tilde Q_i}(s_i,b_i)}, \\
    \Delta_h^{(2)}(s_h) &\defeq \EE_{s_h}\sbr{\sum_{i=h}^H \gamma^{i-h} \kl\infdivx[]{\nu_i(\cdot\given s_i)}{\tilde\nu_i(\cdot\given s_i)}}, 
\end{aligned}
\end{equation*}
where $\EE_{s_h, b_h}$ is a short hand of $\EE^{\pi, M^*}\sbr{\cdot\given s_h, b_h}$ and $\EE_{s_h}$ is a short hand of $\EE^{\pi, M^*}\sbr{\cdot\given s_h}$, and $\EE^{\pi, M^*}$ denotes the expectation with respect to the randomness of the trajectory induced by $(pi, \nu^{pi})$ under the true model $M^*$. 
\begin{lemma}[Difference  in  A-Functions]\label{lem:AQV-func diff}
We follow the same notation  as in \cref{lem:performance diff}. 
For any $h \in [H]$ and $(s_h, b_h) \in \cS\times \cB$, we have  
    \begin{align*}
        &A_h(s_h,b_h)-\tilde A_h(s_h, b_h) = \rbr{\EE_{s_h, b_h}-\EE_{s_h}} \sbr{\Delta_h^{(1)}(s_h, b_h) - \gamma\eta^{-1}\Delta_{h+1}^{(2)} (s_{h+1})} + \eta^{-1}\kl\infdivx[]{\nu_h}{\tilde \nu_h}, 
    \end{align*}
    where  $\kl\infdivx[]{\nu_h}{\tilde\nu_h} = \eta\cdot \EE_{s_h}\bigsbr{A_h-\tilde A_h}$. 
    In particular, in the myopic case, 
    we have 
    \$
    A_h(s_h,b_h)-\tilde A_h(s_h, b_h) = \rbr{\EE_{s_h, b_h}-\EE_{s_h}} \sbr{( r_h^{\pi} - \tilde r_h^{\pi} ) (s_h, b_h)} + \eta^{-1}\kl\infdivx[]{\nu_h}{\tilde \nu_h},
    \$ 
    where $\tilde r$ is an estimate of $r$ that generates $\tilde \nu_h$. 
\begin{proof}
    To simplify the notation, we omit  $s_h, b_h$ in the functions and  omit the subscript in $\la \cdot , \cdot \ra_{\cB}$.
By the optimality of $\nu_h$ and $\tilde \nu_h$, we can write the 
    the difference of V-functions  as follows:
    \begin{align*}
        V_h-\tilde V_h &= \inp[]{\nu_h}{Q_h} + \eta^{-1} \cH(\nu_h) - \inp[]{\tilde \nu_h}{\tilde Q_h} - \eta^{-1} \cH(\tilde \nu_h)\nend
        & =\inp[]{\nu_h}{Q_h-\tilde Q_h} +\inp[]{\nu_h-\tilde\nu_h}{\tilde Q_h} - \inp[]{\nu_h}{Q_h-V_h} + \inp[]{\tilde\nu_h}{\tilde Q_h -\tilde V_h}, 
    \end{align*}
where we use the fact that     
$\eta^{-1}\cH(\nu_h) = - \eta^{-1}\dotp{\nu_h}{\log \nu_h}= -\dotp{\nu_h}{Q_h - V_h}$. 
Note that here both $\tilde V_h$ and $V_h$ are real numbers. 
Then, by direct calculation, 
we have 
\$
\inp[]{\nu_h-\tilde\nu_h}{\tilde Q_h} - \inp[]{\nu_h}{Q_h-V_h} + \inp[]{\tilde\nu_h}{\tilde Q_h -\tilde V_h} = - \inp[\big]{\nu_h}{Q_h - V_h - \bigrbr{\tilde Q_h -\tilde V_h}}, 
\$ 
where we use the fact that   $\inp[]{\tilde \nu_h-\nu_h}{\tilde V_h}=0$ since $\tilde V_h$ does not depend on $b_h$,
Hence, we can write $V_h-\tilde V_h $ as 
    \#  
    V_h-\tilde V_h &= \inp[]{\nu_h}{Q_h-\tilde Q_h} - \inp[\big]{\nu_h}{Q_h - V_h - \bigrbr{\tilde Q_h -\tilde V_h}}\nend
        & = \inp[]{\nu_h}{Q_h-\tilde Q_h} - \eta^{-1} \kl\infdivx[]{\nu_h}{\tilde \nu_h}, \label{eq:V diff}
    \# 
   where the  equality holds by noting that $\kl\infdivx[]{\nu_h}{\tilde\nu_h}= \eta \odotp{\nu_h}{A_h -\tilde A_h} $.
   Thus, by \eqref{eq:V diff}, 
   For the A-function, we have
    \begin{align}
        A_h -\tilde A_h & =Q_h - V_h - \bigrbr{\tilde Q_h -\tilde V_h}= \orbr{\EE_{s_h, b_h}-\EE_{s_h}}\sbr{Q_h-\tilde Q_h}+\eta^{-1}\kl\infdivx[]{\nu_h}{\tilde \nu_h}.  \label{eq:A diff-1}
    \end{align}
Meanwhile, by the Bellman equation $Q_h = r_h^{\pi} + \gamma \cdot P_h ^{\pi} V_{h+1}$, 
    we have 
    \begin{align}
        Q_h -\tilde Q_h &= \EE_{s_h, b_h}\sbr{r_h^\pi + \gamma \tilde V_{h+1} - \tilde Q_h + \gamma\orbr{V_{h+1} - \tilde V_{h+1}}}\nend
        &=\EE_{s_h, b_h}\sbr{r_h^\pi + \gamma \tilde V_{h+1} - \tilde Q_h + \gamma \orbr{Q_{h+1}-\tilde Q_{h+1}} - \gamma \eta^{-1} \cdot \kl\infdivx[]{\nu_{h+1}}{\tilde\nu_{h+1}}}\nend
        &= \EE_{s_h, b_h}\sbr{\sum_{i=h}^{H}\gamma^{i-h}\bigrbr{r_i^\pi + \gamma \tilde V_{i+1} - \tilde Q_i - \gamma\eta^{-1}\cdot \kl\infdivx[]{\nu_{i+1}}{\tilde\nu_{i+1}}}},  \label{eq:Q diff}
    \end{align}
where the second equality follows from 
 \eqref{eq:V diff}, and the 
last equality follows from an recursive argument from step $h$ to $H$.
    Plugging \eqref{eq:Q diff} into \eqref{eq:A diff-1}, we have  
    \begin{align*}
        A_h - \tilde A_h &= \rbr{\EE_{s_h, b_h}-\EE_{s_h}}\sbr{\sum_{i=h}^{H}\gamma^{i-h}\bigrbr{r_i^\pi + \gamma \tilde V_{i+1} - \tilde Q_i - \gamma\eta^{-1}\kl\infdivx[]{\nu_{i+1}}{\tilde\nu_{i+1}}}} \nend
        &\qquad + \eta^{-1}\kl\infdivx[]{\nu_h}{\tilde \nu_h}\nend
        & = \rbr{\EE_{s_h, b_h}-\EE_{s_h}} \sbr{\Delta_h^{(1)}(s_h, b_h) - \gamma\eta^{-1}\Delta_{h+1}^{(2)} (s_{h+1})} + \eta^{-1}\kl\infdivx[]{\nu_h}{\tilde \nu_h}, 
    \end{align*}
    which completes the proof.
\end{proof}
\end{lemma}

\subsection{Proof of \Cref{prop:error-prop}}
 \label{proof:prop:error-prop}

 In this subsection, we prove \Cref{prop:error-prop}, which relates the difference between value functions and quantal response policies to the distance between models and leader's policies.

\begin{proof}
    Recall that 
    $(\tilde Q, \tilde V, \tilde U, \tilde W, \tilde \nu)$ are the value functions and 
quantal response policy associated with $\tilde \pi$, under model $\tilde M$. 
Similarly, $ ( Q,  V,  U,  W,  \nu)$ are the corresponding terms associated with a policy $\pi$ under model $M$.  

By Bellman's equation of the follower and triangle inequality, we have 
    \begin{align*}
        \onbr{Q_h - \tilde Q_h}_\infty 
        &\le \bignbr{r_h - \tilde r_h}_\infty + \gamma \cdot \onbr{\tilde V_{h+1}}_\infty \cdot \sup_{s_h, a_h, b_h\in\cS\times\cA\times\cB}\bignbr{P_h(\cdot\given s_h, a_h, b_h)-\tilde P_h(\cdot\given s_h, a_h, b_h)}_1 \nend
        &\qquad + \gamma\cdot  \bignbr{V_{h+1} -\tilde V_{h+1}}_\infty + \bignbr{\tilde r_h + \gamma \tilde P_h \tilde V_{h+1}}_\infty \cdot \sup_{s_h, b_h\in\cS\times\cB}\nbr{\pi_h(\cdot\given s_h, b_h) - \tilde \pi_h(\cdot\given s_h, b_h)}_1\nend
        &\le 2(1 + \gamma B_A) \epsilon + \gamma \cdot \onbr{V_{h+1} -\tilde V_{h+1}}_\infty, 
    \end{align*}
where the last inequality follows from the fact that $
\onbr{\tilde V_{h+1}}_\infty  \leq B_A, \onbr{\tilde r_h+\gamma \tilde P_h\tilde V_{h+1}}_\infty\le 1 + \gamma B_A$, where for the upper bound of the V function, we 
Moreover, by the leader's Bellman equation, we have 
    \begin{align*}
        \bignbr{U_h - \tilde U_h}_\infty 
        &\le \bignbr{u_h - \tilde u_h}_\infty + \bignbr{\tilde W}_\infty \cdot \bignbr{P_h(\cdot\given s_h, a_h, b_h)-\tilde P_h(\cdot\given s_h, a_h, b_h)}_1 + \bignbr{W_{h+1} -\tilde W_{h+1}}_\infty \nend
        &\le (1 + H) \epsilon + \bignbr{W_{h+1} -\tilde W_{h+1}}_\infty.
    \end{align*}
    We next look at the quantal response difference, where we invoke equation (53) of Lemma 5.1 in \citet{chen2022adaptive}, which says that
    \begin{align}
        \nbr{\nu_h(\cdot\given s_h) - \tilde \nu_h(\cdot\given s_h)}_1 \le 4\eta \bignbr{Q_h(s_h, \cdot) -\tilde Q_h(s_h, \cdot) }_\infty.\label{eq:error-prop-nu}
    \end{align}
    Moreover, for the update of the V function, we have
    \begin{align*}
        \bigabr{V_h(s_h) -\tilde V_h(s_h)} &\le \max\cbr{\bigabr{\bigdotp{\nu_h(\cdot\given s_h)}{Q_h(s_h, \cdot) - \tilde Q_h(s_h, \cdot)}_\cB}, \bigabr{\bigdotp{\tilde\nu_h(\cdot\given s_h)}{Q_h(s_h, \cdot) - \tilde Q_h(s_h, \cdot)}_\cB}}\nend
        &\le \bignbr{Q_h(s_h, \cdot) -\tilde Q_h(s_h, \cdot)}_\infty, 
    \end{align*}
    where the first inequality holds by noting that $V_h(s_h) \defeq \max_{\nu'\in\Delta(\cB)} \dotp{\nu'}{Q_h (s_h,\cdot)}_\cB + \eta^{-1}\cH(\nu')$. Therefore, we have $\bignbr{V_h-\tilde V_h}_\infty \le \bignbr{Q_h-\tilde Q_h}_\infty$.
    And for the W function, we have
    \begin{align*}
        \bignbr{W_h - \tilde W_h}_\infty 
        &\le \bignbr{U_h -\tilde U_h}_\infty + \bignbr{\tilde U_h}_\infty \cdot \sup_{s_h\in\cS} \bignbr{\nu_h \otimes \pi_h (\cdot,\cdot\given s_h)- \tilde\nu_h \otimes\tilde\pi_h(\cdot,\cdot\given s_h)}_1\nend
        &\le \bignbr{U_h -\tilde U_h}_\infty + \bignbr{\tilde U_h}_\infty \cdot \rbr{\bignbr{\nu_h- \tilde\nu_h }_1 + \bignbr{\pi_h- \tilde\pi_h }_1}\nend
        &\le \bignbr{U_h -\tilde U_h}_\infty + 4\eta H  \bignbr{Q_h -\tilde Q_h}_\infty + H \epsilon, 
    \end{align*}
    where the last inequality holds by using \eqref{eq:error-prop-nu} and note that $\tilde U_h$ is bounded by $H$. As a result, we have that 
    \begin{align*}
        \begin{bmatrix}
            \bignbr{U_h - \tilde U_h}_\infty \\
            \bignbr{Q_h - \tilde Q_h}_\infty 
        \end{bmatrix}
        \le 
        \begin{bmatrix}
            1 & 4\eta H \\
            0 & \gamma
        \end{bmatrix}
        \begin{bmatrix}
            \bignbr{U_{h+1} - \tilde U_{h+1}}_\infty \\
            \bignbr{Q_{h+1} - \tilde Q_{h+1}}_\infty 
        \end{bmatrix}
        + \epsilon
        \begin{bmatrix}
            1+2H \\
            2(1+\gamma B_A)
        \end{bmatrix}. 
    \end{align*}
Solving this matrix inequality, we have 
\begin{align*}
    \bignbr{Q_h - \tilde Q_h}_\infty &\le \epsilon 2(1+\gamma B_A) \eff_{H-h+1}(\gamma) \le \epsilon 2(1+\gamma B_A) \eff_{H}(\gamma), \nend
    \bignbr{U_h - \tilde U_h}_\infty &\le \epsilon H \rbr{4\eta H 2(1+\gamma B_A)\eff_H(\gamma) + 1 + 2H}.
\end{align*}
Furthermore, we can establish that 
\begin{align*}
    \bignbr{V_h - \tilde V_h }_\infty 
    &\le \bignbr{Q_h - \tilde Q_h}_\infty \le \epsilon 2(1+\gamma B_A) \eff_{H}(\gamma), \nend
    \bignbr{W_h - \tilde W_h}_\infty &\le \epsilon (H+1) \rbr{4\eta H 2(1+\gamma B_A)\eff_H(\gamma) + 1 + 2H}, 
\end{align*}
and also for the quantal response, 
\begin{align*}
    D_\H^2(\tilde \nu_h(\cdot\given s_h), \nu_h(\cdot\given s_h)) \le \nbr{\nu_h(\cdot\given s_h) - \tilde \nu_h(\cdot\given s_h)}_1 \le 4 \eta \epsilon 2(1+\gamma B_A) \eff_{H}(\gamma).
\end{align*}
Therefore, we conclude the proof. 
\end{proof}




\subsection{Other Auxiliary Lemmas}
\begin{lemma}[Lemma A.4 in \citet{foster2021statistical}]\label{lem:freeman-variation}
  For any sequence of real-valued random variables $(X_t)_{t\in[T]}$ adapted to a filtration $(\cF_t)_{t\in[T]}$, it holds with probability at least $1-\delta$ for all $t\in[T]$ that 
  \begin{align*}
      \sum_{i=1}^{t}X_i\le \sum_{i=1}^t \log\rbr{\EE\sbr{e^{X_i}\biggiven \sF_{i-1}}} + \log(\delta^{-1}).
  \end{align*}
\end{lemma}

\begin{lemma}[Bernstein inequality]\label{lem:bernstein}
  For independent random variables $Z_1,\dots, Z_i,\dots, Z_T$ such that $|Z_i-\EE Z_i|\le B$ and $\Var[Z_i]\le L \EE[Z_i]$, we have with probability at least $1-\delta$ that
  \begin{align*}
    \abr{\frac 1 T \sum_{i=1}^T Z_i - \EE\sbr{Z_i}} &\le \frac{1}{2T}\sum_{i=1}^T \EE [Z_i] + \frac{(2L +4B/3)\log(2\delta^{-1})}{T}.
  \end{align*}
\end{lemma}
  \begin{proof}
    By a standard Bernstein inequality, 
\begin{align*}
  \abr{\frac 1 T \sum_{i=1}^T Z_i - \EE\sbr{Z_i}} &\le \sqrt{\frac{4\sum_{i=1}^T\Var[Z_i]\log(2\delta^{-1})}{T^2}} + \frac{4 B \log(2\delta^{-1})}{3T} \nend
    &\le \sqrt{\frac{4 L\sum_{i=1}^T \EE[Z_i]\log(2\delta^{-1})}{T^2}} + \frac{4 B\log(2\delta^{-1})}{3T}\nend
    &\le \frac{1}{2T}\sum_{i=1}^T \EE [Z_i] + \frac{(2L +4B/3)\log(2\delta^{-1})}{T}.
  \end{align*}
  Therefore, we conclude the proof.
\end{proof}

\begin{lemma}[Freedman's inequality, Lemma 9 in \citet{agarwal2014taming}]\label{lem:freedman}
  Let $X_1, X_2, \dotsc, X_T$ be a sequence of real-valued random variables adapted to a filtration $\{\sF_t\}_{t\in[T]}$.
  Assume for all $t \in \{1,2,\dotsc,T\}$, $X_t \leq R$ and
  $\EE[X_t|\sF_{t-1}] = 0$.
  Define $S := \sum_{t=1}^T X_t$ and $V := \sum_{t=1}^T
  \EE[X_t^2|\sF_{t-1}]$.
  For any $\delta \in (0,1)$
  and $\lambda \in [0,1/R]$,
  with probability at least $1-\delta$,
  \begin{equation}
    S \leq (e-2)\lambda V + \frac{\log(1/\delta)}{\lambda}\le \lambda V + \frac{\log(1/\delta)}{\lambda}.\label{eq:freedman ineq}
  \end{equation}
  \end{lemma}

  \begin{corollary}[Martingale concentration derived from Freedman's inequality] \label{cor:martigale concentration}
    Let $X_1, X_2, \dotsc, X_T$ be a sequence of real-valued random variables adapted to a filtration $\{\cF_t\}_{t\in[T]}$.
    Suppose that $|X_t - \EE\sbr{X_t\given \sF_{t-1}}| \le R$ for all $t\in[T]$.
    For any $\delta \in (0,1)$,
  with probability at least $1-\delta$,
  \begin{align}
    \abr{\sum_{t=1}^T X_t - \EE\sbr{X_t\given \sF_{t-1}}} \le  R\sqrt{T} \log(2e\delta^{-1}). \label{eq:martingale-1}
  \end{align}
  Moreover, suppose that $X_t\in[0, R]$ for any $t\in[T]$, with probability at least $1-\delta$, 
  \begin{align}
    \frac 1 2\sum_{t=1}^T \EE\sbr{X_t\given \sF_{t-1}} - 2R\log(2\delta^{-1}) \le \sum_{t=1}^T X_t \le \frac 3 2 \sum_{t=1}^T \EE\sbr{X_t\given \sF_{t-1}} + 2R\log(2\delta^{-1}). \label{eq:martingale-2} 
  \end{align}
  \end{corollary}

\begin{proof}
  We let $Z_t = X_t - \EE[X_t\given \sF_{t-1}]$ and it is obvious that $\EE[Z_t\given \sF_{t-1}] = 0$ and $|Z_t|\le R$.
  For \eqref{eq:martingale-1}, we use \eqref{eq:freedman ineq} for sequence $\{Z_t\}_{t\in[T]}$, and it holds with probability at least $1-\delta/2$
  \begin{align}
    \sum_{t=1}^T X_t - \EE\sbr{X_t\given \sF_{t-1}} 
    &\le  \lambda \sum_{t=1}^T \EE[Z_t^2\given \sF_{t-1}] + \lambda^{-1} \log(2\delta^{-1})\label{eq:martingale-proof-1}\\
    &\le  \lambda T R^2 + \lambda^{-1} \log(2\delta^{-1}). \notag
  \end{align}
  We let $\lambda = (R\sqrt T)^{-1}$ an have with probability $1-\delta/2$ that
  \begin{align*}
    \sum_{t=1}^T X_t - \EE\sbr{X_t\given \sF_{t-1}} \le  R\sqrt{T} \log(2e\delta^{-1}), 
  \end{align*}
  Using \eqref{eq:freedman ineq} again for $\{-Z_t\}_{t\in[T]}$, we obtain the opposite side, which implies that \eqref{eq:martingale-1} holds with probability at least $1-\delta$. 

  For \eqref{eq:martingale-2}, we notice from \eqref{eq:martingale-proof-1} that with probability at least $1-\delta/2$,
  \begin{align*}
    \sum_{t=1}^T X_t - \EE\sbr{X_t\given \sF_{t-1}} 
    &\le  \lambda \sum_{t=1}^T \EE[Z_t^2\given \sF_{t-1}] + \lambda^{-1} \log(2\delta^{-1})\nend
    &\le  \lambda \sum_{t=1}^T \EE[X_t^2\given \sF_{t-1}] + \lambda^{-1} \log(2\delta^{-1})\nend
    & \le \lambda R \sum_{t=1}^T \EE[X_t\given \sF_{t-1}] + \lambda^{-1} \log(2\delta^{-1}), 
  \end{align*}
  where the last inequality holds by the nonnegativity of $X_t$. Applying this inequality to $\{-Z_t\}_{t\in[T]}$ as well, we conclude with probability at least $1-\delta$ that
  \begin{align*}
    \abr{\sum_{t=1}^T X_t - \EE\sbr{X_t\given \sF_{t-1}}} \le \lambda R \sum_{t=1}^T \EE[X_t\given \sF_{t-1}] + \lambda^{-1} \log(2\delta^{-1}).
  \end{align*}
  We take $\lambda = 1/(2R)$, which gives the result in \eqref{eq:martingale-2}.
\end{proof}

\begin{lemma}[Elliptical potential  lemma for vectors, Proposition 1 in \citet{carpentier2020elliptical}]\label{lem:elliptical potential}
Let $u_1,\ldots,u_T$ be a sequence of arbitrary vectors in $\mathbb{R}^d$ such that $\|u_i\|_2 \leq 1$. For any $1\leq t\leq T$, we define
\[
V_t = \sum_{s=1}^{t-1} u_s u_s^\top +\lambda I\, .
\]
It then holds that
\[
  \sum_{t=1}^T \|u_t\|_{V_{t}^{-1}} \le \sum_{t=1}^T \|u_t\|_{V_{t+1}^{-1}} \leq \sqrt{Td\log\left(\frac{T+d\lambda}{d\lambda}\right)}, 
\]
and also that
\[
  \sum_{t=1}^T \|u_t\|_{V_{t}^{-1}}^2 \le 2 \log\rbr{{\det(V_T)}}
\]
\end{lemma}

\begin{lemma}[Elliptical Potential Lemma for matrices]
  \label{lem:potential-matrix}
  Suppose $U_0 = \lambda I_d$, $U_t = U_{t-1} + X_t$ for $X_t\in\SSS_+^d$, and $\trace\rbr{X_t} \leq L$, then
  \begin{equation*}
    \sum_{t=1}^T \sqrt{\trace\rbr{U_{t-1}^{-1} X_t}} \leq \sqrt{\frac{LT/\lambda}{\log\rbr{1+L/\lambda}}\cdot d \log\rbr{1+\frac{LT}{\lambda d}}}, 
  \end{equation*}
  and also
  \begin{align*}
    \sum_{t=1}^T {\trace\rbr{U_{t-1}^{-1} X_t}} \le \frac{L/\lambda}{\log\rbr{1+L/\lambda}}\cdot d \log\rbr{1+\frac{LT}{\lambda d}}
  \end{align*}
\end{lemma}

\begin{proof}
  First, we have the following decomposition,
  \[
      U_t = U_{t-1} + X_tX_t^{\top} = U_{t-1}^{\frac{1}{2}}(I + U_{t-1}^{-1/2}X_tU_{t-1}^{-1/2})U_{t-1}^{\frac{1}{2}}.
  \]

  Taking the determinant on both sides, we get
  \[
  \det(U_t) = \det(U_{t-1}) \det(I + U_{t-1}^{-1/2}X_tU_{t-1}^{-1/2}),
  \]
  which follows from $\det(I_d + X)\ge 1+\trace(X) $ for $X\in\SSS_+^d$ that
  \[
     \det(U_t) = \det(U_{t-1}) \rbr{1 + \trace\rbr{U_{t-1}^{-1} X_t}}.
\]
  By taking advantage of the telescope structure, we have
  \[
    \begin{split}
           & \sum_{t=1}^T \log\rbr{1 + \trace\rbr{U_{t-1}^{-1} X_t}} \leq \log \frac{\det(U_T)}{\det(U_0)} \leq d\log \left(1 + \frac{LT}{\lambda d}\right),
    \end{split}
  \]
  where the last inequality follows from the fact that $\mbox{Tr}(U_T) \leq \mbox{Tr}(U_0) + LT = \lambda d + LT$, and thus $\det(U_T) \leq (\lambda + LT/d)^d$. Moreover, by noting that $\log(1+x)\ge \log(1+B)\cdot x/B$, we have 
  \begin{align*}
    \sum_{t=1}^T \trace\rbr{U_{t-1}^{-1} X_t} &\le \frac{L/\lambda}{\log\rbr{1+L/\lambda}} \cdot \sum_{t=1}^T \log\rbr{1+\trace\rbr{U_{t-1}^{-1} X_t}} \nend
    &\le \frac{L/\lambda}{\log\rbr{1+L/\lambda}}\cdot d \log\rbr{1+\frac{LT}{\lambda d}}.
  \end{align*} 
  Therefore, Cauchy-Schwarz inequality implies,
  \[
    \sum_{t=1}^T \sqrt{\trace\rbr{U_{t-1}^{-1} X_t}} \leq \sqrt{T\sum_{t=1}^T \trace\rbr{U_{t-1}^{-1} X_t}} \leq \sqrt{\frac{LT/\lambda}{\log\rbr{1+L/\lambda}}\cdot d \log\rbr{1+\frac{LT}{\lambda d}}}.
  \]
  Therefore, we conclude the proof.
\end{proof}

\begin{lemma}[Bounding Cumulative Error with Eluder Dimension, adapted from Lemma 41 in \citet{jin2021bellman}]
  \label{lem:de-regret}
  Given a function class $\cG$ defined on $\cX$ with $|g(x)|\le B$ for all $(g,x)\in\cG\times\cX$, and a family of finite signed measures $\sP$ over $\cX$. 
  Suppose sequence $\{g_t\}_{t=1}^{T}\subset \cG$ and $\{\rho_t\}_{t=1}^{T}\subset\sP$ satisfy that for all $t\in[T]$,
  $\sum_{i=1}^{t-1} (\EE_{\rho_i} [g_t])^2 \le \beta$. Then for all $t\in[T]$ and $\omega>0$ we have for the first order cumulative error that
  \begin{align}\label{eq:DE error-1st order}
      \sum_{i=1}^{t} \abr{\int_{\cX} g_i \rd \rho_i} \le 2\sqrt{\dim_\DE (\cG,\sP,\omega)\beta t}+\min\{t,\dim_\DE (\cG,\sP,\omega)\}B +t\omega, 
  \end{align} 
  and for the second order cumulative error that
  \begin{align}\label{eq:DE error-2nd order}
      \sum_{i=1}^t \rbr{\int_{\cX} g_i \rd \rho_i}^2 \le \dim_\DE(\cG,\sP,\omega)\beta \log(eT) + \min\{t,\dim_\DE (\cG,\sP,\omega)\}B^2 + t\omega^2.
  \end{align}
  \begin{proof}
      We first invoke the following proposition (extended to signed measures) from \citet{jin2021bellman} whose proof can be found in the proof of Proposition 43 in \citet{jin2021bellman}.
      \begin{proposition}\label{prop:ed-sequence length}
          Given a function class $\cG$ defined on $\cX$, and a family of signed measures $\sP$ over $\cX$. 
          Suppose sequences $\{\phi_i\}_{i=1}^{T}\subset \cG$ and $\{\rho_i\}_{i=1}^{T}\subset\sP$ satisfy that for all $t\in[T]$,
          $\sum_{i=1}^{t-1} (\int_\cX \phi_t \rd \rho_i)^2 \le \beta$. Then for all $t\in[T]$,
      $$
          \sum_{i=1}^{t} \ind \cbr{\abr{\int_\cX \phi_i \rd \rho_i} > \epsilon } \leq \rbr{\frac{\beta}{\epsilon^2}+1}\dim_\DE (\cG,\sP,\epsilon).
      $$
      \label{prop:de-regret-prop}
      \end{proposition}
  The following proof is largely the same as those in \citet{jin2021bellman}. For completeness, we still present them here.
  Fix $t \in [T]$ and let $d = \dim_\DE (\cG,\sP,\omega)$. Sort the sequence $\{|\int_{\cX}\phi_1 \rd \rho_1|,\dots,|\int_{\cX}\phi_t \rd \rho_t|\}$ in a decreasing order and denote it by $\{e_1,\dots,e_t\}$ such that $e_1 \geq e_2 \geq \dots \geq e_t$. We have for the cumulative reward of order $k\in\{1, 2\}$ that
$$
  \sum_{i=1}^t |\EE_{\rho_i} [\phi_i]|^k = \sum_{i=1}^t e_i^k = \sum_{i=1}^t e_i^k \ind\big \{e_i \leq \omega \big \} + \sum_{i=1}^t e_i^k \ind\big \{e_i > \omega\big\} \leq t\omega^k + \sum_{i=1}^t e_i^k \ind\big \{e_i > \omega\big\}.
$$
For $i \in [t]$, we want to prove that if $e_i > \omega$, then we have $e_i \leq \min\{\sqrt{{d\beta}/\rbr{i-d}},B\}$. Assume $i\in[t]$ satisfies $e_i >  \omega$. 
Then there exists $\alpha$ such that $e_i >\alpha\ge  \omega$.
By Proposition \ref{prop:de-regret-prop}, we have
$$
i \le 
\sum_{i=1}^t \ind\big \{e_i > \alpha\big\} 
\le \rbr{ \frac{\beta}{\alpha^2} + 1 } \dim_\DE (\cG,\sP,\alpha)
\le 
\rbr{ \frac{\beta}{\alpha^2} + 1 } \dim_\DE (\cG,\sP,\omega),
$$
  where the last inequality holds by definition of the distributional eluder dimension.
The inequality directly implies that $\alpha\le \sqrt{{d\beta}/\rbr{i-d}}$.
Besides, recall $e_i \leq B$, so we have   $e_i \leq \min\{\sqrt{{d\beta}/\rbr{i-d}},B\}$.
Finally, we have
\begin{equation*}
\begin{aligned}
    & \sum_{i=1}^t e_i^k \ind\big \{e_i > \omega\big\}  \nend 
          &\quad \leq \min\{d,t\}B^k+\sum_{i=d+1}^t \rbr{\frac{d\beta}{i-d}}^{k/2} \nend
          &\quad \leq \min\{d,t\}B^k + \ind\cbr{k=1}\cdot \sqrt{d\beta}\int_{0}^t \frac{1}{\sqrt{x}} dx +\ind\cbr{k=2} \cdot  \rbr{d\beta \int_{1}^t x^{-1} \rd x + d\beta}\\
    &\quad \leq \min\{d,t\}B^k + \ind\cbr{k=1}\cdot 2\sqrt{d\beta t} + \ind\cbr{k=2} \cdot d\beta \rbr{\log t +1},
\end{aligned}
\end{equation*}
which completes the proof.
  \end{proof}
\end{lemma}
Note that the above proof as well as \Cref{prop:ed-sequence length} should hold for both the eluder dimension and the distributional eluder dimension.
\end{document}